\documentclass[final]{colt2025}
\usepackage[capitalise,nameinlink]{cleveref}

\title[Towards Fair Representation: Clustering and Consensus]{Towards Fair Representation: Clustering and Consensus}
\usepackage{times}

\coltauthor{
 \Name{Diptarka Chakraborty} \Email{diptarka@comp.nus.edu.sg}\\
 \addr National University of Singapore
 \AND
 \Name{Kushagra Chatterjee} \Email{kushagra.chatterjee@u.nus.edu}\\
 \addr National University of Singapore
 \AND
 \Name{Debarati Das} \Email{dxd5606@psu.edu}\\
 \addr Pennsylvania State University
 \AND
 \Name{Tien Long Nguyen} \Email{tfn5179@psu.edu}\\
 \addr Pennsylvania State University
 \AND
 \Name{Romina Nobahari} \Email{nobahariromina@gmail.com}\\
 \addr Sharif University of Technology
}

\usepackage{thmtools}
\usepackage{thm-restate}
\usepackage{comment}
\usepackage{enumitem}

\crefname{theorem}{Theorem}{theorems}
\crefname{lemma}{Lemma}{lemmas}
\crefname{proposition}{Proposition}{propositions}
\crefname{subsubsubappendix}{Appendix}{Appendices}
\Crefname{subsubsubappendix}{Appendix}{Appendices}

\newtheorem{claim}{Claim}
\newtheorem{subclaim}[claim]{Claim}
\newtheorem{fact}{Fact}

\newtheorem{property}{Property}

\makeatletter
\renewenvironment{proof}[1][\proofname]%
{%
   \par\noindent{\bfseries\upshape #1:\ }%
}%
{\jmlrQED}
\makeatother

\LinesNumbered

\usepackage{mathtools}
\allowdisplaybreaks
\DeclarePairedDelimiter\ceil{\lceil}{\rceil}
\DeclarePairedDelimiter\floor{\lfloor}{\rfloor}

\usepackage{tikz}
\usetikzlibrary{calc, graphs, graphs.standard, shapes, arrows, positioning, decorations.pathreplacing, decorations.markings, decorations.pathmorphing, fit, matrix, patterns, shapes.misc}

\newcommand{\m}[1]{\mathcal{#1}}
\newcommand{\abs}[1]{\left\lvert#1\right\rvert}

\newcommand{\Pmod}{\overset{p}{\equiv}}
\newcommand{\iprime}{i^{\prime}}

\newcommand{\ceilenv}[1]{\left\lceil #1 \right\rceil}

\newcommand{\inpset}{\m{D}}
\newcommand{\inpsetd}{\m{D} = \{ D_1, D_2, \ldots, D_\phi\}}
\newcommand{\inpcl}[1]{D_{#1}}
\newcommand{\fairset}{\m{F}}
\newcommand{\fairsetd}{\m{F} = \{ F_1, F_2, \ldots, F_\zeta\}}

\newcommand{\fairmop}{\m{F}^{*}}
\newcommand{\fairmopd}{\m{F}^{*}= \{ F^*_1, F^*_2, \ldots, F^*_\iota\}}

\newcommand{\optset}{\m{C}^*}

\newcommand{\gencl}[1]{C_{#1}}
\newcommand{\out}{\m{T}}
\newcommand{\outd}{\m{T} = \{ T_1, T_2, \ldots, T_\gamma\}}
\newcommand{\outcl}[1]{T_{#1}}
\newcommand{\outg}{\m{T}^{cm}}

\newcommand{\outcut}{\m{T}^{c}}

\newcommand{\outpq}{\m{Q}}
\newcommand{\outpqcl}[1]{Q_{#1}}
\newcommand{\outpqg}{Q^{cm}}
\newcommand{\outpqcc}{\m{Q}^{cc}}

\newcommand{\opt}{\textsc{OPT}}
\newcommand{\optr}{\textsc{OPT}^{r}}
\newcommand{\optb}{\textsc{OPT}^{b}}
\newcommand{\optbal}{$\opt^{bal}$}
\newcommand{\optfair}{$\opt^{fair}$}

\newcommand{\blue}[1]{\textit{blue}(#1)}
\newcommand{\red}[1]{\textit{red}\left(#1\right)}

\newcommand{\surp}[1]{\textit{s}{(#1)}}
\newcommand{\smp}{s} 
\newcommand{\dmp}{d} 
\newcommand{\ratio}{\rho} 
\newcommand{\blu}{\textit{blue}}
\newcommand{\re}{\textit{red}}
\newcommand{\bsurp}[1]{\textit{$s_b$}{(#1)}}
\newcommand{\rsurp}[1]{\textit{$s_r$}{(#1)}}

\newcommand{\defi}[1]{\textit{d}(#1)}

\newcommand{\dist}{\textit{dist}}

\newcommand{\bdefi}[1]{\textit{$d_b$}(#1)}
\newcommand{\rdefi}[1]{\textit{$d_r$}(#1)}

\newcommand{\pcls}{\mathcal{E}}
\newcommand{\pclsr}{\mathcal{E}^{r}}
\newcommand{\pclsb}{\mathcal{E}^{b}}

\newcommand{\pcl}[1]{E_{#1}}
\newcommand{\pclr}[1]{E^{r}_{#1}}
\newcommand{\pclb}[1]{E^{b}_{#1}}
\newcommand{\pcard}{m}
\newcommand{\tb}{\textsc{TypeBlue}}
\newcommand{\tr}{\textsc{TypeRed}}
\newcommand{\intracost}[1]{\operatorname{intra}(#1)}
\newcommand{\intercost}[1]{\operatorname{inter}(#1)}

\newcommand{\cut}{\textsc{cut}}
\newcommand{\nc}{\textsc{newcut}}
\newcommand{\merge}{\textsc{merge}}
\newcommand{\cuta}{\textsc{cut-algo}}
\newcommand{\mergea}{\textsc{merge-algo}}
\newcommand{\cuto}{\textsc{cut-OPT}}
\newcommand{\mergeo}{\textsc{merge-OPT}}
\newcommand{\ccost}{\textit{cut-cost}}
\newcommand{\mcost}{\textit{merge-cost}}
\newcommand{\ccostf}[1]{\kappa (#1)}
\newcommand{\mcostf}[1]{\mu (#1)}
\newcommand{\costf}[1]{\textit{cost} (#1)}
\newcommand{\rcut}{\textsc{RedCut}}
\newcommand{\bcut}{\textsc{BlueCut}}
\newcommand{\rmerge}{\textsc{RedMerge}}
\newcommand{\bmerge}{\textsc{BlueMerge}}
\newcommand{\rcuta}{\textsc{RCut}}
\newcommand{\bcuta}{\textsc{BCut}}
\newcommand{\rmergea}{\textsc{RMerge}}
\newcommand{\bmergea}{\textsc{BMerge}}

\newcommand{\algog}{\text{AlgoforGeneral()}}
\newcommand{\algoc}{\text{AlgoforCut()}}
\newcommand{\algom}{\text{AlgoforMerge()}}
\newcommand{\algmf}{\text{MakeClustersFair()}}
\newcommand{\algopqgen}{\text{AlgoforGeneralPQ()}}
\newcommand{\algopqc}{\text{AlgoforCut-Cut()}}
\newcommand{\algopqcm}{\text{AlgoforCut-Merge()}}
\newcommand{\algopqmc}{\text{AlgoforMerge-Cut()}}
\newcommand{\algopqm}{\text{AlgoforMerge-Merge()}}

\newcommand{\makeitfair}{\text{MakeItFair}}
\newcommand{\greedymerge}{\text{GreedyMerge}}
\newcommand{\fairpart}{\text{FairPart}}
\newcommand{\nfairpart}{\text{NotFairPart}}

\newcommand{\mop}{\m{T}^*}
\newcommand{\mopq}{\m{Q}^*}
\newcommand{\mopd}{\m{T}^* = \{ T_1^*, T_2^*, \ldots, T_\psi^*\}}
\newcommand{\mopqd}{\m{Q}^* = \{Q_1^*, Q_2^*, \ldots Q_\tau^*\} }
\newcommand{\mopcl}[1]{T^*_{#1}}
\newcommand{\costone}[1]{\textit{cost}_1(#1)}
\newcommand{\costtwo}[1]{\textit{cost}_2(#1)}
\newcommand{\costthree}[1]{\textit{cost}_3(#1)}
\newcommand{\costfour}[1]{\textit{cost}_4(#1)}
\newcommand{\equ}{\Leftrightarrow}
\newcommand{\n}{\nonumber}

\newcommand{\bal}{\textsf{Balanced clustering}}
\newcommand{\fair}{\textsf{Fair clustering}}

\newcommand{\mopdef}{\bal}
\newcommand{\mopqdef}{\bal}

\newcommand{\somef}{7.5}
\newcommand{\cc}[1]{\operatorname{\kappa\kappa}(#1)}
\newcommand{\cm}[1]{\operatorname{\kappa\mu}(#1)}
\newcommand{\mc}[1]{\operatorname{\mu\kappa}(#1)}
\newcommand{\mm}[1]{\operatorname{\mu\mu}(#1)}
\newcommand{\mcf}{7.5}
\newcommand{\mmf}{7}
\newcommand{\costfive}[1]{\textit{cost}_5(#1)}
\newcommand{\algopqmm}{\textsf{AlgoforMerge-Merge}}
\newcommand{\nbcut}{\textsc{newbcut}}
\newcommand{\nrcut}{\textsc{newrcut}}
\newcommand{\clof}{\textsc{closest $p$-fair}}
\newcommand{\npc}{\textbf{NP}\text{-complete}}

\newcommand{\thrp}{\textsc{3partition}}

\newcommand{\kushagra}[1]{\textcolor{red}{Kushagra: #1}}

\usepackage{dirtytalk}
\begin{document}

\maketitle

\begin{abstract}%
Consensus clustering, a fundamental task in machine learning and data analysis, aims to aggregate multiple input clusterings of a dataset, potentially based on different non-sensitive attributes, into a single clustering that best represents the collective structure of the data. In this work, we study this fundamental problem through the lens of fair clustering, as introduced by Chierichetti et al. [NeurIPS'17], which incorporates the disparate impact doctrine to ensure proportional representation of each protected group in the dataset within every cluster. Our objective is to find a consensus clustering that is not only representative but also fair with respect to specific protected attributes. To the best of our knowledge, we are the first to address this problem and provide a constant-factor approximation.

As part of our investigation, we examine how to minimally modify an existing clustering to enforce fairness -- an essential postprocessing step in many clustering applications that require fair representation. We develop an optimal algorithm for datasets with equal group representation and near-linear time constant factor approximation algorithms for more general scenarios with different proportions of two group sizes. We complement our approximation result by showing that the problem is \texttt{NP}-hard for two unequal-sized groups. Given the fundamental nature of this problem, we believe our results on Closest Fair Clustering could have broader implications for other clustering problems, particularly those for which no prior approximation guarantees exist for their fair variants.
\end{abstract}

\section{Introduction}\label{sec:introduction}

Machine learning plays a crucial role in modern decision-making processes, including recommendation systems, economic opportunities such as loan approvals, and recidivism prediction, e.g.~\cite{kleinberg2016inherent, kleinberg2018human}. However, these machine learning-driven processes risk being biased against marginalized communities based on sensitive attributes, such as gender or race~\cite{kay2015unequal, bolukbasi2016man}, due to biases inherent in the training data. This highlights the need to study and design fair algorithms to address disparities resulting from historical marginalization. In recent years, there has been significant literature on algorithmic fairness, focusing on achieving \emph{demographic parity}~\cite{dwork2012fairness} or \emph{equal opportunity}~\cite{hardt2016equality}.

Clustering is a foundational unsupervised learning task that involves partitioning a set of data points in a metric space into clusters, with the goal of minimizing certain objective functions specific to the application. Each data point can be viewed as an individual with certain protected attributes, which can be represented by assigning a color to each point. \cite{chierichetti2017fair} pioneered the concept of \emph{fair clustering}, where each point in the dataset is colored either red or blue. The goal was to partition the data while maintaining balance in each cluster -- the ratio of blue to red points in each cluster reflects the overall ratio in the entire set. 
This balance is important for addressing disparate impact and ensuring fair representation. Since it was introduced, different variants of fair clustering have been studied, including $k$-center/median/means clustering~\cite{chierichetti2017fair, HuangJV19}, scalable clustering~\cite{BackursIOSVW19}, proportional clustering~\cite{ChenFLM19}, overlapping multiple groups~\cite{BeraCFN19}, group-oblivious models~\cite{esmaeili2020probabilistic}, correlation clustering~\cite{pmlr-v108-ahmadian20a, ahmadi2020fair, ahmadian2023improved} and 1-clustering over rankings~\cite{wei22, chakraborty2022}, among others.

In this paper, we consider \emph{consensus clustering}, a fundamental problem in machine learning and data analysis, and initiate its study under fairness constraints.  Given a set of clusterings over a dataset, potentially based on different non-sensitive attributes, the goal of consensus clustering is to aggregate them into a single clustering that best represents the collective structure of the data while minimizing a specified objective function. The choice of objective functions varies depending on the context, with two of the most common being the \emph{median objective} that minimizes the sum of distances, and the \emph{center objective} that minimizes the maximum distance. Here, the distance between two clusterings is measured by counting the pairs of points that are grouped together in one clustering but not in the other (see Section~\ref{sec:prelim} for a formal definition). The consensus clustering problem has a myriad of applications in different domains, such as gene integration in bioinformatics~\cite{filkov2004integrating, filkov2004heterogeneous}, data mining~\cite{topchy2003combining}, community detection~\cite{lancichinetti2012consensus}. The problem (both respect to median and center objective) is not only known to be \texttt{NP}-hard~\cite{kvrivanek1986np, swamy2004correlation}, but also known to be \texttt{APX}-hard, i.e., unlikely to possess any $(1+\epsilon)$-approximation (for any $\epsilon >0$) algorithm even when there are only three input clusterings~\cite{BonizzoniVDJ08}. While several heuristics have been considered to generate reasonable solutions (e.g.,~\cite{goder2008consensus, monti2003consensus, wu2014k}), so far, we only know of an $11/7$-approximation algorithm for the median objective~\cite{ailon2008aggregating}, and a better than 2-approximation for the center objective~\cite{DK25}.

Unfortunately, the consensus algorithms mentioned above do not account for fairness in clustering.
For example, take the task of detecting community structures within a social network. There may be various valid partitions available based on users' non-sensitive attributes, such as age group or food preferences. However, when creating an overall community partition, it is important for each community to be fair, meaning each group should be fairly represented according to specific protected attributes like race or gender. Specifically, we refer to the concept of fair clustering as introduced by~\cite{chierichetti2017fair}, where a clustering of red-blue points is deemed fair if each cluster individually reflects the population ratio of red and blue points. It raises a key computational question: given a set of clusterings, how can we derive a fair clustering that effectively aggregates the input clusterings? To our knowledge, the current work is the first to investigate this computational fair aggregation challenge.

En route, we consider another innate, perhaps even more fundamental, computational question concerning fairness in clustering -- to transform an arbitrary clustering into its nearest fair clustering. This question is intriguing on its own merit. We might have an effective clustering algorithm for specific tasks, but due to potential biases in the training data, the result might not be fair, even if it is a good clustering of the input. In such cases, we seek to achieve a fair clustering by making minimal changes to the output clusters. This type of postprocessing is evidently essential in most clustering applications because it is imperative to ensure fairness among the clusters produced. In many scenarios, clustering algorithms can inadvertently result in biased groupings that do not adequately represent all protected classes or groups within the dataset. Such skewed representations can lead to unfair treatment or outcomes, particularly when the clusters are used for decision-making or analytical purposes. Therefore, to prevent these biases and promote equity, it is necessary to apply postprocessing techniques that adjust the clusters to reflect a fair distribution of the protected attributes. This ensures that each group is proportionately represented and that the clustering results do not inadvertently favor one group over another, thereby upholding principles of fairness and non-discrimination in the analysis. Similar fairness-related questions have been examined in other contexts, such as ranking candidates~\cite{CelisSV18, chakraborty2022, kliachkin2024fairness}. Surprisingly, this computational task has not garnered significant attention within the broader clustering literature. In this study, we also address this fundamental question of fair clustering. As we will demonstrate later, this question is closely linked to the challenge of creating a fair consensus.

\subsection{Our Contribution}
\paragraph{Closest Fair Clustering} We begin by introducing the study of \emph{Closest Fair Clustering} -- the problem of modifying an arbitrary input clustering to achieve fairness with the minimal number of changes. The input dataset consists of points having a color, either red or blue.
Starting with an arbitrary clustering of this colored data set, the goal is to identify the nearest/closest fair clustering. 
First, we study the case where the two colored subsets are of equal size and propose an optimal algorithm for finding the closest fair clustering in this setting (\cref{sec:equiproportion}).

\begin{restatable} {theorem} {closestequifair} \label{thm:closest-1-1-fair}
    There is an algorithm that, given a clustering over a set of $n$ red-blue colored points with an equal number of red and blue points, outputs a closest {\fair} in $O(n \log n)$ time.
\end{restatable}

Next, we explore a more general scenario where the data set is not perfectly balanced; specifically, the ratio of blue to red points is an arbitrary constant. Due to the symmetry between blue and red points, we assume, without loss of generality, that this ratio is at least one throughout this paper. It is important to note that fair versions of various clustering problems turn out to be \texttt{NP}-hard for such arbitrary ratios, particularly when the ratio exceeds a certain (small) constant greater than one, as demonstrated in~\cite{chierichetti2017fair}. However, that does not imply that the closest fair clustering problem is also \texttt{NP}-hard. We show that the closest fair clustering problem is \texttt{NP}-hard even for any integral ratio strictly greater than one (see~\cref{sec:NP-hard}). We establish the \texttt{NP}-hardness result by providing a reduction from the \emph{$3$-Partition} problem.

Next, we design a near-linear time algorithm that provides a constant-factor approximation for the closest fair clustering problem for arbitrary integral ratios. We refer to $\alpha$-approximation as \emph{$\alpha$-close}, for $\alpha \ge 1$ (see~\cref{sec:prelim} for a formal definition).

\begin{restatable} {theorem} {closestpfair} \label{thm:closest-p-1-fair}
    Consider an integer $p > 1$. There is an algorithm that, given a clustering over a set of $n$ red-blue colored points where the ratio between the total number of blue and red points is $p$, outputs a $17$-close {\fair} in time $O(n \log n)$.
\end{restatable}

To demonstrate the aforementioned result, we employ a two-stage approach. In the first step, we ensure that each input cluster becomes \emph{balanced} -- the size of the blue subpart in each cluster is a multiple of $p$. We develop a $3.5$-approximation algorithm for this purpose (\cref{sec:multiple-of-p-clustering}). Subsequently, we introduce an approximation algorithm that transforms any balanced clustering into a fair clustering that is approximately close, more specifically, $3$-close (\cref{sec:bal-to-fair}). By integrating these two steps (\cref{sec:approx-fair}), we achieve a $17$-approximation guarantee, as stated in the theorem above.

The closest fair clustering problem becomes even more complex when the population ratio between blue and red points is fractional (represented as $p/q$ for two coprime numbers $p$ and $q$). It is worth remarking that the closest fair clustering problem is \texttt{NP}-hard for any arbitrary ratio strictly greater than one (see~\cref{sec:NP-hard}). To design an approximation algorithm, we apply the previously mentioned two-stage approach again. Since the second step is effective regardless of whether the ratio is integral or fractional, we can utilize it unchanged. Our attention must now shift to the initial balancing step, where we need to ensure that the blue subpart of each cluster is a multiple of $p$ and the red subpart is a multiple of $q$.

A straightforward approach to achieve balancing is to apply the balancing step for the integral ratio case twice sequentially: first, to make blue subparts a multiple of $p$ and then to make red subparts a multiple of $q$. However, using reasoning similar to that in the proof of~\cref{thm:combine-p-blue-fair} in~\cref{sec:approx-fair}, we can argue that an $\alpha$-approximation balancing algorithm for the integral case results in a $(\alpha^2 + 2\alpha)$-approximate balancing algorithm for the fractional case. This yields a $19.25$-approximation just for the balancing step, given that our integral balancing achieves a $3.5$-approximation. Therefore, combining both the balancing and the process of achieving a fair clustering from a balance clustering results in an $80$-approximation factor for the closest fair clustering.

However, we propose a more sophisticated balancing algorithm and, through an intricate analysis, manage to avoid the squared increase in the approximation factor, achieving an approximation factor of only $7.5$ (\cref{sec:multiple-of-p-q-clustering}). Finally, by combining both stages, we develop a $33$-approximation algorithm, significantly improving upon the naive bound of an $80$-approximation.

\begin{restatable} {theorem} {closestpqfair} \label{thm:closest-p-q-fair}
    Consider two integers $p,q > 1$. There is an algorithm that, given a clustering over a set of $n$ red-blue colored points where the ratio between the total number of blue and red points is $p/q$, outputs a $33$-close {\fair} in time $O(n \log n)$.
\end{restatable}

\noindent \textbf{Closest Fair Clustering using Fair Correlation Clustering.} It is worth emphasizing the connection between the closest fair clustering and the fair correlation clustering. In correlation clustering, we are given a labeled complete graph where each edge is labeled either $+$ or $-$. The cost of a clustering (of the nodes) is defined as the summation of the number of intercluster $+$ edges and the intracluster $-$ edges. The \emph{fair correlation clustering} problem asks to find a fair clustering with minimum cost. It is easy to observe that given a clustering (on $n$ points), we can find its closest fair clustering by solving the fair correlation clustering problem on the following instance: Create a complete graph on $n$ nodes, where an (undirected) edge $(u,v)$ is labeled $+$ if both $u,v$ belong to the same cluster in the input clustering; otherwise it is labeled $-$.

Due to the above reduction, by deploying the current state-of-the-art approximation algorithms for the fair correlation clustering, we immediately get approximation algorithms for the closest fair clustering problem; however, this leads to a much worse approximation factor. The correlation clustering problem (even on complete graphs) is already \texttt{NP}-hard (in fact, \texttt{APX}-hard), and its fair variant remains \texttt{NP}-hard even when the blue-to-red ratio is one~\cite{ahmadi2020fair}. In contrast, we provide an exact algorithm for the closest fair clustering problem when the ratio is one (\cref{thm:closest-1-1-fair}). Thus, we, in fact, get a separation between these two problems -- the closest fair clustering problem and the fair correlation clustering problem -- for the ratio of one. If we had used the state-of-the-art fair correlation clustering algorithm~\cite{pmlr-v108-ahmadian20a, ahmadi2020fair}, we would have only obtained an $O(1)$ approximation factor for the closest fair clustering problem in this case.

For arbitrary ratios, even when the ratio is an integer $p$, the state-of-the-art fair correlation clustering algorithms~\cite{pmlr-v108-ahmadian20a, ahmadi2020fair} give us an approximation factor of $O(p^2)$. In contrast, we get an approximation factor of $17$ (\cref{thm:closest-p-1-fair}) for the closest fair clustering problem. Note that the approximation factor obtained by our algorithm is entirely independent of the blue-to-red ratio (i.e., independent of $p$). Further, due to the hidden constant of $O(\cdot)$ notation in the approximation factor of the algorithms of the fair correlation clustering~\cite{pmlr-v108-ahmadian20a, ahmadi2020fair}, the implied approximation bound is a much larger constant even for $p=2$ compared to our algorithm.

\paragraph{Fair Consensus Clustering} Next, we focus on the challenge of combining multiple input clusterings into a single clustering while ensuring fairness in the resulting output. We address this fair consensus clustering problem under a generalized aggregation objective function, known as the \emph{generalized mean objective} or simply the \emph{$\ell$-mean objective}\footnote{In the existing literature, this exponent parameter is usually referred to as $q$, but we use $\ell$ here to avoid potential confusion with the parameters representing the blue-to-red group size ratio.}. This generalized mean objective encompasses a variety of classical optimization objectives, ranging from the \emph{median} when $\ell = 1$ to the \emph{center} when $\ell = \infty$, and it has attracted considerable interest in the broader clustering literature~\cite{ChlamtacMV22}. The consensus clustering problem, even without fairness constraints, is \texttt{APX}-hard~\cite{BonizzoniVDJ08}. When incorporating fairness constraints, the problem becomes only harder, making it inevitable to incur a constant (multiplicative) approximation factor.

We introduce a generic approximation algorithm for fair consensus clustering by leveraging the solution to the closest fair clustering as a foundational element (\cref{sec:approx-consensus}). Notably, our approach to fair consensus clustering is effective regardless of the parameter $\ell$ in the generalized mean objective. As a result, it is applicable to both the median objective (when $\ell = 1$) and the center objective (when $\ell = \infty$), as well as other more generalized objectives within this spectrum.

Using the result of~\cref{thm:closest-1-1-fair} together with our generic fair consensus approximation algorithm, we achieve a $3$-approximation to the fair consensus clustering for the case when the whole population is perfectly balanced.

\begin{restatable} {theorem} {consensusequifair} \label{thm:consensus-1-1-fair}
    Consider any $\ell \ge 1$. There is an algorithm that, given a set of $m$ clusterings each over a set of $n$ red-blue colored points with an equal number of red and blue points, outputs a $3$-approximate $\ell$-mean fair consensus clustering in time $O(m^2 n^2)$.
\end{restatable}

Next, by applying the result from~\cref{thm:closest-p-1-fair} in conjunction with our general fair consensus approximation algorithm, we attain a $19$-approximation for fair consensus clustering in scenarios where the entire population has an integral ratio of blue to red points.

\begin{restatable} {theorem} {consensuspfair} \label{thm:consensus-p-1-fair}
    Consider any $\ell \ge 1$ and an integer $p > 1$. There is an algorithm that, given a set of $m$ clusterings each over a set of $n$ red-blue colored points where the ratio between the total number of blue and red points is $p$, outputs a $19$-approximate $\ell$-mean fair consensus clustering in time $O(m^2 n^2)$.
\end{restatable}

Finally, by utilizing the result from \cref{thm:closest-p-q-fair} along with our general fair consensus approximation algorithm, we obtain a $35$-approximation for fair consensus clustering when the entire population exhibits an arbitrary fractional ratio of blue to red points.

\begin{restatable} {theorem} {consensuspqfair} \label{thm:consensus-p-q-fair}
    Consider any $\ell \ge 1$ and two integers $p,q > 1$. There is an algorithm that, given a set of $m$ clusterings each over a set of $n$ red-blue colored points where the ratio between the total number of blue and red points is $p/q$, outputs a $35$-approximate $\ell$-mean fair consensus clustering in time $O(m^2 n^2)$.
\end{restatable}

We summarize our main results below.

\begin{table}[!h]\label{tab:result}
\begin{center}
    \begin{tabular}{|c||c|c|c|}
        \hline
        Clustering Problem & Perfectly Balanced ($1:1$) & Integral Ratio ($p:1$) & Fractional Ratio ($p:q$)\\ \hline
        Closet Fair & Optimal (\cref{thm:closest-1-1-fair}) & 17-approx. (\cref{thm:closest-p-1-fair}) & 33-approx. (\cref{thm:closest-p-q-fair})\\
                  \hline
                  \hline
        Fair Consensus & 3-approx. (\cref{thm:consensus-1-1-fair}) & 19-approx. (\cref{thm:consensus-p-1-fair}) & 35-approx. (\cref{thm:consensus-p-q-fair})\\
                  \hline
                  \hline
    \end{tabular}
    \caption{Table summarizing main results}
\end{center}
\end{table}


\section{Preliminaries}\label{sec:prelim}

\paragraph{Notations. }In this paper, we consider clusterings over a set $V$ of red-blue colored points. We use $n$ to denote the size of $V$. Throughout this paper, we consider the points in $V$ are colored either red or blue. For any subset of points $S \subseteq V$, we use $\red{S}$ (resp. $\blue{S}$) to denote the subset of all red (resp. blue) points in $S$. Further, without loss of generality, we assume that the number of blue points in the whole set $V$ is at least that of red points, i.e., $|\blue{V}| \ge |\red{V}|$. For positive integers $a,b,p$, we use $a \Pmod b$ to denote that $a$ is congruent to $b$ under modulo $p$.

\paragraph{Consensus clustering. }Between two clusterings $\m{M}, \m{C}$, we define their distance, denoted by $\dist(\m{M}, \m{C})$, to be the number of unordered pairs $i,j \in V$ that are clustered together by one but separated by another. The consensus clustering problem asks to aggregate a set of clustering over the points $V$. In this paper, we consider consensus clustering with respect to a generalized objective function.

\begin{definition}[Generalized Mean Consensus Clustering]
    Consider an exponent parameter $\ell \in \mathbb{R}$. Given a set of $m$ clusterings $\inpset_1,\inpset_2,\cdots,\inpset_m$ over a point set $V$, the \emph{$\ell$-mean Consensus Clustering} problem asks to find a clustering $\m{C}$ (not necessarily from the set of input clusterings) over $V$ that minimizes the objective function $\left( \sum_{i=1}^m \left(\dist(\inpset_i, \m{C})\right)^\ell \right)^{1/\ell}$.
\end{definition}

The generalized mean objective is well-studied in the clustering literature~\cite{ChlamtacMV22}. In the context of consensus clustering, for $\ell=1$, the above problem (note, the objective is now the sum of distances) is referred to as the \emph{median consensus clustering}~\cite{ailon2008aggregating}, and for $q=\infty$, the above problem (note, the objective is now the maximum distance) is the \emph{center consensus clustering}~\cite{DK25}.

\paragraph{Fair clustering. }
\begin{definition}[$\fair$]
    Consider a clustering $\fairsetd$ on a set of red-blue colored points where the ratio between the number of blue and red points is $p/q$, for two integers $p,q \ge 1$. $\fairset$ is called a $\fair$ if for every cluster $F_i \in \fairset$, the ratio between the number of blue and red points $|\blue{F_i}| / |\red{F_i}| = p/q$.  
\end{definition}

Next, we define the closest fair clustering.

\begin{definition}[Closest $\fair$]
    Given a clustering $\inpsetd$, a clustering $\fairmopd$ is called its \emph{closest $\fair$} if for all $\fair$ $\fairset$, $\dist(\inpset, \fairmop) \le \dist(\inpset, \fairset)$. 
\end{definition}

Let $\fairmop$ be an (arbitrary) closest $\fair$. Then, we use \optfair (or simply $\opt$ when it is clear from the context) to denote $\dist(\inpset,\fairmop)$. For any $\alpha \ge 1$, a clustering $\fairset$ is called an \emph{$\alpha$-close} $\fair$ if $\dist(\inpset, \fairset) \leq \alpha \cdot \dist(\inpset, \fairmop)$.

In this paper, we focus on the \emph{fair consensus clustering} problem.

\begin{definition}[Generalized Mean Fair Consensus Clustering]
    Consider an exponent parameter $\ell \in \mathbb{R}$. Given a set of $m$ clusterings $\inpset_1,\inpset_2,\cdots,\inpset_m$ over a point set $V$, the \emph{$\ell$-mean Fair Consensus Clustering} problem asks to find a {\fair} $\m{F}$ (not necessarily from the set of input clusterings) over $V$ that minimizes the objective function $\left(\sum_{i=1}^m \left(\dist(\inpset_i, \m{F})\right)^\ell \right)^{1/\ell}$.
\end{definition}

\paragraph{Balanced clustering. } One of the intermediate steps we use in making an input clustering fair is first to make it \emph{balanced} and then convert a balanced clustering to a fair one.
\begin{definition}[$\bal$]
    Consider a clustering $\m{Q} = \{Q_1, Q_2, \ldots, Q_{\psi}\}$ on a set of red-blue colored points where the irreducible\footnote{Note, a fraction $p/q$ is called \emph{irreducible} if $p$ and $q$ are coprime, i.e., $\textsf{gcd}(p,q)=1$.} ratio between the number of blue and red points is $p/q$, for two integers $p,q \ge 1$. $\m{Q}$ is called a $\bal$ if for every cluster $Q_i \in \m{Q}$, the number of blue points $|\blue{Q_i}|$ is a multiple of $p$, and the number of red points $|\red{Q_i}|$ is a multiple of $q$.  
\end{definition}

Next, we define the closest balanced clustering.

\begin{definition}[Closest $\bal$]
    Given a clustering $\inpsetd$, a clustering $\mopqd$ is called its \emph{closest $\bal$} if for all $\bal$ $\m{Q}$, $\dist(\inpset, \mopq) \le \dist(\inpset, \m{Q})$. 
\end{definition}

Let $\mopq$ be an (arbitrary) closest $\bal$. Then, we use \optbal (or simply $\opt$ when it is clear from the context) to denote $\dist(\inpset,\mopq)$. For any $\alpha \ge 1$, a clustering $\m{Q}$ is called an \emph{$\alpha$-close} $\bal$ if $\dist(\inpset, \m{Q}) \leq \alpha \cdot \dist(\inpset, \mopq)$.

\section{Technical Overview}\label{sec:tech-overview}
Given a set $ V $ of $ n $ points, a clustering $ \mathcal{D} = \{D_1, D_2, \dots, D_\phi\} $ is a partition of $ V $ into disjoint subsets. 
Given two clusters, we define their distance as the total number of pairwise disagreements between them.
In this work, we assume that each point in $ V $ is associated with a color from the set $\{\text{red, blue}\}$. For a cluster $ D_i $, let $ \blu(D_i) $ denote the blue points in $ D_i $, and $ \re(D_i) $ denote the red points in $ D_i $.  
Given a clustering $\mathcal{D}$, with the ratio between the number of blue and red points being $p/q$ for some integer $p,q\ge 1$,  we call $\mathcal{D}$ a \emph{fair clustering} if, for each $ i \in [\phi] $, the ratio between $|\blu(D_i)|$ and $|\re(D_i)|$ is exactly $ p/q $. Next, we discuss how efficiently we can transform a given clustering $ \mathcal{D} $ into its nearest fair clustering under various settings of $ p $ and $ q $.

\paragraph{Optimal Fair Clustering for Equi-Proportionally Balanced Input.} We begin by discussing the case where $ p = q = 1 $ and present a near-linear time algorithm that, given a clustering $ \mathcal{D} $ of $ V $, efficiently finds a closest fair clustering $ \mathcal{C}^* $ (i.e., an optimal fair clustering). Below, we provide a brief outline of this algorithm.  

Before describing the algorithm, we first establish key properties of the optimal fair clustering. In particular, we show that there exists an optimal fair clustering that satisfies specific structural properties. These properties serve as a guide for transforming $ \mathcal{D} $ into $ \mathcal{C}^* $.  

(i) Our first observation considers an arbitrary cluster $ D_i \in \mathcal{D} $. Without loss of generality, assume $ |\blu(D_i)| \geq |\re(D_i)|$. We define the \emph{maximal fair cluster} of $ D_i $, denoted by $ D_i^{\text{max},\text{fair}} $, as the largest subset of $ D_i $ that contains exactly $ |\re(D_i)| $ red and blue points.  
Our first claim is that for each $ i \in [\phi] $, the cluster $ D_i^{\text{max},\text{fair}} $ appears as a cluster in $ \mathcal{C}^* $. The intuition behind this is as follows: it is always beneficial to retain at least $|\re(D_i)|$ red and blue points from $ D_i $ together. Otherwise, if they are distributed across multiple clusters, the overall cost would increase, as these points originate from the same input cluster.  
Furthermore, if $ D_i^{\text{max},\text{fair}} $ were instead included as a subset of a larger cluster in $ \mathcal{C}^* $, then since $ \mathcal{C}^* $ is fair—the larger cluster would need additional blue points from another cluster to maintain fairness. Importantly, the number of foreign blue points added would be equal to the excess blue points removed when forming the maximal fair cluster, ensuring that the overall cost remains unchanged.

(ii) Next, consider a cluster $ C^*_j \in \mathcal{C}^* $. We claim that $ C^*_j $ originates from at most two clusters in the input clustering $ \mathcal{D} $.  
To see why this holds, assume for contradiction that $ C^*_j $ is formed from three distinct clusters $ D_o, D_s, D_t \in \mathcal{D} $. From our earlier observations, each of these clusters can contribute only blue or only red points to $ C^*_j $. Without loss of generality, suppose $ D_o $ contributes $ b_o $ blue points, while $ D_s $ and $ D_t $ contribute $ r_s $ and $ r_t $ red points, respectively. Since $\mathcal{C}^*$ is fair, we must have $ b_o = r_s + r_t $.  
Now, assume without loss of generality that $ r_s \leq r_t $. We can create a new clustering by partitioning $ C^*_j $ into two clusters:  
The first cluster contains $ r_s $ red points from $ D_s $ and $ r_s $ blue points from $ D_o $.  
The second cluster contains $ r_t $ red points from $ D_t $ and $ r_t $ blue points from $ D_o $.  
In this transformation, partitioning $ D_o $ increases the cost by $ r_s \cdot r_t $. However, since the red points from $ D_s $ and $ D_t $ are now separated, the cost decreases by the same amount, $ r_s \cdot r_t $. Thus, the overall cost remains unchanged. Moreover, by construction, the new clustering remains fair.  
In our algorithm, we extend this argument to a more general setting where $ C^*_j $ originates from more than three clusters, showing that a similar reasoning applies to establish the claim.

Using these two properties, we construct $ \mathcal{C}^* $ from $ \mathcal{D} $ through a two-step process. 

In \textbf{Step 1:} Following observation (i), for each input cluster $ D_i $, we extract the maximal fair cluster $ D_i^{\text{max},\text{fair}} $.  
Define the remaining points as $ R_i = D_i \setminus D_i^{\text{max},\text{fair}} $.  
Note that each $ R_i$ forms a monochromatic cluster.

In \textbf{Step 2}, we introduce a greedy strategy to combine $ R_1, R_2, \dots, R_\phi $ into a set of fair clusters. 
We begin by selecting the smallest cluster $ R_i $. Without loss of generality, assume $ R_i $ consists of blue points. We then identify the smallest red cluster $ R_j $ and form a new cluster by combining $ R_i $ with $ |R_i| $ red points from $ R_j $. This process is repeated iteratively for the remaining monochromatic clusters. (See~\cref{fig:equiproportion}.) 

\begin{figure}[t]
\centering
    \includegraphics[width=11cm]{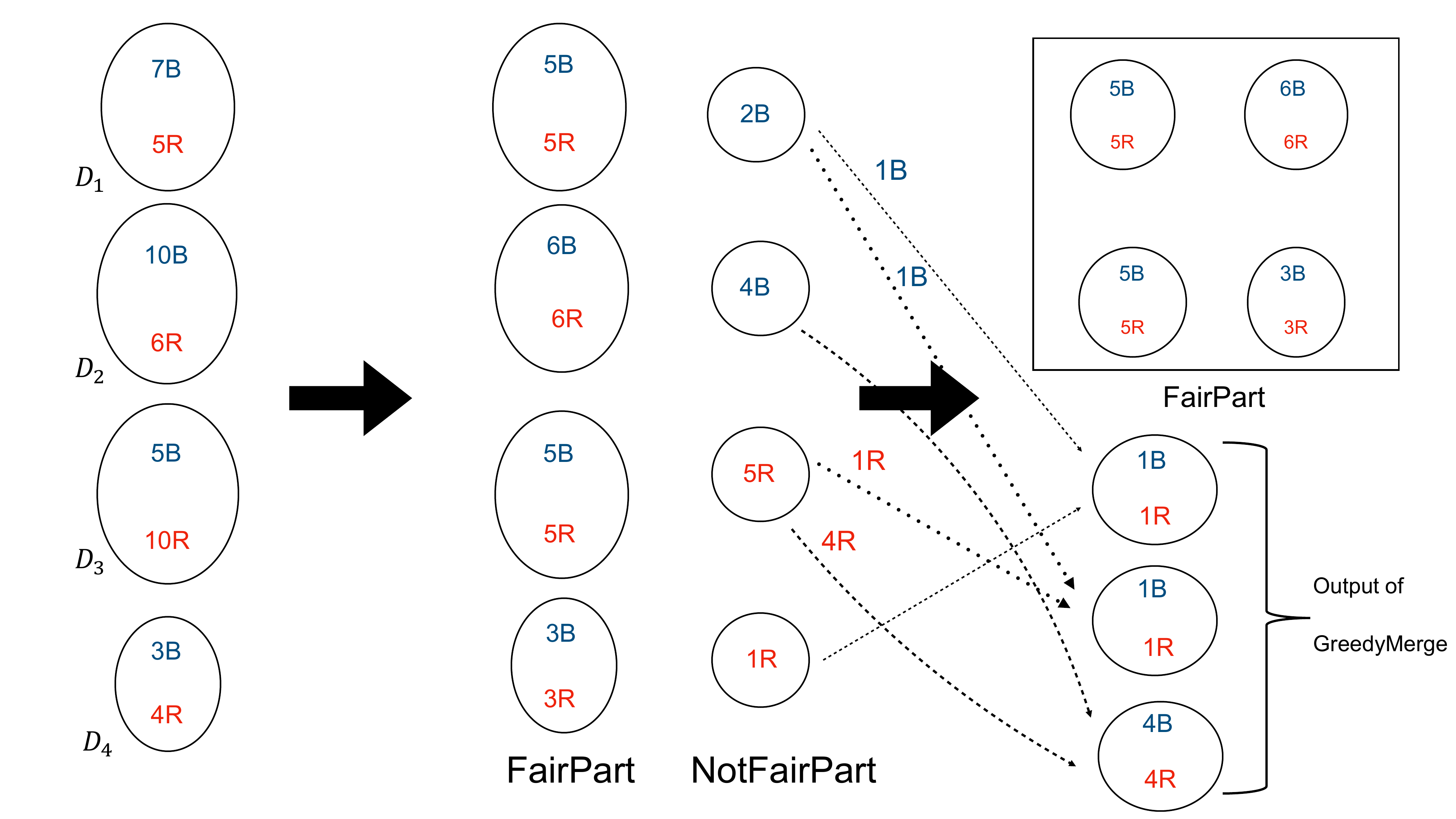}
    \caption{\label{fig:equiproportion}Visualization of the algorithm to find a closest fair clustering for equi-proportionally balanced input. In each cluster, \textcolor{blue}{$xB$} (resp., \textcolor{red}{$yR$}) denotes that the number of \textcolor{blue}{blue} (resp., \textcolor{red}{red}) points is $x$ (resp., $y$).
    }
\end{figure}

By construction, this merging strategy preserves the fairness property. Additionally, since $ \mathcal{C}^* $ is fair and each $ D_i^{\text{max},\text{fair}} $ is included in $ \mathcal{C}^* $, the total number of red points in $ \bigcup_{i \in [\phi]} R_i $ matches the total number of blue points. Thus, our greedy strategy ensures a fair partition of all remaining points.  

Next, we show the optimality of this greedy strategy by leveraging Property (ii) of the optimal fair cluster discussed earlier. We claim that selecting the smallest $ R_i $ at each step is indeed optimal.  
Intuitively, it is preferable to keep all points of $ R_i $ together when they originate from the same input cluster.
Further following Property (ii), the optimal fair cluster can contain points from at most two input clusters. Now, to make $ R_i $ fair while keeping it intact, we must merge it with a cluster $ R_j $ of the opposite color such that $ |R_j| \geq |R_i| $. This condition is best satisfied by processing clusters in \emph{non-decreasing} order of size, which justifies the rationale behind our greedy strategy.

\paragraph{Constant Approximation for Closest Fair Clustering (General $p,q$). } We now consider a more general version of the problem for arbitrary values of $ p $ and $ q $. Without loss of generality, we assume $ p \geq q $. The algorithm proceeds in two steps:

\begin{enumerate}
    \item \textbf{Balancing Step:} Given an input clustering $\mathcal{D} = \{D_1, D_2, \dots, D_\phi\}$, we transform it into a new clustering $\mathcal{T} = \{T_1, T_2, \dots, T_\gamma\}$ such that each cluster $T_i$ contains a number of blue points that is a multiple of $ p $ and a number of red points that is a multiple of $ q $.

    \item \textbf{Making-Fair Step}: We further process the clusters $ T_1, T_2, \dots, T_\gamma $ to obtain a fair clustering $\mathcal{F} = \{F_1, F_2, \dots, F_\zeta\}$.
\end{enumerate}

Let $\mathcal{F}^*$ denote an optimal fair clustering closest to the input clustering $\mathcal{D}$, with an associated distance of $\text{OPT}$. By definition, $\mathcal{F}^*$ is also balanced. 

If we propose an $\alpha$-approximation algorithm for the balancing step, then the distance between $\mathcal{D}$ and $\mathcal{T}$ is at most $\alpha \cdot \text{OPT}$. Furthermore, since the distance between $\mathcal{F}^*$ and $\mathcal{T}$ is at most $(1+\alpha) \cdot \text{OPT}$ (by the triangle inequality), getting a $\beta$-approximation algorithm for the making-fair step ensures that the distance between $\mathcal{T}$ and $\mathcal{F}$ is at most $(1+\alpha) \beta \cdot \text{OPT}$.

Thus, the total distance between the input clustering $\mathcal{D}$ and the final output fair clustering $\mathcal{F}$ is at most $(\alpha + \beta + \alpha \beta) \cdot \text{OPT}$, leading to an overall approximation ratio of $(\alpha + \beta + \alpha \beta)$. This guarantees a constant-factor approximation when $\alpha$ and $\beta$ are constants.

\vspace{2mm}
\noindent \textbf{$3.5$-approximation Balancing Algorithm for $p:1$:} We begin by analyzing the balancing step for the case where $ q = 1 $. For this, we present a $ 3.5 $-approximation algorithm. 

For every cluster $D_i$, we define the \emph{surplus} of $D_i$ as $(|\blu(D_i)| \mod p)$ and its deficit as $(p-(|\blu(D_i)| \mod p))$. To balance each cluster, we either need to remove surplus points or add the required number of deficit points from another cluster. For each cluster, we define the \emph{cut cost} as the cost of removing surplus points and the \emph{merge cost} as the cost of adding deficit points. 

Next, we define our cut and merge strategy. We start by classifying clusters into two categories. Call a cluster \emph{merge cluster} if its surplus is at least $p/2$; i.e., its merge cost is smaller than its cut cost. Otherwise, call it a \emph{cut cluster}. For a merge cluster $D_i$, it is more efficient to add $(p-(|\blu(D_i)| \mod p))$ points (its deficit), whereas, for a cut cluster $D_i$, it is optimal to remove $(|\blu(D_i)| \mod p)$ points (its surplus). A first-line approach is to remove surplus points from cut clusters and merge them into merge clusters, ensuring that the number of blue points in each cluster becomes a multiple of $p$. We start with this approach, referred to as \emph{AlgoGeneral}. However, we cannot guarantee that the total surplus across all clusters matches the total deficit. As a result, after this step, we may be left with only cut clusters or only merge clusters.

If only merge clusters remain, we cannot satisfy all merge requests since we need to cut some clusters to provide the necessary points. This is challenging, as cutting is now costlier than merging for the remaining unbalanced clusters. The key intuition here is that if cutting is necessary, we need a nontrivial strategy to optimize the cost. To achieve this, we sort the merge clusters in non-increasing order of $cut cost- merge cost$, ensuring optimal cutting decisions. An important observation is that the total deficit $W$ across all clusters is a multiple of $p$. This follows from the existence of a fair clustering, which guarantees that there exists a way to process the input set of clusters such that the number of blue points in each cluster is a multiple of $p$. Additionally, if we cut $k$ points from a cluster $D_i$ to fulfill the $k$ deficit of another cluster $D_j$, we effectively balance a total deficit of $p$ points. This follows intuitively because if $k$ is the surplus of $D_i$, then its deficit is $p-k$, ensuring that the combined deficit of $D_i$ and $D_j$ sums up to $p$. Since the total deficit is $W$, we repeat this operation exactly $W/p$ times, which is crucial for bounding the overall approximation cost. (See~\cref{fig:merge-p}.) 

\begin{figure}[t]
\centering
    \includegraphics[width=11cm]{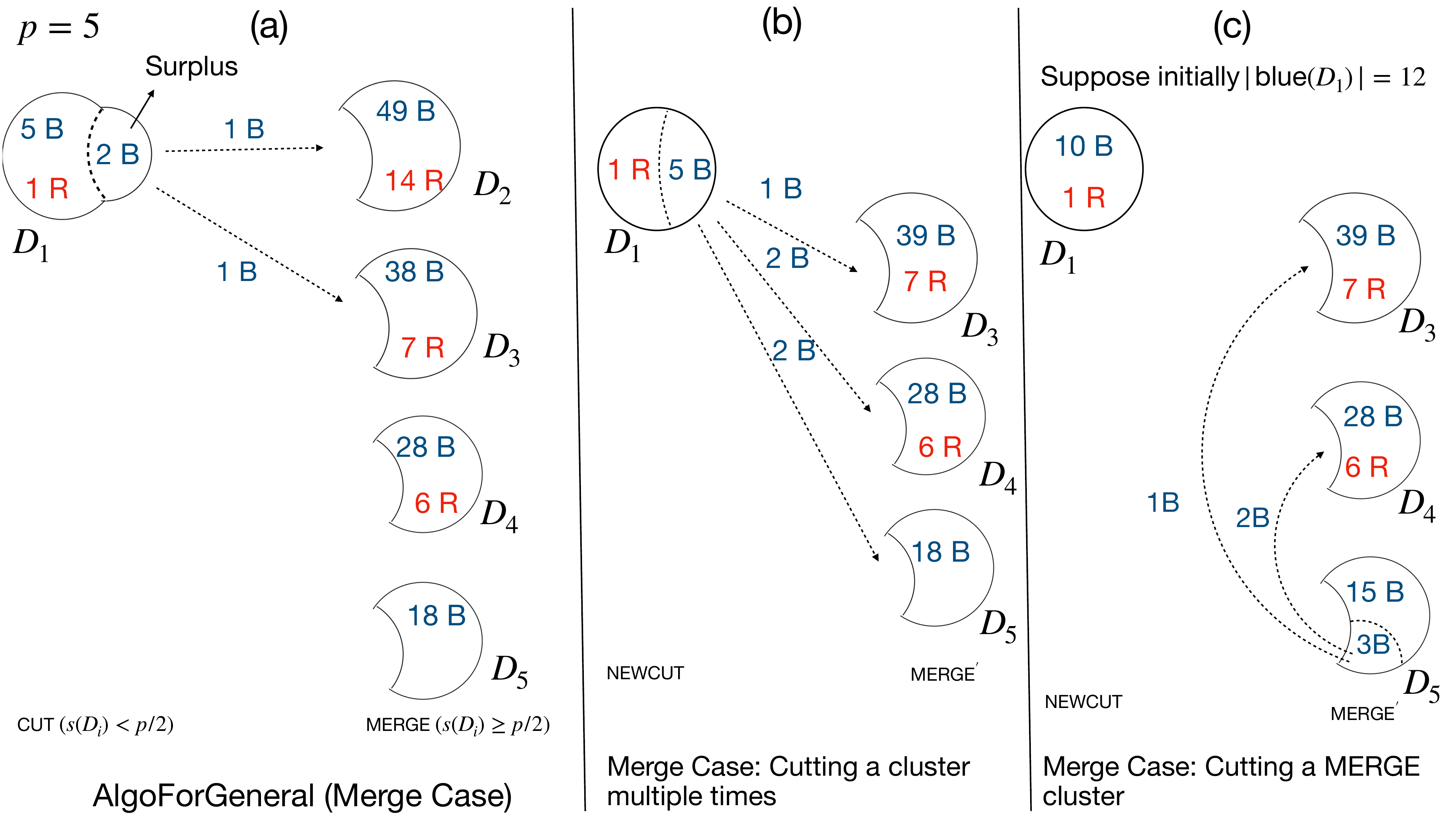}
    \caption{\label{fig:merge-p}Visualization of the scenario when only merge clusters remain after executing AlgoGeneral, which converts (a) to (b). Now, the total deficit $W=5$, and thus, depending on the $cut cost- merge cost$, either an already cut cluster needs to be cut further (depicted in (b)), or a single merge cluster needs to be cut (depicted in (c)). In each cluster, \textcolor{blue}{$xB$} (resp., \textcolor{red}{$yR$}) denotes that the number of \textcolor{blue}{blue} (resp., \textcolor{red}{red}) points is $x$ (resp., $y$).
    }
\end{figure}

Next, we consider the simpler case where, after the initial cut-merge processing, we are left only with cut clusters, each having a surplus of size $<p/2$. Here, we simply remove these surplus points from each cluster and combine them to form clusters of size exactly $p$, containing only blue points. A key question is why forming clusters of size $p$ is the optimal choice rather than larger clusters. The reason is that ensuring size $p$ may require partitioning surplus points from specific clusters, incurring additional costs. For example, consider the case where there are three cut clusters with surpluses of $p/3$, $p/3$, and $p/2-1$. Following our strategy to form a cluster of size $p$, we must partition at least one existing cluster. One might wonder whether it is more efficient to combine even more clusters and create a cluster of size $2p$ instead of $p$. However, through careful analysis, we show that accommodating clusters of size $>p$ results in higher merge costs than the partition cost. Thus, forming clusters of size exactly $p$ is the optimal strategy for combining surplus points efficiently. (See~\cref{fig:cut-p}.) 

\begin{figure}[t]
\centering
    \includegraphics[width=11cm]{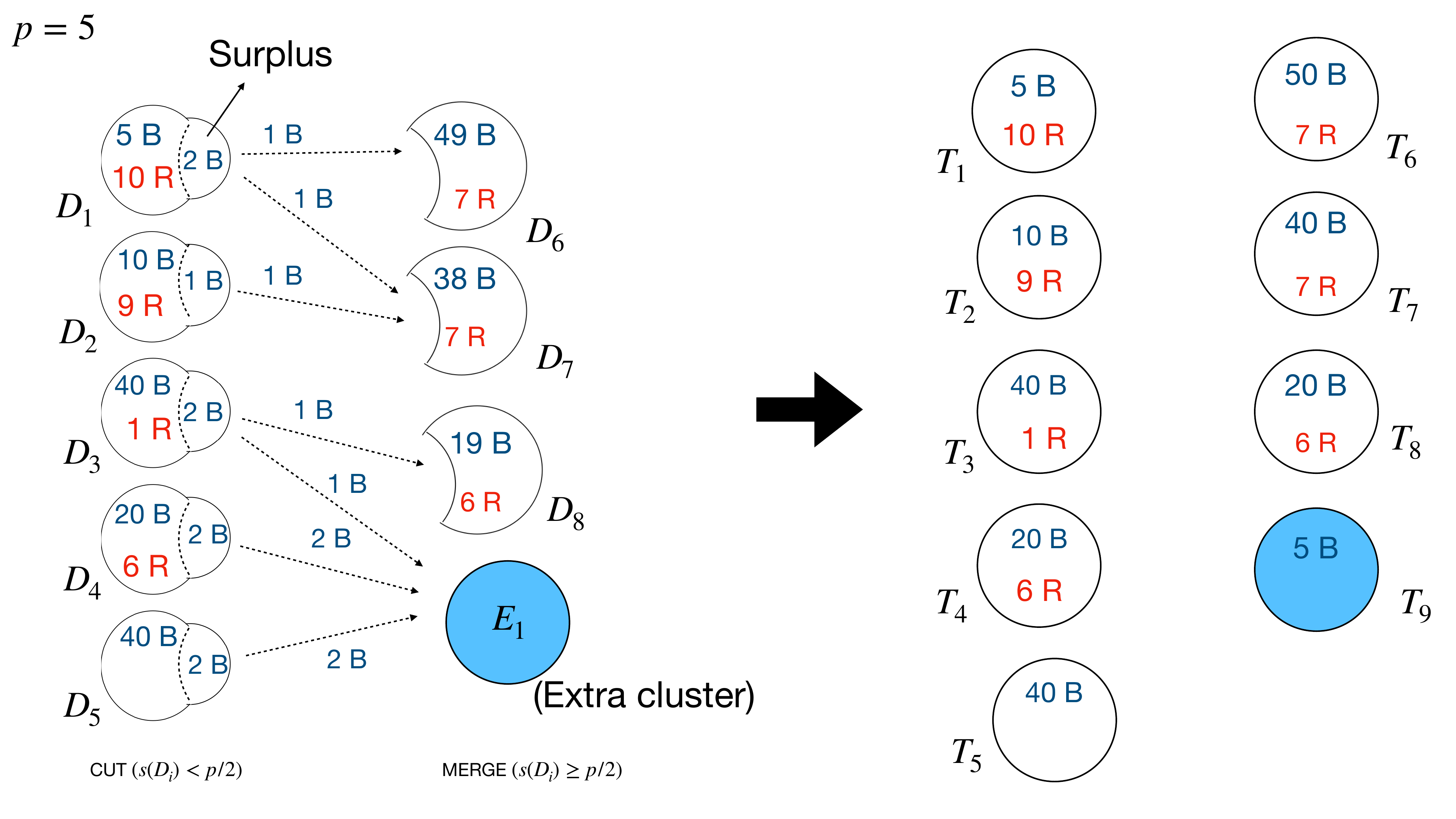}
    \caption{\label{fig:cut-p}Visualization of the scenario when only cut clusters remain after executing AlgoGeneral. The extra surplus points of clusters $D_3, D_4, D_5$ form an extra cluster $E_1$ of size $p$, containing only blue points. In each cluster, \textcolor{blue}{$xB$} (resp., \textcolor{red}{$yR$}) denotes that the number of \textcolor{blue}{blue} (resp., \textcolor{red}{red}) points is $x$ (resp., $y$).
    }
\end{figure}

\vspace{2mm}

\noindent \textbf{$7.5$-approximation Balancing Algorithm for $p:q$:} We now turn our attention to balancing when both $p$ and $q$ are greater than 1. We build on the fundamental concept of the balancing step previously outlined for the $p:1$ scenario. In this context, we aim to ensure that the blue subpart of each cluster is a multiple of $p$ and the red subpart is a multiple of $q$. This involves examining surpluses and deficits for both the blue and red subparts separately, similar to the $p:1$ case. Consequently, each cluster $D_i$ can be categorized into one of four types: blue merge - red cut, blue cut - red merge, blue cut - red cut, and blue merge - red merge, where blue (resp., red) merge denotes a deficit of at most $p/2$ points (resp., $q/2$ points) and blue (resp., red) cut denotes a surplus of at most $p/2$ points (resp., $q/2$ points). Depending on its type, we employ our previously outlined balancing algorithm by addressing the blue and red components of a cluster $D_i$ separately. 


However, with two colors, our analysis now encounters additional challenges in ensuring an approximation bound. For example, consider a cluster $D_i$ of type blue merge - red merge. An optimal strategy might involve merging external blue points (denoted as $B_i$) to address the blue deficit and external red points (denoted as $R_i$) to address the red deficit. However, this introduces an intercluster cost between the sets $B_i$ and $R_i$, calculated as $|R_i|\cdot |B_i|$. In this context, it can be noted that $|R_i| \le q/2 \le |\re(D_i)|$, which means the additional cost is bounded above by $\opt_{D_i}$, where $\opt_{D_i}$ represents the optimal cost incurred by the points within cluster $D_i$. To understand this, note that $\opt_{D_i}$ must cover at least the cost of resolving its blue deficit, which is at least $|\re(D_i)|\cdot |B_i|$. Summing these costs across all such clusters results in an additional cost that is bounded above by the overall optimal cost. 

For the remaining three cases, we show that by adapting the techniques and arguments from the \( p:1 \) case, we achieve the claimed approximation guarantee. Finally, by summing all these costs, we establish that the overall approximation factor is bounded by \( 7.5 \).

\vspace{2mm}

\noindent \textbf{Making the Balance Clustering Fair:}
Given a balanced clustering $\mathcal{T} = \{T_1, T_2, \dots, T_\gamma\}$ here we discuss the step to convert it to a fair clustering $\mathcal{F} = \{F_1, F_2, \dots, F_\zeta\}$. For this, we present a $ 3$-approximation algorithm. 

Without loss of generality, assume $ p \geq q $. We define a cluster $ T_i $ as $\tr$ if it has surplus red points, meaning $ |\blue{\outcl{i}}| < (p/q) |\red{\outcl{i}}| $.  
Let $ S_i $ denote the surplus of red points in $ T_i $, given by $ |S_i| = |\red{\outcl{i}}| - (q/p) |\blue{\outcl{i}}| $.
Otherwise, we define $ T_i $ as $\tb$ if it has a deficit of red points, meaning $ |\blue{\outcl{i}}| > (p/q) |\red{\outcl{i}}| $.  
Let $ D_i $ represent the red point deficit in $ T_i $, given by $ |D_i| = (q/p) |\blue{\outcl{i}}| - |\red{\outcl{i}}| $.

Next, we discuss how to eliminate surplus red points and redistribute them to address the deficits. To achieve this, we simply remove the surplus points from clusters in $ \tr $ and reassign them to clusters in $ \tb $ to compensate for the deficit.  
A crucial observation is that the existence of a fair clustering guarantees that the total surplus from all clusters in $ \tr $ matches the total deficit across all clusters in $ \tb $, thereby making the redistribution process well-defined.
We now proceed to analyze the approximation.

We first claim that since $ p \ge q $, removing surplus points and redistributing them to fulfill deficits is indeed optimal.  
A key subtlety we need to address is the case where a specific surplus $ S_i $ is distributed across multiple clusters to fulfill their deficits. This redistribution involves splitting $ S_i $, which incurs an additional cost. However, even in an optimal fair clustering, $ S_i $ must be transformed into a fair cluster, requiring the merging of $ (p/q)|S_i| $ blue points with $ S_i $. 
The cost of this merging is $ (p/q)|S_i|^2 $, while the maximum splitting cost of $ S_i $ is at most $ |S_i|^2/2 $. Since $ p \ge q $, the splitting cost of $ S_i $ remains bounded by its merging cost, ensuring the overall $3$-approximation guarantee. In the case where multiple surplus clusters contribute to fulfilling a deficit $ D_i $, our algorithm incurs an intra-cluster cost of at most $ |D_i|^2/2 $. Since $ p \ge q $, a similar argument as before ensures that, despite this additional cost, we still remain within the claimed approximation guarantee.

\paragraph{Fair Consensus Clustering. }
Next, we propose an algorithm for \emph{Fair Consensus Clustering} that minimizes the \(\ell\)-mean objective for any \(\ell \geq 1\).  
Formally, given \( m \) clusterings \(\inpset_1, \inpset_2, \dots, \inpset_m\) on a set \( V \), the goal is to find a fair clustering \(\m{F}^*\) that minimizes  
$\left(\sum_{j=1}^m \left(\dist(\inpset_j, \m{F}^*)\right)^\ell \right)^{1/\ell}$.
To achieve this, we design a simple algorithm leveraging our \emph{Closest Fair Clustering} algorithm discussed above.

We begin by computing, for each \(\inpset_i\), an \(\alpha\)-close fair clustering \(\m{F}_i\) using the \emph{Closest Fair Clustering algorithm}. Then, we output the fair clustering \(\m{F}_k\) that minimizes the overall \(\ell\)-mean objective.  
Next, we show that this approach guarantees a \((2+\alpha)\)-approximation for \emph{Fair Consensus Clustering}.

Let \(\m{F}^*\) be an optimal fair consensus clustering, and let \(\m{D}_i\) be the input clustering closest to \(\m{F}^*\). Using our \(\alpha\)-approximate \emph{Closest Fair Clustering} algorithm, we obtain a fair clustering \(\m{F}_i\) such that $\dist(\m{D}_i, \m{F}_i) \leq \alpha \dist(\m{D}_i, \m{F}^*)$.
 By applying the triangle inequality, we get $\dist(\m{F}_i, \m{F}^*) \leq (1+\alpha) \dist(\m{D}_i, \m{F}^*)$.
Now, for any other input clustering \(\m{D}_j\), we have  

\[
\dist(\m{D}_j, \m{F}_i) \leq \dist(\m{D}_j, \m{F}^*) + \dist(\m{F}_i, \m{F}^*) \leq (2+\alpha) \dist(\m{D}_j, \m{F}^*),
\]  

since by assumption, \(\m{D}_i\) is the input closest to \(\m{F}^*\), ensuring \(\dist(\m{D}_i, \m{F}^*) \leq \dist(\m{D}_j, \m{F}^*)\).  

Thus, we conclude that there exists an input clustering \(\m{D}_i\) whose corresponding fair clustering \(\m{F}_i\) achieves a \((2+\alpha)\)-approximation. Since our algorithm selects the fair clustering $\m{F}_k$ that minimizes the overall objective, it guarantees an overall approximation of \((2+\alpha)\).  

Moreover, as our proposed \emph{Closest Fair Clustering} algorithm achieves a constant approximation factor \(\alpha\), this ensures that our \emph{Fair Consensus Clustering} algorithm also provides a constant-factor approximation.



\section{Conclusion}\label{sec:conclusion}

In this paper, we initiate the study of closest fair clustering and fair consensus clustering problems. We focus on datasets where each point is associated with a binary protected attribute (red or blue). First, we study the problem of transforming an arbitrary clustering into a fair clustering while minimizing modifications. For perfectly balanced datasets with an equal number of red and blue points, we propose an optimal algorithm. We then extend our approach to more general cases with arbitrary ratios and develop constant-factor approximation algorithms.
Building on our closest fair clustering algorithm, we develop a constant-factor approximation algorithm for Fair Consensus Clustering. 

A natural open question is whether the approximation factors for both problems can be improved. For the closest fair clustering problem, directly achieving a better approximation is of particular interest. In contrast, the fair consensus clustering problem is effectively solved through a reduction to the closest fair clustering, making it compelling to explore whether a tighter reduction can be achieved in terms of the approximation guarantee.  
Another promising direction is extending these methods to handle non-binary and potentially overlapping protected groups, broadening the applicability of fair clustering frameworks.

\acks{Diptarka Chakraborty was supported in part by an MoE AcRF Tier 2 grant (MOE-T2EP20221-0009), an MoE AcRF Tier 1 grant (T1 251RES2303), and a Google South \& South-East Asia Research Award. Kushagra Chatterjee was supported by an MoE AcRF Tier 2 grant (MOE-T2EP20221-0009). Debarati Das was supported in part by NSF grant 2337832. A part of the research was done when Romina Nobahari was doing an internship at the National University of Singapore, where she was supported by an MoE AcRF Tier 2 grant (MOE-T2EP20221-0009).}


\bibliography{reference}

\newpage
\appendix

\paragraph{Organization of Appendix. }First, in~\cref{sec:equiproportion}, we introduce an exact algorithm for identifying the closest fair clustering when the dataset contains an equal number of blue and red points. Subsequently, in~\cref{sec:approx-fair}, we outline a two-stage algorithmic framework that establishes constant factor approximation guarantees for the closest fair clustering problem in more general scenarios where the ratio of blue to red points is arbitrary. Following this, in~\cref{sec:multiple-of-p-clustering}, we present an approximation algorithm for the initial balancing step when the ratio is an integer, and in~\cref{sec:multiple-of-p-q-clustering}, we provide an approximation algorithm for cases where the ratio is fractional. An approximation algorithm for converting a balanced cluster into a fair one, which is the second step, is detailed in~\cref{sec:bal-to-fair}. In~\cref{sec:approx-consensus}, we leverage the results of the closest clustering problem to achieve constant-factor approximation guarantees for the fair consensus clustering problem. In~\cref{sec:NP-hard}, we establish \texttt{NP}-hardness of the closest fair clustering problem for arbitrary ratio of blue to red points.

\section{Optimal Fair Clustering for Equi-Proportionally Balanced Input}\label{sec:equiproportion}

Given an input clustering $\inpsetd$ consisting of 
an equal number of red and blue points, we propose a linear-time algorithm to compute a closest fair clustering $\optset$ to $\inpset$. Formally, we show the following. 

\closestequifair*

\subsection{Details of the Algorithm} \label{subsec:algorithms.for.equiproportion} 


Our algorithm (\cref{alg:FindClosestFair}) consists of two key operations: \emph{cut} and \emph{merge}. Given a cluster $ \gencl{i} $, the cut operation removes a subset of points $ S $ from $ \gencl{i} $, resulting in two cluster $ \gencl{i}\setminus S $ and $ S $. Conversely, the merge operation adds a subset of points $ M $ to $ \gencl{i} $, forming a new cluster $ \gencl{i}\cup M$.

We rely on two procedures: $ \makeitfair $ (\cref{alg:MakeItFair}) and $\greedymerge$ (\cref{alg:GreedyMerge}). Our algorithm starts with iterating over all the clusters $ \inpcl{i}\in \inpset $, and for each $ \inpcl{i} $, applies $ \makeitfair(\inpcl{i}) $, which do the following. Cut a fair subset of maximum size from the cluster $ \inpcl{i} $, and let $ F =\min(|\red{\inpcl{i}}|,|\blue{\inpcl{i}}|)  $. In other words, after cutting, this new cluster contains $ F $ red points and $ F $ blue points. Denote this cluster by $\fairpart_{\inpcl{i}}$. Points in $\inpcl{i}$ outside $ \fairpart_{\inpcl{i}} $, if there remains, form a not fair cluster $ \nfairpart_{\inpcl{i}} $ of size $ |\red{\inpcl{i}}- \blue{\inpcl{i}}| $, in which every point has the same color. Consider the two sets
\begin{align}
    \text{NotFairRed} &= \{ \nfairpart_{\inpcl{i}}| \text{ points of } \nfairpart_{\inpcl{i}} \text{are colored red} \} \nonumber \\
    \text{NotFairBlue} &= \{ \nfairpart_{\inpcl{i}}| \text{ points of } \nfairpart_{\inpcl{i}} \text{are colored blue} \}. \nonumber 
\end{align}
Note that all the clusters that remain unfair belong to either $ \text{NotFairRed} $ or $ \text{NotFairBlue} $; hence the number of red points in $ \text{NotFairRed} $ matches the number of blue points in $ \text{NotFairBlue} $. The algorithm now applies $ \greedymerge(\text{NotFairRed}, \text{NotFairBlue})$, which first sorts clusters in these sets in non-decreasing order based on their clusters' size. Finally, it iteratively cuts subsets of clusters from $ \text{NotFairRed} $ and merges them into subsets of clusters from $ \text{NotFairBlue} $ so that the newly created clusters are fair. At the end of this procedure, the whole clustering achieves fairness. We provide the pseudocodes for detailed algorithms.

\textbf{Runtime Analysis of \cref{alg:FindClosestFair}}: In the subroutine call \cref{alg:MakeItFair} of \cref{alg:FindClosestFair} we consider each cluster $D_i \in \inpset$ and divide it into two parts FairPart and NotFairPart. For each cluster $D_i \in \inpset$ this would take $O(|D_i|)$ time because we need to read the whole cluster $D_i$ to split it into FairPart and NotFairPart. Hence, the subroutine call \cref{alg:MakeItFair} would take at most $O(\sum_{D_i \in \inpset}|D_i|) = O(n)$ time. In the subroutine call \cref{alg:GreedyMerge} of \cref{alg:FindClosestFair} the input is a set of monochromatic clusters $\m{R} = \{R_1, R_2, \ldots, R_t\}$ and $\m{B} = \{B_1, B_2, \ldots, B_m\}$. Sorting clusters in $ \m{R} $ and $ \m{B} $ costs $O(|\m{R}|\log(|\m{R}|)) $ and $ O(|\m{B}|\log(|\m{B}|)) $, which are $ O(n\log(n)) $ time. At each step of~\cref{alg:GreedyMerge}, it reads two sets $R_j$ and $B_k$ and merges them. This would take $O(\sum_{j = 1}^t|R_j| + \sum_{k = 1}^m |B_k|) = O(n)$ time. Hence, combining the time complexities of two subroutine calls, we get that the time complexity of \cref{alg:FindClosestFair} is $O(n\log(n))$.

\begin{algorithm2e}
\DontPrintSemicolon
\caption{FindClosestFair($\inpset$)}\label{alg:FindClosestFair}
\KwData{Input set of clusters $\inpset = \{\inpcl{1}, \inpcl{2}, \cdots, \inpcl{k}\}$.}
\KwResult{A fair clustering $\fairset $}.
Initialize three sets FairClusters = $\emptyset$, NotFairRed = $\emptyset$ and NotFairBlue = $\emptyset$ \;
\For{$i = 1$ to $k$}{
    \If{$\inpcl{i}$ is not fair}{
        (FairPart, NotFairPart) = $\makeitfair(\inpcl{i})$ \;
        \If{color of NotFairPart is red}{
            Add NotFairPart in the set NotFairRed \;
        }
        \Else{
            Add NotFairPart in the set NotFairBlue \;
        }
        Add FairPart in the set FairClusters \;
    }
    \Else{
        Add $\inpcl{i}$ in the set FairClusters \;
    }
}
$\fairset = \greedymerge(\{\text{NotFairRed}\}, \{\text{NotFairBlue}\})$

\Return{$\fairset \cup \text{FairClusters}$}
\end{algorithm2e}

\begin{algorithm2e}
\DontPrintSemicolon
    \caption{$\makeitfair(\inpcl{i})$}\label{alg:MakeItFair}
    \KwData{An unfair cluster $\inpcl{i}$, that is, the number of red vertices is not equal to the number of blue vertices in $\inpcl{i}$.}
\KwResult{Output: It splits and returns two parts of cluster $\inpset$, a fair part, and a not fair part} 
\If{$|\red{\inpcl{i}}| > |\blue{\inpcl{i}}|$} {
    Split $\inpcl{i}$ into two parts - NotFairpart and FairPart \;
    Construct a set, NotFairPart from $C$ such that $|\text{NotFairPart}| = |\red{\inpcl{i}}| - |\blue{\inpcl{i}}|$ and each vertex in NotFairPart is \emph{colored red} \;
    FairPart = $\inpcl{i} \setminus $ NotFairPart \;
    }
    \Else {
        Split $\inpcl{i}$ into two parts - NotFairpart and FairPart (same as before) \;
        Construct a set, NotFairPart from $\inpcl{i}$ such that $|\text{NotFairPart}| = |\blue{\inpcl{i}}| - |\red{\inpcl{i}}|$ and each vertex in NotFairPart is \emph{colored blue} \;
        FairPart = $\inpcl{i} \setminus $ NotFairPart \;
    }
    \Return{$($FairPart, NotFairPart$)$}
\end{algorithm2e}

\begin{algorithm2e}
\DontPrintSemicolon
\caption{$\greedymerge(\text{Set}_{1},\text{Set}_{2})$}\label{alg:GreedyMerge}
\KwData{$\text{Set}_{1}$ is a set of clusters of red colored vertices and $\text{Set}_{2}$ is a set of clusters of blue colored vertices}
\KwData{Output: Return a fair set of clusters $\fairset$.} 
    Sort the clusters in $\text{Set}_{1}$ in non-decreasing order based on their sizes \;
    Sort the clusters in $\text{Set}_{2}$ in non-decreasing order based on their sizes \;
    Initialize a set, $\fairset = \emptyset$ \;
        \For{CRed in $ \text{Set}_{1} $ and CBlue in $\text{Set}_{2}$} {
            \If{$|CRed| > |CBlue|$} {
                Take a subset, $CRed' \subseteq CRed$ such that $|CRed'| = |CBlue|$ \;
                $CRed = CRed \setminus CRed'$ \;
                Let, $Fair = CBlue \cup CRed'$ \;
                $\text{Set}_{2} = \text{Set}_{2} \setminus CBlue$ \;
            }
            \Else {
                Take a subset, $CBlue' \subseteq CBlue$ such that $|CBlue'| = |CRed|$ \;
                $CBlue = CBlue \setminus CBlue'$ \;
                Let, $Fair = CBlue' \cup CRed$ \;
                $\text{Set}_{1} = \text{Set}_{2} \setminus CRed$ \;
            }
            Add the cluster $Fair$ to the set $\fairset$ \;
        }
        \Return{$\fairset$}.
\end{algorithm2e}

\subsection{Analysis of the Algorithm}
In this section, we prove~\cref{thm:closest-1-1-fair}. 

Let $ \fairset $ be the output of~\cref{alg:FindClosestFair} on input $ \inpset $. It is easy to see that $ \fairset $ is a fair clustering, and the~\cref{alg:FindClosestFair} runs in $ O(n) $ time. Therefore, it suffices to show that
\begin{align}
    \dist( \fairset, \inpset ) = \dist( \optset, \inpset ), \label{eq:closest-1-1-fair}
\end{align}
where $ \optset $ is any closest fair clustering to $ \inpset $.

To accomplish this, we show that we can construct a fair clustering $ \m{C}^{f} $ from $ \optset $ such that $ \dist( \m{C}^{f}, \inpset ) = \dist( \optset, \inpset ) $, while the properties of $ \m{C}^{f} $, which are established through the construction, ensure that $ \dist(\m{C}^{f}, \inpset) \geq \dist( \fairset, \inpset ) $. By the definition of $ \optset $, it follows the $ \dist( \optset, \inpset ) = \dist(\fairset, \inpset) $. The detailed construction is provided in the following section.

\subsubsection{Optimal Fair Clustering Properties \& Greedy Merge Technique}
Let $ \m{C}^{0} $ be an arbitrary clustering of $ V $. We construct a clustering $ \m{C}^{f} $ from $ \m{C}^{0} $ by applying a sequence of four fairness-preserving transformations:
\[
    \m{C}^{(0)} \to \m{C}^{(1)} \to \m{C}^{(2)} \to \m{C}^{(3)} \to \m{C}^{(4)}=\m{C}^{f},
\]
such that
\[\forall_{0 \leq i \leq 3} ,\dist(\m{C}^{(i+1)},\inpset) \leq \dist(\m{C}^{(i)},\inpset),\]
and $\m{C}^{(i)}$ satisfies all the $i$ properties (detailed below) established prior to that step. Since these transformations preserve fairness, if $ \m{C}^{0} $ is fair, then all $ \m{C}^{i} $ are fair.

We now describe the properties and the transformations.

Let $ \m{C} $ be an arbitrary clustering of $ V $. Recall that the input clustering is $ \inpsetd $.

\begin{property}\label{prop:one}
For each cluster $C_i \in \m{C}$, if there exists no index $k$ such that $C_i \subseteq \inpcl{k}$, then the sets $R_{i,k} \subseteq \red{C_i}$ and $B_{i,k} \subseteq \blue{C_i}$ cannot both be non-empty while simultaneously satisfying the conditions
\begin{align}
    R_{i,k}\subseteq \inpcl{K}, b_{i,k}\subseteq \inpcl{K}. \nonumber
\end{align}
\end{property}
We construct a clustering $ \m{C}^{1} $ from a clustering $ \m{C}^{0} $ as follows.
\paragraph{Construction of $ \m{C}^{1} $:}For each $k \in [|\m{D}|]$ and each cluster $ \gencl{i}\in \m{C}^{(0)}$, let $R_{i,k}$ and $B_{i, k}$ be the maximal subsets of red and blue points in $ \gencl{i} $ such that:

\begin{align}
    R_{i,k} \subseteq \red{\gencl{i}}, B_{i,k} \subseteq \blue{\gencl{i}} , R_{i,k} \cup B_{i,k} \subseteq \inpcl{k}.\nonumber
\end{align}
For all $i,\ k$, cut two subsets $ R'_{i,k} $ and $ B'_{i,k} $ of size $\min(|R_{i,k}|, |B_{i,k}|)$ from $R_{i,k}$ and $B_{i,k}$ respectively, and construct a new cluster as $R'_{i,k} \cup B'_{i,k}$. The new clustering $ \m{C}^{1} $ consists of these new clusters $ R'_{i,k}\cup B'_{i,k} $ and the clusters which are the remaining of $ \gencl{i} $. It can be seen that $ \m{C}^{1} $ satisfies~\cref{prop:one}.

\begin{lemma}\label{lem:both-red-blue}
    The constructed clustering $ \m{C}^{1} $ satisfies~\cref{prop:one} and if $ \m{C}^{0} $ is fair, then $ \m{C}^{1} $ is fair and 
    \begin{align}
        \dist(\m{C}^{(1)}, \inpset)\le  \dist(\m{C}^{(0)}, \inpset).\nonumber
    \end{align} 
\end{lemma}
We employ the following fact.
\begin{fact}\label{fact:twice-max-blue-red}
    Let $\gencl{i}$ be a fair cluster. Then, for any subset $ S \subseteq \red{\gencl{i}} $ (or $ S \subseteq \blue{\gencl{i}} $), we have $|\gencl{i}| \geq 2 |S|.$
\end{fact}
Indeed, as $ \gencl{i} $ is fair, $ |\red{\inpcl{i}}| = |\blue{\inpcl{i}}| $. Therefore, $ |\gencl{i}| = 2|\red{\gencl{i}}| \geq 2|S|$.
\begin{proof}[Proof of~\cref{lem:both-red-blue}]
    Suppose that $ \m{C}^{0} $ is fair, then each $ C_{i}\in \m{C}^{0} $ is fair. Since each newly constructed cluster $ R'_{i,k}\cup B'_{i,k} $ contains an equal number of red and blue points, it is fair. Furthermore, an equal number of red and blue points are split from each $\gencl{i}$; the remaining clusters also preserve fairness.

    It remains to show $ \dist(\m{C}^{(1)}, \inpset)\le  \dist(\m{C}^{(0)}, \inpset)\nonumber $.
        
    As for each $k \in [|\m{D}|]$, we move $\min(|R_{i,k}|,|B_{i,k}|)$ points from each cluster $\gencl{i}$ to form new clusters, therefore
    \begin{align*}
        &\>\>\dist(\m{C}^{(2)}, \inpset) -  \dist(\m{C}^{(1)}, \optset) \nonumber\\
        &= \sum_{\gencl{i} \in \m{C}^{(1)}} \big(\sum_{ 0\leq k \leq|\m{D}| }  \min (|R_{i,k}|,|B_{i,k}|) \abs{|R_{i,k}| - |R_{i,k}|}
         \\ 
        &\>\> - \sum_{0 \leq j < k \leq |\m{D}|} 2 \min (|R_{i,k}|,|B_{i,k}|)(|R_{i,j}|+|B_{i,j}|) \big)\\
         & = \sum_{\gencl{i} \in \m{C}^{(1)}} \min (|R_{i,k}|,|B_{i,k}|) \big(\sum_{ 0\leq k \leq|\m{D}| }  \abs{|R_{i,k}| - |B_{i,k}|} - \sum_{0\leq k \leq|\m{D}| } 2 \abs{|R_{i,k}| + |B_{i,k}|} \big)\\
         & = \sum_{\gencl{i} \in \m{C}^{(1)}} \min (|R_{i,k}|,|B_{i,k}|) \big(\sum_{ 0\leq k \leq|\m{D}| }  \abs{|R_{i,k}| - |B_{i,k}|} - \sum_{j \neq k } (|R_{i,j}|+|B_{i,j}|) \big)\\
         & = \sum_{\gencl{i} \in \m{C}^{(1)}} \min (|R_{i,k}|,|B_{i,k}|) \big(\sum_{ 0\leq k \leq|\m{D}| }  \abs{|R_{i,k}| - |B_{i,k}|} - |\gencl{i}| + |R_{i,k}|+|B_{i,k}| )\big)\\
         & = \sum_{\gencl{i} \in \m{C}^{(1)}} \min (|R_{i,k}|,|B_{i,k}|) \big(\sum_{ 0\leq k \leq|\m{D}| }  \abs{|R_{i,k}| - |B_{i,k}|} - |\gencl{i}| + |R_{i,k}|+|B_{i,k}| )\big)\\
         & = \sum_{\gencl{i} \in \m{C}^{(1)}} \min (|R_{i,k}|,|B_{i,k}|) \big(\sum_{ 0\leq k \leq|\m{D}| }  - |\gencl{i}| + 2 \max (|R_{i,k}|,|B_{i,k}| )\big)\\
         & \leq 0,
    \end{align*}
    where the inequality is due to~\cref{fact:twice-max-blue-red}, $|\gencl{i}| \geq 2 \max (|R_{i,k}|,|B_{i,k}| ))$.



\end{proof}
\begin{property}\label{prop:two}
For each $k \in [|\m{D}|]$, there exists at most one cluster $C_i  \in \m{C}$ such that $C_i \subseteq \inpcl{k}$.
\end{property}
We construct a clustering $ \m{C}^{2} $ from a clustering $ \m{C}^{1} $ as follows.
\paragraph{Construction of $ \m{C}^{2} $:}For each $k \in [|\m{D}|]$, let $\m{I}_k \subseteq \m{C}^{1}$ be the maximal size subset of clusters such that for each $\gencl{i} \in \m{I}_k$, $\gencl{i} \subseteq \inpcl{k}$. Merge all $\gencl{i} \in \m{I}_k$. The new clustering $ \m{C}^{2} $ consists of these new merged clusters and the remaining clusters of $ \m{C}^{1} $. It can be seen that $ \m{C}^{2} $ preserves~\cref{prop:one} while also satisfying~\cref{prop:two}.
\begin{lemma}    
    The new clustering $\m{C}^{(2)}$ satisfies~\cref{prop:one} and~\cref{prop:two}, and if $\m{C}^{(1)}$ is fair then $\m{C}^{(2)}$ is also fair, and
    \begin{align}
        \dist(\m{C}^{(2)}, \inpset)\le  \dist(\m{C}^{(1)}, \inpset). \nonumber
    \end{align} 
\end{lemma} 

\begin{proof}
Suppose that $ \m{C}^{1} $ is fair. Then each cluster $ \gencl{i} $ in clustering $\m{C}^{(1)}$ is fair, the new clustering $\m{C}^{(2)}$, which consists of a union of a subset of these clusters, must also be fair. This follows from the fact that fairness is preserved under the union of fair clusters.

We now show $ \dist(\m{C}^{(2)}, \inpset)\le  \dist(\m{C}^{(1)}, \inpset) $. Indeed, since $\m{C}^{(2)}$ results from merging clusters, it follows that
\begin{align*}
    \dist(\m{C}^{(2)}, \inpset) -  \dist(\m{C}^{(1)}, \inpset) = - \sum_{i<j, C_i, C_j \in \m{I}_k} |C_i| |C_j| \leq 0,
\end{align*}
which is equivalent to $\dist(\m{C}^{(2)}, \inpset)\le  \dist(\m{C}^{(1)}, \inpset)$.
\end{proof}

\begin{property}\label{prop:three}
For each $C_i \in \m{C}$, one of the following conditions hold:
\begin{itemize}
    \item [1] There exists a $k \in [|\m{D}|]$ such that $C_i \subseteq \inpcl{k}$.
    \item [2] There exists a partition of $C_i$ into two disjoint subsets $R_{i,k}$ and $B_{i,j}$, such that there exists indices $k,j \in [|\m{D}|]$ (where $k\neq j$) satisfying:\[R_{i,k} \subseteq \inpcl{k}, B_{i,j} \subseteq \inpcl{j}.\]
\end{itemize}
\end{property}
Suppose that $ \m{C}^{2} $ is fair. We construct a clustering $ \m{C}^{3} $ from a clustering $ \m{C}^{2} $ as follows.
\paragraph{Construction of $\m{C}^{3}$:}For each $k \in [|\m{D}|]$ and each cluster $\gencl{i} \in \m{C}^{(2)}$, let $R_{i,k}$ and $B_{i, k}$ be the maximal subsets of red and blue points in $ \gencl{i} $ such that: 
\begin{align}
    R_{i,k} \subseteq \red{\gencl{i}}, B_{i,k} \subseteq \blue{\gencl{i}} , R_{i,k} \cup B_{i,k} \subseteq \inpcl{k}. \nonumber
\end{align}
From~\cref{prop:two}, for each $ i,\ k $, at least one of $R_{i,k}$ or $B_{i,k}$ must be empty. Now we define the new clustering $\m{C}^{(3)}$ by applying~\cref{alg:GreedyMerge} to each $\gencl{i} \in \m{C}^{(2)}$ on the input cluster sets:
\begin{align}
    \bigcup_{k \in [|\m{D}|]} R_{i,k}, \bigcup_{k \in [|\m{D}|]} B_{i,k}. \nonumber
\end{align}
    
\begin{lemma}\label{clm:two_input_cluster_at_most}
    If $ \m{C}^{2} $ is fair, then the new clustering $\m{C}^{(3)}$ satisfies~\cref{prop:one},~\cref{prop:two}, and~\cref{prop:three}. Moreover, $ \m{C}^{3} $ is also fair and
    \begin{align}
        \dist(\m{C}^{(3)}, \inpset)\le  \dist(\m{C}^{(2)}, \inpset). \nonumber
    \end{align}
\end{lemma}
\begin{proof}
    Suppose that $ \m{C}^{2} $ is fair. 
    
    We show that $ \m{C}^{3}$ is also fair. Indeed, since ~\cref{alg:GreedyMerge} selects an equal number of red and blue points and merges them together, the resulting clustering remains fair. Consequently, $\m{C}^{(3)}$ preserves fairness. 
    
    It is straightforward to verify that~\cref{prop:one} and~\cref{prop:two} continue to hold, as no new cluster $\gencl{}$ has been introduced such that there exists an index $ k $ satisfying $\gencl{} \subseteq \inpcl{k}$. 

    To see that~\cref{prop:three} also holds, recall that each cluster constructed by ~\cref{alg:GreedyMerge} consists of two disjoint subsets of $R_{i,k}\subseteq \inpcl{k}$ and $B_{i,j} \subseteq \inpcl{j}$ for some indices $ i,j,k $. 

    Finally, we show that $ \dist(\m{C}^{(3)}, \inpset)\le  \dist(\m{C}^{(2)}, \inpset) $. Define $\alpha_{j,k}(\gencl{i})$ as the number of blue and red points that~\cref{alg:GreedyMerge} merges from the subsets of $R_{i,j}$ and $B_{i,j}$. Note that for each $ i,\ j$,~\cref{alg:GreedyMerge} merges a subset of $R_{i,j}$ and $B_{i,j}$ at most once. We have
    \begin{align*}
        &\>\>\>\dist(\m{C}^{(3)}, \inpset) - \dist(\m{C}^{(2)}, \inpset) \\
        &= \sum_{\gencl{i} \in \m{C}^{(2)}} \bigg[\sum_{0\leq j , k\leq |\m{D}|}\bigg(\frac{1}{2}\alpha_{j,k}(\gencl{i})(R_{i,j} - \alpha_{j,k}(\gencl{i}) + \frac{1}{2}\alpha_{j,k}(\gencl{i})(B_{i,j} - \alpha_{j,k}(\gencl{i})\\
        &\>\>\>\>\>\>\>\>\>\>\>\>\>\>\>\>\>\>\>\>\>\>\>\>\>\>\>\>\>\>\>\>\>\>\>\>\>\>\>\>\> -\dfrac{1}{2}\alpha_{j,k}(|\gencl{i}| - R_{i,j}) - \dfrac{1}{2}\alpha_{j,k}(|\gencl{i}| - B_{i,k}) \bigg)\bigg]\\
        & = \sum_{\gencl{i} \in \m{C}^{(2)}}  \frac{1}{2}\alpha_{j,k}(\gencl{i}) \bigg[\sum_{0\leq j , k\leq |\m{D}|}2\bigg( -|\gencl{i}| +R_{i,j} + B_{i,k} - \alpha_{j,k}(\gencl{i})\bigg)\bigg]\\
        & \leq 0,
    \end{align*}
    where the inequality is due to $ |\gencl{i}| \geq R_{i,j} + B_{i,k} $.
    \end{proof}

\begin{property}\label{prop:four}
For each $k \in [|\m{D}|]$, there exists a cluster $C_i$ such that $C_i$ is the maximal fair subset of $\inpcl{k}$.
\end{property}
Before introducing the construction for $ \m{C}^{4} $ from the clustering $ \m{C}^{3}$, we establish a lower bound on the distance from an input clustering consisting solely of monochromatic clusters to any of its fair clustering. This serves as a key component in bounding the distance from $ \inpset $ to $ \m{C}^{4} $. We begin with an input clustering $\mathcal{C} = \mathcal{R} \cup \mathcal{B}$, where $\mathcal{R} = \{ R_{1}, R_{2}, \dots, R_{t} \}$ consists of monochromatic clusters containing only red points, and $\mathcal{B} = \{ B_{1'}, B_{2'}, \dots, B_{m'} \}$ consists of monochromatic clusters containing only blue points. Moreover we have $\sum_{i \in [t]} {R}_i=\sum_{j \in [m']} {B}_j$. This ensures the existence of a fair clustering of $\mathcal{C}$.

\begin{lemma}\label{lem:opt-equ}
The distance from \(\mathcal{C}\) to any fair clustering \(\tilde{\mathcal{C}}\) of the input clustering \(\mathcal{C}\) that satisfies~\cref{prop:three} is:
    
    \[\dfrac{1}{2} \sum_{i \in [t]} |{R}_i|^2+\dfrac{1}{2} \sum_{j \in [m']} |{B}_j|^2.\]
    
\end{lemma}
\begin{proof}
    Consider any arbitrary fair clustering $\m{\tilde{C}}$ satisfying~\cref{prop:three}. Each cluster $\tilde{C}_k \in \m{\tilde{C}}$ can be further decomposed into two disjoint subsets, $\tilde{R}_{k,i}$ and $\tilde{B}_{k,j}$ where $ \tilde{R}_{k,i} \subseteq R_i$ and $\tilde{B}_{k,j} \subseteq B_j$ for some $i\in [t]$ and $j\in [m']$. 

Further as $\m{\tilde{C}}$ is fair $|\tilde{R}_{k,i}|=|\tilde{B}_{k,j}|$. Thus, each cluster $\tilde{C}_k \in \m{\tilde{C}}$ pays $|\tilde{R}_{k,i}||\tilde{B}_{k,j}|=\dfrac{1}{2}(|\tilde{R}_{k,i}|^2+ |\tilde{B}_{k,j}|^2)$. Further to separate $\tilde{R}_{k,i}$ from $R_i$ the cost required is $\dfrac{1}{2}|\tilde{R}_{k,i}||R_i\setminus \tilde{R}_{k,i}|$; similarly to separate $\tilde{B}_{k,j}$ from $B_j$ the cost required is $\dfrac{1}{2}|\tilde{B}_{k,j}||B_j\setminus \tilde{B}_{k,j}|$.
Thus the total distance from $ \m{C} $ to $\m{\tilde{C}}$ is: 
    
    \begin{align*}
        \dist(\m{\tilde{C}}, \m{C}) & = \sum_{i \in [t]} \sum_{\tilde{R}_{k,i} \subseteq R_i} \dfrac{1}{2}[|\tilde{R}_{k,i}|^2+ |\tilde{R}_{k,i}|(|R_i|-|\tilde{R}_{k,i}|)]+ 
        \sum_{j \in [m']} \sum_{\tilde{B}_{k,j} \subseteq B_j} \dfrac{1}{2}[|\tilde{B}_{k,j}|^2+ |\tilde{B}_{k,j}|(|B_j|-|\tilde{B}_{k,j}|)]\\
        & = \sum_{i \in [t]} \sum_{\tilde{R}_{k,i} \subseteq R_i} \dfrac{1}{2}|\tilde{R}_{k,i}||R_{i}|+\sum_{j \in [m']} \sum_{\tilde{B}_{k,j} \subseteq B_j} \dfrac{1}{2}|\tilde{B}_{k,j}||B_{j}|\\
         & = \sum_{i \in [t]} \dfrac{1}{2}|R_{i}|^2+\sum_{j \in [m']} \dfrac{1}{2}|B_{j}|^2.
    \end{align*}
\end{proof}

Let $ \m{C}^{3} $ is a fair clustering satisfying~\cref{prop:one},~\cref{prop:two}, and~\cref{prop:three}. We construct a clustering $ \m{C}^{4} $ from $ \m{C}^{3} $ as follows.
\paragraph{Construction of $ \m{C}^{4}$:}By~\cref{prop:one} and~\cref{prop:two}, for each $ k $, in $ \m{C}^{3} $, there exists at most one cluster containing both subset of $\red{\inpcl{k}}$ and a subset of $\blue{\inpcl{k}}$. If such a cluster exists (denoted as $\gencl{\ell_k}$), then $\gencl{\ell_k} \subseteq \inpcl{k}$.

Next, based on~\cref{prop:three}, any remaining cluster can be partitioned into two disjoint subsets $R_{i,k}$ and $B_{i,j}$, such that there exist indices $k, j \in [|\m{D}|]$ (where $k \neq j$) satisfying:
\begin{align}
    R_{i,k} \subseteq \red{\inpcl{k}}, \quad B_{i,j} \subseteq \blue{\inpcl{j}}.\nonumber
\end{align} 
Each such cluster is then split into its disjoint subsets $R_{i,k}$ and $B_{i,j}$.

For each $k$, we define two disjoint subsets:
\begin{itemize}
    \item $\m{I}_k^r \subseteq \bigcup_{\forall \gencl{i} \in \m{C}} R_{i,k}$, which is the largest subset of clusters $R_{i,k}$ such that $R_{i,k} \in \inpcl{k}$;
    \item $\m{I}_k^b \subseteq \bigcup_{\forall \gencl{i} \in \m{C}} B_{i,k}$, which is the largest subset of clusters $B_{i,k}$ such that $B_{i,k} \in \inpcl{k}$.
\end{itemize}
    For each $k$, we sort the non-empty clusters in $\m{I}_k^r$ in non-increasing order based on their size and denote the $i$-th cluster in this ordering as $R_i(k)$. Similarly, we sort the clusters in $\m{I}_k^b$ in non-increasing order based on their size and denote the $i$-th cluster in this ordering as $B_i(k)$.  

We then compute:
\begin{align}
    \min \left( \sum_{R_i(k) \in \m{I}_k^r} |R_i(k)|, \sum_{B_i(k) \in \m{I}_k^b} |B_i(k)| \right).\nonumber
\end{align}
Without loss of generality, assume that: $\sum_{R_i(k) \in \m{I}_k^r} |R_i(k)|$ is the smaller of the two sums.

We then identify the index $t$ such that
\begin{align}
    \sum_{i = 0}^{t} |B_i(k)| \leq \sum_{R_i(k) \in \m{I}_k^r} |R_i(k)| < \sum_{i = 0}^{t+1} |B_i(k)|.\nonumber
\end{align}  

For all $R_{i,k} \in \m{I}_k^r$ we select ${R_i(k)}$ and ${B_i(k)}$ from indices $1$ to $t$ along with a portion $\alpha_k$ from $B_{t+1}(k)$, where:
\begin{align}
    \alpha_k = \sum_{i = 0}^{t+1} |B_i(k)|-\sum_{R_{i}(k) \in \m{I}_k^r}|R_i(k)|. \nonumber
\end{align} 
We proceed by merging these subsets with $\gencl{\ell_k}$, provided that $\gencl{\ell_k}$ exists. The resultant cluster, denoted by $\gencl{\ell_k}$, represents the maximal fair subset of $\inpcl{k}$.\\
To construct the collection of all such maximal fair clusters, we define the set:\[\m{M} = \{\gencl{\ell_k}| k \in [|\m{D}|]\},\] where $\m{M}$ contains all the maximal fair clustering derived from the input clusters.

As for each $k \in [|\m{D}|]$, the maximal fair part is removed, and by the end of this process, at least one of the sets of $\m{I}_k^r$ or $\m{I}_k^b$ is empty. If $\m{I}_k^r$ is empty, then all $B_i(k)$ from $\m{I}_k^b$ are merged into a cluster $B_k$, where $|B_k| = |\red{\inpcl{k}}| - |\blue{\inpcl{k}}|$. Conversely, if $\m{I}_k^b$ is empty, then all $R_i(k)$ from $\m{I}_k^r$ are merged into a cluster $R_k$, where $|R_k| = |\red{\inpcl{k}}| - |\blue{\inpcl{k}}|$. Define

\begin{itemize}
    \item $\m{R} = \{ R_k \mid \m{I}_k^r \text{is non-empty}\}$, which consists of red monochromatic clusters $R_k$.
    \item $\m{B} = \{ B_k \mid \m{I}_k^b \text{is non-empty}\}$, which consists of blue monochromatic clusters $B_k$.
\end{itemize}

Finally, we apply~\cref{alg:GreedyMerge} to the input subsets: \[\text{Set}_1 =\m{R}, \text{Set}_2= \m{B}\] Suppose that $\fairset$ denotes the output of the algorithm.\\
We define $\m{C}^{(4)}$ as the union of maximal fair clusters $\m{M}$ and $\fairset$, formally expressed as:\[\m{C}^{(4)} = \m{M} \cup \fairset.\]

\begin{lemma}
    The new clustering $\m{C}^{(4)}$ is fair and satisfies all properties~\cref{prop:one},~\cref{prop:two},~\cref{prop:three}, and~\cref{prop:four}. Moreover 
$\dist(\m{C}^{(4)}, \inpset)\le  \dist(\m{C}^{(3)}, \inpset)$.
\end{lemma}

\begin{proof}
    It is straightforward to verify that for each $k \in [|\m{D}|]$, $\gencl{\ell_k}$ remains fair as it is the maximal fair subset of $\inpcl{k}$. The rest of the clusters are the result of ~\cref{alg:GreedyMerge}, which construct fair clusters. Therefore $\m{C}^{(4)}$ remains fair.
    
    Furthermore, considering the input structure of~\cref{alg:GreedyMerge}, for each $k \in [|\m{D}|]$ either $R_k$ or $B_k$ is necessarily empty. This guarantees that~\cref{prop:one},~\cref{prop:two} hold. \\Additionally, given the mechanism by which~\cref{alg:GreedyMerge} operates, it is straightforward to verify that~\cref{prop:three} is also satisfied. Moreover, since $\m{C}^{(4)}$ is defined as the union of maximal fair clusters $\m{M}$ and $\fairset$, if follows that ~\cref{prop:four} is also satisfied.
    
    It remains to show that $ \dist(\m{C}^{(4)}, \inpset)\le  \dist(\m{C}^{(3)}, \inpset) $. Applying~\cref{lem:opt-equ}, we establish that the distance from $\mathcal{C} = \mathcal{R} \cup \mathcal{B}$ to any fair clustering $\fairset $ (the output of ~\cref{alg:GreedyMerge}) is
    \begin{align*}
        \dfrac{1}{2} \sum_{R_i \in \mathcal{R}} |{R}_i|^2+\dfrac{1}{2} \sum_{B_i \in \mathcal{B}} |{B}_j|^{2} &= \dfrac{1}{2}\sum_{k \in [|\m{D}|]}\left(   |\red{\inpcl{k}}| -  |\blue{\inpcl{k}}| \right)^{2} \nonumber \\
        &= \dfrac{1}{2}\sum_{k \in [|\m{D}|]}\left( \sum_{C_i \in \m{C}} |R_i(k)| -  \sum_{C_i \in \m{C}} |B_i(k)| \right)^{2}. \nonumber
    \end{align*}

    Therefore, we have
    \begingroup
\allowdisplaybreaks
    \begin{align*}
        &\>\>\dist(\m{C}^{(4)}, \inpset) -  \dist(\m{C}^{(3)}, \inpset) \nonumber \\
        &=\sum_{k \in [|\m{D}|]} \bigg[- \dfrac{1}{2} \sum_{C_i \in \m{C}} |R_i(k)|^{2} -\dfrac{1}{2} \sum_{C_i \in \m{C}} |B_i(k)|^{2} - \sum_{i< j}|R_i(k)||R_j(k)| -\sum_{i<j }|B_i(k)||B_j(k)|-\sum_{i,j }|R_i(k)||B_j(k)| \\
        &\> +2 \min (\sum_{C_i \in \m{C}} |R_i(k)|, \sum_{C_i \in \m{C}} |B_i(k)|)\left( \max (\sum_{C_i \in \m{C}} |R_i(k)|, \sum_{C_i \in \m{C}} |B_i(k)|) -  \min (\sum_{C_i \in \m{C}} |R_i(k)|, \sum_{C_i \in \m{C}} |B_i(k)| \right)\\
        &\> + \dfrac{1}{2} \left(   \sum_{C_i \in \m{C}} |R_i(k)| -  \sum_{C_i \in \m{C}} |B_i(k)| \right)^2\bigg]\\
        & \le \sum_{k \in [|\m{D}|]} \bigg[- \dfrac{1}{2} \sum_{C_i \in \m{C}} |R_i(k)|^{2} -\dfrac{1}{2} \sum_{C_i \in \m{C}}|B_i(k)|^{2} - \sum_{i< j}|R_i(k)||R_j(k)| -\sum_{i<j }|B_i(k)||B_j(k)| -\sum_{i,j }|R_i(k)||B_j(k)|\\
        &\> +2 \min (\sum_{C_i \in \m{C}} |R_i(k)|, \sum_{C_i \in \m{C}} |B_i(k)|) \max (\sum_{C_i \in \m{C}} |R_i(k)|, \sum_{C_i \in \m{C}} |B_i(k)|)\\
        &\> + \dfrac{1}{2} \left( \sum_{C_i \in \m{C}} |R_i(k)| -  \sum_{C_i \in \m{C}} |B_i(k)| \right)^2\bigg]\\
        & =\sum_{k \in [|\m{D}|]} \bigg[- \dfrac{1}{2} \sum_{C_i \in \m{C}} |R_i(k)|^{2} -\dfrac{1}{2} \sum_{C_i \in \m{C}} |B_i(k)|^{2} - \sum_{i< j}|R_i(k)||R_j(k)| -\sum_{i<j }|B_i(k)||B_j(k)| -\sum_{i,j }|R_i(k)||B_j(k)|\\
        &\> + 2(\sum_{C_i \in \m{C}} |R_i(k)|)(\sum_{C_i \in \m{C}} |B_i(k)|) + \dfrac{1}{2} \left(   \sum_{C_i \in \m{C}} |R_i(k)| -  \sum_{C_i \in \m{C}} |B_i(k)| \right)^2\bigg]\\
        &=\sum_{k \in [|\m{D}|]} \bigg[- \dfrac{1}{2} \sum_{C_i \in \m{C}} |R_i(k)|^{2} -\dfrac{1}{2} \sum_{C_i \in \m{C}} |B_i(k)|^{2} - \sum_{i< j}|R_i(k)||R_j(k)| -\sum_{i<j }|B_i(k)||B_j(k)| \\
        &\> + 2(\sum_{C_i \in \m{C}} |R_i(k)|)(\sum_{C_i \in \m{C}} |B_i(k)|) + \dfrac{1}{2} \sum_{C_i \in \m{C}} |R_i(k)|^{2} +\dfrac{1}{2} \sum_{C_i \in \m{C}} |B_i(k)|^{2}  +\sum_{i<j }|B_i(k)||B_j(k)|\\
        &\> + \sum_{i< j}|R_i(k)||R_j(k)|- (\sum_{C_i \in \m{C}} |R_i(k)|)(\sum_{C_i \in \m{C}} |B_i(k)|)\bigg]\\
        & = -\sum_{i,j }|R_i(k)||B_j(k)| + (\sum_{C_i \in \m{C}} |R_i(k)|)(\sum_{C_i \in \m{C}} |B_i(k)|)\\
        &= 0.
    \end{align*}

\endgroup

\end{proof}

\begin{proof}[Proof of~\cref{thm:closest-1-1-fair}]
    Given a fair optimal clustering $\m{C}^*$, we transform it to $\m{C}^{f^*}$ (say) such that it satisfies all the properties,~\cref{prop:one},~\cref{prop:two},~\cref{prop:three}, and~\cref{prop:four} and $\dist(\m{D}, \m{C}^{f^*}) \leq \dist(\m{D}, \m{C}^*)$.

    Now, we prove $\dist(\inpset, \m{F}) \leq \dist(\inpset, \m{C}^{f^*})$ where $\m{F}$ is the output of the \cref{alg:FindClosestFair}.

    In \cref{alg:MakeItFair}, we divide a cluster $D_i \in \inpset$ into two sets, FairPart and NotFairPart. Note that FairPart is the maximal fair subset of $D_i$. Let's denote the maximal fair subset of $D_i$ by $M_i$ and $M = \bigcup_{i \in [|\inpset|]}M_i$.

    Hence,
    \begin{align}
        \dist(\inpset, \m{F}) = \dist(\inpset, \{ M_k \mid k \in [|\inpset|]\}) + \dist(\inpset, \m{F} \setminus \{ M_k \mid k \in [|\inpset|]\} ) + \textit{cost}(M, V \setminus M) \label{eqn:equi-one}
    \end{align}

    where $\textit{cost}(M, V \setminus M)$ counts the number of pairs $(u,v)$ such that $u \in M$ and $v \in V \setminus M$ but $u$ and $v$ are present in the same cluster $D_k$(say) in $\inpset$. Recall that $V$ is the domain.

    Now in $\m{C}^{f^*}$ by \cref{prop:four} we get for each $i \in [|D|]$ there exists $C_j \in \m{C}^{f^*}$ such that $C_j$ is the maximal fair subset of $D_i$. Since $C_j$ is a maximal subset of $D_i$, we can assume $C_j = M_i$ (maximal FairPart that we cut from $D_i$) because $\dist(\inpset, \n{C}^{f^*})$ is independent of points present in $C_j$ it depends only on the size of $C_j$.

    Hence,
    \begin{align}
        \dist(\inpset, \m{C}^{f^*}) = \dist(\inpset, \{ M_k \mid k \in [|\inpset|]\}) + \dist(\inpset, \m{C}^{f^*} \setminus \{ M_k \mid k \in [|\inpset|] \}) + \textit{cost}(M, V \setminus M) \label{eqn:equi-two}
    \end{align}

    By comparing \cref{eqn:equi-one} and \cref{eqn:equi-two} we get that to show $\dist(\inpset, \m{F}) \leq \dist(\inpset, \m{C}^{f^*})$ we need to prove $\dist(\inpset, \m{F} \setminus \{ M_k \mid k \in [|\inpset|]\} ) \leq \dist(\inpset, \m{C}^{f^*} \setminus \{ M_k \mid k \in [|\inpset|] \})$. 

    The set $D_i \setminus M_i$ is either a red monochromatic set or a blue monochromatic set. If $D_i \setminus M_i$ is red monochromatic, we denote it with $R_i$, and if $D_i \setminus M_i$ is blue monochromatic, then we denote it with $B_i$. Thus initially we have a set of red monochromatic clusters $\m{R} = \{R_1, R_2, \ldots R_t \}$ and blue monochromatic clusters $\m{B} = \{B_1, B_2, \ldots, B_m\}$ where both $t$ and $m$ are less than $|\inpset|$. 

    Now, $\m{C}^{f^*} \setminus \{ M_k \mid k \in [|\inpset|] \}$ is an optimal fair clustering on $\m{R} \cup \m{B}$ that satisfies \cref{prop:three}. 
    
    Thus, from \cref{lem:opt-equ} we get 
    \begin{align}
        \dist(\inpset, \m{C}^{f^*} \setminus \{ M_k \mid k \in [|\inpset|] \}) \geq \frac{1}{2}\sum_{i \in [t]}|R_i|^2 + \frac{1}{2} \sum_{j \in [m]}|B_j|^2\label{eqn:equi-three}
    \end{align}
    
    Again since \cref{alg:GreedyMerge} takes $\m{R} \cup \m{B}$ as input and construct a fair clustering $\m{F} \setminus \{ M_k \mid k \in [|\inpset|] \}$ that satisfies $\cref{prop:three}$ we get,
    \begin{align}
        \dist(\inpset, \m{F} \setminus \{ M_k \mid k \in [|\inpset|]\} ) \leq \frac{1}{2}\sum_{i \in [t]}|R_i|^2 + \frac{1}{2} \sum_{j \in [m]}|B_j|^2 \label{eqn:equi-four}
    \end{align}

    Hence, from \cref{eqn:equi-one}, \cref{eqn:equi-two}, \cref{eqn:equi-three} and \cref{eqn:equi-four} we conclude $\dist(\inpset, \m{F}) \leq \dist(\inpset, \m{C}^{f^*})$.
    


\end{proof}

\section{Constant Approximation for Closest Fair Clustering}\label{sec:approx-fair}
In~\cref{sec:equiproportion}, we show how to find an optimal closest fair clustering to any input clustering when the total number of blue and red points are equal in the data set. Now, we focus on a more general case when the (irreducible) ratio between the blue and red points is $p/q$ for some positive integers $p,q \ge 1$, and provide an approximation algorithm. To show our results, we take a two-stage process. First, we design approximation algorithms that find a close balanced cluster -- for each cluster, the number of blue points is a multiple of $p$, and the number of red points is a multiple of $q$. Next, we design an approximation algorithm to find a close fair clustering to an already balanced clustering. Finally, we combine both these algorithms to establish our approximation guarantee for the closest fair clustering problem.

In this section, we first describe the combination step, i.e., how to use an approximation algorithm for the closest balanced clustering problem and an approximation algorithm for the closest fair clustering problem when the inputs are balanced clustering. 

\begin{theorem}
    \label{thm:combine-p-blue-fair}
    Consider any $\alpha,\beta \ge 1$. Suppose there is a $t_1(n)$-time algorithm Algorithm-A() that, given any clustering on $n$ points, produces an $\alpha$-close $\bal$, and a $t_2(n)$-time algorithm Algorithm-B() that, given any $\bal$ on $n$ points, produces a $\beta$-close $\fair$. Then there is an $O(t_1(n) + t_2(n) + n)$-time algorithm that, given any clustering on $n$ points, produces an $\left(\alpha + \beta +\alpha \beta \right)$-close $\fair$.
\end{theorem}

\begin{proof}
Consider the following simple algorithm: Given a clustering $\inpset$ as input, first run Algorithm-A() to generate a $\bal$ $\m{T}$, and then run Algorithm-B() with $\m{T}$ as input to generate a $\fair$ $\m{F}$. Finally, output $\m{F}$.

It is straightforward to see that the running time of the above algorithm is $O(t_1(n) + t_2(n) + n)$. We now claim that $\m{F}$ is an $\left(\alpha + \beta +\alpha \beta \right)$-close $\fair$ of the input clustering $\inpset$.

Let $\fairmop$ be an (arbitrary) closest $\fair$ of the input $\inpset$, and $\mop$ be an (arbitrary) closest $\bal$ of the input $\inpset$. Since by the definition, any $\fair$ is also a $\bal$, we get
\begin{equation}
    \label{eq:opt-fair-bal}
    \dist(\inpset, \mop) \le \dist(\inpset,\fairmop).
\end{equation}

Now, by the guarantee provided by Algorithm-A(), we get
\begin{align}
    \label{eq:output-bal}
    \dist(\inpset,\m{T}) &\le \alpha \cdot \dist(\inpset, \mop) \nonumber\\
    &\le \alpha \cdot \dist(\inpset, \fairmop) &&\text{(By \cref{eq:opt-fair-bal})}.
\end{align}

Next, let $\m{C}^*$ be an (arbitrary) closest $\fair$ of $\m{T}$ (recall, $\m{T}$ is an output of Algorithm-A()). Then, by the guarantee provided by Algorithm-B(), we get
\begin{align}
    \label{eq:outbal-optfair}
    \dist(\m{T},\m{F}) &\le \beta\cdot \dist(\m{T},\m{C}^*)\nonumber\\
    & \le \beta \cdot \dist(\m{T},\fairmop) &&\text{(Since $\fairmop$ is a valid $\fair$)}\nonumber\\
    &\le \beta \cdot \left(\dist(\m{T},\inpset) + \dist(\inpset,\fairmop) \right) &&\text{(By the triangle inequality)}\nonumber\\
    &\le \beta (\alpha + 1)\cdot \dist(\inpset,\fairmop)&&\text{(By \cref{eq:output-bal})}.
\end{align}

Thus, finally, we deduce that

\begin{align*}
    \dist(\inpset,\m{F}) & \le \dist(\inpset,\m{T}) + \dist(\m{T},\m{F}) &&\text{(By the triangle inequality)}\\
    & \le \alpha \cdot \dist(\inpset, \fairmop) + \beta (\alpha + 1)\cdot \dist(\inpset,\fairmop) &&\text{(By \cref{eq:output-bal} and \cref{eq:outbal-optfair})}\\
    & = (\alpha + \beta +\alpha \beta) \cdot \dist(\inpset, \fairmop)
\end{align*}
and that concludes the proof.
\end{proof}
  
Now, as a corollary of the above theorem, we get our $O(1)$-approximation results for the closest fair clustering problem. In~\cref{sec:bal-to-fair}, we present an algorithm to make an already balanced cluster fair; more specifically, we show the following theorem.

\begin{restatable} {theorem} {makepclusterfair} \label{thm:make.p.cluster.fair}
    Let $ \ratio = p/q $ where $ p \neq q $ are coprime positive integers. There exists a linear time algorithm that, given any $ \mopqdef $ $ \out $, finds a fair clustering $ \fairset $ such that
    \begin{align}
        \dist(\out, \fairset) \leq 3\dist(\out, \fairmop),\nonumber
    \end{align}
    where $ \fairmop $ is a closest fair clustering to $ \out $.
\end{restatable}

\paragraph{Approximation guarantee when $p>1, q=1$. }In~\cref{sec:multiple-of-p-clustering}, we present an approximation algorithm to make any arbitrary clustering with the ratio between blue and red points being $p$ for some integer $p > 1$; more specifically, we show the following result.

\begin{restatable}{theorem} {multiplep}\label{thm:main-multiple-of-p}
Consider an integer $p >  1$. There is an algorithm that, given a clustering $\inpset$ over a set of $n$ red-blue colored points where the ratio between the total number of blue and red points is $p$, outputs a $3.5$-close {\bal} of $\inpset$ in time $O(n \log n)$.
\end{restatable}

Now, as an immediate corollary of \cref{thm:combine-p-blue-fair}, by \cref{thm:main-multiple-of-p} and \cref{thm:make.p.cluster.fair}, we get the following result.

\closestpfair*

\paragraph{Approximation guarantee when $p,q>1$. }Next, in~\cref{sec:multiple-of-p-q-clustering}, we present an approximation algorithm to make any arbitrary clustering with the ratio between blue and red points being $p/q$ for some integer $p,q > 1$, where $p,q$ are coprime; more specifically, we show the following result.

\begin{restatable}{theorem}{multiplepq}\label{thm:main-multiple-of-pq}
  Consider two integers $p,q > 1$. There is an algorithm that, given a clustering $\inpset$ over a set of $n$ red-blue colored points where the irreducible ratio between the total number of blue and red points is $p/q$, outputs a $7.5$-close {\bal} of $\inpset$ in time $O(n \log n)$.
\end{restatable}

Also, as an immediate corollary of \cref{thm:combine-p-blue-fair}, by \cref{thm:main-multiple-of-pq} and \cref{thm:make.p.cluster.fair}, we get the following result.

\closestpqfair*

\section{\texorpdfstring{Closest balanced clustering for $p>1, q=1$}{Closest balanced clustering for p>1, q=1}}
\label{sec:multiple-of-p-clustering}

In this section, given a clustering $\inpsetd$ of a set of red-blue colored points, where the irreducible ratio of blue to red points is $ p $ for some integer $ p > 1 $, we construct a balanced clustering $\outd$. That is, in the resulting clustering $\outd$, every cluster $\outcl{i} \in \out$ satisfies the condition that $|\blue{\outcl{i}}|$ is a multiple of $ p $. Our main result for this section is the following:

\multiplep*

We provide the algorithm to find a $3.5$-close $\mopdef$ in \cref{subsec:algorithms} and provide the analysis of the algorithm in \cref{subsec:analysis}

\subsection{Details of the Algorithm} \label{subsec:algorithms} 
Before describing our algorithm, let us first introduce a few notations that we use later.

For any input cluster $D_i$, if its number of blue points is not a multiple of $p$, then we define its surplus as the minimum number of blue points whose removal makes the number of blue points a multiple of $p$, and we define its deficit as the minimum number of blue points whose addition makes the number of blue points to be a multiple of $p$. Formally, we define them as follows.

\begin{itemize}
    \item Surplus of $D_i$, $\surp{D_i}$: For any cluster $D_i \subseteq V$, $\surp{D_i} = |\blue{D_i}| \mod p$.
    \item Deficit of $D_i$, $\defi{D_i}$: For any cluster $D_i \subseteq V$ if $p \nmid |\blue{D_i}|$ then $\defi{D_i} = p - \surp{D_i}$ otherwise $\defi{D_i} = 0$.
\end{itemize}

We call the operation of taking some points, $S$ (say) from a cluster $\inpcl{i} $, that is $\inpcl{i} = \inpcl{i} \setminus S$ as \emph{cut} and the operation of merging some points $M$ (say) to a cluster $\inpcl{i}$ that is $\inpcl{i} = \inpcl{i} \cup M$ as \emph{merge}. Let us define some more terms that we need to describe our algorithms.

\begin{itemize}
    \item Cut cost, $\ccostf{\inpcl{i}}$: For any cluster $\inpcl{i} \in \inpset$, $\ccostf{\inpcl{i}} = (\surp{\inpcl{i}} \cdot (|\inpcl{i}| - \surp{\inpcl{i}})$.


    \item Merge cost, $\mcostf{\inpcl{i}}$: For any cluster $\inpcl{i} \in \inpset$, $\mcostf{\inpcl{i}} = d(\inpcl{i}) \cdot |\inpcl{i}|$.

\end{itemize}


\paragraph{Description of the algorithm\\\\}


Given a set, $\inpset$, we partition the clusters $\inpcl{i} \in \inpset$ into two types based on the $\surp{\inpcl{i} }$ value of the cluster $\inpcl{i} $. If $\surp{\inpcl{i} } \leq p/2$ we call these clusters as $\cut$ clusters and if $\surp{\inpcl{i} } > p/2$ we call these clusters as $\merge$ clusters. Thus

\begin{itemize}
    \item $\cut = \{\inpcl{i} \mid \surp{\inpcl{i}} \leq p/2\}$.
    \item $\merge = \{\inpcl{i} \mid \surp{\inpcl{i}} > p/2\}$.
\end{itemize}

We first sort the clusters $\inpcl{i} \in \merge$ based on the value of $(\ccostf{\inpcl{i}} - \mcostf{\inpcl{i}})$ in \emph{non-increasing} order. Our algorithm $\algog$ cuts a set of $\surp{\inpcl{i} }$ blue points from each cluster $\inpcl{i}  \in \cut$ and merges them to clusters in $\merge$ sequentially starting from the beginning in the sorted list. In each cluster $\inpcl{j} \in \merge$, we merge a set of $\defi{\inpcl{j}}$ blue points. Now, we have two cases: after a certain point of time, the remaining set of clusters in $\cut$ becomes empty, and we are only left with some clusters in $\merge$; in this case, we call our subroutine $\algom$. To remove any ambiguity, we refer to the remaining clusters in $\merge$ at this point as $\merge'$, and the set of clusters created from $\cut$ after cutting the surplus points $s(D_i)$ from each cluster $D_i \in \cut$ be $\nc$. 

Another case is when the remaining set of clusters in $\merge$ becomes empty, and we are only left with some clusters in $\cut$. In this case, we call our subroutine $\algoc$. Again, To remove any ambiguity, we refer to the remaining set of clusters in $\cut$ as $\cut'$.

The subroutine $\algom$ takes two sets $\nc, \merge'$ as input. In $\algom$ we calculate the sum of $\defi{\inpcl{i}}$ for each of the remaining clusters $\inpcl{i} \in \merge'$ (where $\merge'$ is the remaining set of $\merge$ clusters). Let, $W = \sum_{D_j \in \merge'} \defi{\inpcl{j}}$. Now, each cluster $D_k \in \nc \cup \merge'$ the set $\blue{D_k}$ is divided into $(|\blue{D_k}| - \surp{D_k})/p$ subsets of size $p$ and one subset of size $\surp{D_k}$. Let us number the subset of size $\surp{D_k}$ as $0$, and all other subsets of size $p$ are numbered arbitrarily. Let us denote the $z$th subset of a cluster $\inpcl{k}$ by $W_{k,z}$.

Each of these subsets is attached with a cost, 

\begin{align}
    \kappa_0(D_j) = \surp{D_j} \left( |D_j| - \surp{D_j} \right) && \text{cutting cost for $0$th subset} \n \\
    \kappa_z(D_j) = p \left( |D_j| - (zp + \surp{D_j}) \right) && \text{cutting cost for $z$th subset of size $p$ where $z \neq 0$} \n
\end{align}

The algorithm $\algom$ is called when the remaining set of clusters are in $\merge$. In this case, ideally, we would like to merge in every cluster because, for the clusters $D_i \in \merge$, we have $\surp{D_i} > p/2$ and thus $\mcostf{D_i} < \ccostf{D_i}$ but to make in each cluster $|\blue{D_i}|$ a multiple of $p$ we must cut from some cluster. So, now we would like to cut from those clusters $D_i$ where $\ccostf{D_i} - \mcostf{D_i}$ is minimum because even if we pay the $\ccost$ for $D_i$, it is not too much compared to the $\mcost$ of $D_i$. After the $\algog$ is over, there are clusters where the number of blue points is already a multiple of $p$, and it may happen that it is better to cut of subset of size $p$ from those clusters rather than from the cluster $D_i$ which has the minimum $\ccostf{\inpcl{i}} - \mcostf{\inpcl{i}}$. Now, we describe the algorithm $\algom$ step by step. 

At each iteration of the algorithm, among the clusters $D_k \in \nc \cup \merge'$, the algorithm chooses the cluster $\inpcl{k}$ if the \textit{cost}$(W_{k,v_k})$ is minimum among all clusters. Here, $v_k$ denotes the least numbered subset that is not cut from $\inpcl{k}$. \textit{cost}$(W_{k,v_k})$ is defined as follows

\begin{align}
    \textit{cost}(W_{k,v_k}) &= \kappa_{v_k}(D_k) \, \, \text{for $v_k \geq 1$} \n \\
    &= \kappa_0(D_k) - \mcostf{D_k} \, \, \text{for $v_k = 0$}
\end{align}

We call $\textit{cost}(W_{k,v_k})$ as the cost attached with the $v_k$th subset of $D_k$.

If multiple clusters have a minimum cost, then the algorithm chooses one of them arbitrarily. After that, the algorithm cuts a subset $W_{k,v_k}$ from the cluster $\inpcl{k}$. Then it deletes $\inpcl{k}$ from the set $\merge'$ and then merges that subset sequentially starting from the first cluster in the set $\merge'$ which is sorted according to the $\ccostf{D_i} - \mcostf{D_i}$ of the clusters $D_i$. We keep on doing the above operation until the total deficit, $W$ of the clusters reaches $0$. Later, we prove that to make the total deficit zero we need to cut exactly $W/p$ such subsets. 

The subroutine $\algoc$ takes the set $\cut'$ as input. In $\algoc$, we cut a set of $\surp{\inpcl{i}}$ points from each of the clusters $\inpcl{i} \in \cut'$, and with these surplus points, we create \emph{extra} clusters of size $p$. We do this operation sequentially until the number of blue points becomes a multiple of $p$ in all the clusters. 

We provide the pseudocode of the algorithms $\algog$, $\algoc$, $\algom$ in \cref{alg:algo-for-general}, \cref{alg:algo-for-cut} and \cref{alg:algo-for-merge} respectively. Let the output of the algorithm $\algog$ be $\out$. 

\textbf{Runtime analysis of $\algog$:} The expensive step in the $\algog$ is that in the first step of the algorithm, we sort the clusters in $\merge$ which would take $O\left(|\merge| \, \, \log (|\merge|)\right) = O(n \log n)$ time. The algorithm uses a point $v \in V$ at most twice, once while cutting the point $v$ from its parent cluster $D_i \in \inpset$ (say) and another when merging the point $v$ to some cluster $D_j \neq D_i$. Hence, the cutting and merging process of the algorithm can take at most $O(n)$ time (recall, $|V|=n$). 

In the subroutine $\algom$, at each iteration, the algorithm needs to find the cluster that has the minimum cost attached to its least-numbered subset. The find-Min operation can be done in $O(1)$ by implementing a priority queue with a Fibonacci heap. Hence, overall the time complexity of $\algog$ is $O(|\merge| \, \, \log (|\merge|) + |V|) = O(n \log n)$. 


\begin{algorithm2e}
\DontPrintSemicolon
\caption{$\algog(\inpset)$}\label{alg:algo-for-general}
\KwData {Input set of clusters $\inpset$, in each $\inpcl{i}  \in \inpset$ the surplus of $\inpcl{i} $, $\surp{\inpcl{i} }$ is within $0 < \surp{\inpcl{i} } < p$}
\KwResult{A set of clusters $\out$, such that in each cluster $\outcl{i} \in \m{T}$ the number of blue points is a multiple of $p$.} 
   $\cut, \merge, \nc, \outg \gets \emptyset$
   
   \For{$\inpcl{i}  \in \inpset$}
   {
        \lIf{$ \surp{\inpcl{i}}=0 $}{
            add $ \inpcl{i}$ to the set $ \nc $\;
        } 
        \ElseIf{$\surp{\inpcl{i} } < p/2$}
        {
             add $\inpcl{i} $ to the set $\cut$ \;
        }
        \Else
        {
            add $\inpcl{i} $ to the set $\merge$ \;
        }
   }
   Sort the clusters in $\inpcl{m} \in \merge$ based on their $(\ccostf{\inpcl{m}} - \mcostf{\inpcl{m}})$ in non-increasing order \;
   
   \While{$\cut \neq \emptyset$ and $\merge \neq \emptyset$}{
        \For{$\inpcl{i} \in \cut$} {
            \For{$\inpcl{j} \in \merge$} {
                Let, $k = min(\surp{\inpcl{i}}, \defi{\inpcl{j}})$. \;
                
                cut a set of $k$ blue points from $\inpcl{i}$ and add to $\inpcl{j}$ \;
                \If{$k = \surp{\inpcl{i}}$}{
                    $\cut = \cut \setminus \{\inpcl{i}\}$. \;
                    $\nc = \nc \cup \inpcl{i}$ \;
                
                }
                \If{$k = \defi{\inpcl{j}}$}{
                    $\merge = \merge \setminus \{\inpcl{j}\}$. \; 
                    $\outg = \outg \cup \inpcl{j}$ \;
                    \label{step:cut-merge}
                }
            }
        }
   }
   \If{$\cut = \emptyset$} {
       \Return{$\outg \cup \algom(\nc, \merge)$} \;
    }
    \Else{
         \Return{$\outg \cup \nc \cup \algoc(\cut)$} \;
    }
\end{algorithm2e}

\begin{algorithm2e}
\DontPrintSemicolon
\caption{$\algoc(\cut')$}\label{alg:algo-for-cut}
\KwData{Set of clusters $\cut'$, in each cluster $\gencl{i} \in \cut'$ the surplus of $\gencl{i}$, $\surp{\gencl{i}} \leq p/2$.}
\KwResult{Set of clusters $\outcut$, such that in each cluster $\outcl{i} \in \outcut$ the number of blue points is a multiple of $p$. }
ExtraSum = $\sum \surp{\inpcl{i} }$ \;

$\pcard = \frac{ExtraSum}{p}$ \;

Initialize $\pcard$ many sets $\pcl{1}, \pcl{2}, \ldots \pcl{\pcard}$ to $\emptyset$ \; 
\tcp{$\pcard$ many extra clusters}
Initialize $\pcard$ many variables $\ell_1, \ell_2, \ldots \ell_n$ to $p$ \;
\tcp{leftover space of extra clusters}
Boolean $flag = 0$ \;
\For{$\gencl{i} \in \cut'$} {
    \For{$j = 1$ to $n$} {
             Add $k = min ( \ell_j, \surp{\gencl{i} })$ many blue points from $\gencl{i} $ to the cluster $\pcl{j}$ \;
             
             $\ell_j = \ell_j - k$ \;
    }
}
\Return{$\cut' \cup \{\pcl{1}, \pcl{2}, \ldots, \pcl{\pcard}\}$}
\end{algorithm2e}

\begin{algorithm2e}
\DontPrintSemicolon
\caption{$\algom(\textsc{newcut},\merge')$}\label{alg:algo-for-merge}
\KwData{Two clustering $\textsc{newcut}$ and $\merge'$ such that for all clusters $D_i \in \nc$ we have $p \mid |\blue{D_i}|$ and for clusters $D_j \in \merge'$ we have $s(D_j)>p/2$ respectively. In $\merge'$, the clusters $D_j \in \merge'$ are sorted based on their $(\ccostf{D_j} - \mcostf{D_j})$}
\KwResult{A clustering $\out$, such that for all clusters $T_i \in \out$, $p \mid \blue{T_i}$.}
  $W\gets \sum_{D_j\in\merge'}\defi{D_j}$ \;
  \For{$D_k \in (\nc \cup \merge')$} {
        Initialize a variable $v_k$ to $0$\;
  }

  \While{$W \neq 0$} {
    Take the cluster $D_k \in (\nc \cup \merge')$ for which $\textit{cost}(W_{k,v_k})$ is minimum\;
    \tcp{ $\textit{cost}(W_{k,v_k}) = \kappa_{v_k}(D_k)$ if $v_k \geq 1$, for $v_k = 0$ $\textit{cost}(W_{k,v_k}) = \kappa_{0}(D_k) - \mcostf{D_k}$}
    \While{$\textit{size}(W_{k,v_k}) \neq 0$} {
        \tcp{ $\textit{size}(W_{k,v_k})$ is the size $v_k$th subset of $D_k$ }
        \For{$D_j \in \merge'$}{
         $\gamma = \min(\defi{D_j}, \textit{size}(W_{k,v_k}))$\;
         cut $\gamma$ many blue points from $D_k$ and merge to $D_j$\;
         \If{$\gamma = \defi{D_j}$}
         {
            $\textit{size}(W_{k,v_k}) = \textit{size}(W_{k,v_k}) - \gamma$\;
         }
         \If{$\gamma = \textit{size}(W_{k,v_k}))$}
         {
            $\defi{D_j} = \defi{D_j} - \gamma$ \;
            $W = W - p$ \;
            $v_{k} = v_{k} + 1$\;
         }
         
        }
    }
  }
  \For{$D_j \in \merge'$} {
    \If{$\defi{D_j} = 0$}{
    update $\merge' = \merge' \setminus \{D_j\}$
    }
  }
  \Return{$\out = \nc \cup \merge'$}
\end{algorithm2e}

\subsection{Approximation Guarantee}\label{subsec:analysis}
In this section, we prove that $\dist(\inpset, \out) \leq 3.5 \, \dist(\inpset, \mop)$ which is stated formally in \cref{thm:main-multiple-of-p}

Our proof is divided into the following two cases:
\begin{itemize}
    \item \emph{Merge} case: When $\sum_{\inpcl{i} \in \cut} \surp{\inpcl{i}} \leq \sum_{\inpcl{j} \in \merge} \defi{\inpcl{j}}$; and
    \item \emph{Cut} case: When $\sum_{\inpcl{i} \in \cut} \surp{\inpcl{i}} > \sum_{\inpcl{j} \in \merge} \defi{\inpcl{j}}$.
\end{itemize}

In \cref{subsubsec:merge-case}, we argue that in the merge case $\dist(\inpset, \out) \leq 3 \, \dist(\inpset, \mop)$.

\begin{restatable} {lemma} {mainmerge}
    If $\sum_{\inpcl{i} \in \cut} \surp{\inpcl{i}} \leq \sum_{\inpcl{j} \in \merge} \defi{\inpcl{j}}$, the algorithm $\algog$ produces a $\mopdef$ $\out$ such that $\dist(\inpset, \out) \leq 3 \; \dist(\inpset, \mop)$.
    \label{lem:main-merge-case}
\end{restatable}

In \cref{subsec:cut-case}, we argue that in the cut case $\dist(\inpset, \out) \leq 3.5 \, \dist(\inpset, \mop)$.

\begin{restatable} {lemma} {maincut}
    If $\sum_{\inpcl{i} \in \cut} \surp{\inpcl{i}} > \sum_{\inpcl{j} \in \merge} \defi{\inpcl{j}}$, the algorithm $\algog$ produces a $\mopdef$ $\out$ such that $\dist(\inpset, \out) \leq 3.5 \; \dist(\inpset, \mop)$.
    \label{lem:main-cut-case}
\end{restatable}

It is straightforward to see that to conclude the proof of \cref{thm:main-multiple-of-p}, it suffices to show \cref{lem:main-merge-case} and \cref{lem:main-cut-case}. In \cref{subsubsec:structure-of-M*}, we discuss some important properties of $\mop$ that we need to prove \cref{lem:main-merge-case} and \cref{lem:main-cut-case}. In \cref{subsubsec:merge-case} we prove \cref{lem:main-merge-case} and in \cref{subsec:cut-case} we prove \cref{lem:main-cut-case}.

\subsubsection{\texorpdfstring{Structure of $\mop$}{Structure of MOP}} \label{subsubsec:structure-of-M*}
To prove \cref{lem:main-merge-case} and \cref{lem:main-cut-case}, we need an important property about the structure of $\mop$. We need to define some notations to state that important property of $\mop$.



Consider any point $v \in V$. Suppose in the clustering $\inpset$, $v \in D_i$, for some $i$, and in the clustering $\mop$, $v \in T^*_j$, for some $j$. Note, the number of pairs $(u,v)$ such that $u \in \inpcl{i}$ but $v \notin \mopcl{j}$ is given by

\begin{align}
\sum\limits_{\mopcl{k} \neq \mopcl{j}} \abs{\inpcl{i} \cap \mopcl{k}} \label{equn:expression-1}
\end{align}

Similarly the number of pairs $(u,v)$ such that $u \notin \inpcl{i}$ but $v \in \mopcl{j}$ is given by

\begin{align}
\sum\limits_{\inpcl{k} \neq \inpcl{i}} \abs{\inpcl{k} \cap \mopcl{j}} \label{equn:expression-2}
\end{align}

Thus by the definition of the distance function, we have
\begin{align}
\dist(\inpset, \mop) = \sum\limits_{v \in V}1/2 \cdot \left(\sum\limits_{k \mid \mopcl{k} \neq \mopcl{j}} \abs{\inpcl{i} \cap \mopcl{k}} + \sum\limits_{k \mid \inpcl{k} \neq \inpcl{i}} \abs{\inpcl{k} \cap \mopcl{j}}\right) \label{equn:consensus-metric}
\end{align}.

The $1/2$ comes before the expressions \ref{equn:expression-1} and \ref{equn:expression-2} because we are counting a pair $(u,v)$ twice inside these expressions. $\dist(\inpset, \mop)$ counts the total number of pairs that are together in $\inpset$ but separated by $\mop$ and the pairs that are in different clusters in $\inpset$ but together in $\mop$. 


Let us denote the cost paid by $\mop$ for a vertex $v \in V$ that belongs to $\inpcl{i} \in \inpset$ for some $i$ and also to $\mopcl{j} \in \mop$ for some $j$ by the notation $\opt_v$ where

\[ \opt_v = 1/2 \cdot \left(\sum\limits_{k \mid \mopcl{k} \neq \mopcl{j}} \abs{\inpcl{i} \cap \mopcl{k}} + \sum\limits_{k \mid \inpcl{k} \neq \inpcl{i}} \abs{\inpcl{k} \cap \mopcl{j}} \right)\].

Let us denote the cost paid by $\mop$ for a cluster $\inpcl{i} \in \inpset$ by

\[ \opt_{\inpcl{i}} = \sum_{v \in \inpcl{i}} \opt_v\].

Now, we are ready to state the important property of $\mop$. We show that if for a cluster $\inpcl{i} \in \inpset$, if $\surp{\inpcl{i}} \leq p/2$ then

\[
    \opt_{\inpcl{i}} \geq s(D_i) (|D_i| - s(D_i)) + \frac{1}{2} s(D_i)(p - s(D_i)) 
\]
that is $\mop$ must pay at least the cost $s(D_i) (|D_i| - s(D_i)) + 1/2 s(D_i)(p - s(D_i))$ for the cluster $\inpcl{i}$ if the surplus of $\inpcl{i}$ is less than or equal to $p/2$. 

Again if $\surp{\inpcl{i}} > p/2$ then 

\[
    \opt_{\inpcl{i}} \geq (p - s(D_i)) (|D_i| - s(D_i)) + \frac{1}{2} s(D_i)(p - s(D_i)) 
\] 

that is $\mop$ must pay at least the cost $(p - s(D_i)) (|D_i| - s(D_i)) + \frac{1}{2} s(D_i)(p - s(D_i))$ for the cluster $\inpcl{i}$ if the surplus of $\inpcl{i}$ is greater than $p/2$. 

We state this formally in the following lemma.

\begin{lemma} \label{lem:main-structure-of-M*}
    For a cluster $\inpcl{i} \in \inpset$ we have
    \begin{enumerate}
        \item if $\surp{\inpcl{i}} \leq p/2$ then $\opt_{\inpcl{i}} \geq s(D_i) (|D_i| - s(D_i)) + \frac{1}{2} s(D_i)(p - s(D_i))$.
        \item if $\surp{\inpcl{i}} > p/2$ then $\opt_{\inpcl{i}} \geq (p - s(D_i)) (|D_i| - s(D_i)) + \frac{1}{2} s(D_i)(p - s(D_i)) $.
    \end{enumerate}
\end{lemma}

For proving \cref{lem:main-structure-of-M*}, we need the help of the following proposition.

\begin{proposition}\label{prop:mod-sum-ineq}
    Consider any integer $p \ge 2$, and $\alpha_1,\ldots,\alpha_t \in [p]$, for $t \ge 1$. Let $ s = \sum_{i=1}^t \alpha_i \mod p$. Then
    $\sum_{i=1}^t \alpha_i \left( p - \alpha_i \right) \ge s \left(p-s\right)$.
\end{proposition}

\begin{proof}
    Suppose, $(\alpha_1 + \alpha_2) \mod p = s_1$, hence $(\alpha_1 + \alpha_2) = kp + s_1$ where $k \in \{0,1\}$.

    First, we prove
    \[ 
        \alpha_1(p - \alpha_1) + \alpha_2(p - \alpha_2) \geq s_1(p - s_1)
    \]
    Now,
    \begin{align}
      &\alpha_1(p - \alpha_1) + \alpha_2(p - \alpha_2) \n \\
      = \, \, &\alpha_1p - \alpha_1^2 + \alpha_2p - \alpha_2^2 \n \\
      = \, \, &p(\alpha_1 + \alpha_2) - (\alpha_1 + \alpha_2)^2 + 2\alpha_1 \alpha_2 \n \\
      = \, \, &p(kp + s_1) - (kp + s_1)^2 + 2\alpha_1\alpha_2 && (\text{Replacing} \, \, (\alpha_1 + \alpha_2) \, \, \text{with} \, \, (kp + s_1) )\n \\
      = \, \, &\left(kp + s_1\right)\left(p - \left(kp + s_1\right)\right) + 2 \alpha_1 \alpha_2 \label{equn:prop-equn}
    \end{align}

    Case $1$: When $k = 0$ then \cref{equn:prop-equn} becomes
    \[
        s_1(p - s_1) + 2\alpha_1\alpha_2 \geq s_1(p - s_1) 
    \]

    Case $2$: When $k = 1$ then \cref{equn:prop-equn} becomes

    \begin{align}
        &(p + s_1)(-s_1) + 2\alpha_1\alpha_2 \n \\
        = \, \, &(p - s_1)s_1 + 2(\alpha_1\alpha_2 - s_1p) \n \\
        = \, \, &(p - s_1)s_1 + 2(\left(s_1p + r(p - r - s_1)\right) - s_1p) \n \\
        &(\text{Replacing} \, \, \alpha_1 \, \, \text{with} \, \, (p - r) \, \, \text{and} \, \, \alpha_2 \, \, \text{with} \, \, (s_1 + r) )\n \\
        = \, \, &(p - s_1)s_1 + 2(\left(s_1p + r(p - \alpha_2)\right) - s_1p) \n \\
        \geq \, \, &(p - s_1)s_1 && (\textbf{as} \, \, \alpha_2 \in [p]) \n
    \end{align}

    Hence, we get 

    \begin{align}
        \alpha_1(p - \alpha_1) + \alpha_2(p - \alpha_2) \geq s_1(p - s_1) && (\text{where} \, \, s_1 = (\alpha_1 + \alpha_2) \mod p) \label{equn:prop-main}
    \end{align}

    Now,

    \begin{align}
        \sum_{i=1}^t \alpha_i \left( p - \alpha_i \right) &= \alpha_1(p - \alpha_1) + \alpha_2(p - \alpha_2) + \sum_{i = 3}^t \alpha_i(p - \alpha_i) \n \\
        &\geq s_1(p - s_1) + \alpha_3(p - \alpha_3) + \sum_{i = 4}^t\alpha_i(p - \alpha_i) \n \\ 
        &(\text{using} \, \, \cref{equn:prop-main}, \, \, \text{assume} \, \, s_1 = (\alpha_1 + \alpha_2) \mod p) \n \\
        &\geq s_2(p - s_2) + \alpha_4(p - \alpha_4) + \sum_{i = 5}^t\alpha_i(p - \alpha_i) \n \\
        &(\text{using} \, \, \cref{equn:prop-main}, \, \, \text{assume} \, \, s_2 = (s_1 + \alpha_3) \mod p) \n \\
        & \vdots \n \\
        &\geq s(p - s) && (\text{Proved}) \n
    \end{align}
\end{proof}

\begin{proposition}\label{prop.bound.the.firsterm}
    Given positive integers $ n $, $ p $, and $ x_{1},x_{2}, \dots, x_{t} $ such that $ n = x_{1}+x_{2}+\dots + x_{t} $. Let $ b, b_{1},b_{2}, \dots, b_{t} <p $ be nonnegative integers satisfying $ b\leq n $, $ b_{i} \leq x_{i}$, and $ b = (b_{1} + b_{2} + \dots + b_{t})(\mod{p}) $. Let $ q_{i} = p-b_{i} (\mod{p}) $. Consider $ A = \sum_{i<j}x_{i}x_{j}  + \sum_{i=1}^{t}(x_{i}-b_{i})q_{i} $. Then
    \begin{enumerate}
        \item If $ b\leq p/2 $, then $ A \geq b(n-b) $. \label{enu.first.term.case1}
        \item If $ b>p/2 $, then $ A \geq (p-b)(n-b) $. \label{enu.first.term.case2}
    \end{enumerate}
\end{proposition}
\begin{proof}
    Without loss of generality, assume that $ x_{1}\leq x_{2}\leq \dots \leq x_{t} $.

    (\ref{enu.first.term.case1}) Suppose that $ b\leq p/2 $. We show that $ A \geq b(n-b) $.
    
    As $ x_{t} = \max\{x_{i}\} $, it follows that $ A \geq \sum_{i<j}x_{i}x_{j} = \sum_{i=1}^{t}x_{i}\left( \dfrac{n-x_{i}}{2} \right) \geq n\left( \dfrac{n-x_{t}}{2} \right) $. Additionally, $ A\geq x_{t}(n-x_{t}) + (x_{t}-b_{t})q_{t} $. Let $ B = n \left( \dfrac{n-x_{t}}{2} \right) $, and $ C = x_{t}(n-x_{t}) + (x_{t}-b_{t})q_{t} $. Then $ A\geq B $ and $ A\geq C $. Therefore, it suffices so that $ B\geq b(n-b) $ or $ C\geq b(n-b) $. We proceed by cases.

    \textbf{Case 1}: $ x_{t}\leq b $.
    \begin{itemize}
            \item If $ n\leq x_{t}+b $, then $ C - b(n-b) > x_{t}(n-x_{t}) - b(n-b) = (x_{t}-b)(n-x_{t}-b) \geq 0 $.
            \item If $ n > x_{t} + b $, then
                \begin{align}
                    B - b(n-b) &= \dfrac{1}{2}(n(n-x_{t}) - 2b(n-b)) \nonumber \\
                               &\geq \dfrac{1}{2}(n(n-x_{t})-b(n-b) - b(n-x_{t})) && (\text{as }b\geq x_{t}) \nonumber \\
                               &= \dfrac{1}{2}(n-b)(n-x_{t}-b) \nonumber \\
                               &\geq 0. && (\text{as }n\geq b \text{ and } n > x_{t}+b)\nonumber
                \end{align}
    \end{itemize}

    \textbf{Case 2:} $ x_{t} > b $.
    \begin{itemize}
        \item If $ n \geq x_{t}+b $, then $ C -
            b(n-b) > x_{t}(n-x_{t}) - b(n-b) = (x_{t}-b)(n-x_{t}-b) \geq 0 $.
        \item If $ n<x_{t}+b $, then $ b>n-x_{t} = x_{1} + x_{2} + \dots + x_{t-1} \geq b_{1} + b_{2} + \dots + b_{t-1} $. Note that $ b  = (b_{1} + b_{2} + \dots + b_{t}) (\mod{p}) $. It follows that $ b_{t} = b-(b_{1} + b_{2} + \dots + b_{t-1}) (\mod{p}) = b - (b_{1} + b_{2} + \dots + b_{t-1})$. Hence, $ 0 < b_{t} < b $ and $ q_{t} = p-b_{t} (\mod p) = p - b_{t} > p-b\geq b $, as $ b\leq p/2 $. This leads to
            \begin{align}
                C - b(n-b) &= (x_{t}-b)(n-x_{t}-b) + (x_{t}-b_{t})q_{t} \nonumber \\
                            &> (x_{t}-b)(n-x_{t}-b) + (x_{t}-b)q_{t} && (\text{as }b_{t}<b) \nonumber \\
                            &= (x_{t}-b)(n-x_{t}-b+q_{t}) \nonumber \\
                            &\geq 0. && (\text{as }x_{t}\geq b, n\geq x_{t}, \text{ and }q_{t}>b)\nonumber
            \end{align}
    \end{itemize}
    (\ref{enu.first.term.case2}) Suppose that $ b>p/2 $. We show that $ A \geq (p-b)(n-b) = a(n-b) $, where $ a=p-b\leq p/2 <b $. As before, we have $ A\geq B $ and $ A\geq C $, where $ B= n\left(\dfrac{n-x_{t}}{2}\right) $, and $ C=x_{t}(n-x_{t}) + (x_{t}-b_{t})q_{t} $. Similarly, we proceed by cases.

    \textbf{Case 1:} $ x_{t} \leq a$.
    \begin{itemize}
        \item If $ n\leq x_{t}+a $, then $ C - a(n-b) \geq x_{t}(n-x_{t}) - a(n-a) = (x_{t}-a)(n-x_{t}-a) \geq 0 $.
        \item If $ n>x_{t}+a $, then
            \begin{align}
                B - a(n-b) &\geq \dfrac{1}{2}(n(n-x_{t})-2a(n-a)) \nonumber \\
                           &\geq \dfrac{1}{2}(n(n-x_{t})-a(n-a)-a(n-x_{t})) && (\text{as } a\geq x_{t}) \nonumber \\
                           &= (n-a)(n-x_{t}-a) \nonumber \\
                           &\geq 0. && (\text{as } n\geq a \text{ and }n\geq x_{t}+a) \nonumber.
            \end{align}
    \end{itemize}
    \textbf{Case 2:} $ x_{t}>a $.
    \begin{itemize}
        \item If $ n\geq x_{t}+a $, then $ C-a(n-b)\geq x_{t}(n-x_{t}) - a(n-a) = (x_{t}-a)(n-x_{t}-a)\geq 0 $.
        \item If $ n<x_{t}+a $, then $ a>n-x_{t} = x_{1} + x_{2} + \dots + x_{t-1} \geq b_{1} + b_{2} + \dots + b_{t-1} $. Note that $ b=(b_{1}+b_{2}+\dots + b_{t}) (\mod{p}) $. It follows that $ b_{t} = b-(b_{1}+b_{2}+\dots +b_{t-1}) (\mod{p}) = b-(b_{1}+b_{2}+\dots + b_{t-1}) $. Hence, $ 0<b_{t}<b $, and $ q_{t} = p-b_{t}(\mod{p}) = p-b_{t} $. Therefore, $ q_{t} = p-b+b_{1}+b_{2}+\dots b_{t-1} \geq p-b = a $. If $ b_{t}\leq a $, it follows that
            \begin{align}
                C - a(n-b) &\geq x_{t}(n-x_{t}) + (x_{t}-b_{t})q_{t} - a(n-b) \nonumber \\
                           &\geq (x_{t}-a)(n-x_{t}-a) - (x_{t}-a)q_{t} && (\text{as }b_{t}<a) \nonumber\\
                           &= (x_{t}-a)(n-x_{t}+q_{t}-a) \nonumber \\
                           &\geq 0. &&(\text{as } x_{t} > a,\ n\geq x_{t} \text{ and } q_{t}\geq a) \nonumber
            \end{align}
            It remains to consider $ b_{t}> a $. We have
            \begin{align}
                &\>C - a(n-b) \nonumber \\
                &= \left( x_{t}(n-x_{t})  - b_{t}(n-a) + x_{t}(b_{t}-a)\right) + b_{t}(n-a)-x_{t}(b_{t}-a) + (x_{t}-b_{t})q_{t} - a(n-b) \nonumber \\
                &= (x_{t}-b_{t})(n-x_{t}-a) + (b_{t}-a)(n-x_{t}) + a(b-b_{t}) + (x_{t}-b_{t})q_{t} \nonumber \\
                &\geq (x_{t}-b_{t})(n-x_{t}-a+q_{t}) + (b_{t}-a)(n-x_{t}) + a(b-b_{t}) \nonumber\\
                &\geq 0. \nonumber
            \end{align}
    \end{itemize}
    This concludes the proof.
\end{proof}

Now we are ready to prove \cref{lem:main-structure-of-M*}

\begin{proof}[Proof of \cref{lem:main-structure-of-M*}]
    Suppose in $\mop$, the cluster $\inpcl{i}$ is split into $t$ parts $X_{i,1}, \ldots, X_{i,t}$; more specifically,
    \begin{itemize}
        \item $\inpcl{i} = \bigcup_{j = 1}^t X_{i,j} $, and
        \item $\forall j \neq k \in [t]$, $ X_{i,j} \cap X_{i,k} = \emptyset$.
    \end{itemize}
     For each part $X_{i,j}$, consider a set $S_{i,j} \subseteq X_{i,j}$ that consists of $(\blue{X_{i,j}} \mod p)$ many blue points. Now, since $\mop$ is a $\mopdef$ we get for the clusters $X_{i,k}$ for $k \in \{1, \ldots, t\}$ at least $(p - |S_{i,k}|)$ blue points from clusters other than $\inpcl{i}$ must be merged. Let, these set of $(p - |S_{i,k}|)$ blue points be $S_{i,k}'$. Hence we get,

     \begin{align}
         \opt_{\inpcl{i}} \geq \sum_{j < k}|X_{i,j}||X_{i,k}| + \sum_{k = 1}^t |X_{i,k} \setminus S_{i,k}||S_{i,k}'| + 1/2 \sum_{k = 1}^t |S_{i,k}|  |S_{i,k}'| \label{equn:main-equn}
     \end{align}
      
     \begin{flalign*}
     &\text{Note that} \, \, \sum_{k = 1}^t |S_{i,k}| \mod p = s_i && \text{recall $\surp{\inpcl{i}} = s_i$}
     \end{flalign*}

    Hence, from \cref{prop:mod-sum-ineq} we get,

    \begin{align}
         &1/2 \sum_{k = 1}^t |S_{i,k}|  |S_{i,k}'| \n \\
         = \, \,& 1/2 \sum_{k = 1}^t |S_{i,k}|  (p - |S_{i,k}|) \n \\
         \geq \, \, &1/2 \, \, s_i  \cdot (p - s_i) \label{equn:main-second-term}
    \end{align}

    It suffices to show that
    \begin{align}
        \sum_{j<k}|X_{i,j}||X_{i,k}| + \sum_{k=1}^{t}|X_{i,k}\setminus S_{i,k}||S'_{i,k}| \geq \surp{\inpcl{i}}(|\inpcl{i}| - \surp{\inpcl{i}}),\label{eq.opt.cut.cost.first.term}
    \end{align}
    when $ \surp{\inpcl{i}}\leq p/2 $, and
    \begin{align}
        \sum_{j<k}|X_{i,j}||X_{i,k}| + \sum_{k=1}^{t}|X_{i,k}\setminus S_{i,k}||S'_{i,k}| \geq (p-\surp{\inpcl{i}})(|\inpcl{i}| - \surp{\inpcl{i}}),\label{eq.opt.merge.cost.first.term}
    \end{align}
    when $ \surp{\inpcl{i}} > p/2 $.

    To this end, we let $ |\inpcl{i}| = n,\ \surp{\inpcl{i}} = b<p $, $ |X_{i,j}| = x_{j} $, $ |S_{i,j}| = b_{j}<p $, and $ |S'_{i,j}| = q_{j} $. Then, $ n = x_{1}+x_{2}+\dots x_{t} $, $ b = (b_{1}+b_{2}+ \dots + b_{t}) (\mod{p}) $, $ q_{j} = (p-b_{j}) (\mod{p}) $, and $ x_{j} \geq b_{j} $, for all $ j $. This allows us to apply~\cref{prop.bound.the.firsterm}, which claims the correctness of~\eqref{eq.opt.cut.cost.first.term} and~\eqref{eq.opt.merge.cost.first.term}.
\end{proof}


\subsubsection{Approximation guarantee in the merge case} \label{subsubsec:merge-case}

In this section, we prove that if $\sum_{\inpcl{i} \in \cut} \surp{\inpcl{i}} \leq \sum_{\inpcl{i} \in \merge} \defi{\inpcl{i}}$, that is in the merge case, the distance between input clustering $\inpset$ and the output clustering $\out$ of the algorithm $\algog$, $\dist(\inpset, \out)$ is at most $3 \, \, \dist(\inpset, \mop)$. That we restate formally below.

\mainmerge*

To prove \cref{lem:main-merge-case}, here we need to recall and also define some costs paid by our algorithm $\algog$. For a cluster $\outcl{k} \in \out$ we define $\pi(\outcl{k})$ as a parent of $\outcl{k}$ iff $\pi(\outcl{k}) \in \inpset$ and $\outcl{k}$ is formed either by cutting some points from $\pi(\outcl{k})$ or by merging some points to $\pi(\outcl{k})$.

\begin{itemize}
    \item $\cuta$: The set of clusters $D_i \in \inpset$ from where the algorithm has cut some blue points.
    \item $\mergea$: The set of clusters $D_i \in \inpset$ to where the algorithm has merged some blue points. We have $\mergea = \inpset \setminus \cuta$.
    \item Cost of cutting the $z$th subset of size $p$ from $\inpcl{i}$ where $z \in \{0, \ldots, (|\blue{\inpcl{i}}| - s_i)/p \}$: 
    \begin{itemize}
        \item $\kappa_0(\inpcl{i}) = s_i (|\inpcl{i}| - s_i)$.
        \item For $z \in \{1, \ldots, (|\blue{\inpcl{i}}| - s_i)/p \}$: $\kappa_z(\inpcl{i}) = p \, (|\inpcl{i}| - (s_i + zp))$.
    \end{itemize}
    Let $y_{i,z}$ take the value $1$ if the algorithm cuts the $z$th subset from $D_i$.
    Hence, we define 
    \begin{align}
     \costone{\m{T}} = \sum_{D_i \in \cuta} \sum_{z = 0}^t y_{i,z} \kappa_z(D_i) \label{eq:cost-one-merge} 
    \end{align}
    \item Cost of merging in a cluster $D_j$ which is $\mcostf{D_j}$.
    Hence, we define 
    \begin{align}
        \costtwo{\m{T}} = \sum_{D_j \in \mergea} \defi{D_j} |D_j| \label{eq:cost-two-merge}
    \end{align}
    \item Inter-cluster cost paid by the vertices that belong to a subset of size $p$ of $\inpcl{i}$ that we cut from $\inpcl{i}$ in the algorithm $\algom$:
    \begin{itemize}
        \item (Intercluster cost) $\costthree{\out}$: $\sum\limits_{\outcl{k} \in \out} |(\outcl{k} \setminus \pi(\outcl{k})) \setminus \inpcl{i}|  |(\outcl{k} \setminus \pi(\outcl{k})) \cap \inpcl{i}|$.
    \end{itemize}
    If the algorithm $\algog$ cuts a subset $W_{i,z}$ from a cluster $D_i$ then
    we can upper bound $\costthree{\out}$ by 
    \begin{align}
        \costthree{\out} \leq |W_{i,z}| (p - |W_{i,z}|) \label{eq:cost-three-merge} 
    \end{align}
    \item Intracluster cost within the vertices that belong to a subset of size $p$ of $\inpcl{i}$ that we cut from $\inpcl{i}$ in the algorithm $\algom$:
    \begin{itemize}
        \item (Intracluster cost) $\costfour{\out}$: $\sum\limits_{\substack{\outcl{k}, \outcl{j} \in \out \\ k \neq j}} |(\outcl{k} \setminus \pi(\outcl{k})) \cap \inpcl{i}| |(\outcl{j} \setminus \pi(\outcl{j})) \cap \inpcl{i})|$.
    \end{itemize}
    If the algorithm $\algog$ cuts a subset $W_{i,z}$ from a cluster $D_i$ and it further gets split into parts $W^{(1)}_{i,z}, W^{(2)}_{i,z}, \ldots, W^{(t)}_{i,z}$ then
    we can upper bound $\costfour{\out}$ by 
    \begin{align}
        \costfour{\out} \leq \frac{1}{2}\sum_{j = 1}^t|W^{(k)}_{i,z}| (|W_{i,z}| - |W^{(k)}_{i,z}|) \label{eq:cost-four-merge}
    \end{align}
\end{itemize}


Now we prove the following claim that we need to prove \cref{lem:main-merge-case}.

\begin{claim}\label{clm:merge-case-structure}
Suppose in $\mop$, a cluster $\inpcl{i} \in \inpset$ gets split  into $t$ parts $X_{i,1}, X_{i,2}, \ldots, X_{i,t}$ more specifically,
    \begin{itemize}
        \item $X_{i,j} \subseteq T_{r_j}^*$ for some $T_{r_j}^* \in \mop$.
        \item $\inpcl{i} = \bigcup\limits_{j = 1}^t X_{i,j}$ and
        \item $X_{i,j} \cap X_{i,k} = \emptyset \, \, \forall j \neq k$.
    \end{itemize}
    then
    \begin{enumerate}
        \item Either $\exists j \in [t]$, such that $X_{i,j} = T_{r_j}^*$ (for some $T_{r_j}^* \in \mop)$ and $\red{\inpcl{i}} \subseteq X_{i,j}$ and $|X_{i,\ell}| < p$ $\forall \, \, \ell \neq j$.
        \item or $|X_{i,\ell}| < p$, $\forall \, \, \ell \in [t]$.
    \end{enumerate}
\end{claim}

Now, to prove the previous claim, we need to prove the other claims below. 

\begin{subclaim}\label{clm:bound-on-part-other-than-Xij}
        Consider a partition $X_{i,j}$ of $\inpcl{i}$. Suppose, $X_{i,j} \subseteq T_{r_j}^*$ (for some $T_{r_j}^* \in \mop$) then $|T_{r_j}^* \setminus X_{i,j}| \leq s_{i,j}$ where $s_{i,j} = |\blue{X_{i,j}}| \mod p$.
\end{subclaim}

\begin{proof}
    Suppose for a partition $X_{i,j}$, $|T_{r_j}^* \setminus X_{i,j}| > s_{i,j}$. To prove this first we construct $\m{M}$ from $\mop$ such that $\dist(\inpset, \m{M}) < \dist(\inpset, \mop)$. This would contradict the fact that $\mop$ is the closest $\mopdef$ and thus we will conclude $|T_{r_j}^* \setminus X_{i,j}| \leq s_{i,j}$.

    \paragraph{Construction of $\m{M}$ from $\mop$:} $\m{M} = \mop \setminus \{ T_{r_j}^*\} \cup \{ (X_{i,j} \setminus S_{i,j}), (T_{r_j}^* \setminus (X_{i,j} \setminus S_{i,j}))\}$.

    That is $\m{M}$ we remove the set $T_{r_j}^*$ from $\mop$ and add two sets $(X_{i,j} \setminus S_{i,j})$ and $(T_{r_j}^* \setminus (X_{i,j} \setminus S_{i,j}))$ where $S_{i,j} \subseteq \blue{X_{i,j}}$ such that $|S_{i,j}| = s_{i,j}$.

    Note that,

    \begin{align}
        \dist(\inpset, \m{M}) = \dist(\inpset, \mop) + |X_{i,j} \setminus S_{i,j}||S_{i,j}| - |X_{i,j} \setminus S_{i,j}||T_{r_j}^* \setminus X_{i,j}| \n
    \end{align}

    This is because, 
        \begin{itemize}
            \item \textbf{Reason behind the term $|X_{i,j} \setminus S_{i,j}||S_{i,j}|$} : The pairs $(u,v)$ such that $u \in (X_{i,j} \setminus S_{i,j})$ and $v \in S_{i,j}$ are not counted in $\dist(\inpset, \mop)$ because $u$ and $v$ are present in the same clusters $\inpcl{i} \in \inpset$ and $T_{r_j}^* \in \mop$. In $\m{M}$, since, these pairs $u$ and $v$ belongs to different clusters $(X_{i,j} \setminus S_{i,j})$ and $(T_{r_j}^* \setminus (X_{i,j} \setminus S_{i,j}))$ respectively, hence, these pairs are counted in $\dist(\inpset, \m{M})$.

            \item \textbf{Reason behind the term $|X_{i,j} \setminus S_{i,j}||T_{r_j}^* \setminus X_{i,j}|$} : The pairs $(u,v)$ such that $u \in (X_{i,j} \setminus S_{i,j})$ and $v \in T_{r_j}^* \setminus X_{i,j}$ are counted in $\dist(\inpset, \mop)$ because such pairs $u$ and $v$ are present in the different clusters in $\inpset$ but in the same cluster $T_{r_j}^* \in \mop$. In $\m{M}$, since, these pairs $u$ and $v$ belongs to different clusters $(X_{i,j} \setminus S_{i,j})$ and $(T_{r_j}^* \setminus (X_{i,j} \setminus S_{i,j}))$ respectively, hence, these pairs are not counted in $\dist(\inpset, \m{M})$.
        \end{itemize}
    Now,
    \begin{align}
        \dist(\inpset, \m{M}) &= \dist(\inpset, \mop) + |X_{i,j} \setminus S_{i,j}||S_{i,j}| - |X_{i,j} \setminus S_{i,j}||T_{r_j}^* \setminus X_{i,j}| \n \\
        &< \dist(\inpset, \mop) + |X_{i,j} \setminus S_{i,j}| s_{i,j} - |X_{i,j} \setminus S_{i,j}| s_{i,j} && (\textbf{as} |S_{i,j}| = s_{i,j} \, \, \text{and} \, \, |\mop \setminus X_{i,j}| > s_{i,j}). \n \\ 
        &= \dist(\inpset, \mop). \n
    \end{align}
\end{proof}

\begin{subclaim}\label{clm:at-most-one-partition-size-more-p}
    There exists at most one partition $X_{i,j}$ such that $|X_{i,j}| \geq p$
\end{subclaim}

\begin{proof}
     Let us assume there exists two partitions $X_{i,j}$ and $X_{i,j'}$ such that $|X_{i,j}| \geq p$, $|X_{i,j'}|\geq p$ and $X_{i,j} \subseteq T_{r_j}^*$ and $X_{i,j'} \subseteq T_{r_{j'}}^*$ for some $T_{r_j}^*$ and $T_{r_j'}^*$ in $\mop$.

      Without loss of generality, assume, $|T_{r_j}^* \setminus X_{i,j}| \leq |T_{r_{j'}}^* \setminus X_{i,j'}|$. Now we construct $\m{M}$ from $\mop$ in such a way that $\dist(\inpset, \m{M}) < \dist(\inpset, \mop)$ which is a contradiction, and thus we will conclude that there exists at most one partition $X_{i,j}$ such that $|X_{i,j}| \geq p$.

      \paragraph{Construction of $\m{M}$ from $\mop$ :}

      \begin{align}
          \m{M} = \mop \setminus \left\{ T_{r_j}^*, T_{r_{j'}}^* \right\} \cup \left\{ \left( T_{r_j}^* \cup (X_{i,j'} \setminus S_{i,j'}) \right), \left(T_{r_{j'}}^* \setminus  (X_{i,j'} \setminus S_{i,j'})\right)\right\} \n
      \end{align}

      That is in $\m{M}$ we remove sets $T_{r_j}^*$ and $T_{r_{j'}}^*$ from $\mop$ and add two sets $\left( T_{r_j}^* \cup (X_{i,j'} \setminus S_{i,j'}) \right)$ and $\left(T_{r_{j'}}^* \setminus  (X_{i,j'} \setminus S_{i,j'})\right)$ where $S_{i,j'} \subseteq \blue{X_{i,j'}}$ such that $|S_{i, j'}| = (\blue{X_{i,j'}} \mod p)$.

      Note that,

      \begin{align}
          \dist(\inpset, \m{M}) = \dist(\inpset, \mop) + |(X_{i,j'} \setminus S_{i,j'})||T_{r_j}^* \setminus X_{i,j}| + |(X_{i,j'} \setminus S_{i,j'})||S_{i,j'}| \n \\
          - |(X_{i,j'} \setminus S_{i,j'})||T_{r_{j'}}^* \setminus X_{i,j'}| - |(X_{i,j'} \setminus S_{i,j'})||X_{i,j}| \n
      \end{align}

      This is because

      \begin{itemize}
          \item \textbf{The reason behind the term $|(X_{i,j'} \setminus S_{i,j'})||T_{r_j}^* \setminus X_{i,j}|$:} The pairs $(u,v)$ such that $u \in (X_{i,j'} \setminus S_{i,j'})$ and $v \in T_{r_j}^* \setminus X_{i,j}$ are not counted in $\dist(\inpset, \mop)$ because $u$ and $v$ are present in different clusters both in $\inpset$ and $\mop$. In $\m{M}$, since, these pairs $u$ and $v$ belong to the same cluster $\left( T_{r_j}^* \cup (X_{i,j'} \setminus S_{i,j'}) \right)$. Hence, these pairs are counted in $\dist(\inpset, \m{M})$.

          \item \textbf{The reason behind the term $|(X_{i,j'} \setminus S_{i,j'})||S_{i,j'}|$:} The pairs $(u,v)$ such that $u \in (X_{i,j'} \setminus S_{i,j'})$ and $v \in S_{i,j'}$ are not counted in $\dist(\inpset, \mop)$ because $u$ and $v$ are present in the same clusters $\inpcl{i} \in \inpset$ and $T_{r_{j'}}^* \in \mop$. In $\m{M}$, since, these pairs $u$ and $v$ belongs to different clusters $\left( T_{r_j}^* \cup (X_{i,j'} \setminus S_{i,j'}) \right)$ and $\left(T_{r_{j'}}^* \setminus  (X_{i,j'} \setminus S_{i,j'})\right)$ respectively. Hence, these pairs are counted in $\dist(\inpset, \m{M})$.

          \item \textbf{The reason behind the term $|(X_{i,j'} \setminus S_{i,j'})||T_{r_{j'}}^* \setminus X_{i,j'}|$:} The pairs $(u,v)$ such that $u \in (X_{i,j'} \setminus S_{i,j'})$ and $v \in T_{r_{j'}}^* \setminus X_{i,j'}$ are counted in $\dist(\inpset, \mop)$ because such pairs $u$ and $v$ are present in the different clusters in $\inpset$ but in the same cluster $T_{r_{j'}}^* \in \mop$. In $\m{M}$, since, these pairs $u$ and $v$ belongs to different clusters $\left( T_{r_j}^* \cup (X_{i,j'} \setminus S_{i,j'}) \right)$ and $\left(T_{r_{j'}}^* \setminus  (X_{i,j'} \setminus S_{i,j'})\right)$ respectively. Hence, these pairs are not counted in $\dist(\inpset, \m{M})$.

          \item \textbf{The reason behind the term $|(X_{i,j'} \setminus S_{i,j'})||X_{i,j}|$:} The pairs $(u,v)$ such that $u \in (X_{i,j'} \setminus S_{i,j'})$ and $v \in X_{i,j}$ are counted in $\dist(\inpset, \mop)$ because such pairs $u$ and $v$ are present in the same cluster $\inpcl{i} \in \inpset$ but in different clusters in  $\mop$. In $\m{M}$, since, these pairs $u$ and $v$ belongs to same cluster $\left( T_{r_j}^* \cup (X_{i,j'} \setminus S_{i,j'}) \right)$. Hence, these pairs are not counted in $\dist(\inpset, \m{M})$.
      \end{itemize}

      Now,
      \begin{align}
          \dist(\inpset, \m{M}) &= \dist(\inpset, \mop) + |(X_{i,j'} \setminus S_{i,j'})||T_{r_j}^* \setminus X_{i,j}| + |(X_{i,j'} \setminus S_{i,j'})||S_{i,j'}|\n \\
          &- |(X_{i,j'} \setminus S_{i,j'})||T_{r_{j'}}^* \setminus X_{i,j'}| - |(X_{i,j'} \setminus S_{i,j'})||X_{i,j}| \n \\
          &= \dist(\inpset, \mop) + |(X_{i,j'} \setminus S_{i,j'})|\left(|T_{r_j}^* \setminus X_{i,j}| - |T_{r_{j'}}^* \setminus X_{i,j'}|\right) + |(X_{i,j'} \setminus S_{i,j'})|\left(|S_{i,j'}| - |X_{i,j}|\right) \n \\
          &< \dist(\inpset, \mop) + |(X_{i,j'} \setminus S_{i,j'})|\left(|T_{r_j}^* \setminus X_{i,j}| - |T_{r_{j'}}^* \setminus X_{i,j'}|\right) + |(X_{i,j'} \setminus S_{i,j'})|\left(|S_{i,j'}| - p\right) \n  \\ 
          &(\textbf{as} \, \, |X_{i,j}| \geq p) \n \\
          &< \dist(\inpset, \mop) \n\\ 
          &(\textbf{as} \, \, |T_{r_j}^* \setminus X_{i,j}| \leq |T_{r_{j'}}^* \setminus X_{i,j'}| \, \, \text{and} \, \, |S_{i,j'}| < p) \n.
      \end{align}   
\end{proof}

Now from \cref{clm:at-most-one-partition-size-more-p}, we get that there exists at most one partition $X_{i,j}$ such that $|X_{i,j}| \geq p$ the following claims are related to the partition $X_{i,j}$. Recall each partition $X_{i,\ell} \subseteq T_{r_\ell}^*$ for all $\ell \in [t]$ where $T_{r_\ell}^* \in \mop$.

\begin{subclaim}\label{clm:geq-p-minus-si}
     $|\blue{\bigcup\limits_{\ell \neq j}X_{i,\ell}}| \geq (p - s_{i,j})$
\end{subclaim}

\begin{proof}
   Suppose for contradiction, 
     \[ |\blue{\bigcup_{\ell \neq j} X_{i,\ell}}| < (p - s_{i,j})\]
    Now we prove,
    \[ |\blue{\bigcup_{\ell \neq j} X_{i,\ell}}| < (p - s_{i,j}) \Rightarrow |\blue{\bigcup_{\ell \neq j} X_{i,\ell}}| = s_i - s_{i,j} \]

    \begin{align}
        &\left(  |\blue{X_{i,j}}| + |\blue{\bigcup_{\ell \neq j}X_{i,\ell}}| \right) \mod p = s_i. \n \\
        \Rightarrow & \, \, |\blue{X_{i,j}}| \mod p + |\blue{\bigcup_{\ell \neq j}X_{i,\ell}}| \mod p = s_i. \n \\
        \Rightarrow &\, \, s_{i,j} + |\blue{\bigcup_{\ell \neq j}X_{i,\ell}}| \mod p = s_i. \n \\
        \Rightarrow &\, \, |\blue{\bigcup_{\ell \neq j}X_{i,\ell}}| \mod p = s_i -s_{i,j} \label{equn:bound-on-blue-points}
    \end{align}
    Since, $|\blue{\bigcup\limits_{\ell \neq j}}X_{i,\ell}| < (p - s_{i,j})$ and due to \cref{equn:bound-on-blue-points} we have
    \[ |\blue{\bigcup\limits_{\ell \neq j} X_{i, \ell}}| = s_i - s_{i,j}\]

    Now we prove if $|\blue{\bigcup_{\ell \neq j}X_{i,\ell}}| = s_i - s_{i,j}$ then there exists a clustering $\m{M}$ such that $\dist(\inpset, \m{M}) < \dist(\inpset, \mop)$ which is a contradiction and thus we conclude our claim which is $|\blue{\bigcup\limits_{\ell \neq j}X_{i,\ell}}| \geq (p - s_i)$.

    \paragraph{Construction of $\m{M}$ from $\mop$ \\\\}

    Note that since  $\blue{T_{r_j}^*} \mod p = 0$ and by \cref{clm:bound-on-part-other-than-Xij} we get $|T_{r_j}^* \setminus X_{i,j}| \leq s_{i,j}$ thus

    \[ |\blue{T_{r_j}^* \setminus X_{i,j}}| = p - s_{i,j} > s_i - s_{i,j}\]

    Let $Y \subseteq \blue{T_{r_j}^* \setminus X_{i,j}}$ such that $|Y| = s_i - s_{i,j}$.

    Here, the construction of $\m{M}$ from $\mop$ is a bit more involved, so first, we provide an informal description of the construction, and then we give the formal description.

    Recall, the cluster $D_i$ is split into $t$ partitions $X_{i,1}, X_{i,2}, \ldots, X_{i,t}$ in $\mop$. Among these partitions $|X_{i,j}| \geq p$ and each of these partitions $X_{i,\ell} \subseteq T^*_{r_\ell}$ for some $T^*_{r_\ell} \in \mop$.

    While constructing $\m{M}$ from $\mop$ we do the following steps

    \begin{enumerate}
        \item Divide $Y$ into $(t - 1)$ disjoint parts $Y_1, Y_2, \ldots, Y_{j - 1},Y_{j + 1}, \ldots, Y_{t}$ such that
        \begin{itemize}
            \item $|Y_\ell| = |\blue{X_{i,\ell}}|$ $\forall \ell \neq j$ and
            \item $\bigcup_{\ell \neq j}Y_\ell = Y$
        \end{itemize}
        \item Remove $t$ clusters $T_{r_\ell}^*$ for all $\ell \in [t]$ from $\mop$.
        \item Add $t$ new clusters to $\m{M}$
        \begin{itemize}
            \item First cluster : $\inpcl{i} \cup (T^*_{r_j} \setminus (X_{i,j} \cup Y))$
            \item Remaining $(t - 1)$ clusters : $(T^*_{r_\ell} \setminus X_{i,\ell}) \cup Y_\ell) \, \,  \forall \ell \neq j$.
        \end{itemize}
        
        We move the red points of $X_{i, \ell}$ for all $\ell \neq j$ to the set $T_{r_j}^*$.
        
        The main part of the above construction is that we have divided $Y$ into $(t - 1)$ disjoint parts $Y_1, Y_2, \ldots, Y_{j - 1}, Y_{j + 1}, \ldots, Y_{t}$ such that $|Y_{\ell}| = |\blue{X_{i, \ell}}| \, \, \forall \ell \neq j$. After that, we swap the points in $\blue{X_{i,\ell}}$ with the points in $Y_\ell$.

    \end{enumerate}

    Formally,

    \[
    \m{M} = \mop \setminus \{T^*_{r_\ell} \mid \ell \neq j \} \cup \left\{\left( D_i \cup \left( T_{r_j}^* \setminus \left( X_{i,j} \cup Y \right) \right) \right),  \left(  \left( T_{r_\ell}^* \setminus X_{i, \ell} \right) \cup Y_\ell \right) \mid \ell \neq j  \right\}
    \]

    where $Y_\ell \subseteq Y$ such that 

    \begin{itemize}
        \item $|Y_\ell| = |\blue{X_{i,\ell}}|$ $\forall \ell \neq j$ and
        \item $Y_{m} \cap Y_{n} = \emptyset$ $\forall m \neq n$ and $m, n \in [t]$.
        \item $\bigcup_{\ell \neq j}Y_\ell = Y$.
    \end{itemize}

    Note that,
    \begin{align}
     \dist(\inpset, \m{M}) = \dist(\inpset, \mop) &+  \left|\bigcup_{\ell \neq j}X_{i,\ell}\right| \left| T^*_{r_j} \setminus (X_{i,j} \cup Y)\right| + |Y| \left| T^*_{r_j} \setminus (X_{i,j} \cup Y)\right| + \sum_{\ell \neq j}|Y_\ell| |T^*_{r_\ell} \setminus X_{i,\ell}| \n \\
     &- \left|\bigcup_{\ell \neq j}X_{i,\ell}\right| \left|X_{i,j}\right| - |Y| |X_{i,j}| - \sum_{\ell \neq j}|\blue{X_{i,\ell}}| |T^*_{r_\ell} \setminus X_{i,\ell}| \n 
     \end{align}

    This is because
    \begin{itemize}
        \item \textbf{Reason behind the term $\left|\bigcup_{\ell \neq j}X_{i,\ell}\right| \left| T^*_{r_j} \setminus (X_{i,j} \cup Y)\right|$}: The pairs $(u,v)$ such that $u \in \bigcup_{\ell \neq j}X_{i,\ell}$ and $v \in T_{r_j}^* \setminus (X_{i,j} \cup Y)$ are not be counted in $\dist(\inpset, \mop)$ because $u$ and $v$ are present in different clusters both in $\inpset$ and $\mop$. In $\m{M}$, since these pairs $u$ and $v$ belong to the same cluster, these pairs are counted in $\dist(\inpset, \m{M})$.
        
        \item \textbf{Reason behind the term $|Y| |T_{r_j}^* \setminus (X_{i,j} \cup Y)|$}: The pairs $(u,v)$ such that $u \in Y$ and $v \in T_{r_j}^* \setminus (X_{i,j} \cup Y)$ may not be counted in $\dist(\inpset, \mop)$ because $u$ and $v$ are present in same clusters both in $\inpset$ and $T_{r_j}^* \in \mop$. In $\m{M}$, since these pairs $u$ and $v$ belong to the different clusters, these pairs are counted in $\dist(\inpset, \m{M})$.
        
        \item \textbf{Reason behind the term $\sum_{\ell \neq j}|Y_\ell| T^*_{r_\ell} \setminus X_{i,\ell}|$}: The pairs $(u,v)$ such that $u \in Y_\ell$ and $v \in T^*_{r_\ell} \setminus X_{i,\ell}$ are not be counted in $\dist(\inpset, \mop)$ because $u$ and $v$ are present in different clusters both in $\inpset$ and $\mop$. In $\m{M}$, since, these pairs $u$ and $v$ belong to the same cluster $T^*_{r_\ell}$ these pairs are counted in $\dist(\inpset, \m{M})$.

        \item \textbf{Reason behind the term $\left|\bigcup_{\ell \neq j}X_{i,\ell}\right| \left|X_{i,j}\right|$}: The pairs $(u,v)$ such that $u \in X_{i,j}$ and $v \in Y$ are counted in $\dist(\inpset, \mop)$ because $u$ and $v$ are present in different clusters $\mop$ but present in the same cluster in $D_i \in \inpset$. In $\m{M}$, since these pairs $u$ and $v$ belong to the same cluster, these pairs are not counted in $\dist(\inpset, \m{M})$.

        \item \textbf{Reason behind the term $|Y||X_{i,j}|$}: The pairs $(u,v)$ such that $u \in Y$ and $v \in X_{i,j}$ are counted in $\dist(\inpset, \mop)$ because $u$ and $v$ are present in same cluster $T_{r_j}^* \in \mop$ but present in different clusters in $\inpset$. In $\m{M}$, since these pairs $u$ and $v$ belong to the different clusters, these pairs are not counted in $\dist(\inpset, \m{M})$.

        \item \textbf{Reason behind the term $\sum_{\ell \neq j}|\blue{X_{i,\ell}}| |T^*_{r_\ell} \setminus X_{i,\ell}|$}: The pairs $(u,v)$ such that $u \in \blue{X_{i,\ell}}$ and $v \in T^*_{r_\ell} \setminus X_{i,\ell}$ are counted in $\dist(\inpset, \mop)$ because $u$ and $v$ are present in same cluster $\mop$ but present in different clusters in $\inpset$. In $\m{M}$, since these pairs $u$ and $v$ belong to the different clusters, these pairs are not counted in $\dist(\inpset, \m{M})$.

    \end{itemize}

    Now,

    \begin{align}
     \dist(\inpset, \m{M}) &= \dist(\inpset, \mop) +  \left|\bigcup_{\ell \neq j}X_{i,\ell}\right| \left| T^*_{r_j} \setminus (X_{i,j} \cup Y)\right| + |Y| \left| T^*_{r_j} \setminus (X_{i,j} \cup Y)\right| + \sum_{\ell \neq j}|Y_\ell| T^*_{r_\ell} \setminus X_{i,\ell}| \n \\
     &- \left|\bigcup_{\ell \neq j}X_{i,\ell}\right| \left|X_{i,j}\right| - |Y| |X_{i,j}| - \sum_{\ell \neq j}|\blue{X_{i,\ell}}| |T^*_{r_\ell} \setminus X_{i,\ell}| \n \\
     &< \dist(\inpset, \mop) \n \\
     &(\textbf{by} \, \, \cref{clm:bound-on-part-other-than-Xij} \, \,  \text{we have}  \, \, T_{r_j}^* \setminus (X_{i,j} \cup Y) = s_{i,j} \, \textbf{and} \, \, |X_{i,j}| > p ) \n
     \end{align}

\end{proof}

\begin{subclaim}\label{clm:equal-to-one-cluster}
    $X_{i,j} = T_{r_j}^*$ (for some $T_{r_j}^* \in \mop$).
\end{subclaim}

\begin{proof}
     Suppose for contradiction $X_{i,j} \subset T_{r_j}^*$. at first we prove $\red{T_{r_j}^* \setminus X_{i,j}} = \emptyset$.

     Suppose $\red{T_{r_j}^* \setminus X_{i,j}} \neq \emptyset$ then we construct $\m{M}$ from $\mop$ such that $\dist(\inpset, \m{M}) < \dist(\inpset, \mop)$ which is a contradiction and thus it will conclude $\red{T_{r_j}^* \setminus X_{i,j}} = \emptyset$.

     \paragraph{Construction of $\m{M}$ from $\mop$ \\\\}

     \[ \m{M} = \mop \setminus \{T_j^*\} \cup \left\{ \left( T_j^* \setminus \red{T_{r_j}^* \setminus X_{i,j}} \right), \red{T_{r_j}^* \setminus X_{i,j}} \right\} \]

     That is in $\m{M}$ we remove the set $T_j^*$ from $\mop$ and add sets $\left( T_j^* \setminus \red{T_{r_j}^* \setminus X_{i,j}} \right)$ and $\red{T_{r_j}^* \setminus X_{i,j}}$.

     Note that,
     \begin{align}
         \dist(\inpset, \m{M}) &= \dist(\inpset, \mop) + |\red{T_{r_j}^* \setminus X_{i,j}}| |\blue{T_{r_j}^* \setminus X_{i,j}}| - |\red{T_{r_j}^* \setminus X_{i,j}}| |X_{i,j}| \n
     \end{align}

     This is because

\begin{itemize}
    \item \textbf{Reason behind the term $|\red{T_{r_j}^* \setminus X_{i,j}}| |\blue{T_{r_j}^* \setminus X_{i,j}}|$ :} The pairs $(u,v)$ such that $u \in \red{T_{r_j}^* \setminus X_{i,j}}$ and $v \in \blue{T_{r_j}^* \setminus X_{i,j}}$ are not counted in $\dist(\inpset, \mop)$ because $u$ and $v$ are present in same clusters both in $\inpset$ and $T_{r_j}^* \in \mop$. In $\m{M}$, since, these pairs $u$ and $v$ belong to the different clusters $\red{T_{r_j}^* \setminus X_{i,j}}$ and $\left( T_j^* \setminus \red{T_{r_j}^* \setminus X_{i,j}} \right)$ respectively thus these pairs are counted in $\dist(\inpset, \m{M})$.

    \item \textbf{Reason behind the term $|\red{T_{r_j}^* \setminus X_{i,j}}| |X_{i,j}|$ :} The pairs $(u,v)$ such that $u \in \red{T_{r_j}^* \setminus X_{i,j}}$ and $v \in X_{i,j}$ are counted in $\dist(\inpset, \mop)$ because $u$ and $v$ are present in the same cluster $T_{r_j}^* \in \mop$ but present in different clusters in $\inpset$. In $\m{M}$, since these pairs $u$ and $v$ belong to the different clusters, these pairs are not counted in $\dist(\inpset, \m{M})$.
\end{itemize}

Now, we have proved that $\red{T_{r_j}^* \setminus X_{i,j}} = \emptyset$. We also know that $|\blue{T_{r_j}^* \setminus X_{i,j}}| = (p - s_{i,j})$ thus $|T_{r_j}^* \setminus X_{i,j}| = (p - s_{i,j})$. From \cref{clm:geq-p-minus-si} we get that $|\blue{\bigcup\limits_{\ell \neq j}X_{i,\ell}}| \geq (p - s_{i,j})$. Let $Y \subseteq \blue{\bigcup\limits_{\ell \neq j}X_{i,\ell}}$ such that $|Y| = (p - s_{i,j})$.

To prove $X_{i,j} = T_{r_j}^*$ (for some $T_{r_j}^* \in \mop$) we construct $\m{M}$ from $\mop$ by swapping the points present in $Y$ with the points in $T_{r_j}^* \setminus X_{i,j}$. After that, we prove $\dist(\inpset, \m{M}) < \dist(\inpset, \mop)$. 

\paragraph{Construction of $\m{M}$ from $\mop$}: Suppose $Y_\ell = Y \cap X_{i,\ell}$ $\forall \ell \neq j$. We divide $T_{r_j}^* \setminus X_{i,j}$ into subsets of small size $Z_1, \ldots, Z_t$ such that $|Z_{\ell}| = |Y_{\ell}|$ $\forall \ell \neq j$. 

$\m{M} = \mop \setminus \left\{ T^*_{r_k} \, \middle\vert \, k \in \left\{1, \ldots, t \right\}\right\} \, \cup \, \left\{ (X_{i,j} \cup Y), (T^*_{r_\ell} \setminus Y_\ell \, \middle\vert \, \forall \ell \neq j)\right\}$.

Note that,

\begin{align}
    \dist(\inpset, \m{M}) < \, \, &\dist(\inpset, \mop) \n \\
    &+ \left(p\sum_{\ell \neq j}|Y_{\ell}|\, \, + p\sum_{\ell \neq j}|Z_{\ell}| + \sum_{\ell_1 < \ell_2} |Z_{\ell_1}| |Z_{\ell_2}|\right) \n \\
    &- \left( |X_{i,j}|\sum_{\ell \neq j}|Z_{\ell}| + |X_{i,j}|\sum_{\ell \neq j}|Y_{\ell}| + \sum_{\ell_1 < \ell_2} |Y_{\ell_1}| |Y_{\ell_2}|\right)
\end{align}

\begin{itemize}
    \item \textbf{Reason behind the terms $p\sum_{\ell \neq j}|Y_{\ell}|$ and $p\sum_{\ell \neq j}|Z_{\ell}|$:} The pairs $(u,v)$ such that $u \in Y_\ell$ and $v \in X_{i,\ell}$ for $\ell \neq j$ are not counted in $\dist(\inpset, \mop)$ because $u$ and $v$ are present in the same clusters both in $\mop$ and $\inpset$. In $\m{M}$, since these pairs $u$ and $v$ belong to the different clusters, these pairs are counted in $\dist(\inpset, \m{M})$. Since, $|X_{i,\ell}| \leq p$ for $\ell \neq j$ thus we have the term $p\sum_{\ell \neq j}|Y_{\ell}|$.

    Again, since The pairs $(u,v)$ such that $u \in Z_\ell$ and $v \in X_{i,\ell}$ for $\ell \neq j$ are not counted in $\dist(\inpset, \mop)$ because $u$ and $v$ are present in different clusters both in $\mop$ and $\inpset$. In $\m{M}$, since, these pairs $u$ and $v$ belong to the same cluster these pairs are counted in $\dist(\inpset, \m{M})$ and thus we have the term $p\sum_{\ell \neq j}|Y_{\ell}|$.

    \item \textbf{Reason behind the term $\sum_{\ell_1 < \ell_2} |Z_{\ell_1}| |Z_{\ell_2}|$ :} The pairs $(u,v)$ such that $u \in Z_{\ell_1}$ and $v \in Z_{\ell_2}$ are not counted in $\dist(\inpset, \mop)$ because $u$ and $v$ are present in the same clusters both in $\mop$ and $\inpset$. In $\m{M}$, since these pairs $u$ and $v$ belong to the different clusters, these pairs are counted in $\dist(\inpset, \m{M})$. 

    \item \textbf{Reason behind the terms $|X_{i,j}|\sum_{\ell \neq j}|Z_{\ell}|$ and $|X_{i,j}|\sum_{\ell \neq j}|Y_{\ell}|$:} The pairs $(u,v)$ such that $u \in X_{i,j}$ and $v \in Z_\ell$ are counted in $\dist(\inpset, \mop)$ because $u$ and $v$ are present in the same cluster $T_{r_j}^* \in \mop$ but present in different clusters in $\inpset$. In $\m{M}$, since these pairs $u$ and $v$ belong to the different clusters, these pairs are not counted in $\dist(\inpset, \m{M})$.

    Similarly, the pairs $(u,v)$ such that $u \in X_{i,j}$ and $v \in Y_\ell$ are counted in $\dist(\inpset, \mop)$ because $u$ and $v$ are present in different clusters in $\mop$ but present in the same cluster in $\inpset$. In $\m{M}$, since these pairs $u$ and $v$ belong to the same cluster, these pairs are not counted in $\dist(\inpset, \m{M})$.

    \item \textbf{Reason behind the term $\sum_{\ell_1 < \ell_2} |Y_{\ell_1}| |Y_{\ell_2}|$ :} The pairs $(u,v)$ such that $u \in Y_{\ell_1}$ and $v \in Y_{\ell_2}$ are counted in $\dist(\inpset, \mop)$ because $u$ and $v$ are present in different clusters in $\mop$ but in same cluster in $\inpset$. In $\m{M}$, since these pairs $u$ and $v$ belong to the same cluster, these pairs are not counted in $\dist(\inpset, \m{M})$. 
\end{itemize}
\end{proof}

\begin{subclaim}\label{clm:red-subseteq}
    $\red{\inpcl{i}} \subseteq X_{i,j}$.
\end{subclaim}

\begin{proof}
    Suppose not then $\red{\bigcup_{\ell \neq j} X_{i,\ell}} \neq \emptyset$

    \paragraph{Construction of $\m{M}$ from $\mop$}:

    $\m{M} = \mop  \setminus \left\{ T^*_{r_\ell} \, \, \middle\vert \, \, (\red{D_i} \cap T^*_{r_\ell}) \neq \emptyset\right\}\cup \left\{ \left(T_{r_j}^* \cup \red{\bigcup\limits_{\ell \neq j} X_{i,\ell}} \right), \left( T^*_{r_\ell} \setminus \red{D_i} \, \,  \middle\vert \, \, i \neq \ell \right) \right\}$

    Note that,

    \begin{align}
        \dist(\inpset, \m{M}) &< \dist(\inpset, \mop) + p\left| \red{\bigcup_{\ell \neq j}X_{i,\ell}}\right| - \left|X_{i,j}\right| \left|\red{\bigcup_{\ell \neq j} X_{i,\ell}}\right| \n \\
        &< \dist(\inpset, \mop) && (\textbf{as} \, \, |X_{i,j}| > p)\n
    \end{align}

    \begin{itemize}
        \item \textbf{Reason behind the term $p\left| \red{\bigcup_{\ell \neq j}X_{i,\ell}}\right|$}: The pairs $(u, v)$ such that $u \in \red{\bigcup_{\ell \neq j} X_{i,\ell}}$ and $v \in bigcup_{\ell \neq j} X_{i,\ell}$ are not counted in $\dist(\inpset, \mop)$ because these pairs $u$ and $v$ lie in the same cluster for both $\inpset$ and $\mop$ but since these pairs are present in different clusters in $\m{M}$ they are counted in $\dist{\inpset, \m{M}}$. Now, since we know that for all $\ell \neq j$, $|X_{i,\ell}| \leq p$ hence the total cost paid for these pairs is at most $p \, \,  |\red{\bigcup_{\ell \neq j} X_{i,\ell}}|$.

        \item \textbf{Reason behind the term $\left|X_{i,j}\right| \left|\red{\bigcup_{\ell \neq j} X_{i,\ell}}\right|$}: The pairs $(u,v)$ such that $u \in X_{i,j}$ and $v \in \red{\bigcup_{\ell \neq j} X_{i,\ell}}$ are counted in $\dist(\inpset, \mop)$ because these pairs $u$ and $v$ lie in the different clusters in $\mop$ but same cluster in $\inpset$. In $\m{M}$, since these pairs lie in the same cluster, these pairs are not counted in $\dist{\inpset, \m{M}}$. Now, since in~\cref{clm:equal-to-one-cluster} we proved that $X_{i,j} = T^*_{r_j}$ there is no other cost paid by the algorithm for merging the points of $\red{\bigcup_{\ell \neq j} X_{i,\ell}}$ to $X_{i,j}$.
    \end{itemize}
    
\end{proof}

Now, we are ready to prove \cref{clm:merge-case-structure}

\begin{proof}[Proof of \cref{clm:merge-case-structure}]
    From \cref{clm:equal-to-one-cluster} we get that $X_{i,j} = T_{r_j}^*$, from \cref{clm:red-subseteq} we get $\red{\inpcl{i}} \subseteq X_{i,j}$ and from \cref{clm:at-most-one-partition-size-more-p} we get $|X_{i,\ell}| < p$ $\forall \, \, \ell \neq j$. Thus, $(i)$ is true.

    If such a partition $X_{i,j}$ does not exist, then $(ii)$ is true. Hence, either $(i)$ is true or $(ii)$ is true.

\end{proof}

We also need \cref{clm:merge-case-lp-bound} and \cref{clm:cost-three-and-cost-four} to prove \cref{lem:main-merge-case}. 

\begin{claim}\label{clm:merge-case-lp-bound}
    Suppose $y_{i,j}$ takes the value $1$ if we cut the $j$th subset from a cluster $\inpcl{i}$ in the subroutine $\algom$ otherwise $0$. Then

    \begin{align}
        \sum_{\inpcl{i} \in \cuta}\left(\sum_{j = 1}^n y_{i,j} \, \kappa_j(\inpcl{i}) + y_{i,0} \, \kappa_0(\inpcl{i}) - \mcostf{\inpcl{i}}\right) + \sum_{\inpcl{i} \in \inpset} \mcostf{\inpcl{i}} \leq \dist(\inpset, \mop). \n
    \end{align}
\end{claim}

\begin{proof}
    From~\cref{clm:merge-case-structure} we get that if in $\mop$, $D_i$ gets split into $X_{i,1}, \ldots, X_{i,t}$ that is 

    \begin{itemize}
        \item $X_{i,j} \subseteq T^*_k$ (for some $T_k^* \in \mop$) $\forall j \in [t]$.
        \item $X_{i,j_1} \cap X_{i,j_2} = \emptyset$ for $j_1 \neq j_2$ and $j_1, j_2 \in [t]$.
        \item $\bigcup\limits_{j \in [t]} X_{i,j}= D_i$.
    \end{itemize}
    
    then

    \begin{enumerate}
        \item Either $\exists j \in [t]$, such that $X_{i,j} = T_{r_j}^*$ (for some $T_{r_j}^* \in \mop)$ and $\red{\inpcl{i}} \subseteq X_{i,j}$ and $|X_{i,\ell}| < p$ $\forall \, \, \ell \neq j$.
        \item or $|X_{i,\ell}| < p$ $\forall \, \, \ell \in [t]$.
    \end{enumerate}

    This implies 
    
    \[\left|  \bigcup\limits_{\ell \neq j} X_{i,\ell} \right| \mod p = s_i\]

    Hence,

    \[\left|  \bigcup\limits_{\ell \neq j} X_{i,\ell} \right| = ap + s_i  \, \, \text{for some constant $a$.} \]

    Let us partition the set $\blue{D_i}$ into subsets $W_{i,0}, W_{i,1}, \ldots , W_{i,m}$ such that $m = \left( |\blue{D_i}| - s_i \right)/p$ where $|W_{i,0}| = s_i$ and $|W_{i,z}| = p$ for all $z > 0$.

    If $\left|  \bigcup\limits_{\ell \neq j} X_{i,\ell} \right| = ap + s_i$ then we say $\mop$ has cut $(a + 1)$ subsets $W_{i,0}, W_{i,1}, \ldots, W_{i,a}$ from $D_i$.

    Now since $|X_{i, \ell}| < p$ $\forall \ell \neq j$ hence the pair $(u,v)$ such that $u \in W_{i,z_1}$ and $v \in W_{i,z_2}$ for $z_1 \neq z_2$ must be counted in $\dist(\inpset, \mop)$.

    Hence, the cost of 

    \begin{enumerate}
        \item Cutting $W_{i,0}$ from $D_i$ is : $\kappa_0(D_i) = s_i(|D_i| - s_i)$
        \item Cutting $W_{i,z}$ from $D_i$ for $z > 0$ is : $\kappa_z(D_i) = p(|D_i| - (zp + s_i))$.
    \end{enumerate}

    Recall in $\mop$, $D_i$ gets split into $X_{i,1}, \ldots, X_{i,t}$. We say a cluster $D_i \in \inpset$ also belongs to a set $\cuto$ (say) if $t > 1$; otherwise, it belongs to a set $\mergeo$ (say). Informally, $\mop$ has cut some points from the clusters $D_i \in \cuto$ and merged some points to the clusters $D_i \in \mergeo$. 
    
    $\mop$ must pay the minimum cost for cutting from the clusters that belong to $\cuto$ and merging to the clusters in $\mergeo$. More specifically,

    \begin{align}
        \dist(\inpset, \mop) \geq \text{min} &\sum_{D_i \in \cuto} \sum_{z = 0}^m x_{i,z} \kappa_z (D_i) + \sum_{D_i \in \mergeo} \mu(D_i) \n \\
        &\text{subject to} \n \\
        &x_{i,z} \leq x_{i,z'} \, \, \text{for $z > z'$} \label{equn:constraint-one}\\
        &\sum_{D_i \in \inpset} \sum_{z = 0}^m x_{i,z} = \frac{W_g}{p} \, \, \text{where $W_g = \sum\limits_{D_i \in \merge}\defi{D_i}$} \label{equn:constraint-two}\\
        &x_{i,z} \in \{ 0, 1\} \label{equn:constraint-three}
    \end{align}

    Here, $m = (|\blue{D_i}| - s_i)/p$ and $x_{i,z}$ takes the value $1$ if $\mop$ cuts the subset $W_{i,z}$ from a cluster $D_i$ otherwise $0$. Constraint \cref{equn:constraint-one} denotes that if $\mop$ cuts a subset $W_{i,z}$ then it must cut all subsets $W_{i,z'}$ such that $z' < z$. Constraint \cref{equn:constraint-two} denotes that $\mop$ cuts exactly $W_g/p$ such subsets where $W_g$ is the sum of the deficits of all clusters that belong to $\merge$. Note the difference between the notations $W = \sum_{D_i \in \merge'}\defi{D_i}$ which is used in our subroutine $\algom$ and the notation $W_g = \sum_{D_i \in \merge}\defi{D_i}$. Since here, we are analyzing the algorithm $\algog$, we are using $W_g$ instead of $W$.

    Now, we prove two claims that explain the second constraint.

    \begin{claim}\label{clm:our-algorithm-cuts-w/p}
        $\algom$ cuts exactly $W_g/p$ subsets from the clusters $D_i \in \inpset$.
    \end{claim}

    \begin{proof}
        Suppose the algorithm cuts a subset $W_{i,z}$ from the cluster $D_i$ and merges it to another set of clusters $D_{m_1}, D_{m_2}, \ldots, D_{m_t}$(say).

        When $\algog$ cuts a partition $W_{i,z}$ the deficit of $D_i$ reduces by $(p - |W_{i,z}|)$ that is it gets updated to
        \begin{align}
            \defi{D_i} = \defi{D_i} - (p - |W_{i,z}|) \label{equn:main-claim-merge-case-one}
        \end{align}
        and when it merges this subset to a set of clusters $D_{m_1}, D_{m_2}, \ldots, D_{m_t}$ then the total deficit of clusters other than $D_i$ reduces by $|W_{i,z}|$ that is it gets updated to
        \begin{align}
            \sum_{D_j \neq D_i}\defi{D_j} = \sum_{D_j \neq D_i}\defi{D_j} - |W_{i,z}| \label{equn:main-claim-merge-case-two}
        \end{align}
        Adding \cref{equn:main-claim-merge-case-one} and \cref{equn:main-claim-merge-case-two}, we get that the total deficit decreases exactly by $p$ for cutting one part.
        \begin{align}
            \sum_{D_i \in \nc \cup \merge'}\defi{D_i} = \sum_{D_i \in \nc \cup \merge'}\defi{D_j} - p .
        \end{align}
         Hence, to decrease the deficit $W_g$ to zero, our algorithm cuts exactly $W/p$ subsets.
    \end{proof}

    \begin{claim}\label{clm:mop-cuts-w/p}
        $\mop$ cuts exactly $W_g/p$ subsets from the clusters $D_i \in \inpset$.
    \end{claim}

    \begin{proof}
        In the clustering, $\mop$ for each cluster $T_i \in \mop$, $\blue{T_i}$ is divisible by $p$ but in the clustering $\inpset$ (input of $\algog$) the total deficit is $W_g$. 

        If $\mop$ cuts a partition $W_{i,z}$ from $\inpcl{i}$ then deficit of $D_i$ would reduce by $(p - |W_{i,z}|)$. That is, it gets updated to 
        \begin{align}
            \defi{D_i} = \defi{D_i} - (p - |W_{i,z}|). \label{equn:main-claim-two-merge-case-one}
        \end{align}
        Now, these $|W_{i,z}|$ points can reduce the total deficit of other clusters by at most $|W_{i,z}|$. That is, it gets updated to at most
        \begin{align}
            \sum_{D_j \neq D_i}\defi{D_j} \geq \sum_{D_j \neq D_i}\defi{D_j} - |W_{i,z}| .\label{equn:main-claim-two-merge-case-two}
        \end{align}
        Adding \cref{equn:main-claim-two-merge-case-one} and \cref{equn:main-claim-two-merge-case-two}, we get that $\mop$ can decrease the total deficit at most by $p$ by cutting one part.
        \begin{align}
            \sum_{D_i \in \nc \cup \merge'}\defi{D_i} \geq \sum_{D_i \in \nc \cup \merge'}\defi{D_j} - p .
        \end{align}
         Hence, to decrease the deficit $W_g$ to zero, $\mop$ must cut at least $W/p$ subsets.

         Since, by \cref{clm:our-algorithm-cuts-w/p}, we get our algorithm can fulfill the deficit of $W_g$ only by cutting $W_g/p$ subsets. Since $\mop$ is the optimal clustering, it must cut at most $W_g/p$ subsets.

         Thus we get $\mop$ cuts exactly $W_g/p$ subsets.
    \end{proof}

    Now,

    \begin{align}
        &\text{min} \sum_{D_i \in \cuto} \sum_{z = 0}^m x_{i,z} \kappa_z (D_i) + \sum_{D_i \in \mergeo} \mu(D_i) \n \\
        \equ \, \, &\text{min} \sum_{D_i \in \cuto} \sum_{z = 0}^m x_{i,z} \kappa_z (D_i) - \sum_{D_i \in \cuto} \mu(D_i) + \sum_{D_i \in \inpset} \mu(D_i) \n \\
        \equ \, \, &\text{min} \sum_{D_i \in \cuto} \left( \sum_{z = 1}^m x_{i,z} \kappa_z (D_i) + x_{i,0} \kappa_0(D_i) - \mu(D_i) \right) + \sum_{D_i \in \inpset} \mu(D_i) \n 
    \end{align}

    Now, to prove our claim

    \begin{align}
        \sum_{\inpcl{i} \in \cuta}\left(\sum_{j = 1}^n y_{i,z} \, \kappa_j(\inpcl{i}) + y_{i,0} \, \kappa_0(\inpcl{i}) - \mcostf{\inpcl{i}}\right) + \sum_{\inpcl{i} \in \inpset} \mcostf{\inpcl{i}} \leq \dist(\inpset, \mop). \n
    \end{align}

    We need to prove that $x_{i,z} = 1$ iff $y_{i,z} = 1$. 

    \begin{claim}
        $x_{i,z} = 1$ iff $y_{i,z} = 1$.
    \end{claim}

    \begin{proof}
        Let's look at the objective function of the above ILP.

    \begin{align}
        \text{min} \sum_{D_i \in \cuto} \left( \sum_{z = 1}^m x_{i,z} \kappa_z (D_i) + x_{i,0} \kappa_0(D_i) - \mu(D_i) \right) + \sum_{D_i \in \inpset} \mu(D_i) \n
    \end{align}

    and the constraint $1$ of the ILP

    \begin{align}
        x_{i,z} \leq x_{i,z'} \, \, \text{for $z > z'$}\n
    \end{align}

    The above constraint implies that if $\mop$ cuts a cluster $D_i$, it must cut the $0$th subset $W_{i,0}$ of $D_i$. Now, to minimize the above objective function, we can assume the cost assigned to the $0$th subset of $D_i$ is $\ccostf{D_i} - \mcostf{D_i}$, (note $\kappa_0(D_i) = \ccostf{D_i}$).

    and the cost assigned to the $z$th subset of $D_i$ is $\kappa_z(D_i)$.

    More specifically,

    \begin{align}
        &\textit{cost}(W_{i,z}) = \ccostf{D_i} - \mcostf{D_i} \, \, \text{for $z = 0$} \n \\
        &\textit{cost}(W_{i,z}) = \kappa_z(D_i) \, \, \text{for $z \geq 1$}
    \end{align}

    Now, minimizing the above objective function is equivalent to minimizing the total cost of subsets that $\mop$ cuts. 
    
    Since, in \cref{clm:mop-cuts-w/p} we get $\mop$ cuts exactly $W_g/p$ subsets so to solve the above ILP exactly our algorithm needs to cut $W_g/p$ subsets that have the minimum cost. Note that for all clusters $D_i \in \cut$, we have 
    
    \begin{align}
    \textit{cost}(W_{i,0}) &= (\ccostf{D_i} - \mcostf{D_i}) \n \\
    &\leq (\ccostf{D_j} - \mcostf{D_j}) \, \, \forall D_j \in \merge
    \end{align}

    This is because 

    \begin{align}
        &\textit{cost}(W_{j,0}) - \textit{cost}(W_{i,0}) \n \\
        &= s_j(|D_j| - s_j) - (p - s_j) |D_j| - s_i(|D_i| - s_i) + (p - s_i)|D_i| \n \\
        &= (2s_j - p) |D_j| - s_j^2 + (p - 2s_i)|D_i| + s_i^2 \label{equn:expression-claim-one} \\
        &\geq 0 \n
    \end{align}

    Since $s_j = (\blue{D_j} \mod p )\geq p/2$ and $s_i = (\blue{D_i} \mod p) < p/2$ we get that all the terms in the \cref{equn:expression-claim-one} are positive except the term $-s_j^2$. Now, $|D_i|p > s_j^2$ because $p > s_j$, and we can assume $|D_i| \geq p$ because if $D_i$ is less than $p$ then $\ccostf{D_i} = 0$ and we can ignore those clusters because we cut no points from a cluster $D_i$ where $|D_i| < p$.
    
    Hence, $\mop$ must cut the subset $W_{i,0}$ from all clusters $D_i \in \cut$.

    Now, in the subroutine $\algom$ at each iteration, we select the subsets $W_{i,z}$ for which $\textit{cost}(W_{i,z})$ is the minimum. From \cref{clm:our-algorithm-cuts-w/p} we get $\algog$ also cuts $W_g/p$ subsets which is same as $\mop$.
    
    Thus, our algorithm also selects the subsets that have the minimum cost and hence $x_{i,j} = 1$ iff $y_{i,j} = 1$. 
    \end{proof}
\end{proof}

Now we prove \cref{clm:cost-three-and-cost-four} where we provide a bound on $\costthree{\out}$ and $\costfour{\out}$ paid by our algorithm $\algog$. In the proof of \cref{clm:cost-three-and-cost-four}, we are going to use the following notations.

Suppose a subset $W_{i,z}$ of $D_i$ gets split into $W^{(1)}_{i,z}, \ldots, W^{(t)}_{i,z}$ and gets merged into clusters $D_{m_1}, \ldots, D_{m_t} \in \merge$. Let $T_{j_1}, \ldots, T_{j_t} \in \out$ be the clusters formed from $D_{m_1}, \ldots, D_{m_t}$ respectively by either merging deficit number of blue points to these clusters.

So, the pairs $(u,v)$ such that $u \in W^{(k)}_{i,z}$ and $v \in T_{j_{k}} \setminus (D_{m_{k}} \cup W^{(k)}_{i,z})$ must be counted in $\dist(\inpset, \out)$. Let us denote the number of such pairs by $\costthree{W_{i,z}}$.

Again, the pairs $u \in W^{(k_1)}_{i,z}$ and $v \in W^{(k_2)}_{i,z}$ must be counted in $\dist(\inpset, \out)$. Let us denote the number of such pairs by $\costfour{W_{i,z}}$.

Now, let us assume 

\begin{align*}
    \costthree{\out} = \sum_{D_i \in \inpset} \sum_{z = 0}^{m} \costthree{W_{i,z}}.
\end{align*}

\begin{align*}
    \costfour{\out} = \sum_{D_i \in \inpset}\sum_{z = 0}^m \costfour{W_{i,z}}
\end{align*}

where $m = (|\blue{D_i}| - s_i)/p$. Now, we prove the following claim. To prove the following claim, we need to recall the following definitions

\begin{itemize}
    \item $\mergea$: The set of clusters $D_m \in \inpset$ where the algorithm $\algog$ has merged the deficit amount of points to it.
    \item $\cuta$: The set of clusters $D_i \in \inpset$ where the algorithm $\algog$ has cut the surplus amount of points from it.
\end{itemize}
    
\begin{claim}\label{clm:cost-three-and-cost-four}
    $\costthree{\out} + \costfour{\out} \leq 2 \, \, \dist(\inpset, \mop)$.
\end{claim}

\begin{proof}
    Case $1$: \textbf{When} $|W_{i,z}| \leq p/2$: In this case, we prove that

    \begin{align*}
        \costthree{W_{i,z}} + \costfour{W_{i,z}} \leq W^{(1)}_{i,z} |D_{m_1}| + \mu(D_{m_2}) + \ldots + \mu(D_{m_{t - 1}}) + W^{(t)}_{i,z} |D_{m_t}|
    \end{align*}

    Thus adding over all clusters $D_i \in \inpset$ and all subsets $W_{i,z} \in D_i$ we get that 

    \begin{align*}
        \sum_{D_i \in \inpset} \sum_{z = 0}^{m} \costthree{W_{i,z}} + \sum_{D_i \in \inpset}\sum_{z = 0}^m \costfour{W_{i,z}} &= \costthree{\out} + \costfour{\out} \n \\ 
        &\leq \sum_{D_m \in \mergea} \mu(D_m) \n \\
        &(\text{Reason provided below)}) \n \\
        &< 2 \, \sum_{D_m \in \mergea}s(D_m) (|D_m| - s(D_m)) + \frac{1}{2} s(D_m)(p - s(D_m)) \n \\
        &< 2 \, \dist(\inpset, \mop) \n \\
        & (\text{from} \, \, \cref{lem:main-structure-of-M*})
    \end{align*}

    Reason: This is because for the clusters $D_{m_k}$ for $k \in [2, (t - 1)]$ the subset $W^{(k)}_{i,z}$ covers the deficit of the clusters $D_{m_k}$ fully that is $T_{j_k} \setminus (D_{m_k} \cup W^{(k)}_{i,z}) = \emptyset$ and for $k = 1$ and $k = t$ we have charged only the merge cost contributed by the parts $W^{(1)}_{i,z}$ and $W^{(t)}_{i,z}$ respectively.

    Thus, now we prove 

    \begin{align*}
        \costthree{W_{i,z}} + \costfour{W_{i,z}} \leq W^{(1)}_{i,z} |D_{m_1}| + \mu(D_{m_2}) + \ldots + \mu(D_{m_{t - 1}}) + W^{(t)}_{i,z} |D_{m_t}|
    \end{align*}

    Recall $\costthree{W_{i,z}}$ is the number of pairs $(u,v)$ such that $u \in W^{(k)}_{i,z}$ and $v \in T_{j_{k}} \setminus (D_{m_{k}} \cup W^{(k)}_{i,z})$. Thus 

    \begin{align}
        \costthree{W_{i,z}} &= \sum_{k = 1}^t\left|W^{(k)}_{i,z}\right| \left|T_{j_{k_1}} \setminus (D_{m_{k_1}} \cup W^{(k)}_{i,z})\right|  \n \\
        &< \sum_{k = 1}^t\dfrac{1}{2} \left( \left|W^{(k)}_{i,z}\right|  \dfrac{p}{2}  \right) && \left( \textbf{as} \, \, \left|T_{j_{k}} \setminus (D_{m_{k}} \cup W^{(k)}_{i,z})\right| < p/2 \right)\n \\
        &\leq \dfrac{1}{2} \sum_{k = 1}^t \left( \left|W^{(k)}_{i,z}\right|  \left|D_{m_k} \right|  \right) \label{equn:equation-one}
    \end{align}

    The $1/2$ in the above expression comes from the fact that when we sum over all the clusters $D_i \in \inpset$ and all subsets $W_{i,z} \in D_i$ we are counting the pairs twice once while considering the subset $W^{(k)}_{i,z}$ and again while considering the subset $T_{j_{k}} \setminus (D_{m_{k}} \cup W^{(k)}_{i,z})$. 

    Recall $\costfour{W_{i,z}}$ is the number of pairs $(u,v)$ such that $u \in W^{(k_1)}_{i,z}$ and $v \in W^{(k_2)}_{i,z}$. Thus

    \begin{align}
        \costfour{W_{i,z}} &= \sum_{\substack{k_1, k_2 = 1 \\ k_1 \neq k_2}}^t\left|W^{(k_1)}_{i,z}\right| \left|W^{(k_2)}_{i,z}\right|  \n \\
        &< \sum_{k = 1}^t \dfrac{1}{2} \left( \left|W^{(k)}_{i,z}\right|  \dfrac{p}{2}  \right) && \left( \textbf{as} \, \, \left|W^{(k_2)}_{i,z}\right| < p/2 \right)\n \\
        &\leq \dfrac{1}{2} \sum_{k = 1}^t \left( \left|W^{(k)}_{i,z}\right|  \left|D_{m_k} \right|  \right) \label{equn:equation-two}
    \end{align}

    The $1/2$ in the above expression comes from the fact that when we sum over all the clusters $D_i \in \inpset$ and all subsets $W_{i,z} \in D_i$ we are counting the pairs twice once while considering the subset $W^{(k_1)}_{i,z}$ and again while considering the subset $W^{(k_2)}_{i,z}$. 

    Thus from \cref{equn:equation-one} and \cref{equn:equation-two} we get that,

    \begin{align}
       \costthree{W_{i,z}} + \costfour{W_{i,z}} &\leq \sum_{k = 1}^t \left( \left|W^{(k)}_{i,z}\right|  \left|D_{m_k} \right|  \right) \n \\
       &= W^{(1)}_{i,z} |D_{m_1}| + \mu(D_{m_2}) + \ldots + \mu(D_{m_{t - 1}}) + W^{(t)}_{i,z} |D_{m_t}| \n
    \end{align}

The reason behind the last expression is as stated before, for the clusters $D_{m_k}$ for $k \in [2, (t - 1)]$ the parts $W^{(k)}_{i,z}$ covers the deficit of the clusters $D_{m_k}$ entirely, that is $T_{j_k} \setminus \left(D_{m_k} \cup W^{(k)}_{i,z}\right) = \emptyset$.

Case $2$: \textbf{When} $|W_{i,z}| > p/2$ : For a cluster $D_m \in \mergea$ the deficit of the cluster $D_m$, $\defi{D_m} < p/2$. This implies the deficit of $D_m$ is either filled by exactly one subset $W_{i,z}$ or by two subsets $W_{i,z} \in D_i$ and $W_{j,z'} \in D_j$.

We divide our proof again into two subcases:

Case $2(a)$: \textbf{When} the deficit of $D_m$ is filled by a subset $W_{i,z} \in D_i$: Let us assume a cluster $\outcl{j} \in \out$ is created from the cluster $\inpcl{m} \in \inpset$ by merging the deficit $\defi{\inpcl{m}} = d_m$ (say) to it. Let, a part of $W_{i,z}$, $W^{(1)}_{i,z}$ (say) be used to fill $d_m$. Hence, $\costthree{W^{(1)}_{i,z}} = 0$ because the deficit is filled up by a subset that belongs to only one cluster $\inpcl{i} \in \inpset$. 
    
Again, $\costfour{W_{i,z}} = (|W^{(1)}_{i,z}|  \, \, (|W_{i,z}| - |W^{(1)}_{i,z}|)\, )$ because when we merge $W^{(1)}_{i,z}$ part of $W_{i,z}$ to $\inpcl{m}$ and the remaining part $(W_{i,z} - W^{(1)}_{i,z})$ is merged to other clusters in $\mergea$. Hence, 
    
    \begin{align}
        \costthree{W^{(1)}_{i,z}} + \costfour{W_{i,z}} &= |W^{(1)}_{i,z}| \, \, (|W_{i,z}| - |W^{(1)}_{i,z}|) \nonumber \\
        &\leq (d_m \cdot (p - d_m)) && (\because |W^{(1)}_{i,z}| = d_m \, \, \text{and} \, \, |W_{i,z}| \leq p) \nonumber \label{equn:algo-pays-case-1} \\
    \end{align}

    From \cref{lem:main-structure-of-M*} we get that for a cluster $D_m \in \mergea$
    
     
     \begin{align}
        \opt_{D_m} &\geq \surp{D_m}(|D_m| - \surp{D_m}) + \frac{1}{2} \surp{D_m}(p - \surp{D_m})\n \\
        &> \frac{1}{2} d_m \, \, (p - d_m) &&(\because |\inpcl{m}| > (p - d_m)) \label{equn:mop-pays}
     \end{align}

     Hence, by \cref{equn:algo-pays-case-1} and \cref{equn:mop-pays} we get, that we can charge the $2\,\,\opt_{D_m}$ for $\costthree{W^{(1)}_{i,z}} + \costfour{W_{i,z}}$.

Case $2(b)$: \textbf{When} deficit of the cluster $\defi{\inpcl{m}} = d_m$ is filled by the subsets of exactly two clusters $W_{i,z} \in D_i$ and $W_{j,z} \in D_j$: Let us assume a cluster $\outcl{j} \in \out$ is created from the cluster $\inpcl{m}$ by merging the deficit $\defi{\inpcl{m}} = d_m$ (say) to it. Let, a part of $W_{i,z}$, $W^{(2)}_{i,z}$ is used to fill the deficit of $\inpcl{m}$. Similarly, let part of $W_{j,z}$, $W^{(1)}_{j,z}$ is used to fill the deficit of $\inpcl{m}$. Hence,

\[
    \costthree{W^{(2)}_{i,z}} + \costthree{W^{(1)}_{j,z}} = |W^{(2)}_{i,z}| \, \, |W^{(1)}_{j,z}|
\]

and 

\[
    \costfour{W_{j,z}} = |W^{(1)}_{j,z}| \, \, (|W_{j,z}| - |W^{(1)}_{j,z}|)
\] 

( note we have already charged the cost $\costfour{W_{i,z}}$ in case $1$ ). 

Hence, 

\[
    \costthree{W^{(2)}_{i,z}} + \costthree{W^{(1)}_{j,z}} + \costfour{W_{j,z}} = (|W^{(1)}_{i,z}| \, \, |W^{(1)}_{j,z}|) + (|W^{(1)}_{j,z}| \, \, (|W_{j,z}| - |W^{(1)}_{j,z}|)).  
\] 

Now, WLOG we can assume $(|W_{j,z}| - |W^{(1)}_{j,z}|) < p/2$
because otherwise there exists a portion $W^{(1')}_{j,z}$ of $(W_{j,z} \setminus W^{(1)}_{j,z})$ that is used to fill the deficit of some cluster fully in $\mergea$ and thus we have charged $\costfour{w_{j,z}}$ in case $(1)$.

    Hence,

    \begin{align}
        \costthree{W^{(2)}_{i,z}} + \costthree{W^{(1)}_{j,z}} + \costfour{W_{j,z}} &= (|W^{(2)}_{i,z}| \, \, |W^{(1)}_{j,z}|) + (|W^{(1)}_{j,z}| \, \, (|W_{j,z}| - |W^{(1)}_{j,z}|)) \nonumber \\
        &\leq |W^{(2)}_{i,z}| \, \, d_m + |W^{(1)}_{j,z}| \, \,  p/2 \n \\ 
        &(\textbf{as} \, \, |W^{(1)}_{j,z}|  < d_m \, \text{and} \,  (|W_{j,z}| - |W^{(1)}_{j,z}|) < p/2) \nonumber \\
        &\leq (|W^{(2)}_{i,z}| + |W^{(1)}_{j,z}|) \, \, (p - d_m) \n \\ 
        &(\textbf{as} \, \, d_m \, \, \text{and} \, \, p/2 \leq (p - d_m)) \nonumber \\
        &= d_m \, \, (p - d_m) \label{equn:algo-pays-case-2} 
    \end{align}

     Hence, by \cref{equn:algo-pays-case-2}, \cref{equn:mop-pays} and \cref{lem:main-structure-of-M*} we get, that we can charge the $2 \, \, \opt_{D_m}$ for $\costthree{W^{(2)}_{i,z}} + \costthree{W^{(1)}_{j,z}} + \costfour{W_{j,z}}$.

     Since,

     \begin{align}
         d_m(p - d_m) &= (p - s(D_m)) s(D_m) \n \\
         &< 2 \, (s(D_m) (|D_m| - s(D_m)) + \frac{1}{2} s(D_m)(p - s(D_m))) \n \\
         &< 2 \opt_{\inpcl{m}} \n
     \end{align}

    Thus, combining the above two cases, we get,

    \begin{align*}
        \sum_{D_i \in \inpset} \sum_{z = 0}^{m} \costthree{W_{i,z}} + \sum_{D_i \in \inpset}\sum_{z = 0}^m \costfour{W_{i,z}} &= \costthree{\out} + \costfour{\out} \n \\ 
        &\leq 2 \, \, \sum_{D_m \in \mergea} \opt_{D_m} \n \\
        &\leq 2 \, \dist(\inpset, \mop) \n \\
        & (\text{from} \, \, \cref{lem:main-structure-of-M*})
    \end{align*}

\end{proof}

Now we are ready to prove \cref{lem:main-merge-case}

\begin{proof}[Proof of \cref{lem:main-merge-case}]
   Suppose $y_{i,j}$ takes the value $1$ if we cut the $j$th subset of size $p$ from a cluster $\inpcl{i}$ in the subroutine $\algom$ otherwise $0$. Hence,

   \begin{align}
       \dist(\m{D},\m{T}) &= \sum_{\inpcl{i} \in \cuta} \sum_{j = 0}^n y_{i,j} \kappa_i(\inpcl{j}) + \sum_{\inpcl{j} \in \mergea} \mcostf{\inpcl{j}} + \costthree{\out} + \costfour{\out}. \n \\
       &= \sum_{\inpcl{i} \in \cuta}\left(\sum_{j = 1}^n y_{i,j} \, \kappa_j(\inpcl{i}) + y_{i,0} \, \kappa_0(\inpcl{i}) - \mcostf{\inpcl{i}}\right) + \sum_{\inpcl{i} \in \inpset} \mcostf{\inpcl{i}} + \costthree{\out} + \costfour{\out}. \n \\
       &\leq \dist(\inpset, \mop) + 2 \, \dist(\inpset, \mop). \n \\
       &= 3\, \, \dist(\inpset, \mop). \n
   \end{align}
   
\end{proof}

\subsubsection{Approximation guarantee in the cut case} \label{subsec:cut-case}

 In this section, we show that, in the cut case, the clustering output by our algorithm $\algog$ has a distance of at most $3.5\opt$ from the input clustering. In particular, we prove the following lemma that we restate below.
\maincut*

To prove \cref{lem:main-cut-case}, we need to define and recall some definitions.


\begin{align}
    \costone{\out}: \sum_{D_i \in \inpset} s_i(|D_i| - s_i) \label{equn:cost-one} 
\end{align}

It is the total cost of cutting the surplus blue points from a cluster $D_i \in \inpset$.

\begin{align}
    \costtwo{\out}: \sum_{D_j \in \inpset}(p - s_j) |D_j| \label{equn:cost-two}
\end{align}

It is the total cost of merging the $\defi{D_j}$ number of blue points to a cluster $D_j \in \inpset$.

Now, let's recall the following definitions


where for a cluster $\outcl{k} \in \out$ we defined $\pi(\outcl{k})$ as a parent of $\outcl{k}$ iff $\pi(\outcl{k}) \in \inpset$ and $\outcl{k}$ is formed either by cutting some points from $\pi(\outcl{k})$ or by merging some points to $\pi(\outcl{k})$.

Thus, our algorithm pays the above four types of costs. 

To prove \cref{lem:main-cut-case}, first, we need to prove some claims.

 We first argue (in~\cref{lem.opt.si.times.p.min.si}) that $\mop$ must pay $\sum_{\inpcl{i} \in \inpset} \frac{s_i (p - s_i)}{2}$ in the cut case where $s_i = \surp{\inpcl{i}}$.

\begin{claim}\label{lem.opt.si.times.p.min.si}
    \begin{align}
        \dist(\inpset, \mop) \geq \sum_{\inpcl{i} \in \inpset}\dfrac{s_{i}(p-s_{i})}{2} \nonumber
    \end{align}
\end{claim}

\begin{proof}
    Suppose there is an input cluster $\inpcl{k} \in \inpset$ that is split into $t$ parts (say) $X_{k,1},\ldots,X_{k,t}$.

    We can assume that the cluster $\inpcl{k}$ consists only of blue points. If it contains red points, the cost paid by $\mop$ for splitting $\inpcl{k}$ into $t$ parts would be more because the red points would increase the size of the cluster $\inpcl{k}$. Let us assume that the surplus of $X_{k,i}$ be $s_{k,j}$ and to make the number of blue points in each cluster a multiple of $p$ $\mop$ must merge at least $(p - s_{k,j})$ blue points to $X_{k,j}$. Hence,

    \begin{align}
        \opt_{D_k} \geq \frac{1}{2} \left(\sum_{j = 1} ^ t s_{k,j} (p - s_{k,j}) \right) \label{equn:one}
    \end{align}

    Now, since $\sum_{j = 1} ^ t s_{k,j} \mod p = s_k$ thus from \cref{prop:mod-sum-ineq} we get,

    \begin{align}
        \frac{1}{2} \left( \sum_{j = 1} ^ t s_{k,j} (p - s_{k,j}) \right) \geq \frac{s_k (p - s_k)}{2} \label{equn:two}
    \end{align}

    Now,

    \begin{align}
        \dist(\inpset, \mop) &= \sum_{D_i \in \inpset} \opt_{D_i} \n \\
        &\geq \sum_{D_i \in \inpset} \frac{s_i (p - s_i)}{2} && (\text{from} \, \cref{equn:one} \,  \text{and} \,  \cref{equn:two}) \n
    \end{align}

    This completes the proof.
\end{proof}

The following lemma is a consequence of~\cref{lem.opt.si.times.p.min.si}
\begin{claim}\label{lem.opt.si.times.p.min.si.breakdown}
    \begin{align}
        \opt \geq \sum_{\inpcl{i} \in \cut\setminus\cut'}\dfrac{1}{2}s_{i}^{2} + \sum_{\inpcl{i} \in \merge}\dfrac{1}{2}d_{i}^{2} + \sum_{\inpcl{i} \in \cut'}\dfrac{s_{i}(p-s_{i})}{2} \nonumber 
    \end{align}
Here, $s_i = \surp{\inpcl{i}}$ and $d_i = \defi{\inpcl{i}}$.
\end{claim}
\begin{proof}
    We have
    \begin{align}
        \sum_{\inpcl{i} \in \inpset}\dfrac{s_{i}(p-s_{i})}{2}  = \sum_{\inpcl{i} \in \cut\setminus\cut'}\dfrac{s_{i}(p-s_{i})}{2}  + \sum_{\inpcl{i} \in \merge}\dfrac{s_{i}(p-s_{i})}{2} + \sum_{\inpcl{i} \in \cut'}\dfrac{s_{i}(p-s_{i})}{2}.\nonumber
    \end{align}
    
    For each $\inpcl{i} \in \cut$, $s_{i} \leq \frac{p}{2}$. Hence $s_{i}(p-s_{i}) \geq s_{i}\frac{p}{2} \geq s_{i}^{2}$.

    For each $\inpcl{i} \in \merge$, $s_{i} \geq \frac{p}{2}$. Hence $s_{i}(p-s_{i}) \geq \frac{p}{2} (p-s_{i}) \geq d_{i}^{2}$  (where $d_i = (p - s_i)$).

    These observations imply that
    \begin{align}
        \sum_{\inpcl{i} \in \cut\setminus\cut'}\dfrac{1}{2}s_{i}^{2} + \sum_{\inpcl{i} \in \merge}\dfrac{1}{2}d_{i}^{2} + \sum_{\inpcl{i} \in \cut'}\dfrac{s_{i}(p-s_{i})}{2} \leq \sum_{\inpcl{i} \in \inpset}\dfrac{s_{i}(p-s_{i})}{2}  \leq \opt,\nonumber
    \end{align}
    where the last inequality is from~\cref{lem.opt.si.times.p.min.si}.

    This completes the proof.
\end{proof}

\begin{claim}\label{lem:cost-one-plus-two-cut-case}
    If $\sum_{\inpcl{i} \in \cut} \surp{\inpcl{i}} > \sum_{\inpcl{j} \in \merge} \defi{\inpcl{j}}$ then $\costone{\out} + \costtwo{\out} \leq 2 \, \opt$.
\end{claim}

\begin{proof}
   By the definition of $\costone{\out}$, we have $\costone{\out} = \sum_{\inpcl{i} \in \cuta}\ccostf{\inpcl{i}}$. Now, if $\sum_{\inpcl{i} \in \cut} \surp{\inpcl{i}} > \sum_{\inpcl{j} \in \merge} \defi{\inpcl{j}}$ that is in the cut case we have $\cuta = \cut$. Recall that $\cuta$ is the set of clusters $\inpcl{i} \in \inpset$ from which we cut some points at any point during the execution of our algorithm $\algog$. Since in the cut case, we only cut from the clusters $\inpcl{i} \in \cut$ that is the clusters $\inpcl{i}$ for which $\surp{\inpcl{i}} \leq p/2$ we have $\cuta = \cut$. Hence, in the cut case, we have, 
   \begin{align}
      \costone{\out} = \sum_{\inpcl{i} \in \cuta}\ccostf{\inpcl{i}} = \sum_{\inpcl{i} \in \cut}\ccostf{\inpcl{i}}. \label{equn:cut-case-first}
   \end{align}
    
    Similarly, in the cut case, we also have $\mergea = \merge$ thus we have

    \begin{align}
    \costtwo{\out} = \sum_{\inpcl{i} \in \mergea}\mcostf{\inpcl{i}} = \sum_{\inpcl{i} \in \merge}\mcostf{\inpcl{i}}. \label{equn:cut-case-second}
    \end{align}
    Hence, from \cref{equn:cut-case-first} and \cref{equn:cut-case-second} we have

    \begin{align}
        \costone{\out} + \costtwo{\out} = \sum_{\inpcl{i} \in \cut}\ccostf{\inpcl{i}} + \sum_{\inpcl{i} \in \merge}\mcostf{\inpcl{i}}. \label{equn:cut-case-third}
    \end{align}

    Now, from \cref{lem:main-structure-of-M*} we get for all clusters $\inpcl{i} \in \cut$ we have 
    
    \begin{align}
        \opt_{\inpcl{i}} &\geq s(D_i)(|D_i| - s(D_i)) + \frac{1}{2} s(D_i) (p - s(D_i)) \n \\
        &> \ccostf{D_i} \n
    \end{align} 
    and for all clusters $\inpcl{i} \in \merge$ we have 
    
    \begin{align}
        \opt_{\inpcl{i}} &\geq (p - s(D_i)) (|D_i| - s(D_i)) + \frac{1}{2} s(D_i) (p - s(D_i)) \n \\
        &\geq 2 \, \mu(D_i) \n
    \end{align} 
    Thus we have

    \begin{align}
        \dist(\inpset, \mop) &= \sum_{\inpcl{i} \in \cut} \opt_{\inpcl{i}} + \sum_{\inpcl{i} \in \merge} \opt_{\inpcl{i}} \nonumber \\
        &\geq \sum_{\inpcl{i} \in \cut} \ccostf{\inpcl{i}} + 2 \sum_{\inpcl{i} \in \merge} \mcostf{\inpcl{i}} .\label{equn:cut-case-fourth}
    \end{align}

    Now, from \cref{equn:cut-case-fourth} and \cref{equn:cut-case-third} we get

    \begin{align}
        \costone{\out} + \costtwo{\out} &\leq 2 \, \, \dist(\inpset, \mop) \n \\
        &\leq 2 \, \, \opt.\n
    \end{align}
\end{proof}

\begin{claim} \label{lem:cost 3 + 4 cut case}
    $\costthree{\out} + \costfour{\out} \leq 1.5 \, \opt$.
\end{claim}
\begin{proof}

    We introduce some notations.

    For each cluster $ \inpcl{i}\in \cut $, denote by $ X_{i,{1}}, X_{i,{2}}, \dots, X_{i,{\ell_{i}}} $ the partition of $ s_{i} $ blue points that are split from $ \inpcl{i} $, and each of these parts are merged to different clusters $ \inpcl{j} $'s in $ \merge $, for $ k=1,2,\dots, \ell_{j} $.

     Let $\pcls$ be the set of clusters of size $p$ formed by $\algoc$. Let $\pcl{j}$ be its $j$th cluster. Represents $ \pcl{j} = Y_{j,1} \cup Y_{j,2} \cup \dots \cup Y_{j,y_{j}} $ where each $ Y_{j,k} $ is a set of blue points coming from a cluster $ \inpcl{j_{k}}\in \cut $.

    We prove this claim by showing the correctness of the following claims.
    \begin{subclaim}\label{clm:cost 4}
        $ \costfour{\out} \leq \sum_{\inpcl{i} \in \cut\setminus\cut'}\dfrac{1}{2}s_{i}^{2} + \sum_{\inpcl{i} \in \cut'}{\left( \sum_{1\leq j< k\leq \ell_{i}} |X_{i,j}||X_{i,k}| \right)}. $
    \end{subclaim}
    \begin{subclaim}\label{clm:cost 3} 
    $ \costthree{\out} \leq \sum_{\inpcl{i}\in \merge}\dfrac{1}{2}d_{i}^{2} + \dfrac{1}{2}\sum_{\pcl{j}\in \pcls}\left(\sum_{k=1}^{y_{j}}|Y_{j,k}|(p-|Y_{j,k}|)\right). $
    \end{subclaim}
    \begin{subclaim}\label{clm:three half cost}
        $$ \sum_{\inpcl{i} \in \cut'}{\left( \sum_{1\leq j< k\leq \ell_{i}} |X_{i,j}||X_{i,k}| \right)} +  \dfrac{1}{2}\sum_{\pcl{j}\in \pcls}\left(\sum_{k=1}^{y_{j}}|Y_{j,k}|(p-|Y_{j,k}|)\right) \leq \dfrac{3}{2}\sum_{\inpcl{i}\in \cut'}\dfrac{s_{i}(p-s_{i})}{2}. $$
    \end{subclaim}
    Given the correctness of the claims above, the lemma can be proved as follows.
    {\small
    \begin{align}
        &\> \costthree{\out} + \costfour{\out} \nonumber \\
        &\leq \sum_{\inpcl{i} \in \cut\setminus\cut'}\dfrac{1}{2}s_{i}^{2} + \sum_{\inpcl{i} \in \merge}\dfrac{1}{2}d_{i}^{2} \nonumber\\
        &\> + \sum_{\inpcl{i} \in \cut'}\left( \sum_{1\leq j<k \leq \ell_{i}} |X_{i,j}||X_{i,k}| \right) + \dfrac{1}{2}\sum_{\pcl{j}\in \pcls}\left(\sum_{k=1}^{y_{j}}|Y_{j,k}|(p-|Y_{j,k}|)\right) && (\text{by~\cref{clm:cost 4} and~\cref{clm:cost 3}}) \nonumber \\
        &\leq \sum_{\inpcl{i} \in \cut\setminus\cut'}\dfrac{1}{2}s_{i}^{2} + \sum_{\inpcl{i} \in \merge}\dfrac{1}{2}d_{i}^{2} + \dfrac{3}{2}\sum_{\inpcl{i}\in \cut'}\dfrac{s_{i}(p-s_{i})}{2} && (\text{by~\cref{clm:three half cost}})\nonumber \\
        &\leq \dfrac{3}{2} \left(\sum_{\inpcl{i} \in \cut\setminus\cut'}\dfrac{1}{2}s_{i}^{2} + \sum_{\inpcl{i} \in \merge}\dfrac{1}{2}d_{i}^{2} + \sum_{\inpcl{i}\in \cut'}\dfrac{s_{i}(p-s_{i})}{2} \right) \nonumber \\
        &\leq \dfrac{3}{2}\opt. && (\text{by~\cref{lem.opt.si.times.p.min.si.breakdown}}) \nonumber
    \end{align}
}
    We now proceed to the proof of each of the above claims to finish the proof:

    \begin{proof}[Proof for~\cref{clm:cost 4}]
        By definition,
        \begin{align}
            \costfour{\out}  &= \sum_{\inpcl{i} \in \cut}{\left( \sum_{1\leq j< k\leq \ell_{i}} |X_{i,j,}||X_{i,k}| \right)} \nonumber\\
            &= \sum_{\inpcl{i} \in \cut \setminus \cut'}{\left( \sum_{1\leq j< k\leq \ell_{i}} |X_{i,j}||X_{i,k}| \right)} + \sum_{\inpcl{i} \in \cut'}{\left( \sum_{1\leq j< k\leq \ell_{i}} |X_{i,j}||X_{i,k}| \right)}\nonumber            
        \end{align}
        Note that $ \sum_{1\leq j \leq \ell_{i}}|X_{i,j}| = s_{i} $, therefore $\sum_{1\leq j<k\leq \ell_{i}} |X_{i,j}| |X_{i,k}| \leq \frac{1}{2}\left( \sum_{1\leq j\leq \ell_{i}} |X_{i,j}| \right)^{2} = \frac{1}{2}s_{i}^{2}$, it follows that
        \begin{align}
            \costfour{\out} \leq \sum_{\inpcl{i} \in \cut\setminus\cut'}\dfrac{1}{2}s_{i}^{2} + \sum_{\inpcl{i} \in \cut'}{\left( \sum_{1\leq j< k\leq \ell_{i}} |X_{i,j}| |X_{i,k}| \right)}. \nonumber 
        \end{align}
    \end{proof}

    \begin{proof}[Proof for~\cref{clm:cost 3}]
        For each cluster $\inpcl{i} \in \merge$, $\algog$ brings $d_{i}$ blue points from some other clusters $\inpcl{\iprime} \in \cut$ into $\inpcl{i}$. Denote by $ Y_{i,1}, Y_{i,2}, \dots, Y_{i,y_{i}} $ the partition of these $ d_{i} $ blue points such that in the following order, $ \algog $ merges into $ \inpcl{i} $: $ Y_{i,1} $ from $ \inpcl{i_{1}} $, $ Y_{i,2} $ from $ \inpcl{i_{2}} $, and so on, $ Y_{i,y_{i}} $ from $ \inpcl{i_{y_{i}}} $. This implies
        \begin{align}
        d_{i} = |Y_{i,1}| + |Y_{i,2}| + \dots + |Y_{i,y_{i}}| \label{eq.defi of Ci},
        \end{align}
        and the cost contributed to $\costthree{\out}$ by this cluster $\inpcl{i}$ is
        \begin{align}
            \dfrac{1}{2} \left( \sum_{j=1}^{y_{i}} |Y_{i,j}|(d_{i} - |Y_{i,j}|) \right) &\leq \dfrac{1}{2y_{i}}d_{i} (y_{i} - 1)d_{i}\label{eq.cost 4.chebyshev ineq}\\
            &\leq \dfrac{1}{2}d_{i}^{2}. \label{eq.cost 4. bound}
        \end{align}

        To obtain~\eqref{eq.cost 4.chebyshev ineq}, we use Chebyshev inequality ($\frac{a_{1}b_{1}+a_{2}b_{2}+\dots a_{n}b_{n}}{n} \leq \frac{a_{1}+a_{2}+\dots a_{n}}{n}\frac{b_{1}+b_{2}+\dots b_{n}}{n}$, for $a_{1}\leq a_{2}\leq \dots \leq a_{n}, b_{1}\geq b_{2}\geq \dots \geq b_{n}$), combining with~\eqref{eq.defi of Ci}.

        Similarly, for each cluster of size $ p $, $ \pcl{j} \in \pcls $, we can assume $\pcl{j}$ is formed by $ Y_{j,1} $ from $ \inpcl{j_{1}} $, $ Y_{j,2} $ from $ \inpcl{j_{2}} $, and so on, $ Y_{j,{y_{j}}} $ from $ \inpcl{j_{y_{j}}} $. Therefore, the cost contributed to $\costthree{\out}$ by $P_{j}$ is 
        \begin{align}
            \dfrac{1}{2}\sum_{k=1}^{y_{j}}|Y_{j,k}|(p-|Y_{j,k}|) \label{eq.cost 3 contributed by Pj}
        \end{align}
        From~\eqref{eq.cost 4. bound} and~\eqref{eq.cost 3 contributed by Pj}, we have 
    $ \costthree{\out} \leq \sum_{\inpcl{i}\in \merge}\dfrac{1}{2}d_{i}^{2} + \dfrac{1}{2}\sum_{\pcl{j}\in \pcls}\left(\sum_{k=1}^{y_{j}}|Y_{j,k}|(p-|Y_{j,k}|)\right). $
    \end{proof}

    \begin{proof}[Proof for~\cref{clm:three half cost}]
        Let $$ A = \sum_{\inpcl{i} \in \cut'}{\left( \sum_{1\leq j< k\leq \ell_{i}} |X_{i,j}||X_{i,k}| \right)} +  \dfrac{1}{2}\sum_{\pcl{j}\in \pcls}\left(\sum_{k=1}^{y_{j}}|Y_{j,k}|(p-|Y_{j,k}|)\right),$$
        and $ A' = \sum_{\inpcl{i}\in \cut'}\dfrac{s_{i}(p-s_{i})}{2} $. We need to show that $ A \leq \dfrac{3}{2}A' $.
        
        In each cluster $\inpcl{i} \in \cut'$, $ \algoc $ splits $s_{i}$ blue points from it, and these $s_{i}$ blue points are split again into at most two parts. Therefore, we can write $ \sum_{1\leq j< k\leq \ell_{i}} |X_{i,j}||X_{i,k}|  = |X_{i,1}|(s_{i}-|X_{i,1}|) = \alpha_{i,1}(s_{i} - \alpha_{i,1}) $, where $ 0<\alpha_{i,1} = |X_{i,1}| \leq s_{i} $.

        Each cluster $ \pcl{j}\in \pcls $ is formed by merging $ Y_{j,1} $ from $ \inpcl{j_{1}} \in \cut' $, $ Y_{j,2} $ from $ \inpcl{j_{2}} \in \cut' $, and so on, $ Y_{j,{y_{j}}} $ from $ \inpcl{j_{y_{j}}} \in \cut' $, in that order. Therefore, $ |Y_{j,k}| = s_{j_{k}} $ for $ 2\leq k< y_{j} $, and $ |Y_{j,1}| = \alpha_{j_{1}} \leq s_{j_{1}} $, $ |Y_{j,y_{j}}| = \beta_{j_{y_{j}}} \leq s_{j_{y_{j}}} $.

        $ A $ can be rewritten as
        \begin{align}
            A &= \sum_{\inpcl{i} \in \cut'}\alpha_{i,1}(s_{i}-\alpha_{i,1}) \nonumber \\
              &\> + \dfrac{1}{2}\sum_{\pcl{j} \in \pcls}\left( \alpha_{j_{1}}(p-\alpha_{j_{1}}) + \sum_{t=2}^{y_{j}-1}s_{j'_{t}}(p-s_{j'_{t}}) + \beta_{j_{y_{j}}}(p-\beta_{j_{y_{j}}}) \right) \nonumber \\
              &=\sum_{P_{j} \in \pcls} \left( \dfrac{1}{2}   \left( \alpha_{j_{1}}(p-\alpha_{j_{1}}) + \sum_{t=2}^{y_{j}-1}s_{j'_{t}}(p-s_{j'_{t}}) + \beta_{j_{y_{j}}}(p-\beta_{j_{y_{j}}}) \right)\nonumber + \alpha_{j_{1}}(s_{j_{1}} - \alpha_{j_{1}}) \right) \nonumber \\
              &=  \sum_{P_{j} \in \pcls} \Bigl( \dfrac{1}{2}   \left( \alpha_{j_{1}}(p-s_{j_{1}}) + \sum_{t=2}^{y_{j}-1}s_{j'_{t}}(p-s_{j'_{t}}) + \beta_{j_{y_{j}}}(p-s_{j_{y_{j}}}) \right)\nonumber\\
              &\> + 3\alpha_{j_{1}}(s_{j_{1}} - \alpha_{j_{1}}) + \beta_{j_{y_{j}}}(s_{j_{y_{j}}}-\beta_{j_{y_{j}}}) \Bigl) \nonumber \\
              &= \sum_{\inpcl{j}\in \cut'} \dfrac{s_{j}(p-s_{j})}{2} + \sum_{\pcl{j}\in \pcls}\left(3\alpha_{j_{1}}(s_{j_{1}} - \alpha_{j_{1}}) + \beta_{j_{y_{j}}}(s_{j_{y_{j}}}-\beta_{j_{y_{j}}}) \right) \nonumber \\
              &= A' + \sum_{\pcl{j}\in \pcls} \dfrac{1}{2}(3\alpha_{j_{1}}(s_{j_{1}} - \alpha_{j_{1}}) + \beta_{j_{y_{j}}}(s_{j_{y_{j}}}-\beta_{j_{y_{j}}})). \nonumber
        \end{align}

        Since $s_{j} \leq \frac{p}{2}$ for all $\inpcl{j} \in \cut$, and $ p = \alpha_{j_{1}} + s_{j_{2}} + \dots + s_{j_{y_{j}}} + \beta_{j_{y_{j}}} $, it follows that
\begin{align}
    A' &\geq \sum_{\pcl{j} \in \pcls} \dfrac{1}{2} (\alpha_{j_{1}} + s_{j_{2}} + \dots + s_{j_{y_{j}}} + \beta_{j_{y_{j}}})\dfrac{p}{2} & \nonumber \\
       &= |\pcls|\dfrac{p^{2}}{4}. \nonumber 
\end{align}
Moreover, $\alpha_{j_{1}}(s_{j_{1}} - \alpha_{j_{1}}) \leq \dfrac{s_{j_{1}}^{2}}{4} \leq \dfrac{p^{2}}{16}$, and $\beta_{j_{y_{j}}}(s_{j_{y_{j}}}-\beta_{j_{y_{j}}})) \leq \dfrac{s_{j_{y_{j}}}^{2}}{4} \leq \dfrac{p^{2}}{16}$. Combining these facts, we get
\begin{align}
    A &\leq A' + |\pcls|\dfrac{1}{2}\left( 3\dfrac{p^{2}}{16} + \dfrac{p^{2}}{16} \right) \nonumber\\
      &= A' + |\pcls|\dfrac{p^{2}}{8} \leq \dfrac{3}{2}A', \nonumber
\end{align}
which is desired.

\end{proof}

\end{proof}

Now we are ready to prove \cref{lem:main-cut-case}.

\begin{proof}[Proof of \cref{lem:main-cut-case}]
   Let $\out$ be the output of $\algog$ in the cut case. Let $\costone{\out}, \costtwo{\out},$ $\costthree{\out},$ $\costfour{\out}$ be the four costs that are paid by our algorithm when in $\algog$, we are left with the clusters present only in $\cut$.~\cref{lem:main-cut-case} is essentially showing that
\begin{align}
    \costone{\out} + \costtwo{\out} + \costthree{\out} + \costfour{\out} \leq 2 \, \opt + 1.5\opt. \nonumber
\end{align}

We complete the proof by applying~\cref{lem:cost-one-plus-two-cut-case}, which gives 
\begin{align}
   \costone{\out} + \costtwo{\out} \leq 2 \, \opt \nonumber,
\end{align}
and applying~\cref{lem:cost 3 + 4 cut case}, which gives
\begin{align}
    \costthree{\out} + \costfour{\out} \leq 1.5\opt \nonumber.
\end{align} 
\end{proof}

Hence, from \cref{lem:cost-one-plus-two-cut-case} and \cref{lem:cost 3 + 4 cut case} we get that $\mop$ pays at most $3.5 \, \, \opt$ in the cut case.

\section{\texorpdfstring{Closest Balanced Clustering for $p,q>1$}{Closest Balanced Clustering for p,q > 1}}\label{sec:multiple-of-p-q-clustering}

In this section, given a clustering $\inpsetd$ of a set of red-blue colored points, where the irreducible ratio of blue to red points is $ p/q$ for some integer $ p,q > 1 $,  it returns a balanced clustering $ \outpq $.
That is, in the resulting clustering $ \outpq $, every
$\outpqcl{i} \in \outpq$, satisfies the condition that $|\red{\outpqcl{i}}|$ is a multiple of $q$ and $|\blue{\outpqcl{i}}|$ is a multiple of $p$. Our main result for this section is the following:

\multiplepq*

We first provide an algorithm to find a $7.5$-close {\bal} of $\inpset$ and then analyze the algorithm in the subsequent subsections.

\subsection{Details of the Algorithm}\label{subsec:Algopq-cut-case}

Before describing our algorithm, let us first define a few notations that we use later.

For any cluster, if its number of red (or blue) points is not a multiple of $q$ (or $p$), then we define its red (or blue) surplus as the minimum number of red( or blue) points whose removal makes the number of red (or blue) points to be a multiple of $q$ (or $p$), and we define its deficit as the minimum number of red (or blue) points whose addition makes the number of red (or blue) points to be a multiple of $q$ (or $p$). Formally, we define them as follows.

\begin{itemize}
    \item Red Surplus of $\inpcl{i}$, $\rsurp{\inpcl{i}}$: For any cluster $\inpcl{i}\subseteq V$, $\rsurp{\inpcl{i}} = |\red{\inpcl{i}}| \mod q$.
    \item Blue Surplus of $\inpcl{i}$, $\bsurp{\inpcl{i}}$: For any cluster $\inpcl{i} \subseteq V$, $\bsurp{\inpcl{i}} = |\blue{\inpcl{i}}| \mod p$.
    \item Red Deficit of $\inpcl{i}$, $\rdefi{\inpcl{i}}$: For any cluster $\inpcl{i} \subseteq V$ if $q \nmid |\red{\inpcl{i}}|$ then $\rdefi{\inpcl{i}} = q - \rsurp{\inpcl{i}}$ otherwise $\rdefi{\inpcl{i}} = 0$.
    \item Blue Deficit of $\inpcl{i}$, $\bdefi{\inpcl{i}}$: For any cluster $\inpcl{i} \subseteq V$ if $p \nmid |\blue{\inpcl{i}}|$ then $\bdefi{\inpcl{i}} = p - \bsurp{\inpcl{i}}$ otherwise $\bdefi{\inpcl{i}} = 0$.
\end{itemize}


Based on the surplus of a cluster, we define four sets of clusters as follows.
\begin{itemize}
    \item $ \rcut = \{ \inpcl{i}\ | \ \rsurp{\inpcl{i}} \leq q/2 \} $;
    \item $ \rmerge = \{ \inpcl{i} \ | \ \rsurp{\inpcl{i}} > q/2 \} $;
    \item $ \bcut = \{ \inpcl{i}\ | \ \bsurp{\inpcl{i}} \leq p/2 \} $;
    \item $ \bmerge = \{ \inpcl{i} \ | \ \bsurp{\inpcl{i}} > p/2 \} $.
\end{itemize}

A cluster belongs to exactly two sets. For example, a cluster with a red surplus greater than $ q/2 $ and a blue surplus smaller than $p/2$, that is, $s_b(D_i) \leq p/2$ and $s_r(D_i) > p/2$. Such clusters belong to $ \bcut $ and $ \rmerge $. Therefore, the input clustering is a partition (union of disjoint sets) $ \inpset = (\rcut\cap \bcut)\cup (\rcut\cap \bmerge) \cup (\rmerge\cap \bcut)\cup (\rmerge\cap \bmerge) $.

We briefly describe our algorithm to find a $7.5$-close {\bal} of $\inpset$.

\paragraph{Description of the algorithm}


Our algorithm $ \algopqgen $ (\cref{alg:algo-for-p/q-general}) aims to, for each cluster $ \inpcl{i}\in \rcut $, cut $ \rsurp{\inpcl{i}} $ red points from $ \inpcl{i} $ and merge these $ \rsurp{\inpcl{i}} $ red points to some clusters $ \inpcl{j}\in \rmerge $. Similarly, $ \bsurp{\inpcl{i}} $ blue points are cut from a cluster $ \inpcl{i}\in \bcut $ and merged to some clusters $ \inpcl{j}\in \bmerge $. 

We describe the implementation in~\cref{alg:algo-for-p/q-general}. We initialize the sets $ \rcuta = \rcut $, $ \rmergea = \rmerge $, $ \bcut = \bcuta $, and $ \bmergea = \bmerge $. Moreover, in $ \rcuta $ and $ \bcuta $, we discard the sets $ \inpcl{i}, \inpcl{j} $ with $ \rsurp{\inpcl{i}}=0 $ and $ \bsurp{\inpcl{j}}=0 $, respectively. The clusters $ \inpcl{i}\in \rmerge $ are sorted in non-increasing order based on $ \rsurp{\inpcl{i}}(|\inpcl{i}| - \rsurp{\inpcl{i}}) - \rdefi{\inpcl{i}}|\inpcl{i}| $. The algorithm iterates through each cluster $ \inpcl{i}\in \rcut$. For each such $ \inpcl{i} $, the algorithm takes a cluster $ \inpcl{j}\in \rmerge $, and transfer at most $\min(\rsurp{\inpcl{i}}, \rdefi{\inpcl{j}})$ red points from $\inpcl{i}$ to $\inpcl{j}$. After this, the number of red points in $ \inpcl{i} $ or $ \inpcl{j} $ must be a multiple of $ q $. If $ q||\red{\inpcl{i}}| $, the algorithm iterates to the next cluster in $ \rcut $. Otherwise, if $q||\red{\inpcl{j}}|$, the algorithm finds another cluster in $ \rmerge $ and repeats transferring red points from $ \inpcl{i} $ into this one. This process proceeds until $\rcuta $ or $ \rmergea $ becomes empty. A similar process is also applied for $ \bcut $ and $ \bmerge $. 


After this, we will be in only one of four cases, which will be defined below. These four cases can be easily identified by the input clustering $\inpset$.

\begin{itemize}
    \item \textbf{Cut-Cut case:} If all the sets are empty or if the nonempty set of clusters is $ \rcuta $ or $ \bcuta$, or both $ \rcuta $ and $ \bcuta $. This case occurs if the following holds 
    \begin{align*}
        \sum_{\inpcl{i} \in \rcut } \rsurp{\inpcl{i}} \geq \sum_{\inpcl{j} \in \rmerge} (q - \rsurp{\inpcl{j}})
    \end{align*}
    and
    \begin{align*}
        \sum_{\inpcl{i} \in \bcut } \bsurp{\inpcl{i}} \geq \sum_{\inpcl{j} \in \bmerge} (p - \bsurp{\inpcl{j}})
    \end{align*}
    \item \textbf{Cut-Merge case:} If the remaining sets of clusters are $ \rcuta $ and $ \bmergea$ . This case occurs when the following holds
    \begin{align*}
        \sum_{\inpcl{i} \in \rcut} \rsurp{\inpcl{i}} > \sum_{\inpcl{j} \in \rmerge} (q - \rsurp{\inpcl{j}})
    \end{align*}
    and
    \begin{align*}
        \sum_{\inpcl{i} \in \bcut} \bsurp{\inpcl{i}} < \sum_{\inpcl{j} \in \bmerge} (p - \bsurp{\inpcl{j}})
    \end{align*}
    
    \item \textbf{Merge-Cut case:} If the remaining sets of clusters are $\rmergea$ and $ \bcuta $. This case occurs when the following holds
    \begin{align*}
        \sum_{\inpcl{i} \in \rcut} \rsurp{\inpcl{i}} < \sum_{\inpcl{j} \in \rmerge} (q - \rsurp{\inpcl{j}})
    \end{align*}
    and
    \begin{align*}
        \sum_{\inpcl{i} \in \bcut} \bsurp{\inpcl{i}} > \sum_{\inpcl{j} \in \bmerge} (p - \bsurp{\inpcl{j}})
    \end{align*}
    
    \item \textbf{Merge-Merge case:} If the remaining set of cluster is $ \rmergea $, or $\bmergea $, or both $ \rmergea $ and $ \bmergea $. This case occurs when the following holds.
    \begin{align*}
        \sum_{\inpcl{i} \in \rcut} \rsurp{\inpcl{i}} \leq \sum_{\inpcl{j} \in \rmerge} (q - \rsurp{\inpcl{j}})
    \end{align*}
    and
    \begin{align*}
        \sum_{\inpcl{i} \in \bcut} \bsurp{\inpcl{i}} \leq \sum_{\inpcl{j} \in \bmerge} (p - \bsurp{\inpcl{j}})
    \end{align*}
\end{itemize}

Each of these cases is handled by using specific subroutines. We provide a brief description of all the subroutines below.

\vspace{0.5em}

\noindent \textbf{$\algopqc$} (\cref{alg:algo-for-p/q-cut}): As mentioned previously we call the subroutine $\algopqc$ in the cut-cut case, that is when the remaining set of clusters $D_i$ belongs to $\rcut$ or $ \bcut$, or both $\rcut$ and $\bcut$. In this case, we cut $\rsurp{D_i}$ red points from every remaining cluster $D_i$ and form extra clusters of size $q$, and we also cut $\bsurp{D_i}$ blue points from every remaining cluster $D_i$ and form extra clusters of size $p$. We do this operation sequentially until, in every cluster, the number of red points is a multiple of $q$ and the number of blue points is a multiple of $p$. 

\vspace{0.5em}

\noindent \textbf{$\algopqcm$}(\cref{alg:algo-for-p/q-cutmerge}) : We call the subroutine $\algopqcm$(\cref{alg:algo-for-p/q-cutmerge}) in the cut-merge case when the remaining set of clusters $D_i$ belong to $\rcut$ and $\bmerge$. In this case, on the red side, we cut $\rsurp{D_i}$ red points from every cluster $D_i$ and form extra clusters of size $q$. On the blue side, since the surplus $\bsurp{D_i} > p/2$, we want to merge the deficit amount of blue points in these clusters. This algorithm is a bit more involved and we describe it below.

    $\algopqgen$ takes a clustering $\inpset$ as input where the clusters $D_i \in \inpset$ can belong to any of the four sets $\rcut$, $\bcut$, $\rmerge$ or $\bmerge$. The Cut-Merge case occurs when the total surplus of the clusters belonging to $\bcut$ is less than the total deficit of the clusters belonging to $\bmerge$, that is 
    
    \[
        \sum_{D_i \in \bcut}\bsurp{D_i} < \sum_{D_i \in \bmerge}\bdefi{D_i}.
    \] 
    
    Before the subroutine $\algopqcm$(\cref{alg:algo-for-p/q-cutmerge}) is called, the algorithm $\algopqgen$ cuts the blue surplus points from the clusters in $\bcut$ and merges these blue points to the clusters in $\bmerge$. At this point, for all clusters in $\bcut$, the number of blue points is a multiple of $p$. To remove any ambiguity, let us call this transformed set of clusters $\nbcut$ and the remaining set of clusters in $\bmerge$ where the number of blue points is still not a multiple of $p$ as $\bmerge'$. 

    The subroutine $\algopqcm$(\cref{alg:algo-for-p/q-cutmerge}) takes two sets $\nbcut, \bmerge'$ as input. In $\algopqcm$(\cref{alg:algo-for-p/q-cutmerge}), we calculate the sum of $\bdefi{\inpcl{i}}$ for each of the remaining clusters $\inpcl{i} \in \bmerge'$. Let, $W_b = \sum_{D_j \in \bmerge'} \bdefi{\inpcl{j}}$. Now, each cluster $D_k \in \nbcut \cup \bmerge'$ the set $\blue{D_k}$ is divided into $(|\blue{D_k}| - \bsurp{D_k})/p$ subsets of size $p$ and one subset of size $\bsurp{D_k}$. Let us number the subset of size $\bsurp{D_k}$ as $0$, and all other subsets of size $p$ are numbered arbitrarily. 

    Each of these subsets is attached with a cost 

    \begin{align}
        \kappa_0(D_j) = \bsurp{D_j} \left( |D_j| - \rsurp{D_j} - \bsurp{D_j} \right) && \text{cutting cost for $0$th subset} \n \\
        \kappa_z(D_j) = p \left( |D_j| - \rsurp{D_j} - (zp + \bsurp{D_j}) \right) && \text{cutting cost for $z$th subset of size $p$ where $z \neq 0$} \n
    \end{align}

    Now, in the algorithm $\algopqcm$(\cref{alg:algo-for-p/q-cutmerge}), we cut $W_b/p$ blue subsets with the minimum cost. Let us call the set of clusters from where we have cut at least one subset as $\m{A}$. We merge these subsets into the clusters present in $\bmerge' \setminus \m{A}$.

\vspace{0.5em}

\noindent \textbf{$\algopqmc$} (\cref{alg:algo-for-p/q-merge-cut}): The description of the algorithm $\algopqmc$ is similar to the description of algorithm $\algopqcm$(\cref{alg:algo-for-p/q-cutmerge}).

\vspace{0.5em}

\noindent \textbf{$\algopqmm$} (\cref{alg:algo-for-merge-merge}): The description of the algorithm $\algopqmm$(\cref{alg:algo-for-merge-merge}) is again similar to the description of the merge part in the algorithm $\algopqcm$(\cref{alg:algo-for-p/q-cutmerge}). Before the subroutine $\algopqmm$(\cref{alg:algo-for-merge-merge}) is called, the algorithm $\algopqgen$ cuts the blue surplus points from the clusters in $\bcut$ and merges these blue points to the clusters in $\bmerge$. At this point, for all clusters in $\bcut$, the number of blue points is a multiple of $p$. To remove any ambiguity, let us call this transformed set of clusters $\nbcut$ and the remaining set of clusters in $\bmerge$ where the number of blue points is still not a multiple of $p$ as $\bmerge'$. Similarly, on the red side, let us call the transformed set of clusters from $\rcut$ after the algorithm $\algopqgen$ be $\nrcut$ and the remaining set of clusters in $\rmerge$ where the number of red points is still not a multiple of $q$ as $\rmerge'$.

    The subroutine $\algopqmm$(\cref{alg:algo-for-merge-merge}) takes four sets $\nbcut,$ $\bmerge',$ $\nrcut,$ $\rmerge'$ as input. First, in $\algopqmm$(\cref{alg:algo-for-merge-merge}), we calculate the sum of $\bdefi{\inpcl{i}}$ for each of the remaining clusters $\inpcl{i} \in \bmerge'$. Let, $W_b = \sum_{D_j \in \bmerge'} \bdefi{\inpcl{j}}$. Now, each cluster $D_k \in \nbcut \cup \bmerge'$ the set $\blue{D_k}$ is divided into $(|\blue{D_k}| - \bsurp{D_k})/p$ subsets of size $p$ and one subset of size $\bsurp{D_k}$. Let us number the subset of size $\bsurp{D_k}$ as $0$, and all other subsets of size $p$ are numbered arbitrarily. 

    Each of these subsets is attached with a cost 

    \begin{align}
        \kappa_0(D_j) = \bsurp{D_j} \left( |D_j| - \bsurp{D_j} \right) && \text{cutting cost for $0$th subset} \n \\
        \kappa_z(D_j) = p \left( |D_j| -  (zp + \bsurp{D_j}) \right) && \text{cutting cost for $z$th subset of size $p$ where $z \neq 0$} \n
    \end{align}

    Now, in the algorithm $\algopqmm$(\cref{alg:algo-for-merge-merge}), we cut $W_b/p$ blue subsets with the minimum cost. Let us call the set of clusters from where we have cut at least one subset as $\m{A}_b$. We merge these blue subsets of size $p$ into the clusters present in $\bmerge' \setminus \m{A}_b$.

    After that,  $\algopqmm$(\cref{alg:algo-for-merge-merge}) does the same thing on the red side. Here, we calculate the sum of $\rdefi{\inpcl{i}}$ for each of the remaining clusters $\inpcl{i} \in \rmerge'$. Let, $W_r = \sum_{D_j \in \rmerge'} \rdefi{\inpcl{j}}$. Now, for each cluster $D_k \in \nrcut \cup \rmerge'$ the set $\red{D_k}$ is divided into $(|\red{D_k}| - \rsurp{D_k})/q$ subsets of size $q$ and one subset of size $\rsurp{D_k}$. Let us number the subset of size $\rsurp{D_k}$ as $0$, and all other subsets of size $q$ are numbered arbitrarily. 

    Each of these subsets is attached with a cost 

    \begin{align}
        \kappa_0(D_j) = \rsurp{D_j} \left( |D_j| - \rsurp{D_j} \right) && \text{cutting cost for $0$th subset} \n \\
        \kappa_z(D_j) = p \left( |D_j| -  (zp + \rsurp{D_j}) \right) && \text{cutting cost for $z$th subset of size $p$ where $z \neq 0$} \n
    \end{align}

    Now, in the algorithm $\algopqmm$(\cref{alg:algo-for-merge-merge}), we cut $W_r/q$ red subsets with the minimum cost. Let us call the set of clusters from where we have cut at least one subset as $\m{A}_r$. We merge these blue subsets of size $p$ into the clusters present in $\bmerge' \setminus \m{A}_r$.

\textbf{Runtime Analysis of $\algopqgen$}: Executing the algorithm $\algopqgen$ is equivalent to executing the algorithm $\algog$(\cref{alg:algo-for-general}) in case of $p>1, q=1$ twice -- once to make the red points in every cluster a multiple of $q$ and again to make the number of blue points in every cluster a multiple of $p$. 

When we make the number of red points in every cluster a multiple of $q$, we need to sort the clusters in $\rmerge$; this takes $O(|\rmerge| \log (\rmerge) ) = O(n \log n)$ time and again to make the number of blue points in every cluster a multiple of $p$ we need to sort the clusters in $\bmerge$ this would again take $O(|\bmerge| \log (\bmerge)) = O(n \log n)$ time.

In this algorithm, a point $v \in V$ is used at most twice -- once while cutting and again while merging. Hence, the cutting and merging process takes at most $O(n)$ time. Thus, the total time complexity is $O(n \log n)$.

\begin{algorithm2e}
\SetAlgoVlined
\DontPrintSemicolon
\small
\caption{$ \algopqgen $}\label{alg:algo-for-p/q-general}
\KwData {Input set of clusters $\inpset$}
\KwResult{A set of clusters $ \outpq $, such that in each cluster $\outpqcl{i} \in \outpq $ the number of blue points is a multiple of $p$, and the number of red points is a multiple of $ q $.} 
    $ \rcuta \gets \{ \inpcl{i}\ | \ 0 <\rsurp{\inpcl{i}} \leq q/2 \} $ and $ \nrcut \gets \{ \inpcl{i}| \rsurp{\inpcl{i}}=0 \} $ \;
    $ \rmergea \gets \inpcl{i} \ | \ \rsurp{\inpcl{i}} \leq q/2 \} $ \;
    $ \bcuta \gets \{ \inpcl{i}\ | \ 0< \bsurp{\inpcl{i}} \leq p/2 \} $  and $ \nbcut \gets \{ \inpcl{i}| \bsurp{\inpcl{i}}=0 \} $ \;
    $ \bmergea \gets \{ \inpcl{i} \ | \ \bsurp{\inpcl{i}} \leq q/2 \} $\;
    $\outpqg \gets \{ \inpcl{i}| \rsurp{\inpcl{i}}=\bsurp{\inpcl{i}}=0  \}$\;

    Sort the clusters $ \inpcl{i}\in \rmerge $ based on their $ \rsurp{\inpcl{i}}(|\inpcl{i}| - \rsurp{\inpcl{i}}) - \rdefi{\inpcl{i}}|\inpcl{i}| $ in non-increasing order.\;

    \While{$ \rcuta \neq \emptyset $ and $ \rmergea \neq \emptyset $}{
        \For{$ \inpcl{i}\in \rcuta $}{
            \For{$ \inpcl{j}\in \rmergea $}{
                $ k \gets \min(\rsurp{\inpcl{i}}, \rdefi{\inpcl{j}} ) $\;
                Cut a set of $ k $ red points from $ \inpcl{i} $ and merge to $ \inpcl{j} $\;
                \uIf{$k = \rdefi{\inpcl{j}}$}{
                    $\rmergea = \rmergea \setminus \{\inpcl{j}\}$ \; 
                    \lIf{ $ \bsurp{\inpcl{j}}=0 $ }{$\outpqg = \outpqg \cup \inpcl{j}$} 
                }
                
                \uIf{$k = \rsurp{\inpcl{i}}$}{
                    $\rcuta \gets \rcuta \setminus \{\inpcl{i}\}$, $ \nrcut \gets \nrcut \cup \{\inpcl{i}\} $ \;
                    \lIf{ $ \bsurp{\inpcl{i}}=0 $ }{$\outpqg = \outpqg \cup \inpcl{i}$}
                    Break
                }
            }
        }
    }
    
    Sort the clusters $ \inpcl{i}\in \bmerge $ based on their $ \bsurp{\inpcl{i}}(|\inpcl{i}| - \bsurp{\inpcl{i}}) - \bdefi{\inpcl{i}}|\inpcl{i}| $ in non-increasing order.\;
    \While{$ \bcuta \neq \emptyset $ and $ \bmergea \neq \emptyset $}{
        \For{$ \inpcl{i}\in \bcuta $}{
            \For{$ \inpcl{j}\in \bmergea $}{
                $ k \gets \min(\bsurp{\inpcl{i}}, \bdefi{\inpcl{j}} ) $\;
                Cut a set of $ k $ blue points from $ \inpcl{i} $ and merge to $ \inpcl{j} $\;
                \uIf{$k = \bdefi{\inpcl{j}}$}{
                    $\bmergea \gets \bmergea \setminus \{\inpcl{j}\}$ \; 
                    \lIf{ $ \rsurp{\inpcl{j}}=0 $ }{$\outpqg = \outpqg \cup \inpcl{j}$} 
                }
                \uIf{$k = \bsurp{\inpcl{i}}$}{
                    $\bcuta \gets \bcuta \setminus \{\inpcl{i}\}$,  $\nbcut \gets \nbcut \cup \{\inpcl{i}\}$ \;
                    \lIf{ $ \rsurp{\inpcl{i}}=0 $ }{$\outpqg = \outpqg \cup \inpcl{i}$} 
                    Break
                }
            }
        }
    }
   \uIf{$\rmergea = \emptyset$ and $ \bmergea = \emptyset $} {
       \Return{$\outpqg \cup \algopqc(\rcuta, \bcuta)$}  \tcp{\cref{alg:algo-for-p/q-cut}} \label{line:algopqgen-refer-algopqc}
    }
    \uElseIf{$ \rmergea = \emptyset $ and $ \bcuta = \emptyset $}{
         \Return{$\outpqg \cup \algopqcm(\rcuta, \nbcut, \bmergea)$}  \tcp{\cref{alg:algo-for-p/q-cutmerge}} \label{line:algopqgen-refer-algopqcm}
    }
    \uElseIf{ $ \rcuta = \emptyset $ and $ \bmergea = \emptyset $}{
        \Return{$ \outpqg \cup \algopqmc(\nrcut, \rmergea, \bcuta) $} \tcp{\cref{alg:algo-for-p/q-merge-cut}}\label{line:algopqgen-refer-algopqmc}
    }\Else {\tcc{\cref{alg:algo-for-merge-merge}}
        \Return{$ \outpqg \cup \algopqm(\nrcut, \rmergea, \nbcut, \bmergea) $}\label{line:algopqgen-refer-algopqm}\; 
    }
\end{algorithm2e}

\begin{algorithm2e}
\DontPrintSemicolon
\caption{$\algopqc(\rcuta, \bcuta)$}\label{alg:algo-for-p/q-cut}
\KwData{ $ \rcuta $: set of cluster $\gencl{i} $ with the red surplus $ 0< \rsurp{\gencl{i}} \leq p/2$;\\ \> \> \> \> \> \>  $ \bcuta $: set of clusters with the blue surplus $ 0< \bsurp{\gencl{i}} \leq p/2$.}
\KwResult{Set of clusters $\outpqcc$, such that in each cluster $\outpqcl{i} \in \outpqcc$ the number of blue points is a multiple of $p$ and the number of red points is a multiple of $q$. }
    $\pcard_{r} = \frac{\sum_{\gencl{i}\in \rcuta}\rsurp{\gencl{i}}}{q}$ \;
    
    Initialize $\pcard_{r}$ many sets $\pclr{1}, \pclr{2} \ldots \pclr{\pcard_{r}}$ to $\emptyset$ \; 
    \tcp{$ \pcard_{r} $ many extra clusters}
    
    Initialize $ \pcard_{r} $ many variables $\ell_1, \ell_2, \ldots \ell_{\pcard_{r}}$ to $q$ \;
    \tcp{leftover space of extra clusters}
    
    Boolean $flag = 0$ \;
    
    \For{$\gencl{i} \in \rcuta$} {
        \For{$j = 1$ to $ \pcard_{r} $} {
             Add $k = min ( \ell_j, \rsurp{\gencl{i} })$ many red points from $\gencl{i} $ to the cluster $ \pclr{j} $ \;
             
             $\ell_j = \ell_j - k$ \;
        }
    }
    $\pcard_{b} = \frac{\sum_{\gencl{i}\in \bcuta}\bsurp{\gencl{i}}}{p}$ \;
    
    Initialize $ \pcard_{b} $ many sets $ \pclb{1}, \pclb{2} \ldots \pclb{\pcard_{b}} $ to $\emptyset$ \; 
    \tcp{$ \pcard_{b} $ many extra clusters}
    
    Initialize $ \pcard_{b} $ many variables $\ell_1, \ell_2, \ldots \ell_{\pcard_{b}}$ to $p$ \;
    \tcp{leftover space of extra clusters}
    
    Boolean $flag = 0$ \;
    
    \For{$\gencl{i} \in \cut'$} {
        \For{$j = 1$ to $ \pcard_{b} $} {
                 Add $k = min ( \ell_j, \bsurp{\gencl{i} })$ many blue points from $\gencl{i} $ to the cluster $ \pclb{j} $ \;
                 
                 $\ell_j = \ell_j - k$ \;
        }
    }
    \Return{$\rcuta \cup \bcuta \cup \{ \pclr{1}, \pclr{2}, \dots, \pclr{\pcard_{r}}\} \cup \{ \pclb{1}, \pclb{2}, \dots, \pclb{\pcard_{b}} \}$}
\end{algorithm2e}

\begin{algorithm2e}
\DontPrintSemicolon
\caption{$\algopqcm(\rcut', \nbcut, \bmerge')$}\label{alg:algo-for-p/q-cutmerge}
\KwData{Three clustering as input. For all $D_i \in \nbcut$ we have $p \mid \blue{D_i}$, for all $D_j \in \rcut'$ and for all $D_k \in \bmerge'$ we have $\rsurp{D_j} < q/2$ and $\bsurp{D_k} < p/2$ respectively.}
\KwResult{A clustering $\outpq$, such that in each cluster $\outpqcl{i} \in \outpq$ the number of blue points is a multiple of $p$ and the number of red points is a multiple of $q$. }

RedExtraSum = $\sum_{D_i \in \rcut'} \rsurp{\inpcl{i} }$ \;
$m = \frac{\text{RedExtraSum}}{q}$ \;
Initialize $m$ many sets $E_{1}, E_{2} \ldots E_{m}$ to $\emptyset$ \;

\tcp{$m$ many extra clusters}

Initialize $m$ many variables $\ell_1, \ell_2, \ldots \ell_m$ to $q$ \;

\tcp{$\ell_i$ denotes leftover space of extra cluster $E_i$.}

Boolean $flag = 0$ \;
\For{$\inpcl{i} \in \rcut'$} {
    \For{$j = 1$ to $m$} {
             Add $k = min ( \ell_j, \rsurp{\inpcl{i} })$ many red points from $\inpcl{i} $ to the cluster $E_{j}$ \;
             $\ell_j = \ell_j - k$ \;
    }
}
$\rcut' = \rcut' \cup \{E_1, E_2, \ldots, E_n\}$

$W_b \gets \sum_{\inpcl{j}\in\bmerge'}\bdefi{\inpcl{j}}$ \;
  \For{$D_k \in (\nbcut \cup \bmerge')$} {
        Initialize a variable $v_k$ to $0$\;
  }

  \For{$a = 1$ to $W/p$} {
    Take the cluster $D_k \in (\nbcut \cup \bmerge')$ for which $\textit{cost}(w_{k,v_k})$ is minimum\;
    
    \tcp{Recall $\textit{cost}(w_{k,z})$ denotes the cost of cutting the $z$th red subset of size $q$ from $D_k$. }
    
    \While{$\textit{size}(w_{k,v_k}) \neq 0$} {
        \For{$\inpcl{j} \in \nbcut \cup \bmerge'$}{
         $\gamma = \min(\bdefi{\inpcl{j}}, \textit{size}(w_{k,v_k}))$\;
         
         cut $\gamma$ many blue points from $D_k$ and merge to $\inpcl{j}$\;
         
         $\textit{size}(w_{k,v_k}) = \textit{size}(w_{k,v_k}) - \gamma$\;
         $\bdefi{\inpcl{j}} = \bdefi{\inpcl{j}} - \gamma$
        }
    }
    \For{$\inpcl{j} \in \bmerge'$} {
  \If{$\bdefi{\inpcl{j}} = 0$}{
  update $\bmerge' = \bmerge' \setminus \{\inpcl{j}\}$
  }
  }
  
  }

  \Return{$\outpq = \rcut' \cup \nbcut \cup \bmerge'$}
\end{algorithm2e}

\begin{algorithm2e}
\DontPrintSemicolon
\caption{$\algopqmc(\nrcut, \rmerge', \bcut')$}\label{alg:algo-for-p/q-merge-cut}
\KwData{Three clustering as input. For all $D_i \in \nrcut$ we have $q \mid \red{D_i}$, for all $D_j \in \rmerge'$ and for all $D_k \in \bcut'$ we have $\rsurp{D_j} > q/2$ and $\bsurp{D_k} \leq p/2$ respectively.}
\KwResult{A clustering $\outpq$, such that in each cluster $\outpqcl{i} \in \outpq$ the number of blue points is a multiple of $p$ and the number of red points is a multiple of $q$. }

BlueExtraSum = $\sum_{D_i \in \bcut'} \bsurp{\inpcl{i} }$ \;
$m = \frac{\text{BlueExtraSum}}{p}$ \;
Initialize $m$ many sets $E_{1}, E_{2} \ldots E_{m}$ to $\emptyset$ \;

\tcp{$m$ many extra clusters}

Initialize $m$ many variables $\ell_1, \ell_2, \ldots \ell_m$ to $q$ \;

\tcp{$\ell_i$ denotes leftover space of extra cluster $E_i$.}

Boolean $flag = 0$ \;
\For{$\inpcl{i} \in \bcut'$} {
    \For{$j = 1$ to $m$} {
             Add $k = min ( \ell_j, \bsurp{\inpcl{i} })$ many blue points from $\inpcl{i} $ to the cluster $E_{j}$ \;
             $\ell_j = \ell_j - k$ \;
    }
}
$\bcut' = \bcut' \cup \{E_1, E_2, \ldots, E_n\} $

$W_r \gets \sum_{\inpcl{j}\in\rmerge'}\rdefi{\inpcl{j}}$ \;
  \For{$D_k \in (\nrcut \cup \rmerge')$} {
        Initialize a variable $v_k$ to $0$\;
  }

  \For{$a = 1$ to $W/p$} {
    Take the cluster $D_k \in (\nrcut \cup \rmerge')$ for which $\textit{cost}(w_{k,v_k})$ is minimum\;
    
    \tcp{Recall $\textit{cost}(w_{k,z})$ denotes the cost of cutting the $z$th blue subset of size $p$ from $D_k$. }
    
    \While{$\textit{size}(w_{k,v_k}) \neq 0$} {
        \For{$\inpcl{j} \in (\nrcut \cup \rmerge')$}{
         $\gamma = \min(\rdefi{\inpcl{j}}, \textit{size}(w_{k,v_k}))$\;
         
         cut $\gamma$ many blue points from $D_k$ and merge to $\inpcl{j}$\;
         
         $\textit{size}(w_{k,v_k}) = \textit{size}(w_{k,v_k}) - \gamma$\;
         $\rdefi{\inpcl{j}} = \rdefi{\inpcl{j}} - \gamma$
        }
    }
    \For{$\inpcl{j} \in \bmerge'$} {
  \If{$\rdefi{\inpcl{j}} = 0$}{
  update $\rmerge' = \rmerge' \setminus \{\inpcl{j}\}$
  }
  }
  
  }
  \Return{$\outpq = \bcut' \cup \nrcut \cup \rmerge'$}
\end{algorithm2e}

\begin{algorithm2e}[t]
\DontPrintSemicolon
\caption{$\algopqmm(\nrcut, \nbcut, \bmerge', \rmerge')$}\label{alg:algo-for-merge-merge}
\KwIn{Clusterings $\nrcut, \nbcut, \bmerge', \rmerge'$ such that:
\begin{itemize}[nosep,leftmargin=1.5em]
    \item $q \mid |\red{D_i}|$ for all $D_i \in \nrcut$; $p \mid |\blue{D_j}|$ for all $D_j \in \nbcut$.
    \item $\rsurp{D_\ell} > q/2$ for $D_\ell \in \rmerge'$, $\bsurp{D_k} > p/2$ for $D_k \in \bmerge'$.
    \item Clusters in $\rmerge', \bmerge'$ are sorted by $(\ccostf{D} - \mcostf{D})$.
\end{itemize}
}
\KwOut{Clustering $\m{Q}$ with $p \mid |\blue{Q_i}|$ and $q \mid |\red{Q_i}|$ for all $Q_i \in \m{Q}$.}

$W_r \gets \sum_{D \in \rmerge'} \defi{D}$\;

\ForEach{$D \in \nrcut \cup \rmerge'$}{Set $v_D \gets 0$\;}

\For{$a \gets 1$ \KwTo $W_r / q$}{
    Pick $D_k$ minimizing $\textit{cost}(w_{k,v_{D_k},r})$\;
    Increment $v_{D_k}$\;
    \While{$\textit{size}(w_{k,v_{D_k},r}) > 0$}{
        \ForEach{$D_j \in \rmerge'$}{
            $\gamma \gets \min(\defi{D_j}, \textit{size}(w_{k,v_{D_k},r}))$\;
            Move $\gamma$ red points from $D_k$ to $D_j$\;
            Update $\textit{size}(w_{k,v_{D_k},r}) \mathrel{-}= \gamma$\;
        }
    }
    Remove any $D_j$ from $\rmerge'$ with $\rdefi{D_j} = 0$\;
}

$W_b \gets \sum_{D \in \bmerge'} \defi{D}$\;

\ForEach{$D \in \nbcut \cup \bmerge'$}{Set $v_D \gets 0$\;}

\For{$a \gets 1$ \KwTo $W_b / p$}{
    Pick $D_k$ minimizing $\textit{cost}(w_{k,v_{D_k},b})$\;
    \While{$\textit{size}(w_{k,v_{D_k},b}) > 0$}{
        \ForEach{$D_j \in \bmerge'$}{
            $\gamma \gets \min(\defi{D_j}, \textit{size}(w_{k,v_{D_k},b}))$\;
            Move $\gamma$ blue points from $D_k$ to $D_j$\;
            Update $\textit{size}(w_{k,v_{D_k},b}) \mathrel{-}= \gamma$\;
        }
    }
    Remove any $D_j$ from $\bmerge'$ with $\bdefi{D_j} = 0$\;
}

\Return{$\m{Q} = \nrcut \cup \nbcut \cup \rmerge' \cup \bmerge'$}
\end{algorithm2e}

We analyze the algorithms for each case separately in~\cref{subsec:bal-analyze-4-cases-p-q}. 
Let $\mopq$ be some closest Balance clustering of $\inpset$.
In~\cref{subsec:bal-define-cost-p-q}, we start by introducing notations and useful bounds for $ \dist(\inpset, \mopq) $, which we later use for the analysis in~\cref{subsec:bal-analyze-4-cases-p-q}. The last section~\cref{subsec:bal-overall-approx} is dedicated to proving~\cref{thm:main-multiple-of-pq}, which combines all the cases to come up with the overall approximation.

\subsection{\texorpdfstring{Bounds for $\dist(\inpset, \mopq)$}{Bounds for dist(X, M\_p,q)}}\label{subsec:bal-define-cost-p-q}

We provide some bounds for $ \dist(\inpset, \mopq) $. Suppose in $\mopq$ a cluster $\inpcl{i} \in \inpset$ (can be of any type) is split into $t$ parts $X_{i,1}, X_{i,2}, \ldots, X_{i,t}$. Each part $ X_{i,j} $ is merged with $ R_{i,j}, B_{i,j} $, the set of red points and blue points outside $ \inpcl{i} $. Then, the distance between $ \inpset $ and $ \mopq $ is   
\begin{align}
    \dist( \inpset, \mopq ) &= \sum_{\inpcl{i}\in \inpset } \left( \sum_{1\leq j<\ell \leq t}|X_{i,j}||X_{i,\ell}| + \sum_{j=1}^{t}\dfrac{|X_{i,j}|(|R_{i,j}|+|B_{i,j}|)}{2} \right). \label{eq:neutral-dist-inpset-mopq}
\end{align}

We have the following lower bound for $ \dist( \inpset, \mopq ) $, which is the generalization of~\cref{lem.opt.si.times.p.min.si}.
\begin{lemma}\label{lem:opt-greater-than-surp-times-defi}
    \begin{align}
        \dist(\inpset, \mopq) \geq \sum_{\inpcl{i}\in \inpset } \dfrac{\rsurp{ \inpcl{i} }(q-\rsurp{\inpcl{i}})+ \bsurp{\inpcl{i}}(p-\bsurp{\inpcl{i}})}{2}. \nonumber
    \end{align}
\end{lemma}
\begin{proof}
    In $ \mopq $, in every cluster $ X_{i,j}\cup R_{i,j}\cup B_{i,j} $, the number of red points is a multiple of $ q $, and the number of blue points is a multiple of $ p $, hence, $ |R_{i,j}| \geq \rdefi{X_{i,j}} $, and $ |B_{i,j}| \geq \bdefi{X_{i,j}} $. Therefore
    \begin{align}
        \sum_{j=1}^{t}\dfrac{|X_{i,j}|(|R_{i,j}|+|B_{i,j}|)}{2} \geq \sum_{j=1}^{t} \dfrac{\rsurp{X_{i,j}}\rdefi{X_{i,j}}}{2} + \sum_{j=1}^{t}\dfrac{\bsurp{X_{i,j}}\bdefi{X_{i,j}}}{2}. \nonumber
    \end{align}
    It remains to show that $ \sum_{j=1}^{t} \dfrac{\rsurp{X_{i,j}}\rdefi{X_{i,j}}}{2} \geq \dfrac{\rsurp{\inpcl{i}}(q-\rsurp{\inpcl{i}})}{2} $, and\\ $ \sum_{j=1}^{t} \dfrac{\bsurp{X_{i,j}}\bdefi{X_{i,j}}}{2} \geq \dfrac{\bsurp{\inpcl{i}}(p-\bsurp{\inpcl{i}})}{2} $. To this end, we apply~\cref{prop:mod-sum-ineq}. The proof is complete.
\end{proof}

Similar to the Cut cost and Merge cost defined in \cref{sec:multiple-of-p-clustering} we define the following costs:

\begin{itemize}
    \item \textbf{Cut-Cut cost:} $\cc{\inpcl{i}} = (\rsurp{\inpcl{i}} + \bsurp{\inpcl{i}}) (|\inpcl{i}| - (\rsurp{\inpcl{i}} + \bsurp{\inpcl{i}})) + \rsurp{\inpcl{i}} \bsurp{\inpcl{i}}$. 
    \item \textbf{Cut-Merge cost:} $\cm{\inpcl{i}} = \rsurp{\inpcl{i}} (|\inpcl{i}| - \rsurp{\inpcl{i}}) + (p - \bsurp{\inpcl{i}}) (|\inpcl{i}| - \rsurp{\inpcl{i}})$.
    \item \textbf{Merge-Cut cost:} $\mc{\inpcl{i}} = (q - \rsurp{\inpcl{i}}) (|\inpcl{i}| - \rsurp{\inpcl{i}}) + \bsurp{\inpcl{i}} (|\inpcl{i}| - \rsurp{\inpcl{i}})$.
    \item \textbf{Merge-Merge cost:} $\mm{\inpcl{i}} = (q - \rsurp{\inpcl{i}}) |\inpcl{i}| + (p - \bsurp{\inpcl{i}}) |\inpcl{i}|$.
\end{itemize}
To bound these costs in terms of $ \dist(\inpset, \mopq) $, we first rearrange the sum in~\eqref{eq:neutral-dist-inpset-mopq} to obtain the following. 

{\small
\begin{align}
    \dist(\inpset, \mopq) &\geq \sum_{\inpcl{i}\in \inpset} \left(\sum_{j < \ell}|X_{i,j}||X_{i,\ell}| +  \rdefi{X_{i,j}}(|X_{i,j}| - \rsurp{X_{i,j}}) + \frac{1}{2} \sum_j \rsurp{X_{i,j}} (q - \rsurp{X_{i,j}}) \right) \nonumber \\
    &= \sum_{\inpcl{i}\in \inpset} \optr_{\inpcl{i}},\nonumber
\end{align}
}
where $ \optr_{\inpcl{i}} =  \sum_{j < \ell}|X_{i,j}||X_{i,\ell}| +  \rdefi{X_{i,j}}(|X_{i,j}| - \rsurp{X_{i,j}}) + \frac{1}{2} \sum_j \rsurp{X_{i,j}} (q - \rsurp{X_{i,j}})$, and
{\small
\begin{align}
    \dist(\inpset, \mopq) &\geq \sum_{\inpcl{i}\in \inpset} \left(\sum_{j < \ell}|X_{i,j}||X_{i,\ell}| +  \bdefi{X_{i,j}}(|X_{i,j}| - \bsurp{X_{i,j}}) + \frac{1}{2} \sum_j \bsurp{X_{i,j}} (p - \bsurp{X_{i,j}}) \right) \nonumber \\
    &= \sum_{\inpcl{i}\in \inpset} \optb_{\inpcl{i}},\nonumber
\end{align}
}
where $ \optb_{\inpcl{i}} = \sum_{j < \ell}|X_{i,j}||X_{i,\ell}| +  \bdefi{X_{i,j}}(|X_{i,j}| - \bsurp{X_{i,j}}) + \frac{1}{2} \sum_j \bsurp{X_{i,j}} (p - \bsurp{X_{i,j}})$.

Now note that, utilizing~\cref{lem:main-structure-of-M*}, we have
\begin{align}
    \text{If } \rsurp{\inpcl{i}} \leq q/2, \text{ then }  \optr_{\inpcl{i}} &> \rsurp{\inpcl{i}}(|\inpcl{i}| - \rsurp{\inpcl{i}}) + \frac{1}{2} \rsurp{\inpcl{i}} (q - \rsurp{\inpcl{i}});  \label{equn:cut-red} \\
    \text{If } \rsurp{\inpcl{i}} > q/2, \text{ then }  \optr_{\inpcl{i}} &> (q-\rsurp{\inpcl{i}})(|\inpcl{i}| - \rsurp{\inpcl{i}}) + \frac{1}{2} \rsurp{\inpcl{i}} (q - \rsurp{\inpcl{i}});  \label{equn:merge-red} \\
    \text{If } \bsurp{\inpcl{i}} \leq p/2, \text{ then } \optb_{\inpcl{i}} &> \bsurp{\inpcl{i}}(|\inpcl{i}| - \rsurp{\inpcl{i}}) + \frac{1}{2} \bsurp{\inpcl{i}} (p - \bsurp{\inpcl{i}}); \label{equn:cut-blue} \\
    \text{If } \bsurp{\inpcl{i}} > p/2, \text{ then } \optb_{\inpcl{i}} &> (p-\bsurp{\inpcl{i}})(|\inpcl{i}| - \bsurp{\inpcl{i}}) + \frac{1}{2} \bsurp{\inpcl{i}} (p - \bsurp{\inpcl{i}}). \label{equn:merge-blue}
\end{align}

Now, we bound those four aforementioned costs in terms of $ \optr_{\inpcl{i}} $ and $ \optb{\inpcl{i}} $.

 \paragraph{Cost $ \cc{\inpcl{i}}$ for a cluster $\inpcl{i}\in \rcut \cap \bcut$.\\}
    In a cluster $ \inpcl{i}\in \rcut\cap \bcut $, $ \rsurp{\inpcl{i}}\leq q/2 $ and $ \bsurp{i}\leq p/2 $. Combining~\eqref{equn:cut-red} and~\eqref{equn:cut-blue}, we obtain
    \begin{align}
        \cc{\inpcl{i}} &= (\rsurp{\inpcl{i}} + \bsurp{\inpcl{i}}) (|\inpcl{i}| - (\rsurp{\inpcl{i}} + \bsurp{\inpcl{i}})) + \rsurp{\inpcl{i}} \bsurp{\inpcl{i}} \n \\
        &\leq \rsurp{\inpcl{i}} (|\inpcl{i}| - \rsurp{\inpcl{i}}) + \bsurp{\inpcl{i}}(|\inpcl{i}| - \bsurp{\inpcl{i}}) + \rsurp{\inpcl{i}} \bsurp{\inpcl{i}} \n \\
        &\leq \rsurp{\inpcl{i}} (|\inpcl{i}| - \rsurp{\inpcl{i}}) + \dfrac{\rsurp{\inpcl{i}}^{2}}{2} + \bsurp{\inpcl{i}}(|\inpcl{i}| - \bsurp{\inpcl{i}}) + \dfrac{\bsurp{\inpcl{i}}^{2}}{2} \n \\
        &\leq  \rsurp{\inpcl{i}} (|\inpcl{i}| - \rsurp{\inpcl{i}}) + \dfrac{1}{2} \bsurp{\inpcl{i}} (p - \bsurp{\inpcl{i}}) \n \\
        &\> + \bsurp{\inpcl{i}}(|\inpcl{i}| - \bsurp{\inpcl{i}}) + \dfrac{1}{2} \rsurp{\inpcl{i}} (q - \rsurp{\inpcl{i}}) \n \\
        &= \optr_{\inpcl{i}} + \optb_{\inpcl{i}}. \label{eq:bound-for-cc}
    \end{align}
    
\paragraph{Cost $ \cm{\inpcl{i}}$ for a cluster $\inpcl{i}\in \rcut \cap \bmerge$.\\}
In a cluster $ \inpcl{i}\in \rcut\cap \bmerge $, $ \rsurp{\inpcl{i}}\leq q/2 $ and $ \bsurp{i}> p/2 $. Combining~\eqref{equn:cut-red} and~\eqref{equn:merge-blue}, we obtain    
 {\small
\begin{align}
    \cm{\inpcl{i}} &= \rsurp{\inpcl{i}} (|\inpcl{i}| - \rsurp{\inpcl{i}}) + (p - \bsurp{\inpcl{i}}) (|\inpcl{i}| - \rsurp{\inpcl{i}})  \n \\
    &= \rsurp{\inpcl{i}} (|\inpcl{i}| - \rsurp{\inpcl{i}}) + (p - \bsurp{\inpcl{i}}) (|\inpcl{i}| - \rsurp{\inpcl{i}} - \bsurp{\inpcl{i}}) + \bsurp{\inpcl{i}} (p - \bsurp{\inpcl{i}}) \n \\
    &\leq \optr_{\inpcl{i}} + 2\optb_{\inpcl{i}}.\label{eq:bound-for-cm}
\end{align}
}

\paragraph{Cost $ \mc{\inpcl{i}}$ for a cluster $\inpcl{i}\in \rmerge \cap \bcut$.\\}
In a cluster $ \inpcl{i}\in \rmerge\cap \bcut $, $ \rsurp{\inpcl{i}}> q/2 $ and $ \bsurp{i}\leq p/2 $. Combining~\eqref{equn:merge-red} and~\eqref{equn:cut-blue}, we obtain    
\begin{align}
    \mc{\inpcl{i}} &= (q-\rsurp{\inpcl{i}} )(|\inpcl{i}| - \rsurp{\inpcl{i}}) +  \bsurp{\inpcl{i}} (|\inpcl{i}| - \rsurp{\inpcl{i}})  \n \\
                   &\leq 2\optr_{\inpcl{i}} + \optb_{\inpcl{i}}. \label{eq:bound-for-mc}
\end{align}
    
\paragraph{Cost $ \mm{\inpcl{i}}$ for a cluster $\inpcl{i}\in \rmerge \cap \bmerge$.\\}
In a cluster $ \inpcl{i}\in \rmerge\cap \bmerge $, $ \rsurp{\inpcl{i}}> q/2 $ and $ \bsurp{i} > p/2 $. Combining~\eqref{equn:merge-red} and~\eqref{equn:merge-blue}, we obtain    
\begin{align}
    \mm{\inpcl{i}} &= (q - \rsurp{\inpcl{i}}) |\inpcl{i}| + (p - \bsurp{\inpcl{i}}) |\inpcl{i}| \n \\
                   &= (q - \rsurp{\inpcl{i}})( |\inpcl{i}|-\rsurp{\inpcl{i}}) + (q - \rsurp{\inpcl{i}})\rsurp{\inpcl{i}} \n \\
                   &\> +  (p - \bsurp{\inpcl{i}})(|\inpcl{i}| - \bsurp{\inpcl{i}}) + (p - \bsurp{\inpcl{i}})\bsurp{\inpcl{i}}\n \\
                   &\leq 2\optr_{\inpcl{i}} + 2\optb{\inpcl{i}}. \label{eq:bound-for-mm}
\end{align}

\subsection{Approximation Guarantee}\label{subsec:bal-analyze-4-cases-p-q}

\subsubsection{Approximation Guarantee in Cut-Cut Case} \label{subsec:p/q-cut-case}

In this section, we show that, in the cut-cut case, the clustering output by our algorithm has a distance of at most $6 \dist(\inpset, \mopq )$. In particular, we prove the following lemma.

\begin{lemma}\label{lem:main-lem-cut-cut}
    $ \dist(\inpset, \outpq) \leq 6 \dist(\inpset, \mopq) $, where $ \outpq $ is the output of $ \algopqgen $ in the cut-cut case (\cref{line:algopqgen-refer-algopqc} of~\cref{alg:algo-for-p/q-general}).
\end{lemma}

Observe that the clustering $ \outpq $ output by $ \algopqgen $ in the cut-cut case can be described as follows. Consider a cluster $ \inpcl{i}\in \inpset $. If $\inpcl{i}\in \rcut $, then in $ \outpq $, there will be $\rsurp{\inpcl{i}}$ red points split from $ \inpcl{i} $. If $ \inpcl{i}\in \rmerge $, then in $ \outpq $, there will be $\rdefi{\inpcl{i}}$ red points from other clusters merged to $ \inpcl{i} $. If there remain red points after merging the split red points from clusters in $ \rcut $ to fill the deficit of red points in $ \rmerge $, clusters of size $ q $ will be created from these extra points. Similarly, if $\inpcl{i}\in \bcut $, then in $ \outpq $, there will be $\bsurp{\inpcl{i}}$ blue points split from $ \inpcl{i} $. If $ \inpcl{i}\in \bmerge $, then in $ \outpq $, there will be $\bdefi{\inpcl{i}}$ blue points from other clusters merged to $ \inpcl{i} $. If there remain blue points after merging the split blue points from clusters in $ \bcut $ to fill the deficit of blue points of clusters in $ \bmerge $, clusters of size $ p $ will be created from these extra points. 

Based on these observations, to prove~\cref{lem:main-lem-cut-cut}, we decompose $ \dist(\inpset, \outpq) = \costone{\outpq} + \costtwo{\outpq} + \costthree{\outpq} + \costfour{\outpq} $, where each cost is defined as follows.
\begin{itemize}
    \item $ \costone{\outpq} $: the number of pairs $ (u,v) $ in which in $ \inpset $, $ u,v \in \inpcl{i} $, while in $ \outpq $, $ u $ is kept in $ \inpcl{i} $ and $ v $ is cut from $ \inpcl{i} $. In this case, $ \rsurp{\inpcl{i}} $ red points of every cluster $ \inpcl{i}\in \rcut $ will be split from $ \inpcl{i} $ and be merged to some clusters $ \inpcl{j}\in \rmerge $ by $ \algopqgen $, or be merged with some red points split from other clusters to create a cluster of size $q$ by $ \algopqc $. Similarly, $ \bsurp{\inpcl{i}} $ blue points of every cluster $ \inpcl{i}\in \bcut$ will be split from $ \inpcl{i} $ and be merged to some clusters $ \inpcl{j}\in \bmerge $ by $ \algopqgen $, or be merged with some blue points split from other clusters to create a cluster of size $p$ by $ \algopqc $. Hence
        {\small
        \begin{align}
            \costone{\outpq} &\leq \sum_{\inpcl{i} \in \rcut \cap \bcut} (\rsurp{\inpcl{i}} + \bsurp{\inpcl{i}}) (|\inpcl{i}| - (\rsurp{\inpcl{i}} + \bsurp{\inpcl{i}})) + \rsurp{\inpcl{i}} \bsurp{\inpcl{i}}\nonumber \\
            & + \sum_{\inpcl{i} \in \rcut \cap \bmerge} \rsurp{\inpcl{i}} (|\inpcl{i}| - \rsurp{\inpcl{i}})\nonumber\\
            & + \sum_{\inpcl{i} \in \rmerge \cap \bcut} \bsurp{\inpcl{i}} (|\inpcl{i}| - \bsurp{\inpcl{i}}). \label{equn:cut-case-first-pq}
        \end{align}
        }
        Note that, for $ \inpcl{i}\in \rcut\cap \bcut $, $ \algopqgen $ splits $ \rsurp{\inpcl{i}} $ red points and $ \bsurp{\inpcl{i}} $ blue points from $ \inpcl{i} $, and thus the cost $ (\rsurp{\inpcl{i}}+\bsurp{\inpcl{i}})(|\inpcl{i}| - \rsurp{\inpcl{i}} + \bsurp{\inpcl{i}}) $. In~\eqref{equn:cut-case-first-pq}, we add an additional term $ \rsurp{\inpcl{i}}\bsurp{\inpcl{i}} $, which upper bounds to case when $ \algopqgen $ may split $ \rsurp{\inpcl{i}} $ red points from $ \bsurp{\inpcl{i}} $ blue points. Points among these $ \rsurp{\inpcl{i}} $ may also be split from each other. The same situation happens for $ \bsurp{\inpcl{i}} $ blue points. That cost is taken into account by $ \costfour{\outpq} $.
    \item $ \costtwo{\outpq}: $ the number of pairs $ (u,v) $ in which in $ \inpset $, $ u\in \inpcl{i} $ and $ v\in \inpcl{j} $, while in $ \outpq $, $ v $ is split from $ \inpcl{j} $ and is merged to $ \inpcl{i} $, so that $ u $ and $ v $ are in the same cluster $ \inpcl{i} $. Observe that in the cut-cut case, every cluster $ \inpcl{i}\in \rmerge $ will be merged with exactly $ \rdefi{\inpcl{i}} $ red points split from some clusters $ \inpcl{j}\in \rcut $. Every cluster $ \inpcl{i}\in \bmerge $ will be merged with exactly $ \bdefi{\inpcl{i}} $ blue points split from some clusters $ \inpcl{i}\in \bcut $. Thus, we have
        {\small
        \begin{align} 
            \costtwo{\outpq}
            & \leq \sum_{\inpcl{i} \in \rcut \cap \bmerge} (p - \bsurp{\inpcl{i}}) (|\inpcl{i}| - \rsurp{\inpcl{i}}) \nonumber\\
            & + \sum_{\inpcl{i} \in \rmerge \cap \bcut} (q - \rsurp{\inpcl{i}}) (|\inpcl{i}| - \bsurp{\inpcl{i}}) \nonumber\\
            & + \sum_{\inpcl{i} \in \rmerge \cap \bmerge}  (q - \rsurp{\inpcl{i}}) |\inpcl{i}| + (p - \bsurp{\inpcl{i}}) |\inpcl{i}| + (q - \rsurp{\inpcl{i}})(p - \bsurp{\inpcl{i}})
            \label{equn:cut-case-second-pq}
        \end{align}
        }
    \item $ \costthree{\outpq} $: the number of pairs $ (u,v) $ in which in $ \inpset $, $ u\in \inpcl{j} $, and $ v\in \inpcl{k} $ where $ \inpcl{j} $ and $ \inpcl{k} $ are distinct clusters, while in $ \outpq $, $ u $ is split from $ \inpcl{j} $, $ v $ is split from $ \inpcl{k} $ and they are both merged to a cluster $ \inpcl{i} $. For each cluster $ \inpcl{i}\in \rmerge $, $ \algopqgen $ merges exactly $ \rdefi{\inpcl{i}} $ red points to it. Note, these $ \rdefi{\inpcl{i}} $ may come from different clusters from $ \rcut $. The same arguments hold for clusters in $ \bmerge $.

        Let $ \rcuta $ and $ \bcuta $ be the input of $ \algopqc $ given by $ \algopqgen $. Let $ \pclsr = \{ \pclr{1}, \pclr{2}, \dots, \pclr{\pcard_{r}}\} $ and $ \pclsb = \{ \pclb{1}, \pclb{2}, \dots, \pclb{\pcard_{b}} \} $ be the set of clusters of size $ q $ and of size $ p $, respectively, formed by $ \algopqc $. For each $ \pclr{i} $, represent $ \pclr{i} = Y^{r}_{i,1} \cup Y^{r}_{i,2} \cup \dots \cup Y^{r}_{i,h_{i}} $, where each $ Y^{r}_{i,t} $ is a set of red points coming from a cluster $ \inpcl{i_{t}}\in \rcuta $. Similarly, for each $ \pclb{i} $, represent $ \pclb{i} = Y^{b}_{i,1} \cup Y^{b}_{i,2} \cup \dots \cup Y^{r}_{i,k_{i}} $, where each $ Y^{b}_{i,t} $ is a set of blue points comming from a cluster $ \inpcl{i_{t}}\in \bcuta $. Employing the same arguments for~\cref{clm:cost 3} in the proof of~\cref{lem:cost 3 + 4 cut case}, we have
    \begin{align}
        \costthree{\outpq} &\leq && \sum_{\inpcl{i}\in \rcut \cap \bmerge} \dfrac{\bdefi{\inpcl{i}}^{2}}{2} + \sum_{\inpcl{i}\in \rmerge \cap \bcut} \dfrac{\rdefi{\inpcl{i}}^{2}}{2}  \nonumber \\
                           &\> + && \sum_{\inpcl{i}\in \rmerge \cap \bmerge} \dfrac{(\rdefi{\inpcl{i}} + \bdefi{\inpcl{i}})^{2}}{2} \nonumber \\
                           &\> + && \dfrac{1}{2}\sum_{\pclr{j}\in \pclsr}\left(\sum_{t=1}^{h_{i}}|Y^{r}_{j,t}|(q-|Y^{r}_{j,t}|)\right) + \dfrac{1}{2}\sum_{\pclb{j}\in \pclsb}\left(\sum_{t=1}^{k_{i}}|Y^{b}_{j,t}|(q-|Y^{b}_{j,t}|)\right). \label{eq:cost3-cut-cut-case}
    \end{align}
\item $ \costfour{\outpq} $: the number of pairs $ (u,v) $ in which $ u, v $ are among points split from a cluster $ \inpcl{i} $, but $ u $ and $ v $ are then separated from each other because $ u $ and $ v $ are merged to different clusters $ \inpcl{j} $ and $ \inpcl{k} $. 

    For each cluster $ \inpcl{i}\in \rcut $, denote by $ R_{i,1}, R_{i,2}, \dots, R_{i,r_{1}} $ the partition of $ \rsurp{\inpcl{i}} $ red points that are split from $ \inpcl{i} $, and each of these parts are merged to different clusters in $ \rmerge $ (by $ \algopqgen $), or be grouped with red points split from other clusters to create clusters of size $ q $ (by $ \algopqc $). Similarly, let $ B_{i,1}, B_{i,2},\dots, B_{i,b_{i}} $ denote the partition of $ \bsurp{i} $ blue points split from $ \inpcl{i}\in \bcut $.
 Using the same arguments for~\cref{clm:cost 4} in the proof of~\cref{lem:cost 3 + 4 cut case}, we have
    \begin{align}
        \costfour{\outpq} &\leq &&\sum_{\inpcl{i}\in \rcut\setminus \rcuta} \dfrac{\rsurp{\inpcl{i}}^{2}}{2} + \sum_{\inpcl{i}\in \rcuta }\left(\sum_{1\leq j<t \leq r_{i} }|R_{i,j}||R_{i,t}| \right) \nonumber \\
                          &\> + && \sum_{\inpcl{i}\in \bcut\setminus \bcuta}\dfrac{\bsurp{\inpcl{i}}^{2}}{2} + \sum_{\inpcl{i}\in \bcuta}\left(\sum_{1\leq j<t \leq b_{i}} |B_{i,j}||B_{i,t}| \right). \label{eq:cost-4-cut-cut-case}
    \end{align}
    Note that, for $ \inpcl{i}\in \rcut\cap \bcut $, $ \rsurp{\inpcl{i}}+\bsurp{\inpcl{i}} $ red and blue points are split from $ \inpcl{i} $. And then, the red part and blue part may be split from each other. The upper bound for this cost $ \rsurp{\inpcl{i}}\bsurp{\inpcl{i}} $ has been considered in $ \costone{\outpq} $.
\end{itemize}

\begin{claim}\label{clm:cost-one-plus-two-cut-case_pq}
    $\costone{\outpq} + \costtwo{\outpq} \leq 4 \, \dist(\inpset, \mopq)$.
\end{claim}


\begin{proof}
    

    Hence, from \cref{equn:cut-case-first-pq} and \cref{equn:cut-case-second-pq} we have:
    \begin{align}
        &\> \> \> \costone{\outpq} + \costtwo{\outpq} \nonumber \\
        &\leq \sum_{\inpcl{i} \in \rcut \cap \bcut} (\rsurp{\inpcl{i}} + \bsurp{\inpcl{i}}) (|\inpcl{i}| - (\rsurp{\inpcl{i}} + \bsurp{\inpcl{i}})) + \rsurp{\inpcl{i}} \bsurp{\inpcl{i}} \nonumber \\
        &\> + \sum_{\inpcl{i} \in \rcut \cap \bmerge} \rsurp{\inpcl{i}} (|\inpcl{i}| - \rsurp{\inpcl{i}}) + (p - \bsurp{\inpcl{i}}) (|\inpcl{i}| - \rsurp{\inpcl{i}}) \nonumber \\
        &\> + \sum_{\inpcl{i} \in \rmerge \cap \bcut} \bsurp{\inpcl{i}} (|\inpcl{i}| - \bsurp{\inpcl{i}}) + (q - \rsurp{\inpcl{i}}) (|\inpcl{i}| - \bsurp{\inpcl{i}}) \nonumber \\
        &\> +\sum_{\inpcl{i} \in \rmerge \cap \bmerge}  (q - \rsurp{\inpcl{i}}) |\inpcl{i}| + (p - \bsurp{\inpcl{i}}) |\inpcl{i}| + (q - \rsurp{\inpcl{i}})(p - \bsurp{\inpcl{i}}) \nonumber \\
        &= \sum_{\inpcl{i}\in \rcut \cap \bcut}\cc{\inpcl{i}} + \sum_{\inpcl{i}\in \rcut \cap \bmerge} \cm{\inpcl{i}} \nonumber \\
        &\> + \sum_{\inpcl{i} \in \rmerge \cap \bcut} \mc{\inpcl{i}} + \sum_{\inpcl{i}\in \rmerge \cap \bmerge} \mm{\inpcl{i}} \nonumber \\
        &\leq \sum_{\inpcl{i}\in \rcut \cap \bcut}(\optr_{\inpcl{i}} + \optb_{\inpcl{i}}) + \sum_{\inpcl{i}\in \rcut \cap \bmerge} (\optr_{\inpcl{i}}+2\optb_{\inpcl{i}}) \nonumber \\
        &\> + \sum_{\inpcl{i} \in \rmerge \cap \bcut} (2\optr_{\inpcl{i}}+\optb_{\inpcl{i}}) + \sum_{\inpcl{i}\in \rmerge \cap \bmerge} (2\optr_{\inpcl{i}} + 2\optb_{\inpcl{i}} ) \label{eq:cost1-cost2-bound1} \\
        &\leq 2\sum_{\inpcl{i}\in \inpset}\optr_{\inpcl{i}} + 2\sum_{\inpcl{i}\in \inpset}\optb_{\inpcl{i}}. \n \\
        &\leq 4\dist( \inpset, \mopq ). \n
    \end{align}

    Note that,~\eqref{eq:cost1-cost2-bound1} is obtained by using~\eqref{eq:bound-for-cc},~\eqref{eq:bound-for-cm},~\eqref{eq:bound-for-mc}, and~\eqref{eq:bound-for-mm}. The proof is complete.

\end{proof}

\begin{claim}\label{clm:cost-three-and-cost-four-pq}
    $\costthree{\outpq} + \costfour{\outpq} \leq 2 \, \, \dist(\inpset, \mopq)$.
\end{claim}
\begin{proof}
    Applying the same arguments for~\cref{clm:three half cost} in the proof of~\cref{lem:cost 3 + 4 cut case}, we have
    {\small
    \begin{align}
        \dfrac{1}{2}\sum_{\pclr{j}\in \pclsr}\left(\sum_{t=1}^{h_{i}}|Y^{r}_{j,t}|(q-|Y^{r}_{j,t}|)\right) + \sum_{\inpcl{i}\in \rcuta }\left(\sum_{1\leq j<t \leq r_{i} }|R_{i,j}||R_{i,t}| \right) \leq \dfrac{3}{2}\sum_{\inpcl{i}\in \rcuta}\dfrac{\rsurp{\inpcl{i}}(q-\rsurp{\inpcl{i}})}{2},\label{eq:boubd-q-cut-cut-case}
    \end{align}
    }
    and
    {\small
    \begin{align}
    \dfrac{1}{2}\sum_{\pclb{j}\in \pclsb}\left(\sum_{t=1}^{k_{i}}|Y^{b}_{j,t}|(q-|Y^{b}_{j,t}|)\right) + \sum_{\inpcl{i}\in \bcuta }\left(\sum_{1\leq j<t \leq b_{i} }|B_{i,j}||B_{i,t}| \right) \leq \dfrac{3}{2}\sum_{\inpcl{i}\in \bcuta}\dfrac{\bsurp{\inpcl{i}}(p-\bsurp{\inpcl{i}})}{2}.\label{eq:bound-p-cut-cut-case}
    \end{align}
    }
    Combining~\eqref{eq:cost3-cut-cut-case},~\eqref{eq:cost-4-cut-cut-case},~\eqref{eq:boubd-q-cut-cut-case}, and~\eqref{eq:bound-p-cut-cut-case}, it follows that
    \begin{align}
        &\> \> \costthree{\outpq} + \costfour{\outpq} \nonumber \\
        &\leq \sum_{\inpcl{i}\in \rcut \cap \bmerge} \dfrac{\bdefi{\inpcl{i}}^{2}}{2} + \sum_{\inpcl{i}\in \rmerge \cap \bcut} \dfrac{\rdefi{\inpcl{i}}^{2}}{2}  \nonumber \\
                           &\> + \sum_{\inpcl{i}\in \rmerge \cap \bmerge} \dfrac{(\rdefi{\inpcl{i}} + \bdefi{\inpcl{i}})^{2}}{2} \nonumber \\
                           &\> + \sum_{\inpcl{i}\in \rcut\setminus \rcuta} \dfrac{\rsurp{\inpcl{i}}^{2}}{2} + \sum_{\inpcl{i}\in \bcut\setminus \bcuta}\dfrac{\bsurp{\inpcl{i}}^{2}}{2} \nonumber \\
                           &\> + \dfrac{3}{2} \sum_{\inpcl{i}\in \rcuta}\dfrac{\rsurp{\inpcl{i}}(q-\rsurp{\inpcl{i}})}{2} + \dfrac{3}{2}\sum_{\inpcl{i}\in \bcuta}\dfrac{\bsurp{\inpcl{i}}(p-\bsurp{\inpcl{i}})}{2}. \nonumber \\
                           &\leq 2\left(\sum_{\substack{R\in \{\rcut,\rmerge\}\\ B\in \{\bcut, \merge\}}} \left(\sum_{\inpcl{i} \in R\cap B} \dfrac{\rsurp{\inpcl{i}}\rdefi{\inpcl{i}}+ \bsurp{\inpcl{i}}\bdefi{\inpcl{i}}}{2} \right)\right), \nonumber
    \end{align}
    where the last inequality is due to the following observations:
    \begin{itemize}
        \item If $ \inpcl{i}\in \rcut $, then $ \rsurp{\inpcl{i}}\rdefi{\inpcl{i}}\geq \rsurp{\inpcl{i}}^{2} $.
        \item If $ \inpcl{i}\in \rmerge $, then $ \rsurp{\inpcl{i}}\rdefi{\inpcl{i}}\geq \rdefi{\inpcl{i}}^{2} $.
        \item If $ \inpcl{i}\in \bcut $, then $ \bsurp{\inpcl{i}}\bdefi{\inpcl{i}}\geq \bsurp{\inpcl{i}}^{2} $.
        \item If $ \inpcl{i}\in \bmerge $, then $ \bsurp{\inpcl{i}}\bdefi{\inpcl{i}}\geq \bdefi{\inpcl{i}}^{2} $.
        \item If $ \inpcl{i}\in \rmerge \cap \bmerge $, then $ \rsurp{\inpcl{i}}\rdefi{\inpcl{i}}+ \bsurp{\inpcl{i}}\bdefi{\inpcl{i}}\geq \rdefi{\inpcl{i}}^{2} + \bdefi{\inpcl{i}}^{2}\geq \dfrac{(\rdefi{\inpcl{i}}+\bdefi{\inpcl{i}})^{2}}{2} $.
    \end{itemize}
    Now we utilize~\cref{lem:opt-greater-than-surp-times-defi}, which shows that
    \begin{align}
        \dist(\inpset, \mopq) \geq \sum_{\substack{R\in \{\rcut,\rmerge\}\\ B\in \{\bcut, \merge\}}} \left(\sum_{\inpcl{i} \in R\cap B} \dfrac{\rsurp{\inpcl{i}}\rdefi{\inpcl{i}}+ \bsurp{\inpcl{i}}\bdefi{\inpcl{i}}}{2} \right).\nonumber
    \end{align}
    It follows that $ \costthree{\outpq} + \costfour{\outpq} \leq 2\dist(\inpset, \mopq) $.
\end{proof}

Now we provide the proof of~\cref{lem:main-lem-cut-cut} 
\begin{proof}[Proof of~\cref{lem:main-lem-cut-cut}]
    Combining~\cref{clm:cost-one-plus-two-cut-case_pq} and~\cref{clm:cost-three-and-cost-four-pq}, it follows that
    \begin{align}
        \dist(\inpset, \mopq) &= \costone{\outpq} + \costtwo{\outpq} + \costthree{\outpq} + \costfour{\outpq} \nonumber \\
                              &\leq 4\dist(\inpset, \mopq) + 2\dist(\inpset,\mopq) = 6\dist(\inpset,\mopq). \nonumber
    \end{align}
\end{proof}


\subsubsection{Approximation Guarantee in Cut-Merge Case}

Cut-merge case (or Merge-Cut Case) occurs when the subroutine $\algopqcm$(\cref{alg:algo-for-p/q-cutmerge}) (or $\algopqmc$) is called. 

In this section, we provide the approximation guarantee of the algorithm $\algopqgen$ in the Cut-Merge case. The analysis of the Merge-Cut case is similar to the analysis of the Cut-Merge case only the roles of the red and blue points get reversed. So, in this analysis, we will focus solely on the Cut-Merge case, while the Merge-Cut case will be excluded from this discussion. We show that the cost paid by the algorithm $\algopqgen$ is at most $\mcf \, \, \dist(\inpset, \mopq)$ where $\mopq$ is the 
closest {\bal} of $\inpset$. More specifically,

\begin{lemma}\label{lem:main-cut-merge-case}
    $\dist(\inpset, \outpq) < \mcf \, \, \dist(\inpset, \mopq)$ where $\outpq$ is the output of the algorithm $\algopqgen$ in the cut-merge case (\cref{line:algopqgen-refer-algopqcm} of~\cref{alg:algo-for-p/q-general}).
\end{lemma}

Similarly, for the Merge-Cut case, formally, we want to prove the following.

\begin{lemma}\label{lem:main-merge-cut-case}
     $\dist(\inpset, \outpq) < \mcf \, \, \dist(\inpset, \mopq)$ where $\outpq$ is the output of the algorithm $ \algopqgen $ in the merge-cut case (\cref{line:algopqgen-refer-algopqmc} of~\cref{alg:algo-for-p/q-general}).
\end{lemma}

We analyze in the following way: first, we calculate the cost paid by our algorithm $\algopqgen$ in the Cut-Merge case, and then we bound that cost with the cost paid by the optimal clustering $\mopq$ in the Cut-Merge case.

\paragraph{1) Cost paid by $\algopqgen$ in Cut-Merge case}

 The algorithm $\algopqcm$(\cref{alg:algo-for-p/q-cutmerge}) pays the following set of disjoint costs

 \begin{itemize}
     \item From each cluster $\inpcl{i} \in \rcut$, $\algopqgen$ cuts the red surplus part $s_r(\inpcl{i})$ from $\inpcl{i}$. Hence, the cost paid for cutting these red surplus points is
     \begin{align}
         \costone{\inpcl{i}}^{\text{red}} = s_r(\inpcl{i}) (|\inpcl{i}| - s_r(\inpcl{i})) \label{equn:cost-paid-one}
     \end{align}

     \item For each cluster $\inpcl{j} \in \rmerge$, $\algopqgen$  merges the deficit amount $\rdefi{\inpcl{j}}$ to these clusters. Hence, the cost paid for merging the deficit to $\inpcl{j}$ is
     \begin{align}
         \costtwo{\inpcl{j}}^{\text{red}} = \rdefi{\inpcl{j}}|\inpcl{j}| \label{equn:cost-paid-two}
     \end{align}
     
     \item The $\rsurp {\inpcl{i}}$ red points that are cut from $\inpcl{i}$ can get merged with the red surplus, $\rsurp{\inpcl{k}}$ of some other cluster $\inpcl{k}$. We call the cost of merging $\rsurp{\inpcl{i}}$ red points with the red surplus of other clusters as $\costthree{D_i}^\text{red}$ (similar to $p:1$ case in \cref{sec:multiple-of-p-clustering}).
     \begin{align}
         \costthree{D_i}^{\text{red}} \leq \rsurp{\inpcl{i}} (q - \rsurp{\inpcl{i}}) \label{equn:cost-paid}
     \end{align}
     
     \item These $\rsurp{\inpcl{i}}$ red points from $\inpcl{i}$ can also further be split into several parts $W_1, W_2, \ldots, W_t$ (say). These parts of $\rsurp{\inpcl{i}}$ points belong to different clusters in $\m{Q}$ (output of $\algopqgen$) and thus would incur some cost. We call this cost as $\costfour{D_i}^\text{red}$ (similar to $p:1$ case in \cref{sec:multiple-of-p-clustering}).
     \begin{align}
         \costfour{\inpcl{i}}^{\text{red}} =  \frac{1}{2} \sum_{j = 1}^t |W_j|(\rsurp{\inpcl{i}} - |W_j|)\label{equn:cost-paid-four}
     \end{align}

     \item Now, from the blue side, we may cut multiple subsets of size $p$ and a single subset of size $\bsurp{D_i}$. Let us assume $W_{i,z}$ denotes the $z$th such subset of the set $\blue{\inpcl{i}}$. $y_{i,z}$ takes the value $1$ if we cut $z$th such subset from $\inpcl{i}$. The cost of cutting $z$th such subset from $\inpcl{i}$ is given as
     \begin{align}
         &\kappa_0(\inpcl{i}) = \bsurp{\inpcl{i}} (|\inpcl{i}| - \rsurp{\inpcl{i}}) \n \\
         &(\text{cost of cutting the $0$th subset})\n \\
         &\kappa_z(\inpcl{i}) = p (|\inpcl{i}| - \rsurp{\inpcl{i}} - zp) \n \\
         &(\text{cost of cutting the $z$th subset for $z \geq 1$})\n 
    \end{align}
    and thus, we define 
    \begin{align}
         &\costone{\inpcl{i}}^{\text{blue}} = \sum_{D_i \in \inpset}\sum_{z = 0}^t y_{i,z} \kappa_z(\inpcl{i}) \, \, (\text{refer} \, \, \cref{eq:cost-one-merge}) \label{equn:cost-paid-four-prime} \\
         &\left( \text{t} = \frac{\blue{\inpcl{i}} - \rsurp{\inpcl{i}} - \bsurp{\inpcl{i}}}{p} \right) \n
     \end{align}
     \item Suppose the algorithm $\algopqgen$ merges at a cluster $D_j \in \bmerge$. Then, the cost paid for merging the deficit to $D_j$ is
    \begin{align}
        \costtwo{\inpcl{j}}^{\text{blue}} = \bdefi{\inpcl{j}} |\inpcl{j}| \, \, (\text{refer} \cref{eq:cost-two-merge}) \label{equn:cost-paid-five}
    \end{align}
     \item The $|W_{i,z}|$ blue points that are cut from $\inpcl{i}$ can get merged with the blue subsets $|W_{j,z'}|$ of some other cluster $\inpcl{j}$. We call the cost of merging a blue subset $W_{i,z}$ of $\inpcl{i}$ with the blue subset $W_{j,z'}$ of another cluster $D_j$ as $\costthree{D_i}^\text{blue}$ (similar to $p:1$ case refer \cref{eq:cost-three-merge}).
     \begin{align}
         \costthree{D_i}^{\text{blue}} \leq  |W_{i,z}|(p - |W_{i,z}|) \label{equn:cost-paid-six}
     \end{align}
     
     \item The $|W_{i,z}|$ blue points that are cut from $\inpcl{i}$ can also further be split into several parts $W_1, W_2, \ldots, W_t$ (say). These parts of $W_{i,z}$ can belong to different clusters in $\m{Q}$ (output of $\algopqgen$) and thus would incur some cost. We call this cost as $\costfour{D_i}^\text{blue}$ (similar to $p:1$ case refer \cref{eq:cost-four-merge}).
     \begin{align}
         \costfour{\inpcl{i}}^{\text{blue}} =  \frac{1}{2} \sum_{j = 1}^t |W_j|(|W_{i,z}| - |W_j|)
     \end{align}

     \item Suppose for a cluster $D_j \in \bmerge \cap \rmerge$, both the blue deficit and the red deficit are filled up by the subsets of some other clusters $D_k$ and $D_\ell$ in $\inpset$ respectively such that $k \neq \ell$ then this would incur some cost and is denoted by
     \begin{align}
         \costfive{\inpcl{j}}^{\text{blue}} = (p - \bsurp{D_j})(q - \rsurp{D_j})
     \end{align}
 \end{itemize}

 Now, let us define the total cost paid by the algorithm over all clusters $\inpcl{i} \in \inpset$.

 \begin{itemize}
     \item $\costone{\m{Q}}^{\text{red}} = \sum_{\inpcl{i} \in \rcut} \costone{\inpcl{i}}^{\text{red}}$
     \item $\costtwo{\m{Q}}^{\text{red}} = \sum_{\inpcl{i} \in \rmerge} \costtwo{\inpcl{i}}^{\text{red}}$
     \item $\costthree{\m{Q}}^{\text{red}} = \sum_{\inpcl{i} \in \inpset} \costthree{\inpcl{i}}^{\text{red}}$
     \item $\costfour{\m{Q}}^{\text{red}} = \sum_{\inpcl{i} \in \inpset} \costfour{\inpcl{i}}^{\text{red}}$
     \item $\costone{\m{Q}}^{\text{blue}} = \sum_{\inpcl{i} \in \inpset} \costone{\inpcl{i}}^{\text{blue}}$
     \item $\costtwo{\m{Q}}^{\text{blue}} = \sum_{\inpcl{i} \in \inpset} \costtwo{\inpcl{i}}^{\text{blue}}$
     \item $\costthree{\m{Q}}^{\text{blue}} = \sum_{\inpcl{i} \in \inpset} \costthree{\inpcl{i}}^{\text{blue}}$
     \item $\costfour{\m{Q}}^{\text{blue}} = \sum_{\inpcl{i} \in \inpset} \costfour{\inpcl{i}}^{\text{blue}}$
     \item $\costfive{\m{Q}}^{\text{blue}} = \sum_{\inpcl{i} \in \inpset} \costfive{\inpcl{i}}^{\text{blue}}$
 \end{itemize}

\paragraph{2) Cost paid by optimal clustering $\mopq$ in cut-merge case}

\begin{claim}
    $\costone{\m{Q}}^\text{Red} + \costtwo{\m{Q}}^\text{Red} + \costthree{\m{Q}}^\text{Red} + \costfour{\m{Q}}^\text{Red} \leq 3.5 \dist(\inpset, \mopq)$.
\end{claim}

\begin{proof}
    Recall in $p:1$ case, where $\m{T}$ is the output of $\algog$(\cref{alg:algo-for-general}) we have

    \begin{align}
        \costone{\m{T}} &= \sum_{\inpcl{i} \in \cut} s_i (|\inpcl{i}| - s_i) && (\text{Refer to} \cref{equn:cost-one} ) \n \\
        &= \costone{\m{Q}}^\text{Red} && (\text{Replacing} \, \, s_i \, \, \text{with} \, \, \rsurp{\inpcl{i}})
    \end{align}
    and again,

    \begin{align}
        \costtwo{\m{T}} &= \sum_{\inpcl{i} \in \merge} \defi{D_i} |\inpcl{i}| && (\text{Refer to} \cref{equn:cost-two} ) \n \\
        &= \costone{\m{Q}}^\text{Red} && (\text{Replacing} \, \, d(D_i) \, \, \text{with} \, \, \rdefi{D_i})
    \end{align}

    Now since
    
    \begin{align}
        \optr_{\inpcl{i}} &\geq \rsurp{\inpcl{i}} (|\inpcl{i}| - \rsurp{\inpcl{i}}) + \frac{1}{2} \rsurp{\inpcl{i}} \defi{D_i} && (\text{from} \cref{equn:cut-red})
    \end{align}

    Hence, we get
    \begin{align}
        \costone{\m{Q}}^\text{Red} + \costtwo{\m{Q}}^\text{red} &\leq \sum_{\inpcl{i} \in \rcut}\optr_{\inpcl{i}} + 2 \, \, \sum_{\inpcl{i} \in \rmerge}\optr_{\inpcl{i}} \n \\
        & \leq 2 \, \, \dist(\inpset, \mopq) \label{equn:claim-one-one}
    \end{align}
    Since the way we distribute the red surplus part into the clusters of size $q$ in $\algopqcm$(\cref{alg:algo-for-p/q-cutmerge}) is similar to the way we distribute the surplus part in algorithm $\algoc$(\cref{alg:algo-for-cut}) in $p:1$ case in \cref{sec:multiple-of-p-clustering} we have

    \begin{align}
        &\costthree{\m{T}} = \costthree{\outpq}^{\text{red}}, \text{and} \n \\
        &\costfour{\m{T}} = \costfour{\outpq}^{\text{red}} \n
    \end{align}

    Hence, we can use the same proof of \cref{lem:cost 3 + 4 cut case} to get the following bound.

    \begin{align}
        \costthree{\m{Q}}^\text{Red} + \costthree{\m{Q}}^\text{Red} \leq 1.5 \, \, \dist(\inpset, \mopq) \label{equn:claim-one-two}
    \end{align}
    Now, from \cref{equn:claim-one-one} and \cref{equn:claim-one-two} we get,

    \begin{align}
       \costone{\m{Q}}^\text{Red} + \costone{\m{Q}}^\text{Red} + \costthree{\m{Q}}^\text{Red} + \costfour{\m{Q}}^\text{Red} \leq 3.5 \, \, \dist(\inpset, \mopq) 
    \end{align}
\end{proof}

\begin{claim}\label{clm:cut-merge-blue}
    $\costone{\m{Q}}^\text{blue} + \costtwo{\m{Q}}^\text{blue} + \costthree{\m{Q}}^\text{blue} + \costfour{\m{Q}}^\text{blue} \leq 3 \, \, \opt$.
\end{claim}

\begin{proof}
    Let $y_{i,z}$ take the value $1$ if we cut the $z$th blue subset of size $p$ from a cluster $\inpcl{i}$.

    Now, let us compare the following costs

    In $p:q$ case (\cref{sec:multiple-of-p-q-clustering}) the cost of cutting the $w_{i,z}$ subset from a cluster $\inpcl{i}$ for $z \geq 1$ is given by, 
    
    \[
        \costone{\inpcl{i}}^\text{blue} \leq p(|\inpcl{i}| - \rsurp{\inpcl{i}} - \bsurp{\inpcl{i}} - zp)
    \]

    whereas in $p:1$ case (\cref{sec:multiple-of-p-clustering}) the cost of cutting the $W_{i,z}$ part from a cluster $\inpcl{i}$ for $z \geq 1$ is given by, 
    
    \[
        \costone{\inpcl{i}} \leq p(|\inpcl{i}| - \bsurp{\inpcl{i}} - zp)
    \]

    Thus,

    \begin{align}
        &\costone{\inpcl{i}}^\text{blue} \leq \costone{\inpcl{i}} \n \\
        \implies &\sum_{\inpcl{i}} \costone{\inpcl{i}}^\text{blue} \leq \sum_{\inpcl{i}} \costone{\inpcl{i}} \n \\
        \implies &\costone{\m{Q}}^\text{blue} \leq \costone{\m{T}} && (\text{where $\m{T}$ is the output of $\algog$(\cref{alg:algo-for-general})} ) \n
    \end{align}
    Now similarly, in $p:q$ case (\cref{sec:multiple-of-p-q-clustering}) we have

    \begin{align}
        \costtwo{\inpcl{i}}^\text{blue} = (p - \rsurp{\inpcl{i}}) (|\inpcl{i}| - \rsurp{\inpcl{i}})
    \end{align}

    whereas in $p:1$ case (\cref{sec:multiple-of-p-clustering}) we have

    \begin{align}
        \costtwo{\inpcl{i}} = (p - s(\inpcl{i})) |\inpcl{i}|
    \end{align}

    Hence,

    \begin{align}
        &\costtwo{\inpcl{i}}^\text{blue} \leq \costtwo{\inpcl{i}} \n \\
        \implies &\sum_{\inpcl{i}} \costtwo{\inpcl{i}}^\text{blue} \leq \sum_{\inpcl{i}} \costtwo{\inpcl{i}} \n \\
        \implies &\costtwo{\m{Q}}^\text{blue} \leq \costtwo{\m{T}} && (\text{where $\m{T}$ is the output of $\algog$(\cref{alg:algo-for-general})} ) \n
    \end{align}

    Thus,

    \begin{align}
        \costone{\m{Q}}^\text{blue} + \costtwo{\m{Q}}^\text{blue} \leq \costone{\m{T}} + \costtwo{\m{T}}.
    \end{align}

    Now, since when $\bsurp{\inpcl{i}} \geq p/2$ we have

    \begin{align}
        \optb_{\inpcl{i}} \geq \defi{D_i} (|\inpcl{i}| - \bsurp{\inpcl{i}}) + \frac{1}{2} \bsurp{\inpcl{i}}\defi{D_i} && (\text{from} \cref{equn:merge-blue})
    \end{align}

    we can use the same proof as in \cref{clm:merge-case-lp-bound} to show

    \begin{align}
        \costone{\m{Q}}^\text{blue} + \costtwo{\m{Q}}^\text{blue} \leq \dist(\inpset, \mopq) \label{equn:claim-one-three}
    \end{align}

    Since the way we merge the blue surplus part from a cluster $\inpcl{i}$ to another set of clusters $\inpcl{j}$ in $\algopqcm$(\cref{alg:algo-for-p/q-cutmerge}) is the same as the way we merge the surplus part in algorithm $\algom$(\cref{alg:algo-for-merge}) in $p:1$ case (\cref{sec:multiple-of-p-clustering}). Thus we have

    \begin{align}
        &\costthree{\m{T}} = \costthree{\outpq}^{\text{blue}}, \text{ and} \n \\
        &\costfour{\m{T}} = \costfour{\outpq}^{\text{blue}} \n
    \end{align}

    Hence, we can use the same proof of \cref{clm:cost-three-and-cost-four} to get the following bound.

    \begin{align}
        \costthree{\m{Q}}^\text{blue} + \costfour{\m{Q}}^\text{blue} \leq 2 \, \, \dist(\inpset, \mopq) \label{equn:claim-one-four}
    \end{align}
    Now, from \cref{equn:claim-one-three} and \cref{equn:claim-one-four} we get,

    \begin{align}
       \costone{\m{Q}}^\text{blue} + \costone{\m{Q}}^\text{blue} + \costthree{\m{Q}}^\text{Red} + \costfour{\m{Q}}^\text{Red} \leq 3 \, \, \dist(\inpset, \mopq) 
    \end{align}

\end{proof}

\begin{claim} \label{clm:new-cost}
    $\costfive{\m{Q}}^\text{blue} \leq \dist(\inpset, \mopq)$
\end{claim}

\begin{proof}
    \begin{align}
        \costfive{\m{Q}}^\text{blue} &= \sum_{\inpcl{i}} (q - \rsurp{\inpcl{i}}) (p - \bsurp{\inpcl{i}}) \n \\
        &\leq \sum_{\inpcl{i}} \, (p - \bsurp{\inpcl{i}}) \rsurp{\inpcl{i}} && (\textbf{as} \, \, \rsurp{\inpcl{i}} > \frac{q}{2}) \n \\
        &\leq \sum_{\inpcl{i}} \, (p - \bsurp{\inpcl{i}}) (|\inpcl{i}| - \bsurp{\inpcl{i}})\n \\
        &\leq \sum_{\inpcl{i}} \, \opt_{\inpcl{i}}^b && (\textbf{from} \cref{equn:merge-blue}) \n \\
        &= \dist(\inpset, \mopq) \n
    \end{align}
    
\end{proof}

Now we are ready to prove \cref{lem:main-cut-merge-case}

\begin{proof}[Proof of \cref{lem:main-cut-merge-case}]
    \begin{align}
        \dist(\inpset, \outpq) &= \costone{\m{Q}}^\text{Red} + \costthree{\m{Q}}^\text{Red} + \costfour{\m{Q}}^\text{Red} \n \\
        &+ \costone{\m{Q}}^\text{blue} + \costone{\m{Q}}^\text{blue} + \costthree{\m{Q}}^\text{blue} + \costfour{\m{Q}}^\text{blue} + \costfive{\m{Q}}^\text{blue} \n \\
        &\leq 3.5 \, \, \dist(\inpset, \mopq) + 3 \, \, \dist(\inpset, \mopq) + \dist(\inpset, \mopq)\n \\
        &= 7.5 \, \, \dist(\inpset, \mopq) \n
    \end{align}
\end{proof}

\subsubsection{Approximation Guarantee in Merge-Merge Case}

In this section, we do the analysis for the algorithm $\algopqmm$(\cref{alg:algo-for-merge-merge}). We show that the cost paid by the algorithm $\algopqmm$(\cref{alg:algo-for-merge-merge}) is at most $\mmf \dist(\inpset, \mopq)$ where $\mopq$ is the closest $p$-blue $q$-red clustering to $D$. More specifically,

\begin{lemma}\label{lem:main-merge-merge-case}
    $\dist(\inpset, \outpq) < \mmf \, \, \dist(\inpset, \mopq)$ where $\outpq$ is the output of the algorithm $\algopqgen$ in the merge-merge case (\cref{line:algopqgen-refer-algopqm} of~\cref{alg:algo-for-p/q-general}).
\end{lemma}

We analyze in the following way: first, we calculate the cost paid by our algorithm \\$\algopqmm$(\cref{alg:algo-for-merge-merge}), and then we calculate the cost paid by the optimal clustering $\mopq$. After that, we bound the cost paid by the algorithm $\algopqmm$(\cref{alg:algo-for-merge-merge}) with the cost paid by $\mopq$.

\paragraph{1) Cost paid by $\algopqmm$(\cref{alg:algo-for-merge-merge})}

 The algorithm $\algopqmm$(\cref{alg:algo-for-merge-merge}) pays the following set of disjoint costs

 \begin{itemize}

    \item From the red side, we cut multiple subsets of size $q$. Let us assume $W_{i,z,r}$ denotes the $z$th red subset of the cluster $\inpcl{i}$. $y_{i,z,r}$ takes the value $1$ if we cut $z$th red subset from $\inpcl{i}$. The cost of cutting $z$th subset from $\inpcl{i}$ is given as
     \begin{align}
         &\kappa_0(\inpcl{i}) = \rsurp{\inpcl{i}} (|\inpcl{i}| - \rsurp{\inpcl{i}}) \n \\
         &\kappa_z(\inpcl{i}) = q (|\inpcl{i}| - \rsurp{\inpcl{i}} - zp) \n \\
         &\costone{\inpcl{i}}^{\text{red}} = \sum_{z = 0}^t y_{i,z,r} \kappa_z(\inpcl{i}) \label{equn:cost-merge-merge-one} \\
         &\left( \text{t} = \frac{\red{\inpcl{i}} - \rsurp{\inpcl{i}}}{q} \right)
     \end{align}
     \item Each of these red subsets $W_{i,z,r} \in \inpcl{i}$ is merged with red subsets of some other clusters $W_{j,z',r} \in \inpcl{j}$. This we again indicate by $\costthree{\inpcl{i}}^{\text{red}}$.
    \item Each of these partitions $W_{i,z,r} \in \inpcl{i}$ is further partitioned into several parts $W_1, W_2, \ldots, W_t$ (say) and merged into different clusters. The pairs between these several parts $W_j$s is indicated by $\costfour{\inpcl{i}}^{\text{red}}$ 
    \item Suppose we merge in a cluster $\inpcl{j} \in \rmerge$. The merge cost paid is denoted by
    \begin{align}
        \costtwo{\inpcl{j}}^{\text{red}} = (q - \rsurp{\inpcl{j}}) |\inpcl{j}| \label{equn:cost-merge-merge-two}
    \end{align}
     
     \item Now, from the blue side, again, we cut multiple subsets of size $p$. Let us assume $W_{i,z,b}$ denotes the $z$th blue subset of the cluster $\inpcl{i}$. $y_{i,z,b}$ takes the value $1$ if we cut $z$th blue subset from $\inpcl{i}$. The cost of cutting $z$th subset from $\inpcl{i}$ is given as
     \begin{align}
         &\kappa_0(\inpcl{i}) = \bsurp{\inpcl{i}} (|\inpcl{i}| - \bsurp{\inpcl{i}}) \n \\
         &\kappa_z(\inpcl{i}) = p (|\inpcl{i}| - \bsurp{\inpcl{i}} - zp) \n \\
         &\costone{\inpcl{i}}^{\text{blue}} = \sum_{z = 0}^t y_{i,z} \kappa_z(\inpcl{i}) \label{equn:cost-merge-merge-three} \\
         &\left( \text{t} = \frac{\blue{\inpcl{i}} - \bsurp{\inpcl{i}}}{p} \right)
     \end{align}
     \item Each of these partitions $W_{i,z,b} \in \inpcl{i}$ is merged with partitions of some other clusters $W_{j,z'} \in \inpcl{j}$. This we again indicate by $\costthree{\inpcl{i}}^{\text{blue}}$ 
    \item Each of these partitions $W_{i,z,b} \in \inpcl{i}$ is further partitioned into several parts $W_1, W_2, \ldots, W_t$(say) and merged into different clusters. The pairs between these several parts $W_j$s is indicated by $\costfour{\inpcl{i}}^{\text{blue}}$ 
    \item Suppose we merge in a cluster $\inpcl{j} \in \bmerge$. The merge cost paid is denoted by
    \begin{align}
        \costtwo{\inpcl{j}}^{\text{blue}} = (p - \bsurp{\inpcl{j}}) |\inpcl{j}| \label{equn:cost-merge-merge-four}
    \end{align}
    \item Suppose, the blue deficit merged to the cluster $\inpcl{j}$ comes from the cluster $\inpcl{i}$ and the red deficit merged to the cluster $\inpcl{j}$ comes from a cluster $D_k$ such that $D_k \neq \inpcl{j}$. This cost is denoted by
    \begin{align}
        \costfive{\inpcl{j}}^{\text{blue}} = (p - \bsurp{\inpcl{j}}) (q - \rsurp{\inpcl{j}})
    \end{align}
 \end{itemize}

 One thing to note is that the algorithm $\algopqmm$ may pay the cost for the pairs in $W_{i,z,r} \times W_{i,z',b}$,  that is in $\dist(\m{D}, \m{Q})$ (recall $\m{Q}$ is the output of $\algopqmm$) a pair $(u,v) \in W_{i,z,r} \times W_{i,z',b}$ may be counted because the subsets $W_{i,z,r}$ and $W_{i,z',b}$ may get merged into different clusters $D_j$ and $D_k$ in $\m{D}$ where $j \neq k$. We are counting these pairs in $\costone{D_i}^\text{red}$.

 Now, let us define the total cost paid by the algorithm over all clusters $\inpcl{i} \in \text{Cut-Merge}$.

 \begin{itemize}
     \item $\costone{\m{Q}}^{\text{red}} = \sum_{\inpcl{i} \in \inpset} \costone{\inpcl{i}}^{\text{red}}$
     \item $\costtwo{\m{Q}}^{\text{red}} = \sum_{\inpcl{i} \in \inpset} \costtwo{\inpcl{i}}^{\text{red}}$
     \item $\costthree{\m{Q}}^{\text{red}} = \sum_{\inpcl{i} \in \inpset} \costthree{\inpcl{i}}^{\text{red}}$
     \item $\costfour{\m{Q}}^{\text{red}} = \sum_{\inpcl{i} \in \inpset} \costfour{\inpcl{i}}^{\text{red}}$
     \item $\costone{\m{Q}}^{\text{blue}} = \sum_{\inpcl{i} \in \inpset} \costone{\inpcl{i}}^{\text{blue}}$
     \item $\costone{\m{Q}}^{\text{blue}} = \sum_{\inpcl{i} \in \inpset} \costtwo{\inpcl{i}}^{\text{blue}}$
     \item $\costthree{\m{Q}}^{\text{blue}} = \sum_{\inpcl{i} \in \inpset} \costthree{\inpcl{i}}^{\text{blue}}$
     \item $\costfour{\m{Q}}^{\text{blue}} = \sum_{\inpcl{i} \in \inpset} \costfour{\inpcl{i}}^{\text{blue}}$
     \item $\costfive{\m{Q}}^{\text{blue}} = \sum_{\inpcl{i} \in \inpset} \costfive{\inpcl{i}}^{\text{blue}}$
 \end{itemize}

\paragraph{2) Cost paid by optimal clustering $\mopq$ in merge-merge case}

\begin{claim}
    $\costone{\m{Q}}^\text{red} + \costone{\m{Q}}^\text{red} + \costthree{\m{Q}}^\text{red} + \costfour{\m{Q}}^\text{red} \leq 3 \, \, \dist(\inpset, \mopq)$.
\end{claim}

\begin{claim}
    $\costone{\m{Q}}^\text{blue} + \costone{\m{Q}}^\text{blue} + \costthree{\m{Q}}^\text{blue} + \costfour{\m{Q}}^\text{blue} \leq 3 \, \, \dist(\inpset, \mopq)$.
\end{claim}

The proof of the previous two claims is the same as the proof given in \cref{clm:cut-merge-blue}. So, we are skipping the proof of the above claims.

We also need to prove the following so that we can proceed to prove the main lemma of this section \cref{lem:main-merge-merge-case}.

\begin{claim}
    $\costfive{\m{Q}}^\text{blue} \leq \dist(\inpset, \mopq)$.
\end{claim}

The proof of the above claim is similar to \cref{clm:new-cost}.

Now we are ready to prove \cref{lem:main-merge-merge-case}.

\begin{proof}[Proof of \cref{lem:main-merge-merge-case}]
    \begin{align}
        \dist(\inpset, \outpq) &= \costone{\m{Q}}^\text{red} + \costone{\m{Q}}^\text{red} + \costthree{\m{Q}}^\text{red} + \costfour{\m{Q}}^\text{red} \n \\
        &+ \costone{\m{Q}}^\text{blue} + \costone{\m{Q}}^\text{blue} + \costthree{\m{Q}}^\text{blue} + \costfour{\m{Q}}^\text{blue} + \costfive{\m{Q}}^\text{blue} \n \\
        &\leq 3 \, \, \dist(\inpset, \mopq) + 3 \, \, \dist(\inpset, \mopq) + \dist(\inpset,\mopq) \n \\
        &= \mmf \, \, \dist(\inpset, \mopq). \n
    \end{align}
\end{proof}

\subsection{Overall Approximation}\label{subsec:bal-overall-approx}
We provide the proof of~\cref{thm:main-multiple-of-pq}.
\begin{proof}[Proof of~\cref{thm:main-multiple-of-pq}]
    Let $ \outpq $ be the output of the algorithm $ \algopqgen $ (\cref{alg:algo-for-p/q-general}). We have four following cases. 

    If $ \outpq $ is the output in the cut-cut case (\cref{line:algopqgen-refer-algopqc} of~\cref{alg:algo-for-p/q-general}), then by~\cref{lem:main-lem-cut-cut}, we have $ \dist(\inpset, \outpq) \leq 6 \, \, \dist(\inpset, \mopq) $.
    
    If $ \outpq $ is the output in the cut-merge case (\cref{line:algopqgen-refer-algopqcm} of~\cref{alg:algo-for-p/q-general}), then by~\cref{lem:main-cut-merge-case}, we have $ \dist(\inpset,\outpq) \leq \mcf \dist(\inpset, \mopq) $.

    If $ \outpq $ is the output in the cut-merge case (\cref{line:algopqgen-refer-algopqmc} of~\cref{alg:algo-for-p/q-general}), then by~\cref{lem:main-merge-cut-case}, we have $ \dist(\inpset,\outpq) \leq \mcf \dist(\inpset, \mopq) $.

    If $ \outpq $ is the output in the merge-merge case (\cref{line:algopqgen-refer-algopqm} of~\cref{alg:algo-for-p/q-general}), then by~\cref{lem:main-merge-merge-case}, we have $ \dist(\inpset,\outpq) \leq \mmf \dist(\inpset, \mopq) $.

    Therefore, in all cases, we have
    \begin{align}
        \dist(\inpset, \outpq) \leq 7.5 \dist(\inpset, \mopq). \nonumber
    \end{align}
\end{proof}

\section{Making the Balanced Clustering Fair}\label{sec:bal-to-fair}
Given a clustering $ \out = \{ \outcl{1}, \outcl{2}, \dots, \outcl{\gamma}\} $ such that in each $ \outcl{i} $, the number of blue points $ |\blue{\outcl{i}}| $ is a multiple $ p $, and the number of red points $ |\red{\outcl{i}}| $ is a multiple of $ q $, our goal is to find a closest fair clustering $ \fairmop $ to $ \out $. Recall that a fair clustering is a clustering such that the number of blue points in any cluster is $ \ratio=\dfrac{p}{q} $ times the number of red points. Our main result in this section is the following theorem.

\makepclusterfair*

Throughout this section, we assume $ p>q $, and hence, $ \ratio > 1 $. The case $ q> p $ can be handled similarly by switching the role of blue and red points.

In~\cref{sec:algo.make.pcluster.fair}, we introduce~\cref{alg:algo-mf}, which takes $ \out $ as its input, and outputs a fair cluster $ \fairset $. We analyze in~\cref{sec:analyze.make.pcluster.fair} that $ \dist(\out, \fairset) \leq 3\dist(\out, \fairmop) $.

\subsection{Algorithms}\label{sec:algo.make.pcluster.fair}
The high-level of~\cref{alg:algo-mf} is the following. The $ \mopqdef \ \out $ is partitioned into two sets. Each cluster $ \outcl{i} $ is chosen to be in one of the two sets based on whether $ |\blue{\outcl{i}}| < \ratio |\red{\outcl{i}}| $ or $ |\blue{\outcl{i}}| > \ratio |\red{\outcl{i}}|  $. If it is the former case, $ \outcl{i} $ is placed into the first set $ \tr $, while if it is the latter case, $ \outcl{i} $ is placed into the second set $ \tb $. After that, a set of $ |\red{\outcl{i}}| - \dfrac{|\blue{\outcl{i}}|}{\ratio} $ red points are cut from each cluster $ \outcl{i} \in \tr $. These points are merged to clusters in the set $ \tb $. Each cluster $ \outcl{j} \in \tb $ are merged with exactly $ \dfrac{|\blue{\outcl{j}|}}{\ratio} - |\red{\outcl{j}}| $ red points.

\textbf{Runtime Analysis of $\algmf$}: First note, each point is used at most twice during the execution of the algorithm $\algmf$ -- once when we cut the red point $v \in V$ from the cluster in the set $\tr$ and again when we merge that red point $v \in V$ to cluster that belongs to the set $\tb$. Hence, the total runtime of the algorithm $\algmf$ would be $O(|V|)$.

\begin{algorithm2e}[htbp]
\DontPrintSemicolon
\caption{$\algmf(\out)$}\label{alg:algo-mf}
\KwData {a $ \mopqdef \ \out $.}
\KwResult{a fair clustering $ \fairset $.} 
    $ \tb, \tr, \fairset \gets \emptyset$ \;
    \For{$ \outcl{i} \in \out $}{
        \If{$ |\blue{\outcl{i}}| < \ratio|\red{\outcl{i}}| $}{
            add $ \outcl{i} $ to the set $ \tr $ \;
    } \ElseIf {$ |\blue{\outcl{i}}| > \ratio|\red{\outcl{i}}| $}{
            add $ \outcl{i} $ to the set $ \tb $ \;
        }
    }

        \For{$ \outcl{i} \in \tr $}{
            \For{$ \outcl{j} \in \tb $}{
                Let $ k = \min(|\red{\outcl{i}}|-\dfrac{|\blue{\outcl{i}}|}{\ratio}, \dfrac{|\blue{\outcl{j}}|}{\ratio} - |\red{\outcl{j}}|) $ \;
                Move $ k $ red points from $ \outcl{i} $ to $ \outcl{j} $ \;
                \If{$ k = |\red{\outcl{i}}|-\dfrac{|\blue{\outcl{i}}|}{\ratio} $}{
                    $ \tr = \tr\setminus \{\outcl{i}\} $\;
                    $ \fairset = \fairset \cup \{\outcl{i}\} $ \;
                }
                \If{$ k = \dfrac{|\blue{\outcl{j}}|}{\ratio} - |\red{\outcl{j}}| $}{
                    $ \tb = \tb\setminus \{\outcl{j}\} $ \;
                    $ \fairset = \fairset \cup \{\outcl{j}\} $ \;
                }
            }
        }
    \Return{$ \fairset $}
\end{algorithm2e}

\subsection{Analysis}\label{sec:analyze.make.pcluster.fair}
In this section, we aim to prove~\cref{thm:make.p.cluster.fair}.

Consider a cluster $ \outcl{i}\in \out $. Suppose that in $ \fairmop $, $ \outcl{i} $ is spread through clusters $ X_{i,1} \cup B_{i,1} \cup R_{i,1}, X_{i,2} \cup B_{i,2} \cup R_{i,2}, \dots, X_{i,t} \cup B_{i,t} \cup R_{i,t} $, such that $ \outcl{i} = \cup_{j=1}^{t}X_{i,j} $, and $ B_{i,j}, R_{i,j} $ are the set of blue points and red points from some clusters $ \outcl{k} $ ($ k \neq i $), respectively. Then the distance between the two clusterings $ \fairmop $ and $ \out $ is
\begin{align}
    \dist(\out, \fairmop) &= \sum_{\outcl{i}\in \out} \left(\sum_{j<k} |X_{i,j}||X_{i,k}| + \sum_{j=1}^{t}\dfrac{|X_{i,j}|(|B_{i,j}|+ |R_{i,j}|)}{2} \right) \nonumber \\
                      &= \sum_{\outcl{i}\in \out} \opt_{\outcl{i}}, \nonumber
\end{align}
where $ \opt_{\outcl{i}} =  \sum_{j<k} |X_{i,j}||X_{i,k}| + \sum_{j=1}^{t}\dfrac{|X_{i,j}|(|B_{i,j}|+ |R_{i,j}|)}{2} $.

Prior to proving the correctness of~\cref{thm:make.p.cluster.fair}, we demonstrate some useful lower bounds for $ \opt_{\outcl{i}} $. In particular, let $ \dmp = \frac{|\blue{\outcl{i}}|}{\ratio} - |\red{\outcl{i}}| $, $ \smp = |\red{\outcl{i}}| - \frac{|\blue{\outcl{i}}|}{\ratio} $, then
\begin{itemize}
    \item If $ \outcl{i}\in \tb $, then $ \opt_{\outcl{i}} \geq \dfrac{1}{2}\dmp |\outcl{i}| $;
    \item If $ \outcl{i}\in \tr $, then $ \opt_{\outcl{i}} \geq \dfrac{1}{2}\smp(|\outcl{i}|-s) $;
    \item If $ \outcl{i}\in \tr $, then $ \opt_{\outcl{i}} \geq \dfrac{\smp_{i}^{2}}{2} $.
\end{itemize}
To show this, we leverage the following lemma.

\begin{lemma}\label{lem:make.estimated.extra.integral}
    Given positive numbers $ x_{1}, x_{2}, \dots, x_{t} $ and a positive integer $ S $ such that $ S \leq \sum_{i=1}^{t}x_{i} $. Then for any $ a_{1}, a_{2}, \dots, a_{n} $ such that $ a_{i} \geq 0 $, there exists non-negative integers $ y_{1}, y_{2}, \dots, y_{t} $ satisfying:
    \begin{itemize}
        \item $ 0\leq y_{i} \leq \ceilenv{x_{i}}$, for $ i=1,2,\dots,t $.
        \item $ \sum_{i=1}^{t}y_{i}= S $.
        \item $ \sum_{i=1}^{t}a_{i}y_{i} \leq \sum_{i=1}^{t}a_{i}x_{i} $.
    \end{itemize}
\end{lemma}
\begin{proof}
    Without loss of generality, assume that $ a_{1}\leq a_{2}\leq \dots \leq a_{n} $.

    We will show that the output of the following algorithm is the desired $ y_i $'s.

    \begin{algorithm2e}[htbp]
        \DontPrintSemicolon
        \caption{Choose $ (y_{1}, y_{2}, \dots, y_{t}) $.}
        \label{alg.make.estimated.extra.integral}
        \SetKwInOut{KwIn}{Input}
        \SetKwInOut{KwOut}{Output}
        $ y_{j} \gets \floor{x_{j}}, \forall j $

        \While{$ \sum_{i=1}^{t}y_{i} > S $}{
            \For{$ i=1 \to t $}{
                \uIf{$ \sum_{i=1}^{t} y_{i} >S $ and $ y_{i}>0 $}{
                    $ y_{i} \gets y_{i} - 1 $ \label{line: decreases}
                }
            }
        }

        \While{$ \sum_{i=1}^{t}y_{i} < S $}{
            \For{$ i=1 \to t $}{
                \uIf{$ \sum_{i=1}^{t} y_{i} <S $ and $ y_{i}< \ceil{x_{i}} $}{
                    $ y_{i} \gets y_{i} + 1 $ \label{line: increases}
                }
            }
        }

        \KwRet{$ (y_{1}, y_{2}, \dots, y_{t}) $}
    \end{algorithm2e}

    Let $ y_{i} $'s be the output of~\cref{alg.make.estimated.extra.integral}.
    %
    
    $ \bullet $ $ 0\leq y_{i} \leq \ceilenv{x_{i}} $, for $ i=1,2,\dots, t $.

    Observe that $ y_{i} $ is decreased by $ 1 $ (\cref{line: decreases}) only if $ y_{i}>0 $, and is increased by $ 1 $ (\cref{line: increases}) only if $ y_{i}<\ceil{x_{i}} $. Hence, each $ y_{i} $ output by this algorithm satisfies $ 0\leq y_{i} \leq \ceil{x_{i}} $.

    $ \bullet $ $ \sum_{i=1}^{t}y_{i} = S $.

    To see this, we consider $ (y_{1},y_{2}, \dots, y_{t}) $ after the execution of each while loop. Initially, $ y_{i} $ is set to $ \floor{x_{i}} $. If $ \sum_{i=1}^{t}y_{i} > S $, the first while loop gradually decreases $ \sum_{i=1}^{t}y_{i} $ by $ 1 $ by decreasing a nonzero $ y_{i} $ by $ 1 $, until $ \sum_{i=1}^{t}y_{i} = y $. Since $ \min\sum_{i=1}^{t}y_{i}=0 $ and $ S\geq 0 $, there must be an iteration where $ \sum_{i=1}^{t}y_{i} = S $. After this iteration, the first while loop is done. The second while loop will not be executed, and the algorithm will return $ (y_{1}, y_{2}, \dots, y_{t}) $ satisfying $ \sum_{i=1}^{t}y_{i} = S $. 

    On the other hand, if $ \sum_{i=1}^{t}\floor{x_{i}}<y $, the first while loop is not executed. The second while loop gradually increases $ \sum_{i=1}^{t}y_{i} $ by $ 1 $ by increasing a $ y_{i} < \ceil{x_{i}} $ by $ 1 $, until $ \sum_{i=1}^{t}y_{i} $. Since $ S\leq \sum_{i=1}^{t}x_{i}\leq \sum_{i=1}^{t}\ceil{x_{i}} $, there always exists some $ y_{i} (\leq \ceil{x_{i}})$ that can be increased whenever $ \sum_{i=1}^{t}y_{t} < S $. Hence, there must be an iteration where $ \sum_{i=1}^{t}y_{i}=S $. After this iteration, the second while loop is done, and the algorithm outputs $ (y_{1}, y_{2}, \dots, y_{t}) $ such that $ \sum_{i=1}^{t}y_{i} = S $. 
    
    
    $ \bullet $ $ \sum_{i=1}^{t}a_{i}y_{i} \leq \sum_{i=1}^{t}a_{i}x_{i} $.

    We first claim that there exists $ m\leq t $ such that for all $ i< m $, we have $ y_{i}\geq x_{i} $, and for all $ i\geq m $, we have $ y_{i}\leq x_{i} $.

    If $ y_{i} \leq x_{i} $  for all $ j $, then $ m = 1 $ satisfies the claim. If there exists $ y_{i} > x_{i}$, since $ \sum_{i=1}^{t}y_{i} < \sum_{i=1}^{t}x_{i} $, there also exists $ y_{i} < x_{i} $. Let $ m $ be the minimum index such that $ y_{m}<x_{m} $, then $ y_{i} \geq x_{i}$ for all $ j<m $. Now we show that $ y_{i}\leq x_{i} $ for all $ j\geq m $.

    To see this, note that in~\cref{alg.make.estimated.extra.integral}, if the second while loop is not executed, then the algorithm only decreases $ y_{i} $. Thus $ y_{i}\leq \floor{x_{i}} \leq x_{i} $ for all $ j $. The second while loop is executed only if the first while loop is not. In that case, the second while loop starts with $ y_{i} = \floor{x_{i}} $. If $ y_{i} $ is increased (\cref{line: increases}), then $ y_{i} = \ceil{x_{i}} > x_{i} $. Assume to the contrary that there exists $ j>m $ such that $ y_{j} > x_{j} $. This implies that before $ y_{j} $ is increased, $ \sum_{i=1}^{t}y_{i} < S $. This means that in the iteration of the for loop where $ i = m $, we have $ \sum_{i=1}^{t}y_{i}< S $ and by definition of $ m $, $ y_{m}<x_{m} $. It follows that $ y_{m} $ is increased and becomes greater than $ x_{m} $, which is a contradiction since in the output, $ y_{m}<x_{m} $.
    
    This concludes that $ y_{i}\leq x_{i} $, for all $ i\geq m $.
    
    Finally, we show that $ \sum_{i=1}^{t}a_{i}y_{i} \leq \sum_{i=1}^{t}a_{i}x_{i} $. We have
    \begin{align}
        \sum_{i=1}^{t}x_{i} &\geq \sum_{i=1}^{t}y_{i} \nonumber \\
        \Leftrightarrow \sum_{i=m}^{t}(x_{i} - y_{i}) &\geq \sum_{i=1}^{m-1} (y_{i} - x_{i}) \geq 0. \nonumber
    \end{align}
    Moreover, note that $ a_{1}\leq a_{2}\leq \dots \leq a_{t} $, therefore
    \begin{align}
        \sum_{i=m}^{t}(x_{i}-y_{j})a_{i} &\geq \sum_{i=m}^{t}(x_{i}-y_{i})a_{m} \nonumber \\
                                         &\geq \sum_{i=1}^{m-1}(y_{i}-x_{i})a_{m} \nonumber \\
                                         &= \sum_{i=1}^{m-1}(y_{i}-x_{i})a_{i},\nonumber
    \end{align}
    which is equivalent to 
    \begin{align}
        \sum_{i=1}^{t}a_{i}y_{i} \leq \sum_{i=1}^{t}a_{i}x_{i}. \nonumber
    \end{align}
    The proof is complete.
\end{proof}

\begin{lemma}\label{lem:cost.of.opt.when.blue>.p.red}
    For each $ \outcl{i}\in \tb $, let $ \dmp = \dfrac{|\blue{\outcl{i}}|}{\ratio} - |\red{\outcl{i}}| $. Then $ \opt_{\outcl{i}} \geq \dfrac{1}{2}\dmp |\outcl{i}| $.
\end{lemma}
\begin{proof}
    By definition, $\opt_{\outcl{i}} = \sum_{1\leq j< k\leq t}|X_{i,j}||X_{i,k}| + \sum_{j=1}^{t}\dfrac{|X_{i,j}|(|B_{i,j}| + |R_{i,j}|)}{2}. \nonumber$

    Denote by $ \blue{X_{i,j}} $ and $ \red{X_{i,j}} $ the set of blue points and the set of red points in $ X_{i,j} $, respectively. Then $ X_{i,j} = \blue{X_{i,j}} \cup \red{X_{i,j}} $, and $ |\blue{\outcl{i}}| = \sum_{j=1}^{t}|\blue{X_{i,j}}| $, $ |\red{\outcl{i}}| = \sum_{j=1}^{t}|\red{X_{i,j}}| $, and $ |\blue{X_{i,j}}|, |\red{X_{i,j}}| \geq 0 $.

    Let $ S = \left\{j: |\blue{X_{i,j}}| > \ratio|\red{X_{i,j}}|  \right\} $. Then $ S $ is a nonempty set because otherwise $ |\blue{\outcl{i}}| = \sum_{j=1}^{t}|\blue{X_{i,j}}| \leq  \ratio(\sum_{j=1}^{t}|\red{X_{i,j}}|) = \ratio |\red{\outcl{i}}| $, which contradicts to the assumption that $ \outcl{i}\in \tb $.

    For each $ j\in S $, consider arbitrary $ \dmp_{j} $ blue points from $ \blue{X_{i,j}} $ ($ d_{j}\leq |\blue{X_{i,j}}| $). Denote $  \dmp' = \sum_{j\in L}\dmp'_{j}  $. Then
    \begin{align}
        \sum_{1\leq j<k \leq t}|X_{i,j}||X_{i,k}| &\geq \sum_{j\in S} (\dmp'_{j}\sum_{k\notin S} |X_{i,k}|) + \dfrac{1}{2}\sum_{j\in S} (\dmp'_{j}\sum_{k\in S\setminus\{j\}} |X_{i,k}| ) \nonumber \\
                                                  &\geq \dfrac{1}{2}\sum_{j\in S}\dmp'_{j}|\outcl{i}| - \dfrac{1}{2}\sum_{j\in S} \dmp'_{j}|X_{i,j}| \nonumber \\
                                                  &= \dfrac{1}{2}(\sum_{j\in S}\dmp'_{j})|\outcl{i}| - \dfrac{1}{2}\sum_{j\in S} \dmp'_{j}|X_{i,j}|. \label{eq. cost paid for dmpj}
    \end{align}

    Let $ B_{j} = |B_{i,j}|,\ R_{j} = |R_{i,j}| $. Let $ \dmp_{j} = \dfrac{|\blue{X_{i,j}}|}{\ratio} - |\red{X_{i,j}}| $. 

    Since $ X_{i,j}\cup B_{i,j}\cup R_{i,j} $ is a cluster in the fair clustering $ \fairmop $, it means that $ p(|\red{X_{i,j}}| + R_{j}) = q(|\blue{X_{i,j}}|+ B_{j}) $. Therefore, if $ j\in S $, then $ B_{j} + R_{j} \geq \dfrac{|\blue{X_{i,j}}|}{\ratio} - |\red{X_{i,j}}| = \dmp'_{j} >0 $. Thus
    \begin{align}
        \sum_{j\in S}\dfrac{|X_{i,j}|(B_{j} + R_{j})}{2} \geq \sum_{j\in S}\dfrac{|X_{i,j}|\dmp'_{j}}{2}. \label{eq. cost paid for merge Rj} 
    \end{align}
    From~\eqref{eq. cost paid for dmpj} and~\eqref{eq. cost paid for merge Rj}, it follows that
    \begin{align}
        \opt_{\outcl{i}} \geq \dfrac{1}{2} (\sum_{j\in S}\dmp_{j})|\outcl{i}| - \dfrac{1}{2}\sum_{j\in S}\dmp_{j}|X_{i,j}| + \dfrac{1}{2}\sum_{j\in S}\dmp_{j}'|X_{i,j}|. \label{eq.bound.opt.dmp} 
    \end{align}
If we can choose $ \dmp_{j} (\leq |\blue{X_{ij}}|) $ such that  $ \sum_{j\in S}\dmp_{j} = \dmp $, and $ \sum_{j\in S}\dmp_{j}|X_{i,j}| \leq \sum_{j\in S}\dmp'_{j}|X_{i,j}| $, then from~\eqref{eq.bound.opt.dmp}, it implies that $ \opt_{\outcl{i}}\geq \dfrac{1}{2}\dmp|\outcl{i}| $, which concludes the proof.

    To this end, note that  $ \dmp_{j} > 0 $ for all $ j\in S $, and $ \dmp_{j}\leq 0 $ if $ j\notin S $. As a result, $ \dmp = \sum_{j=1}^{t}\dmp_{j} \leq \sum_{j\in S}\dmp_{j} $. Finally, applying~\cref{lem:make.estimated.extra.integral} (by setting $ \dmp_{j} = x_{j},\ |X_{i,j}| = a_{j} $), there exists integers $ \dmp'_{j} $ satisfying:
    \begin{itemize}
        \item $ 0\leq \dmp'_{j} \leq \ceilenv{\dmp_{j}}$, which implies $ \dmp'_{j}\leq |\blue{X_{i,j}}| $ as $ \dmp_{j} \leq \dfrac{|\blue{X_{i,j}|}}{\ratio} $ and $ \ratio >1 $.
        \item $ \sum_{j\in S}\dmp'_{j} = d $.
        \item $ \sum_{j\in S}\dmp'_{j}|X_{i,j}| \leq \sum_{j\in S}\dmp_{j}|X_{i,j}| $.
    \end{itemize}
\end{proof}

%
\begin{lemma}\label{lem:cost.of.opt.when.blue<.p.red}
For each $ \outcl{i}\in \tr $, let $ \smp = |\red{\outcl{i}}| - \dfrac{|\blue{\outcl{i}}|}{\ratio} $. Then $ \opt_{\outcl{i}} \geq \dfrac{1}{2}\smp(|\outcl{i}|-s) $. 
\end{lemma}
\begin{proof}
    By definition, $\opt_{\outcl{i}} = \sum_{1\leq j< k\leq t}|X_{i,j}||X_{i,k}| + \sum_{j=1}^{t}\dfrac{|X_{i,j}|(|B_{i,j}| + |R_{i,j}|)}{2}. \nonumber$

    Denote by $ \blue{X_{i,j}} $ and $ \red{X_{i,j}} $ the set of blue points and the set of red points in $ X_{i,j} $, respectively. Then $ X_{i,j} = \blue{X_{i,j}} \cup \red{X_{i,j}} $, and $ |\blue{\outcl{i}}| = \sum_{j=1}^{t}|\blue{X_{i,j}}| $, $ |\red{\outcl{i}}| = \sum_{j=1}^{t}|\red{X_{i,j}}| $, and $ |\blue{X_{i,j}}|, |\red{X_{i,j}}| \geq 0 $.

    Let $ L = \left\{j: |\blue{X_{i,j}}| < \ratio|\red{X_{i,j}}|  \right\} $. Then $ L $ is a nonempty set because otherwise $ |\blue{\outcl{i}}| = \sum_{j=1}^{t}|\blue{X_{i,j}}| \geq  \ratio(\sum_{j=1}^{t}|\red{X_{i,j}}|) = \ratio |\red{\outcl{i}}| $, which contradicts to the assumption that $ \outcl{i}\in \tr $.

    For each $ j\in L $, consider arbitrary $ \smp'_{j} $ red points from $ \red{X_{i,j}} $. Denote $  \smp' = \sum_{j\in L}\smp'_{j}  $. Then
    \begin{align}
        \sum_{1\leq j<k \leq t} |X_{i,j}||X_{i,k}| &\geq \sum_{j\in L} \left(\smp'_{j} \sum_{k\notin L}|X_{i,k}|\right) + \sum_{j\in L} \left(\smp'_{j} \sum_{k\in L\setminus \{j\} }(|X_{i,k}| - \smp'_{k})\right) \nonumber \\
        &=\sum_{j\in L}\smp'_{j}\left(|\outcl{i}|-\smp'_{j}\right) - \sum_{j\in L} \smp'_{j}\left(|X_{i,j}-\smp'_{j}|\right) \nonumber \\
        &= \smp'\left(|\outcl{i}|-\smp'\right) - \sum_{j\in L} \smp'_{j}\left(|X_{i,j}|-\smp'_{j}\right).\label{eq. cost paid for smpj} 
    \end{align}

    For each $ j $, let $ B_{j} = |B_{i,j}|,\ R_{j} = |R_{i,j}| $. Let $ \smp_{j} = |\red{X_{i,j}}| - \dfrac{|\blue{X_{i,j}}|}{\ratio} $.

    Since $ X_{i,j}\cup B_{i,j}\cup R_{i,j} $ is a cluster in the fair clustering $ \fairmop $, it means that $ p(|\red{X_{i,j}}| + R_{j}) = q(|\blue{X_{i,j}}|+ B_{j}) $. Therefore, if $ j\in L $, then $ B_{j} + R_{j} \geq \ratio |\red{X_{i,j}}| - |\blue{X_{i,j}}| = \ratio\smp_{j} >0 $. Thus 
    \begin{align}
        \sum_{j\in L} \dfrac{|X_{i,j}|(B_{j} + R_{j})}{2} \geq \dfrac{\ratio}{2}\sum_{j\in L}|X_{i,j}|\smp_{j} \geq \dfrac{1}{2}\sum_{j\in L}|X_{i,j}|\smp_{j} .\label{eq. cost paid for merge Bj} 
    \end{align}
    From~\eqref{eq. cost paid for smpj} and~\eqref{eq. cost paid for merge Bj}, it follows that
    \begin{align}
        \opt_{\outcl{i}} &\geq \smp'\left(|\outcl{i}|-\smp'\right) - \sum_{j\in L} \smp_{j}(|X_{i,j}|-\smp_{j}) + \dfrac{1}{2}\sum_{j\in L}\smp'_{j}|X_{i,j}| \nonumber \\
                         &\geq \dfrac{1}{2} \left(\smp'\left(|\outcl{i}|-\smp'\right) - \sum_{j\in L} \smp_{j}(|X_{i,j}|-\smp_{j}) + \sum_{j\in L}\smp'_{j}|X_{i,j}|\right). \label{eq.cost.bound.smp}
    \end{align}

    If we can choose $ \smp_{j}\ (\leq |\red{X_{i,j}}|) $ such that $ \sum_{j\in L}\smp_{j} = \smp $ and $ \sum_{j\in L} \smp_{j}|X_{i,j}| \leq \sum_{j\in L}\smp'_{j}|X_{i,j}| $, then from~\eqref{eq.cost.bound.smp}, it implies that $ \opt_{\outcl{i}}\geq \dfrac{1}{2}\smp(|\outcl{i}| - \smp) $, which completes the proof.

    To this end, note that $ \smp_{j} > 0 $ for all $ j\in L $, and $ \smp_{j}\leq 0 $ for all $ j\notin L $. Consequently, $ \smp = \sum_{j=1}^{t}\smp_{j} \leq \sum_{j\in L}\smp_{j} $. Finally, applying~\cref{lem:make.estimated.extra.integral} (by setting $ \smp_{j} = x_{j}, |X_{i,j}| = a_{j} $), there exists integers $ \smp'_{j} $ satisfying
    \begin{itemize}
        \item $ 0\leq \smp'_{j}\leq \ceilenv{\smp_{j}} < |\red{X_{i,j}}|$.
        \item $ \sum_{j\in S}\smp'_{j} = \smp $.
        \item $ \sum_{j\in L}\smp_{j}|X_{i,j}| \leq \sum_{j\in L}\smp'_{j}|X_{i,j}| $.
    \end{itemize}

\end{proof} 

\begin{lemma}\label{lem:cost.split.the.split.when.blue<.p.red}
    For each cluster $ \outcl{i}\in \tr $, let $ \smp_{i} = |\red{\outcl{i}}| - \dfrac{|\blue{\outcl{i}|}}{\ratio} $. Then $$ \opt_{\outcl{i}} \geq \dfrac{\smp_{i}^{2}}{2}. $$
\end{lemma}
\begin{proof}
    By definition,
    \begin{align}
        \opt_{\outcl{i}} = \sum_{1\leq j< k\leq t}|X_{i,j}||X_{i,k}| + \sum_{j=1}^{t}\dfrac{|X_{i,j}|(|B_{i,j}| + |R_{i,j}|)}{2}. \nonumber
    \end{align}
    Denote by $ r_{j}, b_{j} $ the number of red points and blue points in $ X_{i,j} $ respectively. Then $ |X_{i,j}| = r_{j} + b_{j} $. Let $ R_{j} = |R_{i,j}| $, and $ B_{j} = |B_{i,j}| $.

    Let $ L = \{j:\ \ratio r_{j} > b_{j} \} $. Then $ L \neq \emptyset $.

    For each $ j\in L $, let $ E_{j} = r_{j} - \dfrac{b_{j}}{\ratio} > 0 $. Since $ X_{i,j}\cup B_{i,j}\cup R_{i,j} $ is a cluster in the fair clustering $ \fairmop $, it means that $ p(r_{j} + R_{j}) = q(b_{j} + B_{j}) $. Therefore, if $ j\in L $, then $ B_{j} \geq \ratio r_{j} - b_{j} >0  $. We have
    \begin{align}
        \opt_{\outcl{i}} &\geq \sum_{j\in L}\dfrac{|X_{i,j}||B_{i,j}|}{2} + \sum_{j<k;\ j,k \in L}|X_{i,j}||X_{i,k}| \nonumber \\
                         &\geq \sum_{j\in L}\dfrac{r_{j}(\ratio r_{j}-b_{j})}{2} + \sum_{j<k;\ j,k \in L} r_{j}r_{k} && (\text{since }|B_{i,j}| \geq \ratio r_{j}-b_{j},\ |X_{i,j}|\geq r_{j})\nonumber \\
                         &\geq \dfrac{1}{2}\left(\sum_{j\in L}  r_{j}(r_{j} - \dfrac{b_{j}}{\ratio}) + 2\sum_{j<k;\ j,k\in L}r_{j}r_{k} \right) && (\text{since }\ratio >1)\nonumber \\
                         &\geq \dfrac{1}{2} \left(\sum_{j\in L} E_{j}^{2} + 2\sum_{j<k; j,k \in L}E_{j}E_{k} \right) && (\text{since } r_{j} \geq r_{j} -\dfrac{b_{j}}{\ratio})\nonumber \\
                         & = \dfrac{(\sum_{j\in L}E_{j})^{2}}{2} .\nonumber
    \end{align}
    Note that, if $ j\notin L $, $ r_{j}-\dfrac{b_{j}}{\ratio}\leq 0 $. Therefore, $ \smp_{i} = |\red{\outcl{i}}| - \dfrac{|\blue{\outcl{i}}|}{\ratio} = \sum_{j=1}^{t}r_{j} - \dfrac{\sum_{j=1}^{t}b_{j}}{\ratio} \leq \sum_{j\in L}(r_{j} - \dfrac{b_{j}}{\ratio}) = \sum_{j\in L}E_{j}$. It follows that $ \opt_{\outcl{i}} \geq \dfrac{(\sum_{j\in L}E_{j})^{2}}{2} \geq \smp_{i}^{2}/2 $.

     This concludes the proof.
\end{proof}

Recall that for each cluster $ \outcl{i}\in \tr $,~\cref{alg:algo-mf} cuts $ |\red{\outcl{i}}| - \dfrac{|\blue{\outcl{i}}|}{\ratio} $ red points from it. Each cluster in $ \outcl{i}\in \tb $ will be merged with $ \dfrac{|\blue{\outcl{i}}|}{\ratio} - |\red{\outcl{i}}| $ red points, which are previously split from clusters in $ \tr $. Note that, during this process, points in $ |\red{\outcl{i}}| - \dfrac{|\blue{\outcl{i}}|}{\ratio} $ red points from the same cluster $ \outcl{i}\in \tr $ may be separated from each other to be merged with different clusters in $ \tb $. Also, $ \dfrac{|\blue{\outcl{i}}|}{\ratio} - |\red{\outcl{i}}| $ points that are merged to a cluster $ \outcl{i}\in \tb $ may come from different clusters in $ \tr $. After this process, every cluster $ \outcl{i} $ has become fair. Therefore, every red point that is cut from $ \tr $ would have been merged to some cluster in $ \tb $.

The distance $ \dist(\out, \fairset) $ is the number of pairs of points that are clustered together in one clustering (either $ \out $ or $ \fairset $) and are separated by the other (either $ \fairset $ or $ \out $). We decompose the number of such pairs into four types of costs. In particular, $ \dist(\out, \fairset) = \costone{\out} + \costtwo{\out} + \costthree{\out} + \costfour{\out} $, where each cost is defined as follows.
\begin{itemize}
    \item $\costone{\out}$: the number of pairs $ (u,v) $ in which in $ \out $, $ u,v \in \outcl{i} $ and $ \outcl{i} \in \tr $, while in $ \fairset $, $ u $ is kept in $ \outcl{i} $ and $ v $ is cut from $ \outcl{i} $. Since for each $ \outcl{i} \in \tr $,~\cref{alg:algo-mf} splits $ \smp_{i} = |\red{\outcl{i}}| - \dfrac{|\blue{\outcl{i}}|}{\ratio} $ from it, it follows that
        \begin{align}
            \costone{\out} = \sum_{\outcl{i}\in \tr}\smp_{i}(|\outcl{i}| - \smp_{i}). \label{eq.cost1.make.p.fair}
        \end{align}
    \item $ \costtwo{\out} $: the number of pairs $ (u,v) $ in which in $ \out $, $ u\in \outcl{i} \in \tb $ and $ v\in \outcl{j} \in \tr $, while in $ \fairset $, $ v $ is split from $ \outcl{j} $ and is merged to $ \outcl{i} $, so that $ u $ and $ v $ are in the same cluster. Since for each $ \outcl{i} \in \tb $,~\cref{alg:algo-mf} merges $ \dmp_{i} = \dfrac{|\blue{\outcl{i}}|}{\ratio} - |\red{\outcl{i}}| $ red points to it, it follows that
        \begin{align}
            \costtwo{\out} = \sum_{\outcl{i}\in \tb}\dmp_{i}|\outcl{i}|. \label{eq.cost2.make.p.fair}
        \end{align}
    \item $ \costthree{\out} $: the number of pairs $ (u,v) $ in which in $ \out $, $ u\in \outcl{j} $, and $ v\in \outcl{k} $ where $ \outcl{j} $ and $ \outcl{k} $ are distinct clusters in $ \tr $, while in $ \fairset $, $ u $ is split from $ \outcl{j} $, $ v $ is split from $ \outcl{k} $ and they are both merged to a cluster $ \outcl{i}\in \tb $. For each cluster $ \outcl{i}\in \tb $,~\cref{alg:algo-mf}  merges exactly $ \dmp_{i} = \dfrac{|\blue{\outcl{i}}|}{\ratio} - |\red{\outcl{i}}| $ red points to it. Suppose that among these points, $ \alpha_{i_{1}} $ points are brought from a cluster $ \outcl{i_{1}}\in \tr $, $ \alpha_{i_{2}} $ points are brought from a cluster $ \outcl{i_{2}}\in \tr $, and so on, $ \alpha_{i_{t}} $ points are brought from a cluster $ \outcl{i_{t}}\in \tr $, such that $ \sum_{j=1}^{t}\alpha_{i_{j}} = \dmp_{i} $, and $ \outcl{i_{1}}, \outcl{i_{2}},\dots, \outcl{i_{t}} $ are distinct clusters. It follows that
        \begin{align}
            \costthree{\out} &= \sum_{\outcl{i}\in \tb } (\sum_{1\leq j <k \leq i_{t}} \alpha_{i_{j}}\alpha_{i_{k}}) \nonumber \\
                             &\leq \sum_{\outcl{i}\in \tb}\dfrac{(\sum_{j=1}^{t}\alpha_{i_{j}})^{2}}{2} = \sum_{\outcl{i}\in \tb}\dfrac{\dmp_{i}^{2}}{2}. \label{eq.cost3.make.p.fair}
        \end{align}
    \item $ \costfour{\out} $: the number of pairs $ (u,v) $ in which $ u, v $ are among $ \smp_{i} = |\red{\outcl{i}}| - \dfrac{|\blue{\outcl{i}}|}{\ratio} $ red points split from a cluster $ \outcl{i}\in \tr $, but $ u $ and $ v $ are then separated from each other because $ u $ and $ v $ are merged to different clusters $ \outcl{j} $ and $ \outcl{k} $ in $ \tb $. Suppose that among these points, $ \alpha_{i_{1}} $ points are merged to a cluster $ \outcl{i_{1}}\in \tb $, $ \alpha_{i_{2}} $ points are merged to a cluster $ \outcl{i_{2}}\in \tb $, and so on, $ \alpha_{i_{t}} $ points are merged to a cluster $ \outcl{i_{t}}\in \tb $, such that $ \sum_{j=1}^{t}\alpha_{i_{j}} = \smp_{i} $, and $ \outcl{i_{1}}, \outcl{i_{2}},\dots, \outcl{i_{t}} $ are distinct clusters. It follows that
        \begin{align}
            \costfour{\out} &= \sum_{\outcl{i}\in \tr } (\sum_{1\leq j <k \leq i_{t}} \alpha_{i_{j}}\alpha_{i_{k}}) \nonumber \\
                             &\leq \sum_{\outcl{i}\in \tr}\dfrac{(\sum_{j=1}^{t}\alpha_{i_{j}})^{2}}{2} = \sum_{\outcl{i}\in \tr}\dfrac{\smp_{i}^{2}}{2}. \label{eq.cost4.make.p.fair}
        \end{align}
\end{itemize}

We are ready to prove~\cref{thm:make.p.cluster.fair}
\begin{proof}[Proof of~\cref{thm:make.p.cluster.fair}]
    Recall that $ \dist(\out, \fairset) = \costone{\out} + \costtwo{\out} + \costthree{\out} + \costfour{\out} $. 

    We show that $ \costone{\out} + \costtwo{\out} \leq \dfrac{1}{2}\dist(\out, \fairmop) $, and $ \costthree{\out} + \costfour{\out} \leq \dist(\out, \fairmop) $.

    Applying~\cref{lem:cost.of.opt.when.blue<.p.red} for each $ \outcl{i}\in \tr $, we have $ \opt_{\outcl{i}} \geq \dfrac{1}{2}\smp_{i}(|\outcl{i}|-\smp_{i}) $. Applying~\cref{lem:cost.of.opt.when.blue>.p.red} for each $ \outcl{i} \in \tb $, we have $ \opt_{\outcl{i}} \geq \dfrac{1}{2}\dmp_{i}|\outcl{i}| $. Consequently, it follows that
    \begin{align}
        \costone{\out} + \costtwo{\out} &= \sum_{\outcl{i}\in \tr}\smp_{i}(|\outcl{i}| - \smp_{i}) + \sum_{\outcl{i}\in \tb}\dmp_{i}|\outcl{i}| && (\text{by~\eqref{eq.cost1.make.p.fair} and~\eqref{eq.cost2.make.p.fair}}) \nonumber \\
                                        &\leq 2\sum_{\outcl{i}\in \tr}\opt_{\outcl{i}} + 2\sum_{\outcl{i}\in \tb}\opt_{\outcl{i}} \nonumber \\
                                        &\leq 2\dist(\out, \fairmop). \nonumber
    \end{align}

    In each $ \outcl{i}\in \tb $, $ |\outcl{i}| \geq |\blue{\outcl{i}}| > \ratio\dmp_{i} $. Therefore, applying~\cref{lem:cost.of.opt.when.blue>.p.red}, it follows that $\opt_{\outcl{i}} \geq \dmp_{i}|\outcl{i}| \geq \ratio \dmp_{i}^{2}$. Using~\eqref{eq.cost3.make.p.fair}, we have $ \costthree{\out} \leq \sum_{\outcl{i}\in \tb}\dfrac{\dmp_{i}^{2}}{2} \leq \sum_{\outcl{i}\in \tb}\dfrac{\opt_{\outcl{i}}}{2\ratio} $.

    For each $ \outcl{i}\in \tr $, applying~\cref{lem:cost.split.the.split.when.blue<.p.red}, we have $ \opt_{\outcl{i}} \geq \dfrac{\smp_{i}^{2}}{2} $. Using~\eqref{eq.cost4.make.p.fair} implies that $ \costfour{\out} \leq \sum_{\outcl{i}\in \tr }\dfrac{\smp_{i}^{2}}{2} \leq \sum_{\outcl{i}\in \tr}\opt_{\outcl{i}} $. Then $ \costthree{\out} + \costfour{\out} $ is upper bounded by
    \begin{align}
        \sum_{\outcl{i}\in \tb}\dfrac{\opt_{\outcl{i}}}{2\ratio} + \sum_{\outcl{i}\in \tr}\opt_{\outcl{i}} \leq \sum_{\outcl{i}\in \out}\opt_{\outcl{i}} = \dist(\out, \fairmop). \nonumber
    \end{align}

     To summarize, we have
     \begin{align}
         \dist(\out, \fairset) &= \costone{\out} + \costtwo{\out} + \costthree{\out} + \costfour{\out}\nonumber \\
                           &\leq 2\dist(\out, \fairmop) + \dist(\out, \fairmop) \nonumber \\
                           &= 3\dist(\out, \fairmop). \nonumber
     \end{align}
     This completes the proof.
\end{proof}

\section{Approximating Consensus Clustering}\label{sec:approx-consensus}

\begin{theorem}
    \label{thm:combine-consensus}
    Consider any $\ell \ge 1$ and any $\alpha \ge 1$. Suppose there is a $t(n)$-time algorithm Algorithm-C() that, given any clustering on $n$ points, produces an $\alpha$-close $\fair$. Then there is an $O(m t(n) + m^2 n^2)$-time algorithm that, given a set of $m$ clusterings on $n$ points, produces a $(2+\alpha)$-approximate $\ell$-mean fair consensus clustering.
\end{theorem}

\begin{proof}
    Let us consider the following algorithm: Suppose we are given $m$ clusterings $\inpset_1,\inpset_2,\cdots,\inpset_m$. Then, for each $\inpset_i$, compute an $\alpha$-close $\fair$ $\m{F}_i$ using Algorithm-C(). Then, output the $\fair$ $\m{F}_k$ that minimizes the $\ell$-mean objective function, i.e.,
    \[
    \m{F}_k = \arg \min_{i} \left(\sum_{j=1}^m \left(\dist(\inpset_j, \m{F}_i)\right)^\ell \right)^{1/\ell}.
    \]

Note, for any $\inpset_j,\m{F}_i$, computing $\dist(\inpset_j, \m{F}_i)$ takes $O(n^2)$ time. Thus, it is easy to see that the overall running time of the above algorithm is $O(m t(n) + m^2 n^2)$. It only remains to argue that the above algorithm achieves $(2+\alpha)$-approximation to the $\ell$-mean fair consensus clustering problem.

Let $\m{F}^*$ be an (arbitrary) optimal fair consensus clustering, and let $\inpset_{i^*}$ be a closest (breaking ties arbitrarily) clustering among $\inpset_1,\inpset_2,\cdots,\inpset_m$, to $\m{F}^*$. We emphasize that we consider $\inpset_{i^*}$ only for the sake of the analysis. Since $\m{F}^*$ is a $\fair$, by the guarantee of Algorithm-C(), we get
\begin{align}
    \label{eq:closest-to-opt}
    \dist(\inpset_{i^*},\m{F}_{i^*}) & \le \alpha \cdot \dist(\inpset_{i^*},\m{F}^*).
\end{align}

We next derive the following

\begin{align*}
    &\left(\sum_{j=1}^m \left(\dist(\inpset_j, \m{F}_{i^*})\right)^\ell \right)^{1/\ell}\\
    & \le \left(\sum_{j=1}^m \left(\dist(\inpset_j,\m{F}^*) + \dist(\m{F}^* ,\m{F}_{i^*}) \right)^\ell \right)^{1/\ell} &&\text{(By the triangle inequality)}\\
    & \le \left(\sum_{j=1}^m \left(\dist(\inpset_j,\m{F}^*) + \dist(\m{F}^* ,\inpset_{i^*}) + \dist(\inpset_{i^*} ,\m{F}_{i^*}) \right)^\ell \right)^{1/\ell} &&\text{(By the triangle inequality)}\\
    & \le \left(\sum_{j=1}^m \left(\dist(\inpset_j,\m{F}^*) + \dist(\m{F}^* ,\inpset_{i^*}) + \alpha \cdot \dist(\inpset_{i^*},\m{F}^*) \right)^\ell \right)^{1/\ell}&&\text{(By \cref{eq:closest-to-opt})}\\
    & \le \left(\sum_{j=1}^m \left(\dist(\inpset_j,\m{F}^*) + \dist(\inpset_j,\m{F}^*) + \alpha \cdot \dist(\inpset_j,\m{F}^*) \right)^\ell \right)^{1/\ell} &&\text{(Since $\inpset_{i^*}$ is closest to $\m{F}^*$)}\\
    &=(2 + \alpha) \left(\sum_{j=1}^m \left(\dist(\inpset_j,\m{F}^*)\right)^\ell \right)^{1/\ell}.
\end{align*}

Now, since our algorithm outputs $\m{F}_k$ that minimizes the $\ell$-mean objective function, the approximation guarantee follows, and that concludes the proof.
\end{proof}

Now, \cref{thm:combine-consensus} together with \cref{thm:closest-1-1-fair} implies the following result.

\consensusequifair*

Similarly, \cref{thm:combine-consensus} together with \cref{thm:closest-p-1-fair} implies the following result.

\consensuspfair*

Also, \cref{thm:combine-consensus} together with \cref{thm:closest-p-q-fair} implies the following result.

\consensuspqfair*

\section{Hardness of Closest Fair Clustering problem}
\label{sec:NP-hard}

In this section, we prove that given a clustering $\m{D}$ over a set of $n$ red-blue colored points where the ratio between the total number of blue and red points is $p$ (where $p$ is an integer), the problem of finding a Closest {\fair} $\m{F}^*$ is \textbf{NP}-hard. Let us start by defining the corresponding decision version of the problem.

\begin{definition}[$\clof$]
    Consider a positive integer $p$. Given a clustering $\m{D}$ over a set of red-blue colored points where the ratio between the total number of blue and red points is $p$, and a non-negative integer $\tau$, decide between the following:
    \begin{itemize}
        \item YES: There exists a {\fair} (on input point set) $\m{F}$ such that $\dist(\m{D}, \m{F}) \le \tau$;
        \item NO: For every {\fair} (on input point set) $\m{F}$, $\dist(\m{D}, \m{F}) > \tau$.
    \end{itemize}
\end{definition}

 We show the following theorem.

\begin{theorem}\label{thm:main}
    For any integer $p \geq 2$, the $\clof$ problem is $\npc$.
\end{theorem}

It is straightforward to see that the problem $\clof$ is in \textbf{NP}. To establish \textbf{NP}-hardness, we provide a polynomial-time reduction from the \emph{3-Partition problem} (which we define below) to the $\clof$ problem. 

\begin{definition}[$3$-Partition problem]
    Given a (multi)set of positive integers $S = \{ x_1, \ldots, x_n\}$, decide whether (YES:) there exists a partition of $S$ into $m$ (disjoint) subsets $S_1, S_2, \ldots, S_m \subseteq S$ where $m = n/3$, such that
    \begin{itemize}
        \item For all $i$, $|S_i| = 3$; and
        \item For all $i$, $\sum_{x_j \in S_i}x_j = T$, where $T = \frac{\sum_{x_j \in S}x_j}{n/3}$,
    \end{itemize}
    or (NO:) no such partition exists.
\end{definition}

Note that, by partition of $S$ we mean $S_i \cap S_j = \emptyset$ for all $i \neq j$ and $\bigcup_{i = 1}^m S_i = S$.

We know that the $3$-Partition problem is in fact \emph{strongly $\npc$}~\cite{garey1975complexity}, i.e., the $3$-Partition problem remains $\npc$ even if all the integers $x_i \in S$ are bounded by a polynomial in $n$; more specifically, $\max_{x_i \in S} x_i \le n^c$ for some non-negative constant $c$. Furthermore, the $3$-Partition problem remains strongly $\npc$ even when for all $x_i \in S$, $x_i \in (T/4, T/2)$, where $T= \frac{3}{n} \sum_{x_i \in S} x_i$~\cite{garey1975complexity}. From now on, we refer to this restricted (with a restriction that each $x_i \in (T/4, T/2)$) variant of the $3$-partition problem as the $\thrp$ problem.

\begin{theorem}(\cite{garey1975complexity}) \label{thm:garey-johnson}
    $\thrp$ is strongly $\npc$.
\end{theorem}

We divide our \textbf{NP}-hardness proof into two parts. First, we prove $\clof$ is \textbf{NP}-hard for any integer $p \geq 3$, and then we prove $\clof$ is \textbf{NP}-hard also for $p = 2$ using a slightly different reduction.

\subsection{\texorpdfstring{Hardness of $\clof$ for $p \geq 3$}{Hardness of clof for p ≥ 3}}
In this subsection, we prove the following lemma.

\begin{lemma}\label{lem:main-one}
    For any integer $p \geq 3$, the $\clof$ problem is \textbf{NP}-hard.
\end{lemma}

\paragraph{Reduction from $\thrp$ to $\clof$ :} Let us consider a $\thrp$ instance $S = \{ x_1, x_2, \ldots, x_n\}$, where for all $i \in [n]$, $x_i \le n^c$, for some non-negative constant $c$. Given the $\thrp$ instance $S = \{ x_1, x_2, \ldots, x_n\}$ we create a $\clof$ instance $( \m{C}, \tau )$ as follows:

\begin{itemize}
    \item $ \m{C} = \{B_1, \ldots, B_{n/3}, R_1,\ldots,R_n \} $, where for each $i \in \{1,\ldots,n/3\}$, $B_i$ is a monochromatic blue cluster (i.e., containing only blue points) of size $pT$ and for each $j \in \{1,\ldots, n\}$, $R_j$ is a monochromatic red cluster (i.e., containing only red points) of size $x_j$ (i.e., $|R_j| = x_j$);

    \item \[ \tau = \frac{1}{2} \sum_{j = 1}^n x_j \left(T - x_j\right) + \frac{n}{3}pT^2.\]
\end{itemize}

Since each $x_i \le n^c$, the size of the instance $( \m{C}, \tau )$ is polynomial in $n$. Moreover, it is straightforward to see that the reduction runs in polynomial time.

Now, we argue that the above reduction maps a YES instance of the $\thrp$ to a YES instance of the $\clof$.

\begin{lemma}\label{lem:yes-instance}
    If $S$ is a YES instance of the $\thrp$, then $(\m{C}, \tau) $ is also a YES instance of the $ \clof$.
\end{lemma}

\begin{proof}
    Suppose $S$ is a YES instance of the $\thrp$. Then there exists a partition $S_1, S_2, \ldots, S_{n/3} \, \, \text{of} \, \, S = \{ x_1, x_2, \ldots, x_n\}$ such that for all $1\le i \le n/3$, $|S_i| = 3$ and
    \[
        \sum_{x_j \in S_i}x_j = T \, \, \text{where} \, \, T = \frac{\sum_{x_k \in S} x_k}{\frac{n}{3}}.
    \]
    Let, $S_i = \{ x_{i_1}, x_{i_2}, x_{i_3}\}$. By our construction of $(\m{C}, \tau)$, we have $|R_{i_j}| = x_{i_j}$, for $j \in \{1,2,3\}$.
   
    Now, we construct a fair clustering $\m{F}$ by merging each $B_i$ with $R_{i_1},R_{i_2}, R_{i_3}$. More formally,
    
    \[
       \m{F} = \left\{ \left(B_i \cup \bigcup_{j \in \{1,2,3\}}R_{i_j} \right)  \, \,  \middle| \, \, i \in [n/3] \right \} .
    \]
    Note that $\m{F}$ is a fair clustering because for each cluster $F \in \m{F}$ we have $|\blue{F}| = pT$ (since $\blue{F}=B_i$ for some $i \in [n/3]$) and 
    
    \[|\red{F}| = \sum_{j \in \{1,2,3\}}|R_{i_j}| = \sum_{j \in \{1,2,3\}}x_{i_j} = T.\] 
    
    Next, we claim that
    \[
        \dist(\m{C}, \m{F}) = \frac{1}{2} \sum_{j = 1}^n x_j  \left(T - x_j\right) + \frac{n}{3}pT^2 = \tau.
    \]

    \begin{itemize}
        \item \textbf{Reason behind the term $\frac{1}{2} \sum_{j = 1}^n x_j \left(T - x_j \right)$}: In $\m{F}$, each red cluster $R_j$ is merged with $(T - |R_j|)$ red points from other clusters. This expression counts the number of such pairs. $1/2$ comes from the fact that we are counting a pair twice, once while considering the cluster $R_j$ and again while considering a cluster $R_{j'}$, for $j' \neq j$, that is merged with $R_j$. (One important thing to note is that since $S$ is a YES instance of the $\thrp$, we do not need to cut any red cluster $R_j$ in $\m{C}$ to form the {\fair} $\m{F}$.)
        \item \textbf{Reason behind the term $(n/3)pT^2$}: To form each cluster of $\m{F}$ we have merged $T$ red points with $pT$ blue points. We have $n/3$ clusters in $\m{F}$, hence the number of such pairs in total is $(n/3)pT^2$.
    \end{itemize}
    Hence, $(\m{C},\tau) $ is a YES instance of the $ \clof$.
\end{proof}

Now, we argue that our reduction maps a NO instance of the $\thrp$ to a NO instance of the $\clof$.

\begin{lemma}\label{lem:no-instance}
 For $p \geq 3$, if $S $ is a NO instance of the $\thrp$, then $(\m{C}, \tau) $ is also a NO instance of the $ \clof$.
\end{lemma}

\begin{proof}
    To prove $(\m{C}, \tau) $ is a NO instance of the $ \clof$ we consider an (arbitrary) closest {\fair} $\m{F}^*$ to $\m{C}$, and then argue that $\dist(\m{C}, \m{F}^*) > \tau$.

    Let us recall, $S = \{x_1,\ldots, x_n\}$, $\m{C} = \{ B_1,\ldots,B_{n/3}, R_1,\ldots,R_n\}$ where $B_i$'s are monochromatic blue clusters each of size $pT$, where $T = \sum_{x_i \in S} x_i / (n/3)$ and for each $\ell \in [n]$, $R_\ell$ is a monochromatic red cluster of size $x_\ell$.

    To prove $\dist(\m{C}, \m{F}^*) > \tau$ we prove the following claim regarding the structure of $\m{F}^*$ when $p \geq 3$.

     \begin{claim}\label{clm:np-main-claim}
        Consider any integer $p \geq 3$. Then, for all $B_i$, there exists $F_a \in \m{F}^*$ such that $\blue{F_a} = B_i$.
    \end{claim}

    For now, let us assume the above claim and prove \cref{lem:no-instance}. Since, the number of monochromatic blue clusters in $\m{C}$ is $n/3$, assuming \cref{clm:np-main-claim} we get $|\m{F}^*| = n/3$.

    According to the definition of consensus metric, $\dist(\m{C}, \m{F}^*)$ counts total number of pairs $(u,v)$ that are together in $\m{C}$ but separated by $\m{F}^*$ and the number of pairs $(u,v)$ that are in different clusters in $\m{C}$ but together in $\m{F}^*$.

    For a cluster $F_a \in \m{F}^*$, let us define $\costf{F_a}$ to be the number of pairs $u,v$ such that $u,v \in F_a$ but in different clusters in $\m{C}$.

    Formally, we define $\costf{F_a}$ as
    \[
        \costf{F_a} := |\{ (u,v) \mid u,v \in F_a \, \, \text{but} \, \, u \in C_i, v \in C_j, C_i, C_j \in \m{C}, i \neq j \}|.
    \]

    It is easy to verify that 

    \[
        \costf{F_a} = pT^2 + \frac{1}{2}\sum_{\ell = 1}^n|R_{a,\ell}|(T - |R_{a,\ell}|)
    \]
    where $R_{a,\ell} = R_\ell \cap F_a$.

    In the above expression, $pT^2$ counts the number of blue-red pairs in $F_a$, note by assuming \cref{clm:np-main-claim} we have $|\blue{F_a}| = pT$ and since $F_a$ is fair we have $|\red{F_a}| = T$. In the above expression $(1/2 \cdot \sum_{\ell = 1}^n|R_{a,\ell}|(T - |R_{a,\ell}|))$ counts the number of red-red pairs $(u,v) \in F_a$ such that $u \in R_{\ell_1}$ and $v \in R_{\ell_2}$ for $\ell_1 \neq \ell_2$.

    It is straightforward to observe

    \[
        \dist(\m{C}, \m{F}^*) \geq \sum_{F_a \in \m{F}^*} \costf{F_a}.
    \]

    Now we prove 

    \[
        \sum_{F_a \in \m{F}^*} \costf{F_a} > \tau.
    \]

    For that purpose, we use the following claim, the proof of which is provided later.
    \begin{claim}\label{clm:exists-ell-lese-ml}
        For each $\ell \in [n]$, let $m_\ell = \max_{a: F_a \in \m{F}^*}|R_{a,\ell}|$. There exists an $\ell \in [n]$ such that $m_{\ell} < x_{\ell}$.
    \end{claim}

    So, from now on, we assume the above claim and proceed with arguing $\sum_{F_a \in \m{F}^*} \costf{F_a} > \tau$.

    \begin{align}
        \sum_{F_a \in \m{F}^*} \costf{F_a} &= \sum_{F_a \in \m{F}^*} \left( pT^2 + \frac{1}{2}\sum_{\ell = 1}^n|R_{a,\ell}|(T - |R_{a,\ell}|) \right) \n \\
        &= \frac{n}{3}pT^2 + \frac{1}{2}\sum_{F_a \in \m{F}^*}\sum_{\ell = 1}^n|R_{a,\ell}|(T - |R_{a,\ell}|) \n \\
        &\geq \frac{n}{3}pT^2 + \sum_{\ell = 1}^n (T - m_\ell) \sum_{F_a \in \m{F}^*}|R_{a,\ell}| \, \, \, (\text{recall}\, \, m_\ell = \max_a|R_{a,\ell}|) \n \\
        &= \frac{n}{3}pT^2 + \sum_{\ell = 1}^n (T - m_\ell) |R_\ell| \n \\
        &= \frac{n}{3}pT^2 + \sum_{\ell = 1}^n (T - m_\ell) x_\ell \, \, \, (\text{recall}\, \, x_\ell = |R_{\ell}|) \n \\
        &> \frac{n}{3}pT^2 + \sum_{\ell = 1}^n (T - x_\ell) x_\ell \n = \tau \n
    \end{align}

    where the last inequality follows from \cref{clm:exists-ell-lese-ml}. Thus we get that $\dist(\m{C}, \m{F}^*) > \tau$, which implies $(\m{C},\tau) $ is a NO instance of the $ \clof$.
    \end{proof}

    We are now left with proving \cref{clm:np-main-claim} and \cref{clm:exists-ell-lese-ml}. Let us first show \cref{clm:exists-ell-lese-ml}.

    \begin{proof}[Proof of \cref{clm:exists-ell-lese-ml}]
        For the sake of contradiction, suppose for all $\ell$, $m_\ell = x_\ell$. Note $m_\ell \leq x_\ell$ because
        \begin{align}
            &m_\ell = \max_a|R_{a,\ell}|,\, \, \text{and} \n \\
            &\sum_{F_a \in \m{F}^*}|R_{a,\ell}| = x_\ell. \label{equn:np-claim-two} 
        \end{align}

        If $m_\ell = x_\ell$, then it immediately follows from the definition,

        \begin{equation}
            \label{eq:np-sta-one}
            \exists a: |R_{a,\ell}| = x_\ell, \text{ and } \forall_{b \ne a}, |R_{b,\ell}| = 0.
        \end{equation}

        Now we create partitions $S_a$ of $S$ in the following way: For each $F_a \in \m{F}^*$,

        \[
            S_a = \{ x_\ell \mid |R_{a,\ell}| = x_\ell \}.
        \]

        Since by \cref{clm:np-main-claim}, there are $n/3$ clusters $F_a \in \m{F}^*$, the number of such subsets would be $n/3$. Clearly, the above subsets $S_a$'s create a partitioning of $S$, that is

        \begin{itemize}
            \item $S_a \cap S_b = \emptyset$ for $a \neq b$, follows from \cref{eq:np-sta-one};
            \item $\bigcup S_a = S$, follows from \cref{equn:np-claim-two} (that says $R_{a,\ell}$ cannot be empty for all $a$).  
        \end{itemize}

        Since for each $a$, $F_a$ is a fair cluster containing $pT$ blue points (by \cref{clm:np-main-claim}) and $R_{a,\ell} = R_\ell \cap F_a$, we have that for all $a$,
        
        \[
            \sum_{x_\ell \in S_a}x_\ell = T.
        \]

        Further, since $x_\ell \in (T/4,T/2)$, we must have $|S_a| = 3$. However, since $S$ is a NO instance of $\thrp$, such a partitioning of $S$ does not exist, leading to a contradiction.

        Hence, there must exist an $\ell$ such that $m_\ell < x_\ell$.
    \end{proof}

    It only remains to prove \cref{clm:np-main-claim}.

    Recall here, our goal is to prove that for every $F_i \in \m{F^*}$, there exists a monochromatic blue cluster $B_j$ such that $\blue{F_i} = B_j$. Before proceeding with the proof, we briefly outline the main steps. First, we show that for any $F_i \in \m{F^*}$, if there exists a monochromatic blue cluster $B_j$ such that $|B_j \cap F_i|$ is a multiple of $p$, then for all $k \neq j$, we must have $B_k \cap F_i = \emptyset$. Next, we prove that for every $F_i \in \m{F^*}$, there exists some $B_j$ such that $|B_j \cap F_i|$ is a multiple of $p$. Combining these two claims, we conclude the proof of \cref{clm:np-main-claim}.

    \paragraph{Proof of \cref{clm:np-main-claim}.} For a cluster $B_i$, suppose it is partitioned into $Y_i^{1}, Y_i^{2}, \ldots, Y^{t}_i$ in $\m{F}^*$ that is
    \begin{itemize}
        \item $Y^{k}_i \subseteq F$ for some $F \in \m{F}^*$ for all $k \in [t]$
        \item For $k \neq \ell \in [t]$ if $Y^{k}_i \subseteq F_a$ and $Y^{\ell}_i \subseteq F_b$ then $a \neq b$ for some $F_a, F_b \in \m{F}^*$.
        \item $Y^{k}_i \cap Y^{\ell}_i = \emptyset$  for  $k \neq \ell$.
        \item $\bigcup_k Y_i^{k} = B_i$.
    \end{itemize}

    Now, we argue that if for a cluster $B_i$, there exists $Y^{k}_i \subseteq F$ for some $F \in \m{F}^*$ such that $p$ divides $|Y^{k}_i|$, then $F$ does not contain blue points from any other cluster $B_j$.

    \begin{claim}\label{clm:claim-np}
        Consider any integer $ p > 1$. Suppose for a cluster $B_i$, there exists $Y^{k}_i \subseteq F$ for some $F \in \m{F}^*$ such that $p$ divides $|Y^{k}_i|$, then $\blue{F} = Y^{k}_i$.
    \end{claim}

    \begin{proof}
        For the sake of contradiction, suppose there exists a cluster $F \in \m{F}^*$ such that $Y^{k}_i \subseteq F$ and $p$ divides $|Y^{k}_i|$, but $B_j \cap F \neq \emptyset$ for some $j \neq i$. In this case, we can construct a {\fair} $\m{M}$ such that $\dist(\m{C}, \m{M}) < \dist(\m{C}, \m{F}^*)$ which contradicts the fact that $\m{F}^*$ is a closest {\fair}. 

        \textbf{Construction of $\m{M}$ from $\m{F}^*$}: Let $Q \subseteq \red{F}$ such that $|Q| = |Y^{k}_i|/p$.

        \[
            \m{M} = \m{F}^* \setminus \{ F\} \cup \{ (Y_i^k \cup Q), F \setminus (Y^k_i \cup Q)\}.
        \]

        That is, in $\m{M}$ we remove the cluster $F$ from $\m{F}^*$ and add two clusters $(Y_i^k \cup Q)$ and $(F \setminus (Y^k_i \cup Q))$.

        Since $F$ is a fair cluster, both the clusters $(Y_i^k \cup Q)$ and $F \setminus (Y_i^k \cup Q)$ are fair.

        Let us now argue that $\dist(\m{C},\m{M}) < \dist(\m{C},\m{F}^*)$. Roughly, the above construction only incurs an additional cost for pairs between $Q$ and $\red{F} \setminus Q$, while saving a cost for pairs between $Q$ and $\blue{F} \setminus Y_i^k$ (note, $\blue{F} \setminus Y_i^k = \blue{F} \setminus B_i$). Further, note that $|\blue{F} \setminus B_i| = p |\red{F} \setminus Q|$. 

        Thus,
        \begin{align*}
            \dist(\m{C}, \m{M}) & \le \dist(\m{C}, \m{F}^*) + |Q| |\red{F} \setminus Q| - p|Q| |\red{F} \setminus Q|\\
            &< \dist(\m{C}, \m{F}^*) &&\text{(for $p > 1$)}.
        \end{align*}
    \end{proof}

    Now, we prove that for all $B_i$ and for all $k \in [t]$, $|Y^{k}_i|$ must be a multiple of $p$. To prove this, we first establish the following two claims. 

    \begin{claim}\label{clm:np-small-claim-one}
        Consider any integer $p > 1$. If for a cluster $F \in \m{F}^*$, $|B_i \cap \blue{F}| < p$ for all $B_i$, then $|\blue{F}| \leq p$.
    \end{claim}

    \begin{proof}
        Consider any cluster $F \in \m{F}^*$. For any $B_i$, let $Y_i := B_i \cap \blue{F} $. WLOG assume $|Y_1| \ge |Y_2| \ge \cdots \ge |Y_n|$. No, suppose for all $i \in [n]$, $|Y_i| < p$. 

        For the sake of contradiction, suppose $|\blue{F}| > p$. Then, there must exist an index $x \in [n]$ such that

        \begin{align}
            &\sum_{i = 1}^{x - 1} |Y_i| \leq p < \sum_{i = 1}^x |Y_i|.\n
        \end{align}

        Note $x \geq 2$ because $|Y_i| < p$ for all $i$.

        Let $Y_x' \subseteq Y_x$ such that $\sum_{i = 1}^{x - 1} |Y_i| + |Y_x'| = p$. Now, since $|Y_1| \ge |Y_x| \ge |Y'_x|$, $|Y_x'| \leq p/2$.

        Next, we construct $\m{M}$ from $\m{F}^*$ in the following way

        \[
            \m{M} = \m{F}^* \setminus \{ F \} \cup \left\{ \left(\bigcup_{i = 1}^{x - 1} Y_i \cup Y_x' \cup Q \right), \left(F \setminus \left(\bigcup_{i = 1}^{x - 1} Y_i \cup Y_x' \cup Q \right) \right) \right\}
        \]

        where $Q \subseteq \red{F}$ such that $|Q| = 1$.

        That is, in $\m{M}$ we remove the set $F$ from $\m{F}^*$ and add two clusters $\left(\bigcup_{i = 1}^{x - 1} Y_i \cup Y_x' \cup Q \right)$ and $\left(F \setminus \left(\bigcup_{i = 1}^{x - 1} Y_i \cup Y_x' \cup Q \right) \right)$.

        It is easy to see that $\m{M}$ is a {\fair}. Next,

        \begin{align}
            \dist(\m{C}, \m{M}) &\le \dist(\m{C}, \m{F}^*) + |Y_x'| |Y_x \setminus Y_x'| + |Q||\red{F} \setminus Q| - \left|\bigcup_{i = 1}^{x - 1}Y_i \right| |Y_x \setminus Y_x'| - |Q| \left| \blue{F} \setminus \left( \bigcup_{i = 1}^{x - 1} Y_i \cup Y_x'\right) \right| \n \\
            &\leq \dist(\m{C}, \m{F}^*) + |Y_x'| (|Y_x| - |Y_x'|) + |Q| |\red{F} \setminus Q| - (p - |Y_x'|) (|Y_x| - |Y_x'|)  - |Q|\cdot  p |\red{F} \setminus Q| \n \\
            &(\text{as} \, \, \left| \blue{F} \setminus \left( \bigcup_{i = 1}^{x - 1} Y_i \cup Y_x'\right) \right| \geq p |\red{F} \setminus Q|) \n \\
            &< \dist(\m{C}, \m{F}^*) \, \, \, (\text{as} \, \, |Y_x'| \leq \frac{p}{2}\, \, \text{and}\, \, |\red{F} \setminus Q| < p |\red{F} \setminus Q|\, \, \text{for $p > 1$} ) \n
        \end{align}
        contradicting the fact that $\m{F}^*$ is a closest {\fair} to $\m{C}$.

        \begin{itemize}
            \item \textbf{Reason behind the term $|Y_x'| |Y_x \setminus Y_x'|$}: The pairs $(u,v)$ such that $u \in Y_x'$ and $v \in Y_x \setminus Y_x'$ are not counted in $\dist(\m{C}, \m{F}^*)$ because the pairs $u$ and $v$ may lie in the same cluster both in $B_i \in \m{C}$ and $F \in \m{F}^*$. These pairs are counted in $\dist(\m{C}, \m{M})$ because they lie in different clusters in $\m{M}$. 
            \item \textbf{Reason behind the term $|Q||\red{F} \setminus Q|$}: The pairs $(u,v)$ such that $u \in Q$ and $v \in \red{F} \setminus Q$ may not be counted in $\dist(\m{C}, \m{F}^*)$ because the pairs $u$ and $v$ may lie in the same cluster both in $\m{C}$ and $\m{F}^*$. These pairs are counted in $\dist(\m{C}, \m{M})$ because they lie in different clusters in $\m{M}$. 
            \item \textbf{Reason behind the term $\left|\bigcup_{i = 1}^{x - 1}Y_i \right| |Y_x \setminus Y_x'| $}: First note that, $\bigcup_{i = 1}^{x - 1}Y_i \neq \emptyset$ as $x \geq 2$. The pairs $(u,v)$ such that $u \in \bigcup_{i = 1}^{x - 1}Y_i$ and $v \in Y_x \setminus Y_x'$ are counted in $\dist(\m{C}, \m{F}^*)$ because the pairs $(u, v)$ lie in different clusters in $\m{C}$ but in the same cluster $F \in \m{F}^*$. These pairs are not counted in $\dist(\m{C}, \m{M})$ because they lie in different clusters both in $\m{C}$ and $\m{M}$. 
            \item \textbf{Reason behind the term $|Q| \left| \blue{F} \setminus \left( \bigcup_{i = 1}^{x - 1} Y_i \cup Y_x'\right) \right| $}: The pairs $(u,v)$ such that $u \in Q$ and $v \in \blue{F} \setminus \left( \bigcup_{i = 1}^{x - 1} Y_i \cup Y_x'\right)$ are counted in $\dist(\m{C}, \m{F}^*)$ because the pairs $(u, v)$ lie in different clusters in $\m{C}$ but in the same cluster $F \in \m{F}^*$. These pairs are not counted in $\dist(\m{C}, \m{M})$ because they lie in different clusters both in $\m{C}$ and $\m{M}$. 
        \end{itemize}
    \end{proof}

    \begin{claim}\label{clm:np-small-claim-two}
        Consider any integer $p > 1$. If there exists a cluster $B_i$ such that $|B_i \cap F| \geq p$ for some $F \in \m{F}^*$, then $|\blue{F} \setminus B_i| \leq p - (|B_i \cap F| \mod p)$.
    \end{claim}

    \begin{proof}
        Let $Y_i = B_i \cap F$. Suppose for the sake of contradiction, $|\blue{F} \setminus B_i| > p - (|Y_i| \mod p)$.

        In that case, we construct $\m{M}$ from $\m{F}^*$ in the following way:

        \[
            \m{M} = \m{F}^* \setminus \{ F\} \cup \{ (Y_i \cup W \cup Q), F \setminus (Y_i \cup W \cup Q) \}
        \]
        where $W \subseteq \blue{F} \setminus Y_i$ such that $|W| = p - (|Y_i| \mod p)$ and $Q \subseteq \red{F}$ such that $|Q| = |Y_i \cup W|/p$ (note, $|Y_i \cup W|$ is a multiple of $p$).

        That is, in $\m{M}$ we remove the set $F$ from $\m{F}^*$ and add two clusters $(Y_i \cup W \cup Q)$ and $F \setminus (Y_i \cup W \cup Q)$. By the construction, $\m{M}$ is a {\fair}.

        Observe that $|Y_i|=|B_i \cap F| \ge p > |W|$. Then,

        \begin{align}
            \dist(\m{C}, \m{M}) &\le \dist(\m{C}, \m{F}^*) + |W| |\blue{F} \setminus (Y_i \cup W)| + |Q| |\red{F} \setminus Q| \n \\
            &- |Y_i||\blue{F} \setminus (Y_i \cup W)| - |Q| |\blue{F} \setminus (Y_i \cup W)| \n \\
            &\leq \dist(\m{C}, \m{F}^*) + |W| |\blue{F} \setminus (Y_i \cup W)| + |Q| |\red{F} \setminus Q|\n\\
            &- |Y_i||\blue{F} \setminus (Y_i \cup W)| - |Q| \cdot p |\red{F} \setminus Q| \n \\
            &< \dist(\m{C}, \m{F}^*)  \, \, (\text{since} \, \, p > 1 \, \, \text{and} \, \, |Y_i| > |W|)\n
        \end{align}
        contradicting the fact that $\m{F}^*$ is a closest {\fair} to $\m{C}$.

        \begin{itemize}
            \item \textbf{Reason behind the term $|W| |\blue{F} \setminus (Y_i \cup W)|$}: The pairs $(u,v)$ such that $u \in W$ and $v \in \blue{F} \setminus (Y_i \cup W)$ may not be counted in $\dist(\m{C}, \m{F}^*)$ because they may lie in same clusters both in $\m{C}$ and $\m{F}^*$. They are counted in $\dist(\m{C}, \m{M})$ because they lie in different clusters in $\m{M}$.
            \item \textbf{Reason behind the term $|Q| |\red{F} \setminus Q|$}: The pairs $(u,v)$ such that $u \in Q$ and $v \in \red{F} \setminus Q$ may not be counted in $\dist(\m{C}, \m{F}^*)$ because they may lie in same clusters both in $\m{C}$ and $\m{F}^*$. They are counted in $\dist(\m{C}, \m{M})$ because they lie in different clusters in $\m{M}$.
            \item \textbf{Reason behind the term $|Y_i||\blue{F} \setminus (Y_i \cup W)|$}: The pairs $(u,v)$ such that $u \in Y_i$ and $v \in \blue{F} \setminus (Y_i \cup W)$ are counted in $\dist(\m{C}, \m{F}^*)$ because they lie in different clusters in $\m{C}$ but the same cluster $F \in \m{F}^*$. They are not counted in $\dist(\m{C}, \m{M})$ because they lie in different clusters both in $\m{C}$ and $\m{M}$.
            \item \textbf{Reason behind the term $|Q| |\blue{F} \setminus (Y_i \cup W)| $}: The pairs $(u,v)$ such that $u \in Q$ and $v \in \blue{F} \setminus (Y_i \cup W)$ are counted in $\dist(\m{C}, \m{F}^*)$ because they lie in different clusters in $\m{C}$ but in the same cluster $F \in \m{F}^*$. They are not counted in $\dist(\m{C}, \m{M})$ because they lie in different clusters both in $\m{C}$ and $\m{M}$.
        \end{itemize}
    \end{proof}

    We are now ready to prove that for all $B_i$, for all $k \in [t]$, $|Y^k_i|$ must be a multiple of $p$.

    \begin{claim}\label{clm:np-all-partition-mop}
        Consider any integer $p > 1$. For all $B_i$, for all $k \in [t]$, $|Y^{k}_i|$ is a multiple of $p$.
    \end{claim}

    \begin{proof}
        Suppose there exists $m_1,\ldots,m_z$ such that $|Y_i^{m_w}|$ is not a multiple of $p$ for $w \in [z]$. In this case, we can construct a {\fair} $\m{M}$ such that $\dist(\m{C}, \m{M}) < \dist(\m{C}, \m{F}^*)$ which would contradict $\m{F}^*$ is a closest {\fair} to $\m{C}$. 
        
        Let, $Y_i^{m_w} \subseteq F_w \in \m{F}^*$ for all $w \in [z]$. WLOG, assume $|Y_i^{m_x}| \geq |Y_i^{m_y}|$ if $m_x < m_y$. Thus $Y_i^{m_1}$ is the partition of $B_i$ of largest size.  

        Before providing the construction of $\m{M}$ let us prove there exists a $B_i$ such that $|\blue{F_1}\setminus Y_i^{m_1}| \leq \left|\bigcup_{w=2}^z Y_i^{m_w} \right|$.

        \begin{itemize}
            \item \textbf{Case 1:} $\exists B_i$, $|Y_i^{m_1}| \geq p$: 
            
            In this case we know by \cref{clm:np-small-claim-two} that $|\blue{F_1} \setminus Y^{m_1}_i| \leq p - (|Y_i^{m_1}| \mod p)$.

        Further, note that $\left| \bigcup_{w=2}^z Y_i^{m_w} \right| \geq p - (|Y_i^{m_1}| \mod p)$. Hence, $|\blue{F_1}\setminus Y_i^{m_1}| \leq \left|\bigcup_{w=2}^z Y_i^{m_w} \right|$.

             \item \textbf{Case 2:} For all $B_j$, $|Y_j^{m_1}| < p$:

            Fix any cluster $B_i$. 

            Since for all $B_j$, $|Y_j^{m_1}| < p$, we get that for all $B_j$, $|B_j \cap F_1| < p$. Then by \cref{clm:np-small-claim-one}, we get $|\blue{F_1}| \leq p$.

            Thus, $|\blue{F_1} \setminus Y_i^{m_1}| \leq p - |Y^{m_1}_i|$.

            Further note that, $|\bigcup_{w = 2}^z Y_i^{m_w}| \geq p - |Y_i^{m_1}|$.

            Hence, $|\blue{F_1}\setminus Y_i^{m_1}| \leq \left|\bigcup_{w=2}^z Y_i^{m_w} \right|$.

        \end{itemize}    

        Next, we provide a construction of $\m{M}$ from $\m{F}^*$. Consider a cluster $B_i$ for which $|\blue{F_1}\setminus Y_i^{m_1}| \leq \left|\bigcup_{w=2}^z Y_i^{m_w} \right|$, the existence of which is already argued. In this case, informally, we want to construct a clustering $\m{M}$ by swapping the locations of blue points in $\blue{F_1}\setminus Y_i^{m_1}$ with an equal number of blue points present in $\bigcup_{w=2}^z Y_i^{m_w}$. We choose an equal number of blue points to $|\blue{F_1}\setminus Y_i^{m_1}|$ in $\bigcup_{w=2}^z Y_i^{m_w}$ greedily starting from $Y_i^{m_2}$.

        Formally, we do the following operation: First, we find an index $x$ such that

        \begin{align}
            \sum_{w = 2}^{x - 1} |Y_i^{m_w}| <|\blue{F_1} \setminus Y_i^{m_1}| \leq \sum_{w = 2}^x |Y_i^{m_w}|.\n 
        \end{align}

        Let, $Y_i^{m_x'} \subseteq Y_i^{m_x}$ such that
        \[
           |\blue{F_1} \setminus Y_i^{m_1}| = \sum_{w = 2}^{x-1}|Y_i^{m_w}| + |Y_i^{m_x'}|. \n
        \]

        Then, we divide the points in $\blue{F_1} \setminus Y_i^{m_1}$ into $(x - 1)$ disjoint subsets $Z_1, Z_2, \ldots, Z_{x-2},Z_{x-1}$ such that for $1 \leq j \leq (x-2)$ we have $|Z_j| = |Y_i^{m_{j + 1}}|$ and $|Z_{x-1}| = |Y^{m_x'}_i|$.

        \textbf{Construction of $\m{M}$ from $\m{F}^*$}:

        \begin{align}
        \m{M} = &\m{F}^* \setminus \{ F_w \mid 1 \leq w \leq x\} \cup \{ (F_w \setminus Y_i^{m_w}) \cup Z_{w - 1} \mid 2 \leq w \leq (x - 1)\} \n \\
        &\cup \{ (F_x \setminus Y^{m_x'}_i) \cup Z_{x - 1} \} \cup \left\{ F_1 \setminus (\blue{F_1} \setminus Y_i^{m_1}) \cup \left(\bigcup_{w = 2}^{x -1}Y_i^{m_w} \right) \cup Y_i^{m_x'}\right\} \n
        \end{align}

        That is, in $\m{M}$, from each $F_w$ for $2 \leq w \leq x-1 $ we remove its subset $Y^{m_w}_i$ and add the set $Z_{w-1}$ to it. From $F_x$ we remove its subset $Y_i^{m_x'}$ and add the set $Z_{x-1}$ to it. From $F_1$ we remove its subset $(\blue{F_1} \setminus Y^{m_1}_1)$ and add the set $\left(\bigcup_{w = 2}^{x -1}Y_i^{m_w} \right) \cup Y_i^{m_x'}$ to it. It is easy to observe that $\m{M}$ is a {\fair}.

        Now,

        \begin{align}
            \dist(\m{C},\m{M}) &\le \dist(\m{C},\m{F}^*) + |Z_{x - 1}||Y_i^{m_x} \setminus Y_i^{m_x'}| + |Y_i^{m_x'}||Y_i^{m_x} \setminus Y_i^{m_x'}|\n\\
            &+ \left(\sum_{r<w \in [x-2]}|Z_r| |Z_w|  + |Z_{x-1}| \sum_{w \in [x-2]}|Z_w|\right) \n \\
            &- |Z_{x - 1}||Y^{m_1}_i| - |Y_i^{m_x'}||Y^{m_1}_i| - \left( \sum_{r < w \in \{2,\cdots,x - 1\}} |Y_i^{m_r}| |Y_i^{m_w}| + |Y_i^{m_x'}| \sum_{w \in \{2,\cdots,x - 1\}} |Y_i^{m_w}|\right) \n \\
             &< \dist(\m{C}, \m{F}^*) \n 
             \end{align}
             where the last inequality follows since

             \begin{align*}
                 &|Y^{m_1}_i| > (|Y^{m_x}_i \setminus Y^{m_x'}_i|,\; \text{ and}\\
                 &\sum_{r<w \in [x - 2]}|Z_r| |Z_w|  + |Z_{x-1}| \sum_{w \in [x-2]}|Z_w|  =  \sum_{r < w \in \{2,\cdots,x - 1\}} |Y_i^{m_r}| |Y_i^{m_w}| + |Y_i^{m_x'}| \sum_{w \in \{2,\cdots,x - 1\}} |Y_i^{m_w}|.
             \end{align*}
        This contradicts the fact that $\m{F}^*$ is a closest {\fair} to $\m{C}$.

        \begin{itemize} 
            \item \textbf{Reason behind the term $|Z_{x - 1}||Y_i^{m_x} \setminus Y_i^{m_x'}|$:} In $\dist(\m{C}, \m{F}^*)$ the pairs $(u,v)$ such that $u \in Z_{x - 1}$ and $v \in Y_i^{m_x}\setminus Y_i^{m_x'}$ are not counted because they are in different clusters both in $\m{C}$ and $\m{F}^*$ but these pairs are counted in $\dist(\m{C}, \m{M})$ because they are in the same cluster $\{ (F_x \setminus Y^{m_x'}_i) \cup Z_{x - 1} \}$ in $\m{M}$.

            \item \textbf{Reason behind the term $|Y_i^{m_x'}||Y_i^{m_x} \setminus Y_i^{m_x'}|$:} In $\dist(\m{C}, \m{F}^*)$ the pairs $(u,v)$ such that $u \in Y_i^{m_x'}$ and $v \in Y_i^{m_x} \setminus Y_i^{m_x'}$ are not counted because they are in same clusters both in $\m{C}$ and $\m{F}^*$ but these pairs are counted in $\dist(\m{C}, \m{M})$ because they are in different clusters in $\m{M}$.

            \item \textbf{Reason behind the term $\left(\sum_{r<w \in [x - 2]}|Z_r| |Z_w|  + |Z_{x-1}| \sum_{w \in [x-2]}|Z_w|\right)$:} In $\dist(\m{C}, \m{F}^*)$ the pairs $(u,v)$ such that $u \in Z_r$ and $v \in Z_w$ for $r < w$ may not be counted because they may be in same clusters both in $\m{C}$ and $\m{F}^*$ but these pairs are counted in $\dist(\m{C}, \m{M})$ because they are in different clusters in $\m{M}$.

            \item \textbf{Reason behind the term $|Z_{x - 1}||Y^{m_1}_i|$:} In $\dist(\m{C}, \m{F}^*)$ the pairs $(u,v)$ such that $u \in Z_{x - 1}$ and $v \in Y_i^{m_1}$ are counted because they are in different clusters in $\m{C}$ but in the same cluster $F_1$ in $\m{F}^*$ but these pairs are not counted in $\dist(\m{C}, \m{M})$ because they are in different clusters both in $\m{C}$ and $\m{M}$.

            \item \textbf{Reason behind the term $|Y_i^{m_x'}||Y^{m_1}_i|$:} In $\dist(\m{C}, \m{F}^*)$ the pairs $(u,v)$ such that $u \in Y_i^{m_x'}$ and $v \in Y_i^{m_1}$ are counted because they are in different clusters in $\m{F}^*$ but in the same cluster $B_i$ in $\m{C}$ but these pairs are not counted in $\dist(\m{C}, \m{M})$ because they are in the same clusters both in $\m{C}$ and $\m{M}$.

            \item \textbf{Reason behind the term $\left( \sum_{r < w \in \{2,\cdots,x - 1\}} |Y_i^{m_r}| |Y_i^{m_w}| + |Y_i^{m_x'}| \sum_{w \in \{2,\cdots,x - 1\}} |Y_i^{m_w}|\right)$:} In $\dist(\m{C}, \m{F}^*)$ the pairs $(u,v)$ such that $u \in Y_i^{m_r}$ and $v \in Y_i^{m_w}$ for $r < w$ where $r,w \in \{ 2,\ldots, (x-1) \}$ or $u \in Y_i^{m_x'}$ and $v \in Y_i^{m_w}$ where $w \in \{ 2,\ldots, x -1 \}$ are counted because they are in the same cluster $B_i$ in $\m{C}$ but different clusters in $\m{F}^*$. These pairs are not counted in $\dist(\m{C}, \m{M})$ because they are in the same cluster 
            in $\m{M}$.
        \end{itemize}
        
    \end{proof}

    Next, we complete the proof of \cref{clm:np-main-claim}.

    \begin{proof}[Proof of \cref{clm:np-main-claim}]
        Suppose a cluster $B_i$ is split into $t$ parts $Y_i^1, Y_i^2, \ldots, Y_i^t$ in $\m{F}^*$. In this claim, we would like to show that $t = 1$. For the sake of contradiction, assume $t \geq 2$.  By \cref{clm:np-all-partition-mop} we get $p$ divides $|Y_i^k|$, for all $k \in [t]$. 
        
        For each $k \in [t]$, let $Y_i^{k} \subseteq F_{a_k}$ for some $F_{a_k} \in \m{F}^*$. By \cref{clm:claim-np}, we know that $\blue{F_{a_k}} = Y_i^k$. Next, we argue that since $t \geq 2$, we can construct a {\fair} $\m{M}$ from $\m{F}^*$ such that $\dist(\m{C}, \m{M}) < \dist(\m{C}, \m{F}^*)$ contradicting the fact that $\m{F}^*$ is a closest {\fair} to $\m{C}$.

        \textbf{Construction of $\m{M}$:}

        \[
            \m{M} = \m{F}^* \setminus \{ F_{a_k} \in \m{F}^* \mid  k \in [t]\} \cup \left\{ \bigcup_{k \in [t]} F_{a_k} \right\}.
        \]

    That is, while constructing $\m{M}$, we remove all the clusters $F_a$ from $\m{F}^*$ if $Y_i^{k} \subseteq F_a$ for some $k \in [t]$ and add one cluster by merging all these clusters. By the construction, clearly $\m{M}$ is a {\fair}.

    Next, we argue that $\dist(\m{C}, \m{M}) < \dist(\m{C}, \m{F}^*)$ for $ p\ge 3$. 

    \begin{align}
        \dist(\m{C}, \m{M}) &\le \dist(\m{C}, \m{F}^*) + \sum_{k \neq \ell} |\blue{F_{a_k}}| |\red{F_{a_\ell}}| + \sum_{k < \ell} |\red{F_{a_k}}||\red{F_{a_\ell}}| - \sum_{k < \ell} |\blue{F_{a_k}}||\blue{F_{a_\ell}}|\n \\
        &= \dist(\m{C}, \m{F}^*) + \sum_{k \ne \ell} \frac{|Y_i^k||Y_i^\ell|}{p} + \sum_{k < \ell} \frac{|Y_i^k|}{p}\frac{|Y_i^\ell|}{p} - \sum_{k < \ell} |Y_i^k||Y_i^\ell| \n \\
        &\left(\text{as }\forall_k \, \, \blue{F_{a_k}} = Y_i^k, \text{by \cref{clm:claim-np}, and } \, \, |\red{F_{a_k}}| = \frac{|Y_i^k|}{p}\right) \n \\
        &= \dist(\m{C}, \m{F}^*) + 2\sum_{k < \ell} \frac{|Y_i^k||Y_i^\ell|}{p} + \sum_{k < \ell} \frac{|Y_i^k|}{p}\frac{|Y_i^\ell|}{p} - \sum_{k < \ell} |Y_i^k||Y_i^\ell| \n \\
        &< \dist(\m{C}, \m{F}^*) \, \, \, (\text{for} \, \, p \geq 3) .\n
    \end{align}

    \begin{itemize}
        \item \textbf{Reason behind the term $\sum_{k \neq \ell} |\blue{F_{a_k}}| |\red{F_{a_\ell}}|$}: The pairs $(u,v)$ such that $u \in \blue{F_{a_k}}$ and $v \in \red{F_{a_\ell}}$ (for $k \neq \ell$) may not be counted in $\dist(\m{C}, \m{F}^*)$ because $u$ and $v$ may lie in different clusters both in $\m{C}$ and $\m{F}^*$. These pairs are counted in $\dist(\m{C}, \m{M})$ because they lie in the same cluster in $\m{M}$.
        \item \textbf{Reason behind the term $\sum_{k < \ell} |\red{F_{a_k}}||\red{F_{a_\ell}}|$}: The pairs $(u,v)$ such that $u \in \red{F_{a_k}}$ and $v \in \red{F_{a_\ell}}$ (for $k < \ell$) may not be counted in $\dist(\m{C}, \m{F}^*)$ because $u$ and $v$ may lie in different clusters both in $\m{C}$ and $\m{F}^*$. These pairs are counted in $\dist(\m{C}, \m{M})$ because they lie in the same cluster in $\m{M}$.
        \item \textbf{Reason behind the term $\sum_{k < \ell} |\blue{F_{a_k}}||\blue{F_{a_\ell}}|$}: The pairs $(u,v)$ such that $u \in \blue{F_{a_k}}$ and $v \in \blue{F_{a_\ell}}$ (for $k < \ell$) are counted in $\dist(\m{C}, \m{F}^*)$ because $u$ and $v$ lie in the same cluster $B_i \in \m{C}$ but different clusters in $\m{F}^*$. These pairs are not counted in $\dist(\m{C}, \m{M})$ because they lie in the same cluster in $\m{M}$.
    \end{itemize}

\end{proof}

We now conclude the proof of~\cref{lem:main-one}.

\begin{proof}[Proof of \cref{lem:main-one}]
    By~\cref{lem:yes-instance} and~\cref{lem:no-instance} we infer that for any $p \geq 3$, $S $ is a YES instance of the $ \thrp$ if and only if $(\m{C}, \tau) $ is a YES instance of the $ \clof$. It now immediately follows from our reduction that for any $p \geq 3$, $\clof$ is \textbf{NP}-hard because $\thrp$ is strongly $\npc$ (\cref{thm:garey-johnson}).
\end{proof}

\begin{remark}\label{remark:np-hard}
    We would like to remark that the above \textbf{NP}-hardness proof can be extended for any fixed $p/q$ ratio on the input blue to red points as long as $p/q > 1+\sqrt{2}$. More specifically, consider any two integers $p,q \ge 2$ such that $p/q > 1+\sqrt{2}$. Then the closest fair clustering problem on red-blue colored points where the ratio between the total number of blue and red points is $p/q$, is \textbf{NP}-hard. We can show this result by making the following modifications to the reduction: Take monochromatic red clusters $R_i$'s of size $q x_i$, set $\tau = \frac{q^2}{2} \sum_{j = 1}^n x_j \left(T - x_j\right) + \frac{n}{3}pqT^2$; everything else remain the same. The proof argument will be similar. One important point to note is that $p/q > 1+\sqrt{2}$ is crucial for the argument used in the proof of \cref{clm:np-main-claim}.
\end{remark}

Now in the following subsection, we prove that $\clof$ is \textbf{NP}-hard also for $p = 2$.

\subsection{\texorpdfstring{Hardness of $\clof$ for $p = 2$}{Hardness of clof for p = 2}}

In this subsection, we prove the following lemma.

\begin{lemma}\label{lem:main-two}
    For $p = 2$, the $\clof$ problem is \textbf{NP}-hard.
\end{lemma}

We again reduce the $\thrp$ problem to the $\clof$ problem for the case when $p = 2$. The reduction follows the same structure as in the case where $p \geq 3$, with the only difference being the choice of the parameter $\tau$. For completeness, we restate the reduction here.

\paragraph{Reduction from $\thrp$ to $\clof$:}  
Consider a $\thrp$ instance $S = \{ x_1, x_2, \ldots, x_n \}$, where for all $i \in [n]$, it holds that $x_i \le n^c$ for some fixed non-negative constant $c$. Given this instance, we construct a corresponding $\clof$ instance $(\mathcal{C}, \tau)$ as follows:

\begin{itemize}
    \item Define the set of clusters as  
    \[
        \mathcal{C} = \{ B_1, \ldots, B_{n/3}, R_1, \ldots, R_n \},
    \]
    where each $B_i$ (for $i \in \{1, \ldots, n/3\}$) is a \emph{monochromatic blue} cluster of size $pT$, and each $R_j$ (for $j \in \{1, \ldots, n\}$) is a \emph{monochromatic red} cluster of size $x_j$, i.e., $|R_j| = x_j$.

    \item Set the value of $\tau$ as  
    \[
        \tau = p\sum_{\ell = 1}^n x_\ell^2 + \frac{p^2}{2} \sum_{\ell = 1}^n x_\ell(T - x_\ell)
    \]
\end{itemize}

Since, for all $i \in [n]$, it holds that $x_i \le n^c$ for some fixed non-negative constant $c$, it is straightforward to see that the above reduction is a polynomial time reduction.

Now, we argue that the above reduction maps a YES instance of the $\thrp$ to a YES instance of the $\clof$.

\begin{lemma}\label{lem:yes-instance-p=2}
    If $S$ is a YES instance of the $\thrp$, then $(\m{C}, \tau) $ is also a YES instance of the $ \clof$.
\end{lemma}

\begin{proof}
    Suppose $S$ is a YES instance of the $\thrp$. Then there exists a partition $S_1, S_2, \ldots, S_{n/3} \, \, \text{of} \, \, S = \{ x_1, x_2, \ldots, x_n\}$ such that for all $1\le i \le n/3$, $|S_i| = 3$ and
    \[
        \sum_{x_\ell \in S_i}x_\ell = T \, \, \text{where} \, \, T = \frac{\sum_{x_\ell \in S} x_\ell}{\frac{n}{3}}.
    \]
    Let, $S_i = \{ x_{i_1}, x_{i_2}, x_{i_3}\}$. By our construction of $(\m{C}, \tau)$, we have $|R_{i_\ell}| = x_{i_\ell}$, for $\ell \in \{1,2,3\}$.
   
    Now, we construct a fair clustering $\m{F}$ by merging $R_{i_\ell}$ with $p|R_{i_\ell}|$ blue points in $B_i$ for $\ell \in \{ 1, 2, 3\}$. More formally,
    
    \[
       \m{F} = \left\{ \left(B_{i_\ell} \cup R_{i_\ell} \right)  \, \,  \middle| \, \, i \in [n/3], \ell \in \{1, 2, 3\} \right \} .
    \]
    where $B_{i_\ell} \subseteq B_i$ such that $|B_{i_\ell}| = p|R_{i_\ell}|$ for $\ell \in \{1, 2, 3\}$.
    
    Note that $\m{F}$ is a $\fair$ because for each cluster $F \in \m{F}$ we have $|\blue{F}| = p|\red{F}|$.
    
    Next, we claim that
    \[
        \dist(\m{C}, \m{F}) = p\sum_{\ell = 1}^n x_\ell^2 + \frac{p^2}{2} \sum_{\ell = 1}^n x_\ell(T - x_\ell) = \tau.
    \]

    \begin{itemize}
        \item \textbf{Reason behind the term $p\sum_{\ell = 1}^n x_\ell^2 $}: In $\m{F}$, each red cluster $R_\ell$ is merged with $p|R_\ell|$ blue points from some cluster $B_i$. This expression counts the number of such blue-red pairs (recall $|R_\ell| = x_\ell)$. (One important thing to note is that since $S$ is a YES instance of the $\thrp$, we do not need to cut any red cluster $R_\ell$ in $\m{C}$ to form the {\fair} $\m{F}$.)
        
        \item \textbf{Reason behind the term $\frac{p^2}{2} \sum_{\ell = 1}^n x_\ell(T - x_\ell)$}: In $\m{F}$, each red cluster $R_\ell$ is merged with $p|R_\ell|$ blue points from some cluster $B_i$. Now, these $p|R_\ell|$ blue points of $B_i$ gets separated from other $pT - p|R_\ell|$ blue points of $B_i$. This expression counts the number of such blue-blue pairs. In the above expression, $1/2$ comes from the fact that we are counting a pair twice, once while considering the subset of $B_i$ that consists of $p|R_\ell|$ points and again while considering a subset of other $(pT - p|R_\ell|)$ blue points of $B_i$.
    \end{itemize}
    
    Hence, $(\m{C},\tau) $ is a YES instance of the $ \clof$.
\end{proof}

Now, we argue that our reduction maps a NO instance of the $\thrp$ to a NO instance of the $\clof$.

\begin{lemma}\label{lem:no-instance-p=2}
 For $p = 2$, if $S $ is a NO instance of the $\thrp$, then $(\m{C}, \tau) $ is also a NO instance of the $ \clof$.
\end{lemma}

\begin{proof}
    Let $\m{F}^*$ be any arbitrary closest clustering to $\m{C}$. We prove that $\dist(\m{C}, \m{F}^*) > \tau$.

    By the definition, $\dist(\m{C}, \m{F}^*)$ counts the total number of pairs $(u,v)$ that are together in $\m{C}$ but separated by $\m{F}^*$ and the number of pairs $(u,v)$ that are in different clusters in $\m{C}$ but together in $\m{F}^*$.

    To count a pair $(u,v)$, we use a charging scheme. For each pair $(u,v)$ we charge the vertex $u$, $1/2$ and the vertex $v$, $1/2$.

    Let us define $\costf{v}$ to be half times the number of vertices $u$ such that $u,v$ are together in $\m{C}$ but separated by $\m{F}^*$, or $u,v$ that are in different clusters in $\m{C}$ but together in $\m{F}^*$. More specifically, 

    \begin{align}
        \costf{v} := \frac{1}{2} | \{ u \mid &u,v \in C, C \in \m{C}\, \, \text{but}\, \, u \in F_a, v\in F_b, F_a, F_b \in \m{F}^*, a \neq b \n \\
        & \text{or} \, \, u,v \in F_a, F_a \in \m{F}^*, u \in C_i, v \in C_j, C_i, C_j \in \m{C}, i \neq j \} | \n 
    \end{align}

    For a cluster $F_a \in \m{F}^*$, let us define $\costf{F_a}$ (by abuse of notation) to be 

    \[
        \costf{F_a} := \sum_{v \in F_a} \costf{v}.
    \]
        
    Observe that 

    \begin{align}
        \costf{F_a} = &p\left( \sum_{\ell = 1}^n |R_{a,\ell}|\right)^2 + \frac{p^2}{2} \left( \sum_{\ell = 1}^n |R_{a,\ell}|\right) \left( T - \sum_{\ell = 1}^n |R_{a,\ell}|\right) \n \\
        &+ \frac{1}{2}\sum_{\ell_1 \neq \ell_2}|R_{a, \ell_1}||R_{a, \ell_2}| + \frac{1}{2}\sum_{\ell = 1}^n|R_{a, \ell}| \left(|R_{\ell}| -|R_{a, \ell}| \right) \label{equn:np-p=2-zero}
    \end{align}
    where $R_{a,\ell} = R_\ell \cap F_a$.

    \begin{itemize}
        \item \textbf{Reason behind the term $p\left( \sum_{\ell = 1}^n |R_{a,\ell}|\right)^2$}: $\left( \sum_{\ell = 1}^n |R_{a,\ell}|\right)$ is the number of red points in $F_a$. Since it is merged with $p\left( \sum_{\ell = 1}^n |R_{a,\ell}|\right)$ blue points, hence the number of blue-red pairs in $F_a$ is $p\left( \sum_{\ell = 1}^n |R_{a,\ell}|\right)^2$.
        \item \textbf{Reason behind the term $\frac{p^2}{2} \left( \sum_{\ell = 1}^n |R_{a,\ell}|\right) \left( T - \sum_{\ell = 1}^n |R_{a,\ell}|\right)$}: By \cref{clm:claim-np} and \cref{clm:np-all-partition-mop} we get that for all $F_a \in \m{F}^*$, $\blue{F_a} \subseteq B_i$ for some $B_i \in \m{C}$. Hence, the points in $\blue{F_a}$ are separated from $pT - |\blue{F_a}|$ points in $B_i$. These blue-blue pairs are counted in this expression
        \begin{align*}
            &\frac{1}{2} \left( p\sum_{\ell = 1}^n |R_{a,\ell}|\right) \left( pT - p\sum_{\ell = 1}^n |R_{a,\ell}|\right) \n \\
            = &\frac{p^2}{2} \left( \sum_{\ell = 1}^n |R_{a,\ell}|\right) \left( T - \sum_{\ell = 1}^n |R_{a,\ell}|\right)
        \end{align*}
        In the above expression, $1/2$ comes from the fact that for a pair $(u,v)$ we are charging $1/2$ for the vertices $u \in F_a$ while considering the cluster $F_a$ and again we will charge $1/2$ for the vertex $v$ while considering a cluster $F_b$ such that $v \in \blue{F_b}$, $\blue{F_b} \subseteq B_i$ and $F_b \neq F_a$.
        \item \textbf{Reason behind the term $\frac{1}{2}\sum_{\ell_1 \neq \ell_2}|R_{a, \ell_1}||R_{a, \ell_2}|$}: This counts the number of red-red pairs $(u,v)$ in $F_a$ such that $u \in R_{\ell_1}$ and $v \in R_{\ell_2}$.
        \item \textbf{Reason behind the term $\frac{1}{2}\sum_{\ell = 1}^n|R_{a, \ell}| \left(|R_{\ell}| -|R_{a, \ell}| \right)$}: The term counts the number of pairs $(u,v) \in R_\ell$ such that $u \in R_\ell \cap F_a$ and $v \in R_\ell \cap F_b$ for $F_b \neq F_a$. In the above expression, $1/2$ comes from the fact that for a pair $(u,v)$ we are charging $1/2$ for the vertices $u \in R_\ell \cap F_a$ while considering the cluster $F_a$ and again we will charge $1/2$ for the vertex $v$ while considering a cluster $F_b$ such that $v \in R_\ell \cap F_b$.
    \end{itemize}

    It is straightforward to observe

    \[
        \dist(\m{C}, \m{F}^*) \geq \sum_{F_a \in \m{F}^*} \costf{F_a}.
    \]

    Now we show that 

    \[
        \sum_{F_a \in \m{F}^*} \costf{F_a} > \tau.
    \]

    Let, $x_{a,\ell} = |R_{a,\ell}|$ and recall $x_\ell = |R_\ell|$.

    \begin{align}
        &\sum_{F_a \in \m{F}^*} \costf{F_a} \n \\
        = \, \, &p \sum_{F_a \in \m{F}^*} \left( \sum_{\ell = 1}^n x_{a,\ell} \right)^2 + \frac{p^2}{2} \sum_{F_a \in \m{F}^*}\left( \sum_{\ell = 1}^n x_{a,\ell}\left( T - \sum_{\ell = 1}^n x_{a,\ell}\right) \right) + \frac{1}{2}\sum_{F_a \in \m{F}^*}\sum_{\ell_1 \neq \ell_2} x_{a,\ell_1} x_{a,\ell_2} \n \\
        &+ \frac{1}{2}\sum_{F_a \in \m{F}^*}\sum_{\ell = 1}^n x_\ell (x_\ell - x_{a,\ell}) \n \\
        = \, \, &p \sum_{F_a \in \m{F}^*} \left( \sum_{\ell = 1}^n x_{a,\ell}^2 + \sum_{\ell_1 \neq \ell_2}x_{a,\ell_1} x_{a,\ell_2}\right) + \frac{p^2}{2} \sum_{F_a \in \m{F}^*}\left( \sum_{\ell = 1}^n x_{a,\ell}\left( T - \sum_{\ell = 1}^n x_{a,\ell}\right) \right) + \frac{1}{2}\sum_{F_a \in \m{F}^*}\sum_{\ell_1 \neq \ell_2} x_{a,\ell_1} x_{a,\ell_2} \n \\
        &+ \frac{1}{2}\sum_{\ell = 1}^n \sum_{F_{a_1} \neq F_{a_2}} x_{a_1, \ell}x_{a_2, \ell} \n \\
        = \, \, &p \sum_{F_a \in \m{F}^*} \sum_{\ell = 1}^n x_{a,\ell}^2 
         + \frac{p^2}{2} \sum_{F_a \in \m{F}^*}\left( \sum_{\ell = 1}^n x_{a,\ell}\left( T - x_{a,\ell}\right) \right) + \frac{1}{2}\sum_{\ell = 1}^n \sum_{F_{a_1} \neq F_{a_2}} x_{a_1, \ell}x_{a_2, \ell} \n \\
         &- \frac{p^2}{2} \sum_{F_a \in \m{F}^*}\sum_{\ell_1 \neq \ell_2} x_{a,\ell_1} x_{a,\ell_2} + \left(p + \frac{1}{2} \right)\sum_{F_a \in \m{F}^*}\sum_{\ell_1 \neq \ell_2} x_{a,\ell_1} x_{a,\ell_2} \n \\
         \geq \, \,  &p \sum_{F_a \in \m{F}^*} \sum_{\ell = 1}^n x_{a,\ell}^2 + \frac{p^2}{2} \sum_{F_a \in \m{F}^*}\left( \sum_{\ell = 1}^n x_{a,\ell}\left( T - x_{a,\ell}\right) \right) + \frac{1}{2}\sum_{\ell = 1}^n \sum_{F_{a_1} \neq F_{a_2}} x_{a_1, \ell}x_{a_2, \ell} \n \\ &(\text{as} \, \, \, p = 2) \label{equn:np-p=2-one} \\
         = \, \, &p \sum_{F_a \in \m{F}^*} \sum_{\ell = 1}^n x_{a,\ell}^2 + \frac{p^2}{2} \sum_{F_a \in \m{F}^*}\left( \sum_{\ell = 1}^n x_{a,\ell}\left( T - x_{\ell} + (x_{\ell} - x_{a, \ell})\right) \right) + \frac{1}{2}\sum_{\ell = 1}^n \sum_{F_{a_1} \neq F_{a_2}} x_{a_1, \ell}x_{a_2, \ell} \n \\
        = \, \, &p \sum_{F_a \in \m{F}^*} \sum_{\ell = 1}^n x_{a,\ell}^2 + \frac{p^2}{2} \sum_{F_a \in \m{F}^*}\sum_{\ell = 1}^n x_{a,\ell} (x_{\ell} - x_{a, \ell}) + \frac{p^2}{2} \sum_{F_a \in \m{F}^*} \sum_{\ell = 1}^n x_{a,\ell}\left( T - x_{\ell}\right) \n \\
        &+ \frac{1}{2}\sum_{\ell = 1}^n \sum_{F_{a_1} \neq F_{a_2}} x_{a_1, \ell}x_{a_2, \ell} \n \\
        = \, \, &p\sum_{\ell = 1}^n \left( \sum_{F_a \in \m{F}^*}  x_{a,\ell}^2 + \frac{p}{2}\sum_{F_a \in \m{F}^*} x_{a, \ell}(x_\ell - x_{a,\ell}) + \frac{1}{2p} \sum_{F_{a_1} \neq F_{a_2}} x_{a_1, \ell}x_{a_2, \ell}  \right) + \frac{p^2}{2} \sum_{\ell = 1}^n \left( T - x_{\ell}\right) \sum_{F_a \in \m{F}^*} x_{a,\ell}\n \\ 
        = \, \, &p\sum_{\ell = 1}^n \left( \sum_{F_a \in \m{F}^*} x_{a,\ell}^2 + \left(\frac{p}{2} + \frac{1}{2p} \right) \sum_{F_{a_1} \neq F_{a_2}} x_{a_1,\ell} x_{a_2, \ell}\right) + \frac{p^2}{2} \sum_{\ell = 1}^n x_\ell (T - x_\ell) \n \\
        > \, \, &p\sum_{\ell = 1}^n \left( \sum_{F_a \in \m{F}^*} x_{a,\ell}^2 +  \sum_{F_{a_1} \neq F_{a_2}} x_{a_1,\ell} x_{a_2, \ell}\right) + \frac{p^2}{2} \sum_{\ell = 1}^n x_\ell (T - x_\ell) \label{equn:np-p=2-two} \\
       = \, \, &p\sum_{\ell = 1}^n x_\ell^2 + \frac{p^2}{2} \sum_{\ell = 1}^n x_\ell (T - x_\ell) = \tau \n
    \end{align}

    The reason behind \cref{equn:np-p=2-two} is that $(p/2 + 1/2p) > 1$ for $p > 1$ and there must exist at least one $\ell$ and $F_{a_1}$, $F_{a_2} \in \m{F}^*$ such that $x_{a_1,\ell}, x_{a_2, \ell} \neq 0$. We prove this in the following claim.

    \begin{claim}
        There exists $\ell, a_1, a_2$ such that $x_{a_1,\ell}, x_{a_2, \ell} \neq 0$.
    \end{claim}

    \begin{proof}
        For the sake of contradiction, suppose for all $\ell$, there exists at most one $F_a \in \m{F}^*$ such that $x_{a,\ell} \neq 0$. 
        
        Since $\sum_{F_a \in \m{F}^*} x_{a,\ell} = x_\ell$, we have a cluster $F_a \in \m{F}^*$ such that $x_{a,\ell} = x_\ell$.
       
        Now we create partitions $S_i$ of $S$ in the following way: For each $B_i \in \m{C}$,

        \[
            S_i = \{ x_\ell \mid x_{a,\ell} = x_\ell \, \, \text{and} \, \, \blue{F_a} \subseteq B_i \}.
        \]

        Note by \cref{clm:claim-np} and \cref{clm:np-all-partition-mop} we get for each $F_a \in \m{F}^*$ there must exist a $B_i$ such that $\blue{F_a} \subseteq B_i$.

        Since there are $n/3$ clusters $B_i \in \m{C}$, the number of such subsets would be $n/3$. Clearly, the above subsets $S_i$'s create a partitioning of $S$, that is

        \begin{itemize}
            \item $S_a \cap S_b = \emptyset$ for $a \neq b$ (since, $B_i \cap B_j = \emptyset$, hence, if $\blue{F_a}\subseteq B_i$ then $\blue{F_a}\not\subseteq B_j$ for $j \neq i$);
            \item $\bigcup S_a = S$ (since $x_{a,\ell}$ cannot be zero for all $a$).  
        \end{itemize}

        Since $|B_i| = pT$ and all clusters $F_a \in \m{F}^*$ such that $\blue{F_a} \subseteq B_i$ are fair we have,

        \[
            \sum_{x_\ell \in S_i}x_\ell = T.
        \]

        Further, since each $x_\ell \in (T/4,T/2)$, we must have $|S_i| = 3$. However, since $S$ is a NO instance of $\thrp$, such a partitioning of $S$ does not exist, leading to a contradiction.

        Hence, we conclude that there must exist $\ell, a_1, a_2$ such that $x_{a_1,\ell}, x_{a_2, \ell} \neq 0$.
    \end{proof}

    Thus we finally get that $\dist(\m{C}, \m{F}^*) > \tau$, which implies $(\m{C},\tau) $ is a NO instance of the $ \clof$, for $p = 2$.
\end{proof}

Similar to the proof of \cref{lem:main-one}, we conclude the proof of \cref{lem:main-two}.

\begin{proof}[Proof of \cref{lem:main-two}]
    By~\cref{lem:yes-instance-p=2} and~\cref{lem:no-instance-p=2} we infer that for any $p = 2$, $S $ is a YES instance of the $ \thrp$ if and only if $(\m{C}, \tau) $ is a YES instance of the $ \clof$. It now immediately follows from our reduction that for any $p = 2$, $\clof$ is \textbf{NP}-hard because $\thrp$ is strongly $\npc$ (~\cref{thm:garey-johnson}).
\end{proof}

\begin{remark}\label{remark:np-hard-one}
    We would like to remark that the above \textbf{NP}-hardness proof can be extended for any fixed ratio (between the number of blue and red points) of $p/q$ where $1 < p/q \leq (1 + \sqrt{2})$ with modifications in the reduction mentioned in \cref{remark:np-hard}, with $\tau$ being set to $\tau = pq\sum_{\ell = 1}^n x_\ell^2 + \frac{p^2}{2} \sum_{\ell = 1}^n x_\ell(T - x_\ell)$. The reason behind this is that \cref{equn:np-p=2-one} holds for any ratio $ p\in [1 - \sqrt{2}, 1 + \sqrt{2}]$, and while expressing $\costf{F_a}$ in \cref{equn:np-p=2-zero} we use \cref{clm:claim-np} and \cref{clm:np-all-partition-mop}, both of which also hold for any $p, q > 1$.

\end{remark}

From \cref{lem:main-one} and \cref{lem:main-two} we conclude \cref{thm:main}.

\end{document}